\documentclass{article}

\usepackage{xcolor}
\definecolor{color1}{HTML}{105e8a}
\definecolor{color2}{HTML}{e99926}
\definecolor{color3}{HTML}{b82a0c}
\definecolor{color4}{HTML}{3e8a10}
\definecolor{color5}{HTML}{80037e}
\definecolor{color6}{HTML}{070707}

\usepackage{hyperref}
\hypersetup{
    colorlinks,
    linkcolor={color1},
    citecolor={color1},
    urlcolor={color2}
}

\usepackage{graphicx}
\usepackage{amsfonts,amsmath,amssymb,amsthm,mathrsfs}
\usepackage{mathrsfs}
\usepackage{comment}
\usepackage{natbib}
\usepackage{wrapfig}
\usepackage{thmtools}
\usepackage{fontawesome}

\usepackage{graphicx}
\usepackage{tikz}
\usetikzlibrary{decorations.pathreplacing}
\usetikzlibrary{calc}
\usepackage{pgf}
\usepackage{pgfplots}
\pgfplotsset{compat=1.18}

\usepackage{tikz-cd}
\usepackage{booktabs}

\usepackage{algorithm,algorithmic}

\usepackage[title,toc,page,header]{appendix}

\usepackage{minitoc}

\doparttoc

\mtcsetdepth{parttoc}{2}

\usepackage{authblk}

\definecolor{tobiComment}{RGB}{150, 10, 150}
\definecolor{fannyComment}{RGB}{150, 10, 250}
\definecolor{geelonComment}{RGB}{255, 127, 0}
\definecolor{fedeComment}{RGB}{200, 100, 100}

\newcommand{\tw}[1]{}
\newcommand{\fy}[1]{}
\newcommand{\geelon}[1]{}
\newcommand{\fede}[1]{}

\usepackage{pifont}

\usepackage{caption}
\usepackage{subcaption}

\newcommand{\bR}[0]{\mathbf{R}}
\newcommand{\cZ}[0]{\mathcal{Z}}
\newcommand{\I}[0]{\mathbf{I}}
\newcommand{\Rhat}[0]{\widehat{\cR}}

\newcommand{\eps}[0]{\varepsilon}

\DeclareMathOperator*{\EE}{\mathbb{E}}
\DeclareMathOperator*{\argmin}{\arg\min}

\newcommand{\inner}[2]{\left\langle #1 , #2 \right\rangle}
\newcommand{\PP}[0]{\mathbb{P}}

\newcommand{\e}[0]{\mathbf{e}}

\usepackage{color}
\usepackage{graphicx} % Required for inserting images
\usepackage{amsfonts,amsmath,amssymb,amsthm}
\usepackage{microtype}
\usepackage{nicefrac}
\usepackage{hyperref}
\usepackage[capitalize]{cleveref}
\usepackage{enumitem}
\usepackage{natbib}
\newtheoremstyle{thmstyle}
  {6pt} % Space above \topsep
  {2pt} % Space below \topsep
  {\itshape} % Body font
  {} % Indent amount
  {\bfseries} % Theorem head font
  {.} % Punctuation after theorem head
  {.5em} % Space after theorem head
  {} % Theorem head spec (can be left empty, meaning `normal')

\newtheoremstyle{defstyle}
  {6pt} % Space above \topsep
  {2pt} % Space below \topsep
  {} % Body font
  {} % Indent amount
  {\bfseries} % Theorem head font
  {.} % Punctuation after theorem head
  {.5em} % Space after theorem head
  {} % Theorem head spec (can be left empty, meaning `normal')

\theoremstyle{thmstyle}
\newtheorem{theorem}{Theorem}
\newtheorem{lemma}{Lemma}
\newtheorem{proposition}{Proposition}
\newtheorem{corollary}{Corollary}

\theoremstyle{defstyle}
\newtheorem{example}{Example}
\newtheorem{definition}{Definition}
\newtheorem{assumption}{Assumption}

\theoremstyle{remark}
\newtheorem{remark}{Remark}

\newcommand{\nrm}[1]{\left\|#1 \right\|}
\newcommand{\prn}[1]{\left(#1 \right)}
\newcommand{\brk}[1]{\left[#1 \right]}
\newcommand{\crl}[1]{\left\{#1 \right\}}
\newcommand{\abs}[1]{\left|#1 \right|}
\newcommand{\floor}[1]{\left\lfloor #1 \right\rfloor}
\newcommand{\ceil}[1]{\left\lceil #1 \right\rceil}

\newcommand{\RR}[0]{\mathbb{R}}
\newcommand{\NN}[0]{\mathbb{N}}
\newcommand{\ZZ}[0]{\mathbb{Z}}
\newcommand{\SSS}[0]{\mathbb{S}}

\newcommand{\cA}[0]{\mathcal{A}}
\newcommand{\cC}[0]{\mathcal{C}}

\newcommand{\cH}[0]{\mathcal{H}}
\newcommand{\cE}[0]{\mathcal{E}}
\newcommand{\cF}[0]{\mathcal{F}}

\newcommand{\cP}[0]{\mathcal{P}}
\newcommand{\cQ}[0]{\mathcal{Q}}

\newcommand{\cN}[0]{\mathcal{N}}
\newcommand{\cR}[0]{\mathcal{R}}
\newcommand{\cX}[0]{\mathcal{X}}
\newcommand{\cY}[0]{\mathcal{Y}}

\newcommand{\one}{\mathbf 1}

\DeclareMathOperator*{\prob}{\PP}

\DeclareMathOperator{\uniformOp}{Uniform}
\newcommand{\uniform}[1]{\uniformOp\prn{#1}}
\DeclareMathOperator{\conv}{conv}

\DeclareMathOperator*{\Var}{Var}

\newcommand{\starhull}[0]{\mathrm{star}}

\newcommand{\fhat}[0]{\widehat{f}}
\newcommand{\fstar}[0]{f^{\star}}

\DeclareMathOperator{\KL}{KL}

\newcommand{\pareto}[0]{\operatorname{Par}}

\newcommand{\rhohat}[0]{\widehat{\rho}}
\newcommand{\rhostar}[0]{\rho^{\star}}

\newcommand{\rhotilt}[0]{\rhohat_{\operatorname{tilt}}}

\newcommand{\Rtilde}[0]{\widetilde{\cR}}

\newcommand{\nameditem}[2][]{%
  \item[#1]\label{#2}%
  \expandafter\gdef\csname itext@#2\endcsname{#1}%
}
\newcommand{\refitemtext}[1]{\csname itext@#1\endcsname}

\newcommand{\cL}[0]{\mathcal{L}}
\newcommand{\radComp}[0]{\mathscr{R}_n}

\newcommand{\Npar}[0]{N_{\operatorname{Par}}}

\newcommand{\bsigma}[0]{\boldsymbol{\sigma}}
\newcommand{\Ber}[0]{\operatorname{Ber}}
\newcommand{\bsigmahat}[0]{\widehat{\bsigma}}
\newcommand{\sH}[0]{\mathscr{H}}

\newcommand{\fR}[0]{\mathfrak{R}}
\newcommand{\n}[0]{\mathbf{n}}
\newcommand{\x}[0]{\mathbf{x}}
\renewcommand{\a}[0]{\mathbf{a}}

\renewcommand{\r}[0]{\mathbf{r}}
\renewcommand{\v}[0]{\mathbf{v}}
\newcommand{\w}[0]{\mathbf{w}}
\newcommand{\Vol}[0]{\operatorname{Vol}}

\newcommand{\elltilde}[0]{\widetilde{\ell}}

\newcommand{\bH}[0]{\mathbf{H}}

\newcommand{\uppermodulus}[0]{\overline{\omega}_Q}
\newcommand{\lowermodulus}[0]{\underline{\omega}_Q}

\newcommand{\tlower}[0]{\underline{\gamma}_Q}
\newcommand{\Fpruned}[1]{\cF_{#1}}
\newcommand{\tadapt}[0]{\gamma_{n}}

\newcommand{\Paretodistance}[0]{d_{\bR}}

\renewcommand{\d}[0]{\operatorname{d}}

\newcommand{\btau}[0]{\boldsymbol{\tau}}

\newcommand{\sigmahat}[0]{\widehat{\sigma}}

\newcommand{\sign}[0]{\operatorname{sign}}
\newcommand{\bu}[0]{\mathbf{u}}

\newcommand{\rhofast}[0]{\rhohat_{\operatorname{Par}}}

\newcommand{\yhat}[0]{\widehat{y}}

\newcommand{\Paretoerm}[1]{\fhat_{\operatorname{Par(#1)}}}
\newcommand{\Thresholderm}[1]{\fhat_{#1}}
\newcommand{\ParetoStar}[1]{\fhat_{\operatorname{star}(#1)}}
\newcommand{\ParetoEW}[0]{f_{\rhohat}}
\newcommand{\AdaptiveThresholdERM}[0]{\fhat_{\operatorname{adapt}}}

\newcommand{\q}[0]{\mathbf{q}}
\newcommand{\cI}[0]{\mathcal{I}}

\newcommand{\range}[0]{\operatorname{range}}
\newcommand{\bqhat}[0]{\widehat{\q}}
\newcommand{\Fhat}[0]{\widehat{F}}

\newcommand{\excessRisk}[2]{\cE_{Q}(#1 , #2)}

\usepackage{xspace}
\newcommand{\HELM}[0]{\texttt{HELM}\xspace}
\newcommand{\VHELM}[0]{\texttt{VHELM}\xspace}
\newcommand{\HELMLite}[0]{\texttt{HELM Lite}\xspace}

\usepackage[utf8]{inputenc}
\DeclareFontFamily{U}{mathb}{\hyphenchar\font45}
\DeclareFontShape{U}{mathb}{m}{n}{
<-6> mathb5 <6-7> mathb6 <7-8> mathb7
<8-9> mathb8 <9-10> mathb9
<10-12> mathb10 <12-> mathb12
}{}
\DeclareSymbolFont{mathb}{U}{mathb}{m}{n}
\DeclareMathSymbol{\llcurly}{\mathrel}{mathb}{"CE}
\DeclareMathSymbol{\ggcurly}{\mathrel}{mathb}{"CF}

\newcommand{\fG}[0]{\mathfrak{G}}
\newcommand{\LMSA}[0]{\operatorname{LMSA}(\Lambda)}
\newcommand{\fLMSA}[0]{\fhat_{\Lambda}}
\newcommand{\Slambda}[1]{s_{#1}}

\newcommand{\g}[0]{\mathbf{g}}
\renewcommand{\b}[0]{\mathbf{b}}
\renewcommand{\r}[0]{\mathbf{r}}

\newif\ifarxiv
\newcommand{\ifarxivelse}[2]{%
  \ifarxiv
    #1%
  \else
    #2%
  \fi
}
\newcommand{\arxivonly}[1]{\ifarxiv #1\fi}
\newcommand{\cameraonly}[1]{\ifarxiv\else #1\fi}

\arxivtrue      % arXiv version

\usepackage{geometry}
\usepackage{parskip}
\title{Hedging on the Frontier: Learning New Tasks with Few Samples}

\author[1]{Tobias Wegel}%\thanks{Corresponding author: \texttt{twegel@ethz.ch}.}}
\author[1]{Federico Di Gennaro\thanks{Equal contribution.}}
\author[2]{Geelon So$^*$}
\author[1]{Fanny Yang}

\affil[1]{Department of Computer Science, ETH Zurich}
\affil[2]{Department of Computer Science and Engineering, UC San Diego}

\date{\today}

\begin{document}
\faketableofcontents

\maketitle

\begin{abstract}
  
When a learner faces a new task with few samples, it must leverage any available side information. 
In practice, this often comes in the form of model evaluations on related tasks in public benchmarks. 
A key question then is how to model task relatedness such that it is both realistic and the benchmark evaluations lead to provable gains. 
Empirically, we observe that \emph{weak monotonicity} is often approximately satisfied: 
if a model dominates another on many benchmarks, it also tends to outperform on the new task. 
We explore the statistical complexity of learning under (approximate) weak monotonicity, leveraging it within two learning paradigms: transfer learning and model selection aggregation. We show that not only can we prune the model class based on monotonicity, but we can also further adapt to the geometry of the available trade-offs by \emph{hedging on the frontier}.

\end{abstract}

\section{Introduction}
Foundation models are designed to perform well across a vast range of downstream applications \citep{bommasani2021OnTO}. 
But, in general, it is hard to determine 
whether one model is strictly more capable than another, or to know beforehand how suitable it is for a new task.
Benchmarks can provide some signal by evaluating these models
on a fairly narrow set of tasks \citep{liang2023holistic, chiang2024chatbotarena, srivastava2023bench, hardt2025emerging}. 
However, models are often ranked inconsistently across different benchmarks \citep{Zhang2024tradeoffs}, which raises the issue of \emph{transferability}: with the proliferation of models and benchmarks, can they be reliably used by a practitioner to solve a new target task with limited data?

\ifarxivelse{
\begin{wrapfigure}{r}{0.45\textwidth}
    \vspace{-0.6cm}
    \centering
    \includegraphics[width=\linewidth]{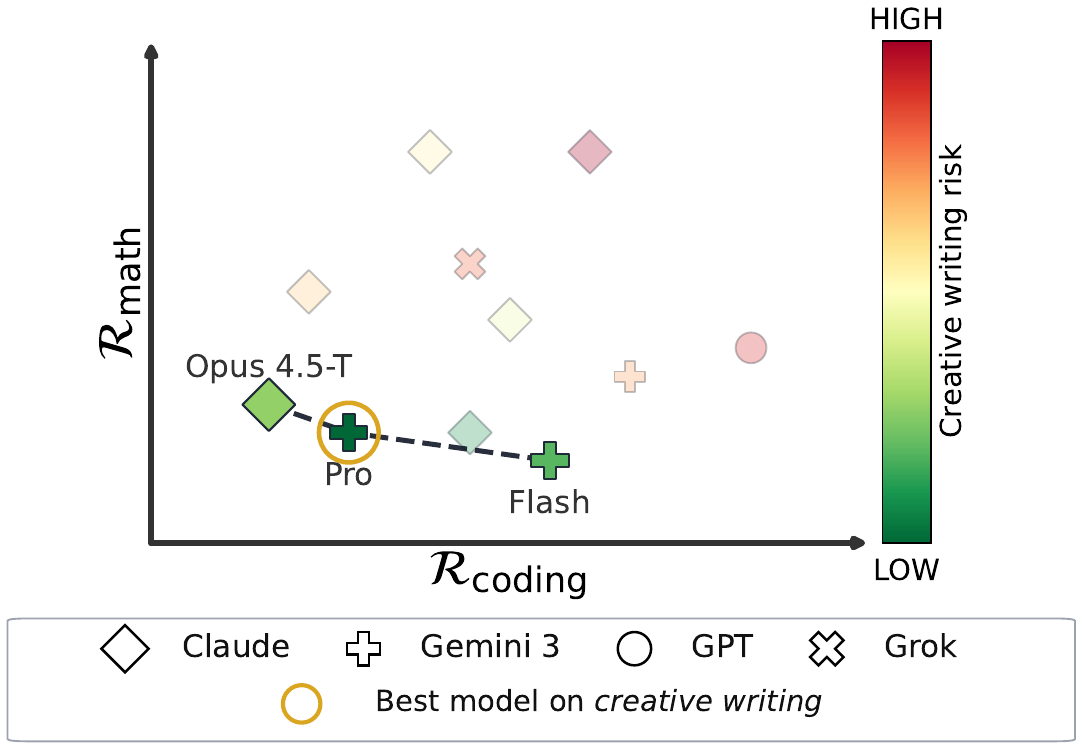}
    \caption{Pareto Frontier (Coding vs Math) with risk scores from a subset of Huggingface \texttt{LMArena} dataset \citep{chiang2024chatbotarena}.\protect\footnotemark
    \vspace{-0.15cm}}
    \label{fig:Pareto-LLMArena}
\end{wrapfigure}
}{
\begin{figure}
    \vspace{-0.2cm}
    \centering
    \includegraphics[width=\linewidth]{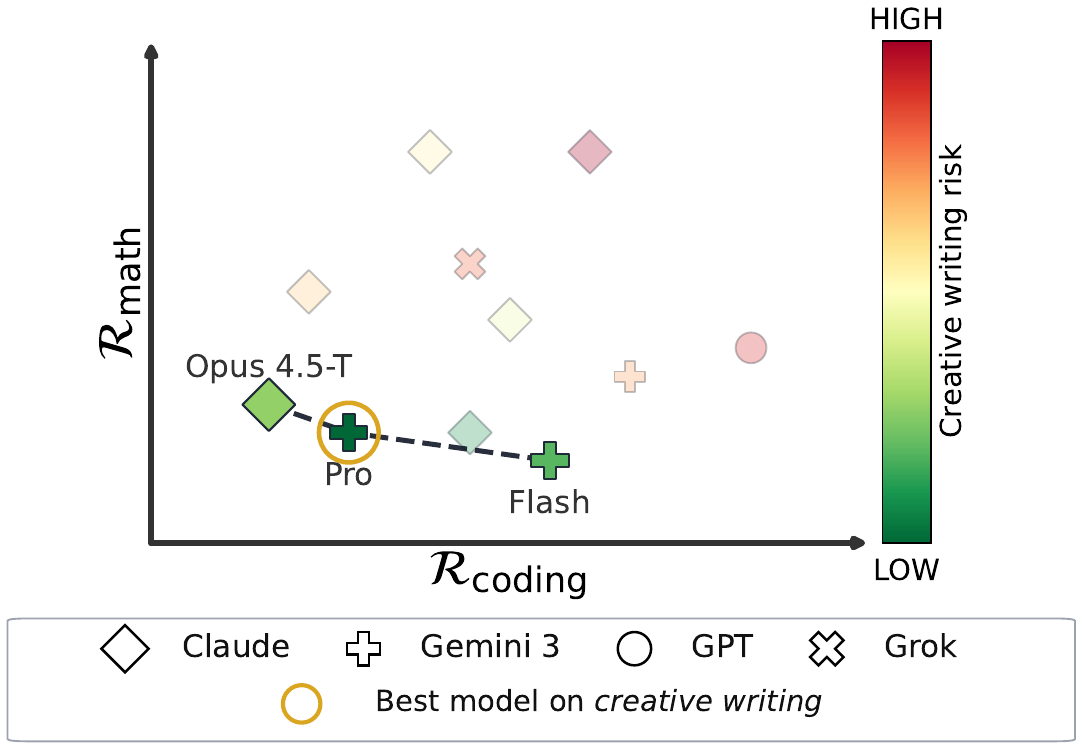}
    \caption{Pareto Frontier (Coding vs Math) with risk scores from a subset of Huggingface \texttt{LMArena} dataset \citep{chiang2024chatbotarena}.\protect\footnotemark
    \vspace{-0.3cm}}
    \label{fig:Pareto-LLMArena}
\end{figure}
}

Underlying the belief that benchmarks are useful is often the implicit assumption that 
performance transfers: a model that dominates another on a related benchmark task will likely outperform the other also in the target task. Indeed, this phenomenon has already been observed in the context of out-of-distribution (OOD) versus in-distribution generalization \citep{miller2021accuracy}, although it fails to hold in most arguably more realistic scenarios \citep{sanyal2024accuracy,salaudeen2025domain}.
Hence, instead of relating just two tasks, in this work, we formalize the weaker assumption of \emph{weak monotonicity}, which relates a target task to multiple source tasks.\footnotemark\ Weak monotonicity admits comparison across a \emph{suite} of benchmarks: a more capable model that is consistently better on a whole set of benchmarks will also tend to be better on the target. It does not imply that there is a single dominating model; rather, there may be an entire \emph{frontier} of potentially suitable models, and a need to be able to select from it using some limited target data. 
This is depicted in \cref{fig:Pareto-LLMArena} using data from the Huggingface \texttt{LMArena}, where math and coding scores separately do not predict the best model for creative writing, but it is contained in the joint frontier. 
Importantly, a key feature is that the assumption of weak monotonicity only further weakens as the number of benchmark tasks increases.
Moreover, we do not always require it to hold strictly: we quantify the degree to which it holds using the \emph{modulus of monotonicity}, a quantity that appears in different variants in our theoretical results. 

For the theoretical results, we assume throughout that the benchmark risks are known exactly. In transfer learning terminology, this corresponds to the idealized setting of infinite source samples; we leave the finite source sample case to future work.
\ifarxivelse{We are then interested in the following theoretical question:}{We then ask the theoretical question:}
\ifarxivelse{
\begin{quote}
    \centering
    \emph{What is the sample complexity of learning under accurate benchmark evaluations and (approximate) weak monotonicity?}
\end{quote}
}{
\emph{What is the sample complexity of learning under accurate benchmark evaluations and (approximate) weak monotonicity?}

}
To answer this question, we distinguish between two classes of functions we may compare to.
Firstly, we consider the \emph{full} hypothesis class for which evaluations are available.
This then turns into the question of (classical) \emph{transfer learning} (also called supervised domain adaptation) \citep{ben2010theory}, where the aim is to achieve small excess risk with respect to the full hypothesis class. 
And secondly, we may compare only to the dictionary of best models for each benchmark alone.
This falls into the realm of \emph{model selection aggregation} \citep{tsybakov2003optimal} in a setting where the dictionary is not arbitrary in that the models are optimal for individual benchmarks. The goal is then to achieve small excess risk with respect to the best model in the dictionary. We introduce model selection aggregation in detail in \cref{sec: aggregation}.

In the existing transfer learning literature (or domain generalization),
most approaches are based on assumptions on the relationship between source and target domains that typically fall into two categories: \emph{Interpolation}-based approaches exploit proximity across source and target domains such as distributional similarity, shared features, or empirical performance correlations \citep{ben2010theory,long2015learning,mansour2008domain,mansour2021theory}. On the other hand, \emph{extrapolation}-based approaches can transfer beyond observed environments by learning invariant or causal representations \citep{arjovsky2019invariant,krueger2021out,peters2016causal,scholkopf2021toward}.
Importantly, learning under weak monotonicity strictly generalizes previous settings in transfer learning, e.g., when the target distribution is a \emph{mixture} of source distributions \citep{mansour2008domain,mansour2021theory}, cf. \cref{subsec:weak-monotonicity}.
Both types of modeling assumptions are often too restrictive in practice and are not reflective of the actual shift between the tasks. For instance, causal methods often do not outperform simple baselines \citep{nastl2024causal,ahuja2021invariance} and DRO style methods fail, e.g., on the prominent \texttt{WILDS} dataset \citep{koh2021wilds}.
For modern foundation models, transfer is often framed as model selection from a \emph{zoo} of pretrained candidates \citep{zhang2023model,dong2022zood}, for example by ranking candidates via transferability scores \citep{nguyen2020leep,you2021logme} 
typically using access to model representations.
However, for state-of-the-art models, we often only have black-box access
to their evaluations \citep{lee2024vhelm}, so that these methods cannot be used.

We assume weak monotonicity partly to address these shortcomings.
The main insight from our work is that under monotonicity, statistically efficient algorithms should not only \emph{prune} models that are dominated. When target data is limited, they must also be selective over the frontier models, and a form of \emph{hedging} can be facilitated by avoiding improper trade-offs on the benchmarks. We show that the effectiveness of hedging is determined by the geometry on the frontier induced by monotonicity.

In short, our contributions can be summarized as follows: 

\begin{itemize}[leftmargin=*,itemsep=-0.5pt,topsep=-2pt]
    \item We are the first to\arxivonly{ introduce and} formally study the setting of learning a target task under (approximate) \emph{weak monotonicity} (\cref{sec: problem setting}), motivated by empirical findings that approximate monotonicity holds in real benchmarking datasets, such as \HELM and \VHELM (\cref{fig:modulus HELM and VHELM,fig:mean excess risks HELM and VHELM,subsec:transfer-learning-HELM-VHELM}).

    \item We introduce the Pareto covering number (\cref{subsec:geometry-covering}), a complexity measure that adapts naturally to the geometry of the Pareto frontier while inherently encoding the monotonicity assumption. We further derive its limiting distribution; this result may be of independent interest.  

    \item We prove sample complexity guarantees in terms of the Pareto covering number and modulus of monotonicity in the transfer learning setting (\cref{subsec:Pareto-ERM}). \arxivonly{Using its limiting distribution, we show that the covering number implicitly endows the frontier with a special prior distribution, yielding a PAC-Bayesian interpretation of the guarantees.} 
    We also present a parameter-free algorithm that enjoys a fall-back guarantee for \arxivonly{strong }monotonicity violations (\cref{sec: beyond exact monotonicity}).  

    \item Under assumptions on the curvature of the Pareto front, we show that monotonicity enables fast rates in model selection aggregation, without requiring a Bernstein condition or strongly convex loss (\cref{sec: aggregation}).
\end{itemize}

\footnotetext{Risks are mapped to $[0,1]$ using models rankings, where rankings are taken from the \texttt{LMArena} \emph{Overview} leaderboard (snapshot taken from data available on Jan. 10$^{\text{th}}$, 2026).}
\footnotetext{We use \emph{benchmark} and \emph{source task} interchangeably.}

\paragraph{Datasets.} While the contributions of this paper are mainly theoretical, \arxivonly{throughout the sections }we also validate our assumptions and methods on two benchmark suites: \HELM (Holistic Evaluation of Language Models)~\citep{liang2023holistic} and \VHELM (Holistic Evaluation of Vision Language Models)~\citep{lee2024vhelm}. Both report evaluations of LLMs (respectively, VLMs) across a collection of scenarios and metrics.\footnote{Data and code are available at \href{https://github.com/FedericoDiGennaro/Hedging-on-the-Frontier}{\faGithub \ \url{https://github.com/FedericoDiGennaro/Hedging-on-the-Frontier}.}}

\section{Problem Setting and Weak Monotonicity}\label{sec: problem setting}

Let $\cF$ be a model class and let $\cR_Q : \cF \to [0,1]$ measure risk for a task of interest.\footnote{\cref{tab:notation} in the appendix provides a full overview of notation.}
Specifically, we let $Q$ be a distribution on an arbitrary space $\cZ$ and $\cR_Q(f)=\EE_{z\sim Q}\ell(f,z)$ for any function $f\in\cF$, where the loss is bounded, $\ell(f,z)\in[0,1]$. We aim to find a model with minimal target risk. Further, we assume full access to a set of $K$ benchmarks $\bR \equiv (\cR_1,\ldots, \cR_K) : \cF \to [0,1]^K$, each of them minimized in $\cF$ by $f_k$.
A learning algorithm $\cA$ observes $n$ i.i.d.\ samples from $Q$, denoted $\crl{z_i}_{i=1}^n\sim Q^{\otimes n}$, and outputs a model $\fhat\equiv\cA(\crl{z_i}_{i=1}^n;\bR,\cF)$ given knowledge of $\bR$ and $\cF$. 
In this paper, we analyze the excess risk $\excessRisk{\fhat}{\cH}:=\cR_Q(\fhat)-\inf_{h\in \cH}\cR_Q(h)$ 
with respect to the following two function classes:
\begin{equation*} 
     \cH = 
     \begin{cases}
         \cF & \text{transfer learning (\cref{sec:transfer-learning}}), \\
         \crl{f_k}_{k=1}^K & \text{model selection aggregation (\cref{sec: aggregation})}.
     \end{cases}
\end{equation*}
We study the sample complexity of achieving $\eps$-excess risk with high probability uniformly over $Q$ in some class of distributions $\cQ$; that is, the minimal sample size $n$ to achieve, for some $\delta\in(0,1)$,
\begin{equation*}
    \inf_{Q\in\cQ}\;\prob_{\crl{z_i}_{i=1}^n\sim Q^{\otimes n}}\prn{\excessRisk{\fhat}{\cH}\leq \eps}\geq 1-\delta.
\end{equation*}
Without any relationship between $\bR$ and $\cR_Q$, the sample complexity is generally equivalent to a standard learner that (only) has access to data $\crl{z_i}_{i=1}^n$ from $Q$,
and the additional availability of the benchmarks does not provide any advantage. 
In the following, we propose a new assumption to model relationships between benchmarks and the new task.

\subsection{Weak Monotonicity}
\label{subsec:weak-monotonicity}

To formalize our assumption, we use the following notation: for vectors $\v,\w\in \RR^K$ we write $\v\preceq \w$ if $v_k\leq w_k$ for all $k\in[K]$, $\v\prec\w$ if $\v\preceq\w$ and $\v\neq \w$, and $\v\llcurly \w$ if $v_k< w_k$ for all $k\in[K]$.
To model the information that benchmark performance contains about the new task, define the \emph{pre}order $\preceq_\bR$ on the hypothesis space $\cF$ \arxivonly{induced by the benchmarks $\bR$} as
\begin{equation*}
    f \preceq_\bR f' \quad :\iff\quad \bR(f)\preceq\bR(f').
\end{equation*}
This is only a preorder, because there can exist hypotheses $f$ and $f'$ that are \emph{incomparable}, i.e., neither of the two comparisons $f\preceq_\bR f'$ or $f\succeq_\bR f'$ holds, and if $\bR$ is not injective, $\preceq_\bR$ is not antisymmetric.
In an ideal setting, the benchmarks provide a reliable signal on the performance of the new task of interest. We formalize this intuition as $\cR_Q$ preserving the preorder $\preceq_\bR$.
\begin{restatable}[Monotonicity]{definition}{defweakmonotonicity}
    We say that $\cR_Q$ is \emph{(weakly) monotonic} with respect to $\bR$ if for all $f, f' \in \cF$, 
    \begin{equation}
        f \preceq_\bR f' \quad \implies \quad \cR_Q(f) \leq \cR_Q(f'). \label{eqn:weak-monotone} \tag{MON}
    \end{equation}
\end{restatable}
\arxivonly{In words, $\cR_Q$ is monotonic with respect to $\bR$ if any function that dominates another on the benchmarks in terms of the preorder $\preceq_\bR$ also has lower risk on $Q$.}
The sense in which this monotonicity is \emph{weak} is that the stronger conditional $\bR(f)\prec\bR(f')$ does not necessarily imply $\cR_Q(f)<\cR_Q(f')$.
Throughout the work, we use monotonicity and \emph{weak} monotonicity interchangeably.
Monotonicity is \arxivonly{directly }connected to \arxivonly{the well-known concept of }\emph{Pareto optimality}.
\begin{restatable}[Pareto optimality]{definition}{defparetooptimality}
\label{def:Pareto-optimality}
    A model $f\in\cF$ is called Pareto optimal \ifarxivelse{with respect to}{w.r.t.} $\cF$ and $\bR$ if there is no other model $f'\in\cF$ so that $\bR(f')\prec \bR(f)$. We denote the set of all Pareto optimal models in $\cF$ \ifarxivelse{with respect to}{w.r.t.} $\bR$ as $\pareto(\cF, \bR)\subseteq\cF$.
\end{restatable}
\arxivonly{In particular, }Pareto optimality and monotonicity are related through the following lemma.
In what follows, to avoid pathologies, we always implicitly assume that
the set $\bR(\cF)\subseteq [0,1]^K$ is compact, and $\cR_Q$ attains its minimum on $\cF$, i.e., there exists $\fstar\in\cF$ such that $ \cR_Q(\fstar)=\inf_{f\in\cF}\cR_Q(f)$.
\begin{restatable}[Monotonicity and Pareto optimality]{lemma}{monopareto}
\label{lem:monotonicity-pareto}
For every $f\in\cF$ there exists an $f'\in\pareto(\cF,\bR)$ with $f'\preceq_{\bR}f$. \ifarxivelse{And }{Moreover, }if $\cR_Q$ satisfies \eqref{eqn:weak-monotone}, then there exists an $\fstar\in\argmin_{f\in\cF}\cR_Q(f)$ that is Pareto optimal.
\end{restatable}

Notice that not all minimizers need to be in the Pareto set, and the fact that the Pareto set contains a minimizer does not directly imply monotonicity.
Beyond \cref{lem:monotonicity-pareto}, we can also show that monotonicity implies that the target risk is not sensitive to changes in $f$ that do not affect the benchmark\arxivonly{ performance}s.
\begin{restatable}[Monotonicity implies sufficiency]{lemma}{monosufficiency} \label{lem:weak-sufficiency}
    Suppose $\cR_Q$ is monotonic with respect to $\bR$. Then, there exists a monotonic map $s : \mathbb{R}^K \to \mathbb{R}$ such that $\cR_Q=s\circ \bR$ on $\cF$.
\end{restatable}

Given \cref{lem:monotonicity-pareto,lem:weak-sufficiency} (proven in \cref{proof:basics-monotonicity}), monotonicity may seem to be a strong assumption. However, the assumption becomes weaker as the number of benchmarks $K$ increases:
in general, the \ifarxivelse{more benchmarks there are}{larger $K$}, the fewer models are comparable in terms of the preorder, and hence the easier it is for $\cR_Q$ to be monotonic with respect to $\bR$. \arxivonly{
Moreover, monotonicity strictly generalizes some existing assumptions and is orthogonal to others, which we discuss next.}

\paragraph{Relation to other transfer models.} Assume that the benchmarks are risks over distributions $P_k$ on $\cZ$, that is, $\cR_k(f) = \EE_{z \sim P_k} \ell_k (f,z)$ for some $\ell_k(f,z)\in[0,1]$. 
It is then easy to show that monotonicity holds if $\ell_k=\ell$ for all $k$ and the target distribution is a mixture of the source distributions $Q\in \conv(P_1,\ldots,P_K)$ 
(see \cref{lem:Q-in-convex-hull} in the appendix).
Monotonicity is therefore at least as general as the assumptions made in classical multi-source domain adaptation with convex mixtures \citep{mansour2008domain,mansour2021theory,hoffman2018algorithms} and group DRO \citep{sagawadistributionally}. However, it is also strictly more general, as one can appreciate from \cref{ex:convex-hull-counterexample} in the appendix. Moreover, as we argue in \cref{subsec:Pareto-ERM,sec:separation-mansour}, our approach can also yield tighter bounds even under the mixture assumption. 

Another paradigm of modeling distribution shifts is based on invariances such as covariate or label shift \citep{shimodaira2000improving}, and, more generally, shared (or invariant) causal structure 
\citep{peters2016causal,rojas2018invariant,arjovsky2019invariant}. 
\ifarxivelse{In general, these methods shine when solving the target task requires some extrapolation, and are to some extent orthogonal to monotonicity; neither implies the other. For example, b}{But, b}oth covariate and label shift are neither necessary nor sufficient for monotonicity (see Lemma \ref{lem:covariate-and-label-shift}). 
And the same holds for the mixtures of conditionals assumption in \citep{zhang2015multi}.
\arxivonly{On a high level, the assumption of causal invariance serves a different purpose: instead of aiming to solve a specific target task, invariant predictors aim to generalize robustly
on a set of possible distributions.}

\subsection{Two Moduli of Monotonicity}
\label{subsec:modulus-of-monotonicity}
In practice, monotonicity may hold only approximately. In this section we introduce two moduli of monotonicity to quantify, for any target distribution, how much it violates monotonicity. 
The moduli are defined using the following notion of distance. 
\begin{restatable}[Pareto distance]{definition}{paretodistance}
\label{def:Pareto-distance}
    The \emph{Pareto distance} on $\cF$ with respect to $\bR$ is the map $\Paretodistance: \cF \times \cF \to \RR$ defined as
    \[\Paretodistance(f,f') = \max_{k \in [K]} \crl{\cR_k(f) - \cR_k(f')}.\]
\end{restatable}
\vspace{-0.2cm}
\ifarxivelse{Notice that, i}{I}n general, this distance is not symmetric and can be negative. 
\ifarxivelse{For a fixed $f_0$, the set $\crl{f\in\cF:\Paretodistance(f,f_0) \leq t}$ includes all functions that are at most $t$ worse than $f_0$ across all objectives $\cR_k$. In contrast, t}{T}he \emph{Pareto ball} $\crl{f\in\cF:\Paretodistance(f_0,f) \leq t}$ includes all functions that are not more than $t$ better than $f_0$ across all objectives, see \cref{fig:Pareto-geometry-tikz}\subref{subfig:Pareto-ball}.
Another way to express the Pareto distance is $\Paretodistance(f,f') = \min \big\{t\in\RR : \bR(f)  - t \one \preceq \bR(f')\big\}$ (see \cref{lem:pareto-distance-axioms}), where $\one=(1,\ldots,1)$ is the all-ones vector.
In this formulation, we can \arxivonly{readily }interpret the Pareto distance as
the smallest \arxivonly{uniform }amount $t\in\RR$ that the objectives $\cR_k(f)$ need to improve
so that \ifarxivelse{the performance}{it} is no worse than $\cR_k(f')$. 
\ifarxivelse{
\begin{figure}
    \centering
    \begin{subfigure}[c]{0.3\linewidth}
        \includegraphics[width=\linewidth]{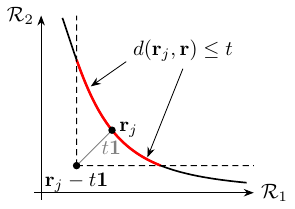}
        \caption{A Pareto ball.}
        \label{subfig:Pareto-ball}
    \end{subfigure}
    \begin{subfigure}[c]{0.3\linewidth}
        \includegraphics[width=\linewidth]{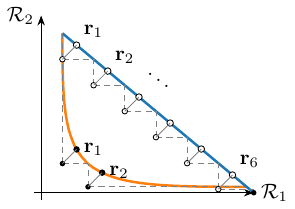}
        \caption{Two Pareto coverings.}
        \label{subfig:Pareto-covering-zoomed-out}
    \end{subfigure}
    \begin{subfigure}[c]{0.3\linewidth}
        \includegraphics[width=\linewidth]{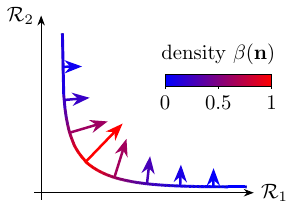}
        \caption{Density from \cref{thm:limiting-distribution}. }
        \label{subfig:Pareto-density}
    \end{subfigure}
    \caption{Using the Pareto quasi-metric from \cref{def:Pareto-cover} to cover the Pareto front. In \cref{subfig:Pareto-ball} the area covered by a single point $\r_j$ (i.e., a Pareto ball). The number of points required adapts to the geometry of the front (\ref{subfig:Pareto-covering-zoomed-out}). And in the limit, as $t\to 0$, the density of the covering $\beta$ depends on the normal vector (\ref{subfig:Pareto-density}), as specified in \cref{thm:limiting-distribution}: If any of the coordinates $n_k$ of the normal vector $\n$ is small, the density is low. \vspace{0cm}}
    \label{fig:Pareto-geometry-tikz}
\end{figure}
}{
\begin{figure*}
    \centering
    \begin{subfigure}[c]{0.3\linewidth}
        \includegraphics[width=\linewidth]{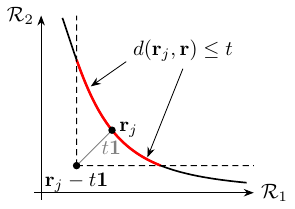}
        \caption{A Pareto ball.}
        \label{subfig:Pareto-ball}
    \end{subfigure}
    \begin{subfigure}[c]{0.3\linewidth}
        \includegraphics[width=\linewidth]{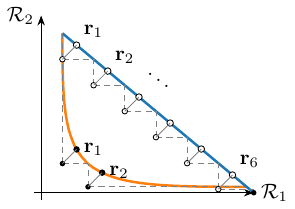}
        \caption{Two Pareto coverings.}
        \label{subfig:Pareto-covering-zoomed-out}
    \end{subfigure}
    \begin{subfigure}[c]{0.3\linewidth}
        \includegraphics[width=\linewidth]{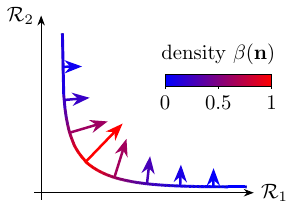}
        \caption{Density from \cref{thm:limiting-distribution}. }
        \label{subfig:Pareto-density}
    \end{subfigure}
    \caption{Using the Pareto quasi-metric from \cref{def:Pareto-cover} to cover the Pareto front. In \cref{subfig:Pareto-ball} the area covered by a single point $\r_j$ (i.e., a Pareto ball). The number of points required adapts to the geometry of the front (\ref{subfig:Pareto-covering-zoomed-out}). And in the limit, as $t\to 0$, the density of the covering $\beta$ depends on the normal vector (\ref{subfig:Pareto-density}), as specified in \cref{thm:limiting-distribution}. \vspace{-0.3cm}}
    \label{fig:Pareto-geometry-tikz}
\end{figure*}
}

With the Pareto distance, we can \arxivonly{now }define two ways \ifarxivelse{in which the assumption of exact weak monotonicity can be relaxed, and that }{to }quantify the \emph{degree} of monotonicity. Our results are expressed in terms of these moduli\arxivonly{, which can be computed for any combination of benchmark and target tasks given enough data}. They are closely related to the modulus of \arxivonly{(single source) }transfer from \citet{hanneke2024adaptive}.
\begin{restatable}{definition}{moduliofmonotonicity}
    \label{def:modulus-of-monotonicity} The \emph{upper modulus of monotonicity} of $\cR_Q$ and $\bR$ is the function $\uppermodulus:[0,1]\to[0,1]$,
    \ifarxivelse{
    \begin{align*}
        \uppermodulus(t):= \sup \big\{\cR_Q(f)&-\cR_Q(f'):\,f,f'\in\cF 
        \text{ such that }\Paretodistance(f,f')\leq t \big\}.
    \end{align*}
    }{
    \begin{align*}
        \uppermodulus(t)= \sup_{f,f'\in\cF
        :\Paretodistance(f,f')\leq t} \{\cR_Q(f)&-\cR_Q(f') \}.
    \end{align*}
    }
    Similarly, the \emph{lower modulus of monotonicity} is defined as the function $\lowermodulus:[0,1]\to[-1,1]$ given by
    \ifarxivelse{
    \begin{align*}
        \lowermodulus(t) = \inf\big\{\cR_Q(f')&-\cR_Q(f):f,f'\in\cF 
        \text{ such that }\Paretodistance(f,f')\leq -t\big\},
    \end{align*}
    }{
    \begin{align*}
        \lowermodulus(t) = \inf_{f,f'\in\cF:\Paretodistance(f,f')\leq -t }\{\cR_Q(f')&-\cR_Q(f)\},
    \end{align*}
    }
    where we set $\lowermodulus(t)=1$ if the set from above is empty.
\end{restatable}
By definition, the upper modulus of monotonicity satisfies for any $t\geq 0$ and $f,f'\in\cF$, that if $\cR_k(f) \leq \cR_k(f')+t$ for all $k$, then $\cR_Q(f) \leq \cR_Q(f')+ \uppermodulus(t)$; it measures how much worse a function can be on the target when it is at most $t$ worse on all source risks. On the other hand, the lower modulus satisfies that if $\cR_k(f)\leq \cR_k(f')-t$ for all $k$, then $\cR_Q(f)\leq \cR_Q(f')-\lowermodulus(t)$; that is, it measures the minimal gain on the target when all source risks are improved by at least $t$.
The lower and upper moduli are both \emph{non-decreasing} functions of $t$, and it \ifarxivelse{follows directly from the definition that}{holds that}
\begin{equation*} 
    \uppermodulus(0)= 0 \iff \lowermodulus(0)\geq 0\iff  \text{\eqref{eqn:weak-monotone}}.
\end{equation*}
\arxivonly{
\begin{wrapfigure}{r}{0.45\textwidth}
    \vspace{-0.7cm}
    \centering
    \includegraphics[width=\linewidth]{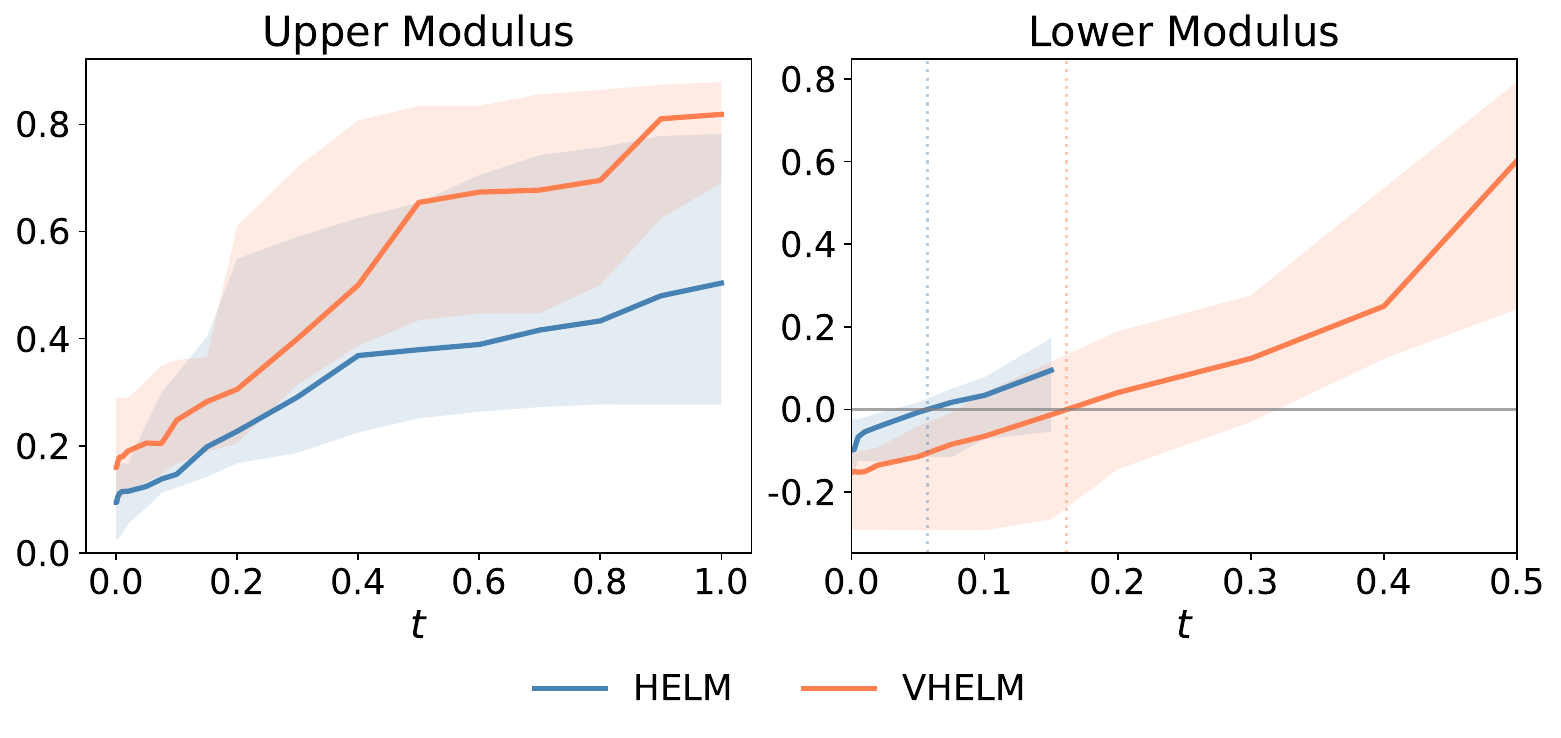}
    \caption{Median and quantiles of moduli (over all combinations of $K=3$ sources and target $Q$) for both \texttt{HELM} and \texttt{VHELM}. We only plot the lower modulus for $\lowermodulus(t)<1$.\protect\footnotemark}
    \label{fig:modulus HELM and VHELM}
    \vspace{-0.2cm}
\end{wrapfigure}
}
The results in \cref{sec:transfer-learning,sec: aggregation} show that a small upper modulus, or respectively a large lower modulus, when $t$ is close to zero, lead to stronger generalization bounds. 
\arxivonly{The values of the moduli at zero, if not exactly zero, give a sense for how close to monotonic the relationship between $\bR$ and $\cR_Q$ is.} 
In particular, $\uppermodulus(0)$ upper bounds the ``approximation error'', in the sense that the best model outside the Pareto set can have risk on $\cR_Q$ that is at most $\uppermodulus(0)$ smaller than the best model inside the Pareto set.
Further, it can be shown that, if $s$ from \cref{lem:weak-sufficiency} is a weighted average, then $\uppermodulus(t)\leq t$ and $\lowermodulus(t)\geq t$ (see \cref{lem:bound-modulus-Lipschitz} in the Appendix\arxivonly{ for a more general bound}).

\cameraonly{
\begin{figure}[H]
    \vspace{-0.3cm}
    \centering
    \includegraphics[width=\linewidth]{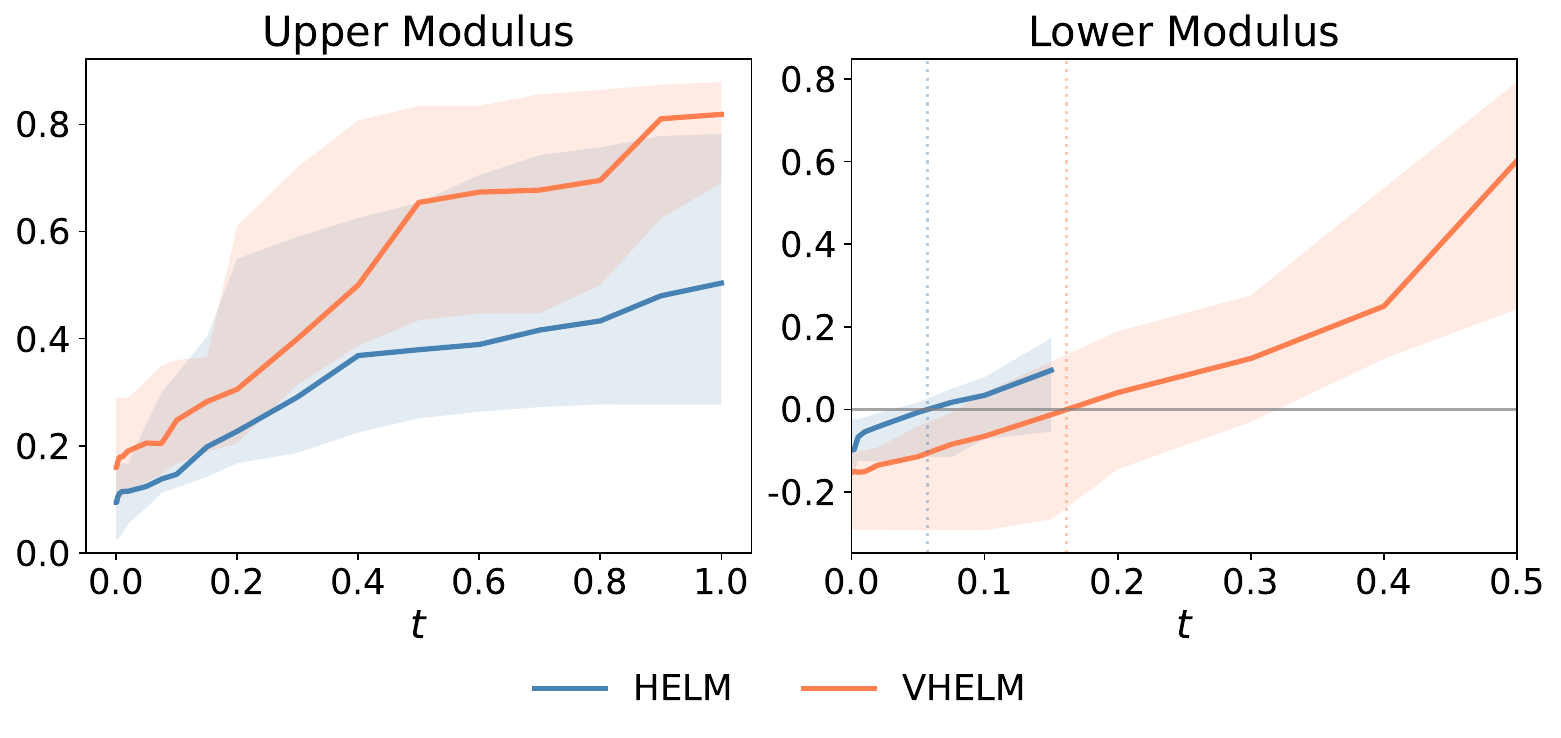}
    \caption{Median and quantiles of moduli (over all combinations of $K=3$ sources and target $Q$) for both \texttt{HELM} and \texttt{VHELM}.\protect\footnotemark \vspace{-0.3cm}}
    \label{fig:modulus HELM and VHELM}
\end{figure}
}
\paragraph{The moduli of \texttt{HELM} and \texttt{VHELM}.}
We now compute the moduli on \texttt{HELM} and \texttt{VHELM}\arxivonly{ to observe to which degree the monotonicity assumption holds in real-data scenarios}.
\cref{fig:modulus HELM and VHELM} (left) reports, for a fixed $K$, the upper modulus $\uppermodulus$ from Definition~\ref{def:modulus-of-monotonicity}, aggregated by taking the median across all combinations of $K$ sources and a target task $Q$. 
From \cref{fig:modulus HELM and VHELM}, we observe that the moduli are quite close to zero at $t=0$, suggesting that the datasets are \emph{almost} monotone. Further, the target performance decreases significantly with benchmark performance. 
\arxivonly{However, $\uppermodulus(0)>0$ and $\lowermodulus(0)<0$ for both \texttt{HELM} and \texttt{VHELM}; this indicates that for some source and target configurations, weak monotonicity does not exactly hold and implies that the optimal model for the new task may not lie in the Pareto set\arxivonly{ (although it may be close to it)},
suggesting that methods that work well for approximate monotonicity may be optimal on these datasets, cf.\ \cref{sec: beyond exact monotonicity}. }
\footnotetext{For larger $t$, the set defining the lower modulus is empty in \HELM and we cut off the plot there.}

\section{A New Complexity Measure: Pareto Covering Numbers}\label{subsec:geometry-covering}

To capture the statistical complexity, we first introduce 
a notion of geometry on the Pareto front\footnote{We can pull the geometry back onto the Pareto set. To do so, identify $\fR$ with the quotient space $\pareto(\cF, \bR)/\bR$. Formally, $f \sim g$ are equivalent when $\bR(f) = \bR(g)$. The equivalence class containing $f \in \pareto(\cF, \bR)$ is in bijection with $\bR(f) \in \fR$. }
\begin{equation*}
    \fR := \big\{\bR(f)\in[0,1]^K : f \in \pareto(\cF, \bR)\big\}.
\end{equation*}
First, we define the \emph{Pareto distance} $d: \fR \times \fR \to \RR$ on the front $\fR$ as the map  \(d(\r,\r') = \Paretodistance(f,f'),\) for $\r=\bR(f),\r'=\bR(f')$ using the Pareto distance
$\Paretodistance$ 
given in \cref{def:Pareto-distance}. This distance is a quasi-metric \emph{on the Pareto front}: it satisfies all the usual axioms of a metric except symmetry (see \Cref{lem:pareto-distance-axioms} and \cref{fig:Pareto-geometry-tikz}). We use it to define the notion of a $t$-Pareto cover and set. 
\begin{definition}[Pareto cover and set]
\label{def:Pareto-cover}
    Let $A \subset \fR$ be a set in the Pareto front. A subset $S \subset A$ is a \emph{$t$-Pareto cover} of $A$ if for all $\r \in A$, there is an $\r_0 \in S$ so that $d(\r_0,\r) \leq  t$.
    The $t$-\emph{Pareto covering number} $\Npar(t,A)$ is the minimal size of any $t$-Pareto cover of $A$.
    A set $\crl{f_1,\ldots,f_N}\subset\cF$ is a \emph{$t$-Pareto set} with respect to $\bR$ in $\cF$ if $\bR(\crl{f_1,\ldots,f_N})$ is a $t$-Pareto cover of $\fR$. We write $\Npar(t,\cF,\bR) := \Npar(t,\fR)$.
\end{definition}

\cref{subfig:Pareto-ball} visualizes a Pareto ball; the part of $\fR$ covered by a single point, and \cref{subfig:Pareto-covering-zoomed-out} shows how a collection of Pareto balls forms a Pareto cover.
A consequence of the definition of a $t$-Pareto set $\crl{f_1,\ldots,f_N}$ and \cref{lem:monotonicity-pareto} is that for every $f\in\cF$, there is some $f_j$ so that $\bR(f_j) \preceq \bR(f)+t\one$ (or equivalently $d_\bR(f_j,f)\leq t$), as visualized in \cref{fig:Pareto-geometry-tikz}. Also, a crude upper bound on the $t$-Pareto covering number is always $\log \Npar(t,\cF,\bR)\leq (K-1)\log(1+1/t)$, as proved in \cref{lem:linf-bound-covering} in \cref{sec:additional-Pareto-covering}. \arxivonly{However, it can be \emph{much} smaller.}

\subsection{The Limiting Distribution}
The next result quantifies the growth rate of the $t$-Pareto covering number as $t \to 0$. This rate depends on the curvature of the Pareto front. 
From a technical perspective, it is a main contribution of this work.

\begin{assumption}[Nice Pareto front] \label{assumption: nice}
    Let $\fR \subset [0,1]^K$ be a compact and smooth $(K-1)$-dimensional submanifold (with boundary) that admits a smooth normal field $\n: \fR \to \SSS^{K-1}$ in its interior, where the unit normal vector at $\r \in \fR$, denoted by $\n(\r)=(n_1(\r),\ldots,n_K(\r))$, is uniformly bounded away from zero; $n_k(\r) > \eta > 0$ for all $\r \in\fR$ and $k \in [K]$. We say such a Pareto front is \emph{nice}. 
\end{assumption}

\begin{theorem}
\label{thm:limiting-distribution}
    Let $\sH$ denote the $(K-1)$-dimensional Hausdorff measure and \cref{assumption: nice} hold. 
    Then, there is a constant $C_K$ that depends only on $K$ so that, for any $A \subset \fR$ that has a boundary $\partial A$ of $\sH$-measure zero, it holds that
    \[\lim_{t \to 0}\,  \Npar(t, A) \cdot t^{K-1} = \mu(A),\]
    where $\mu$ is the following measure on the Pareto front $\fR$ with Borel measurable sets $B$:
    \ifarxivelse{
    \begin{align*} 
        \mu(B) &= \!\int_B \!\beta\big(\n(\r)\big)\d\sH(\r) \quad \text{ with }\quad \beta(\n) = C_K\frac{ \prod_{k \in [K]} n_k}{(\one \cdot \n)^{K-1}}.
    \end{align*}
    }{
    \begin{align*} 
        \mu(B) &= \!\int_B \!\beta\big(\n(\r)\big)\d\sH(\r),\quad \beta(\n) = C_K\frac{ \prod_{k \in [K]} n_k}{(\one \cdot \n)^{K-1}}.
    \end{align*}
    }
\end{theorem}

We provide further comments on the role of the assumptions at the start of the proof in \cref{proof:limiting-distribution}.\arxivonly{\par}
It may initially be surprising that the density $\d\mu/\d\sH$ is not uniform: the measure depends on the normal direction $\n(\r)$ and puts more mass on parts of the Pareto front with larger $\beta$.
\arxivonly{Intuitively, $\beta$ is larger when trade-offs are somewhat symmetric (decreasing one risk increases another by roughly the same amount, and vice versa) and small where at least one $n_k$ is close to zero.}
In fact, a small $\beta$ is closely related to a large gain-to-loss ratio for some pair of benchmarks; this is called an improper trade-off \citep{geoffrion1968proper} and is generally undesirable from a multi-objective perspective. Hence, the limiting distribution avoids putting weight on (nearly) improper solutions, implying that the upcoming algorithms \emph{hedge} using only proper trade-offs. The density depending on the normal vectors is visualized in \cref{subfig:Pareto-density}. 

\arxivonly{
Our theorem gives a precise characterization of the limit as $t\to 0$ for a phenomenon that can already be easily observed for two Pareto fronts in two dimensions for moderately large $t$. In \cref{subfig:Pareto-covering-zoomed-out}, we see that for a Pareto front with solely proper trade-offs (blue; where the normal vectors are large in all coordinates), the Pareto cover is uniformly spread across the front and requires many points. 
When the Pareto front has larger curvature (orange), fewer points are in regions where some $n_k$ is close to zero, and the trade-offs are improper. In total, potentially far fewer points are needed to cover the Pareto front with large curvature. As we shall see in \cref{subsec:Pareto-ERM}, this alleviates statistical complexity.
}

The measure $\mu$ also naturally induces a distribution over the Pareto set:
Without loss of generality (by \cref{lem:weak-sufficiency}), consider the quotient space $\pareto(\cF,\bR)/\bR$.
By pulling back $\mu$ from the Pareto front, we define the \emph{prior distribution} $\pi=\bR^\star \mu /\mu(\fR)$ over the quotient space with density
\begin{equation}
\label{eqn:prior-distribution}
    \frac{\d\pi}{\d (\bR^\star \sH)}(f) \propto \beta\big(\n(\bR(f))\big), \quad f \in \pareto(\cF,\bR)/\bR.
\end{equation}
As we will see, this enables a PAC-Bayesian perspective of the learning problem, with the prior naturally encoding where a learner should expect to find the solution of $\cR_Q$.

\arxivonly{
\begin{wrapfigure}{r}{0.25\textwidth}
    \centering
    \vspace{-0.6cm}
    \input{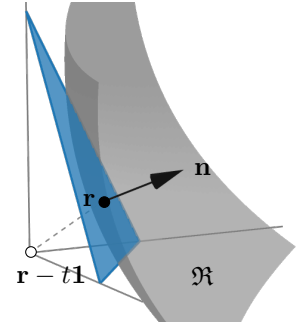}
    \vspace{-0.5cm}
    \caption{The Pareto ball in the tangent space of a Pareto front is a simplex.\vspace{-0.4cm}}
    \label{fig:tangent}
    \vspace{-0.5cm}
\end{wrapfigure}
\paragraph{Proof outline of \cref{thm:limiting-distribution}.}
The proof of \cref{thm:limiting-distribution}, provided in \cref{proof:limiting-distribution}, is based on two main steps. In the first step, we introduce technical machinery that helps reduce the problem from \emph{smooth} Pareto fronts to \emph{linear} ones. The second step computes covering numbers for linear fronts.

The core lemma in the first step is \cref{lem:piecewise-approximation}. It constructs a piecewise-linear approximation of the Pareto front by cutting up the Pareto front into many small pieces, and approximating each piece by a corresponding (linear) tangent space, as in \cref{fig:tangent}. As we eventually compute covering numbers on each linear front, we need to show that this linearization procedure approximately preserves the quasi-metric structure. Indeed, arbitrarily little local distortion can be incurred by cutting the Pareto front into sufficiently fine pieces.

For the second step, we compute the limiting distribution for linear Pareto fronts in \cref{prop:covering-to-minkowski}. Given a $(K-1)$-dimensional linear front and scale~$t$, there is a simplex (a convex combination of $K$ points in that hyperplane) so that $t$-Pareto balls correspond to translations of that simplex. And so, for linear fronts, the problem of constructing a Pareto cover is equivalent to a classic problem from geometry of covering a space by translates of a fixed simplex: the \emph{translative covering} \citep{naszodi2018flavors} of the front.
The periodic translative covering density of simplices turns out to be unique, which then implies that the covering density reduces to the reciprocal of the volume of the simplex, yielding the final formula of $\beta$ in \cref{thm:limiting-distribution}.
}

\ifarxivelse{
\subsection{Discussion}
\label{subsec:covering-discussion}

\paragraph{Computation of a Pareto covering.}
In general, \emph{exactly} computing a minimal $t$-Pareto cover is computationally hard \citep{zitzler2008quality,papadimitriou2000approximability,chvatal1979greedy}. 
For finite hypothesis classes, the problem becomes much simpler to analyze: Let $M$ be the number of models, $K$ the number of source benchmarks, and $P \leq M$ the size of the Pareto set. 
For $t=0$ (used in our experiments), we only need to compute the exact Pareto set. This can be done by checking pairwise dominance across all models, which requires at most $O(KM^2)$ comparisons. In practice, this may already yield a good candidate set.
For $t>0$, one can first restrict attention to the Pareto set and then build a $t$-Pareto cover on top using standard greedy set-cover heuristics to obtain a small approximate cover with a logarithmic size overhead and computational cost $O(KP^2)$ \citep{chvatal1979greedy}. Notice that a logarithmic size overhead is essentially irrelevant in the upcoming bounds.
}{
\subsection{Comparison with Other Coverings}
\label{subsec:covering-discussion}
}

\arxivonly{\paragraph{Comparison with other coverings.}}
Instead of the non-uniform Pareto covering, one may also construct uniform coverings using norm balls, such as $\ell_\infty$ or $\ell_2$-balls \citep{zhang2024gliding}.
Alternatively, one could also cover the simplex in $\ell_1$-norm and then minimize a weighted sum of the risks \citep{mansour2021theory}. We compare the coverings in more detail in \cref{subsec:covering-comparison}.
A discussion of other Pareto front approximations can be found in \citep{zitzler2008quality,vassilvitskii2005efficiently,papadimitriou2000approximability,compton2026ratio}. \arxivonly{They usually consider ratio covers, where for some $\alpha>1$ and for every $\r\in\fR$ there is a $\r_0$ in the covering with $\r_0\preceq \alpha \r$.}

We now explicitly calculate the density in \cref{eqn:prior-distribution} for a family of example Pareto fronts, and then numerically validate \cref{thm:limiting-distribution} on this example\cameraonly{ (cf.\ \cref{fig:Pareto-covering-small})}. \ifarxivelse{We further compare with the other coverings discussed above.}{We compare this distribution with the other coverings in \cref{subsec:covering-comparison}.}
\cameraonly{
\begin{figure}
    \centering
    \includegraphics[width=\linewidth]{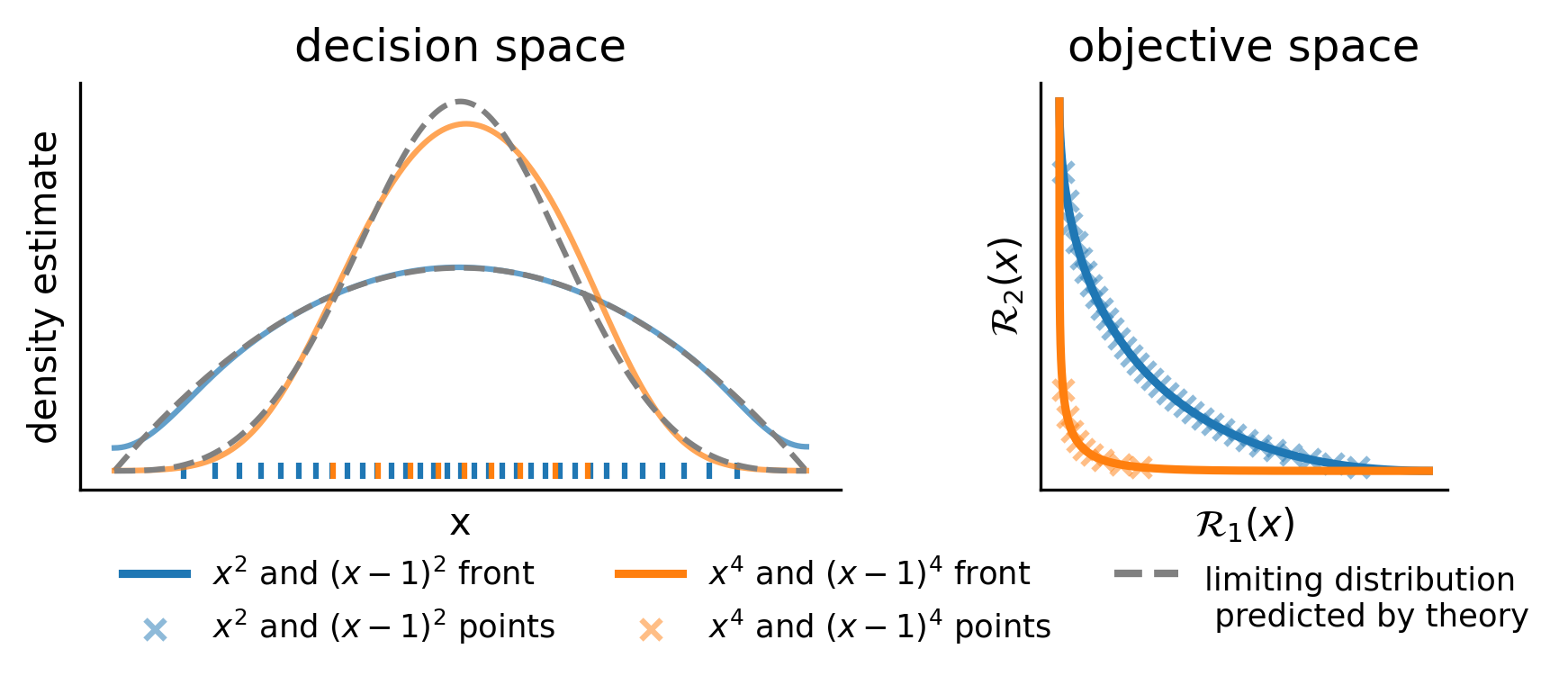}
    \caption{Limiting distribution from \cref{ex:limiting-distribution}, cf.\ \cref{fig:Pareto-covering}.}
    \vspace{-0.2cm}
    \label{fig:Pareto-covering-small}
\end{figure}
}
\begin{example}
\label{ex:limiting-distribution}
    Let $\cF=[\eps,1-\eps]$ and $K=2$. Then, take the risks $\cR_1(x)=\abs{x}^p$ and $\cR_2(x)=\abs{x-1}^p$ with $p\in\NN$. The density of $\pi$ with respect to Lebesgue measure $\lambda$ is proportional to the harmonic mean of the derivatives
    $$
    \frac{\d\pi}{\d\lambda}(x) \propto \frac{px^{p-1}(1-x)^{p-1}}{x^{p-1}+(1-x)^{p-1}}.
    $$
\end{example}
\arxivonly{
In \cref{fig:Pareto-covering}, we visualize the setting of \cref{ex:limiting-distribution} for the values $p=2$ and $p=4$. \cref{subfig:eps-pareto-set-fronts} shows the Pareto fronts with a numerically computed minimal Pareto covering as well as the two different coverings from above. \cref{subfig:eps-pareto-set-distributions} shows that each covering method induces a different distribution over the decision space, and depends differently on the Pareto front geometry.
Clearly, the Pareto covering number puts more mass on parts corresponding to the ``elbow'' of the Pareto fronts. \cref{subfig:eps-pareto-set-scale} shows that, unsurprisingly, all coverings scale as $t^{-(K-1)}=t^{-1}$, but there is a constant factor gap between them, with the Pareto covering being much smaller at the same scale. Therefore, the other two coverings can be ``unnecessarily large''. Finally, \cref{subfig:eps-pareto-set-curve,subfig:eps-pareto-set-distributions} together show that only the Pareto covering adapts to the Pareto front geometry as one may expect: the $\ell_2$-norm covering number \emph{increases} as more favorable trade-offs become available ($p=4$), and even puts more mass on improper regions. At the same time, the simplex $\ell_1$-covering number is completely agnostic to the geometry of the front, and while the induced distributions are more similar to those of the Pareto covering, they still do not avoid improper regions in the case $p=4$.
\begin{figure}[t]
    \begin{minipage}[c]{0.72\textwidth}
        \begin{subfigure}{\linewidth}
            \includegraphics[width=\linewidth, trim={0cm 0cm 0cm 0cm}, clip]{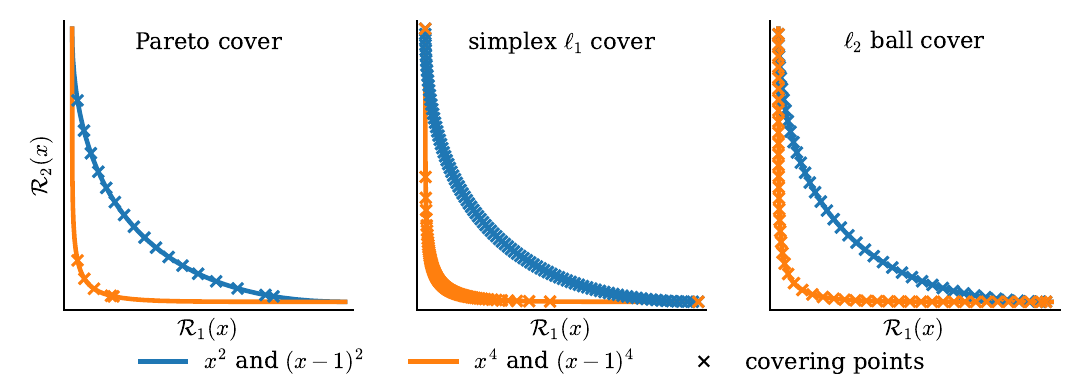}
            \caption{Three different covers of two Pareto fronts at the same scale $t$.}
            \label{subfig:eps-pareto-set-fronts}
        \end{subfigure}
        \begin{subfigure}{\linewidth}
            \includegraphics[width=\linewidth, trim={0cm 0.2cm 0cm 0cm}, clip]{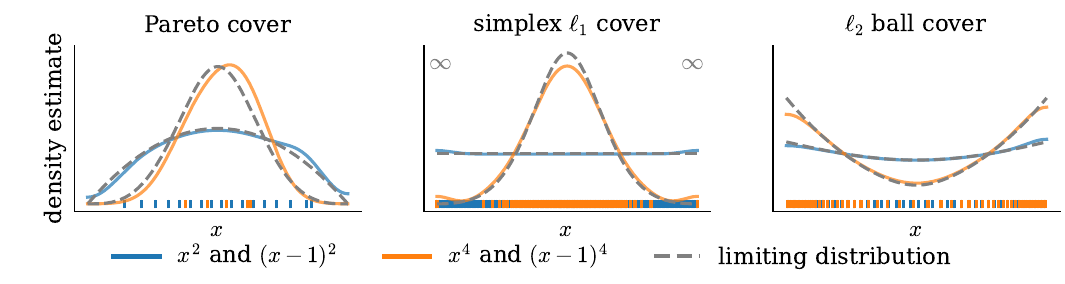}
            \caption{The density estimates and limiting distributions in decision space.}
            \label{subfig:eps-pareto-set-distributions}
        \end{subfigure}
    \end{minipage}
    \hfill
    \begin{minipage}[c]{0.275\textwidth}
    \begin{subfigure}{\linewidth}
        \includegraphics[width=\linewidth, trim={0cm 0.4cm 0cm 0cm}, clip]{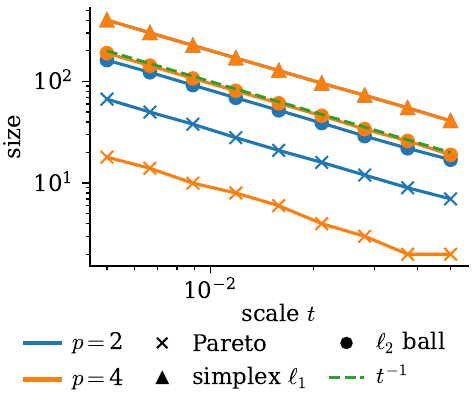}
        \caption{Cover size dependence on $t$.}
        \label{subfig:eps-pareto-set-scale}
    \end{subfigure}
    \begin{subfigure}{\linewidth}
        \includegraphics[width=\linewidth, trim={0cm 0.5cm 0cm 0cm}, clip]{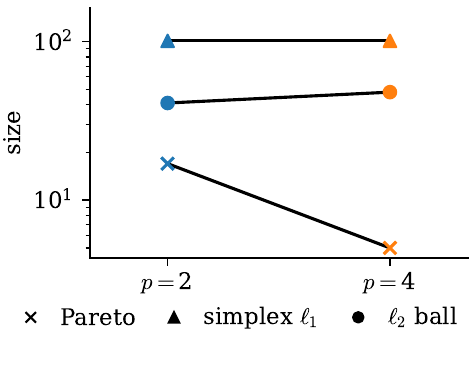}
        \caption{Cover size dependence on $p$.}
        \label{subfig:eps-pareto-set-curve}
    \end{subfigure}
    \end{minipage}
    \caption{(\subref{subfig:eps-pareto-set-fronts}) For a fixed scale $t$ and two different Pareto fronts from \cref{ex:limiting-distribution}, we plot a Pareto covering and two different covers from \cref{subsec:covering-discussion,subsec:covering-comparison}. 
    (\subref{subfig:eps-pareto-set-distributions}) We plot a kernel density estimate and the limiting distribution from \cref{thm:limiting-distribution} (for the other two limiting distributions, see \cref{subsec:covering-comparison}). 
    (\subref{subfig:eps-pareto-set-scale}) As a function of the scale $t$, the covering numbers are all of order $t^{-(K-1)}=t^{-1}$, but the constant factor depends on the geometry and the covering method.
    (\subref{subfig:eps-pareto-set-curve}) The Pareto covering number decreases with more favorable trade-offs, whereas the simplex $\ell_1$ cover is indifferent, and the $\ell_2$-norm cover even increases.}
    \label{fig:Pareto-covering}
\end{figure}
}

\section{Transfer Learning under Monotonicity}\label{sec:transfer-learning}

In this section, we bound the statistical complexity of learning under (approximate) monotonicity.

\subsection{The Statistical Complexity of Transfer Learning}
\label{subsec:Pareto-ERM}
In this section, we discuss two estimators based on the \emph{empirical risk}, defined as the average loss on the i.i.d.\ sample $\crl{z_i}_{i=1}^n \sim Q^{\otimes n}$, that is, $\Rhat_Q(f):= \frac{1}{n}\sum_{i=1}^n \ell(f,z_i)$. 

\paragraph{Pareto ERM.} The first algorithm
is a simple two-step procedure: for some fixed $t>0$, 
\begin{enumerate}[leftmargin=*,itemsep=-0.1pt,topsep=-2pt,label=(\roman*)]
    \item build a minimal $t$-Pareto set $\crl{f_1,\ldots,f_N}$, and
    \item compute \arxivonly{the ERM} $\smash{\Paretoerm{t}\in\argmin_{f\in\crl{f_1\ldots,f_N}}\Rhat_Q(f)}$.
\end{enumerate}
We call $\Paretoerm{t}$ the \emph{Pareto Empirical Risk Minimizer}\arxivonly{ (Pareto ERM)}.

\paragraph{Pareto EW.}
For the second algorithm, we consider again the quotient space $\pareto(\cF,\bR)/\bR$, and let \cref{assumption: nice} hold.
Then, we can let \(\pi\) be the prior distribution supported on \(\pareto(\cF,\bR)/\bR\) from \cref{eqn:prior-distribution}, induced by the limiting distribution in \cref{thm:limiting-distribution}. The Exponential Weights (EW) posterior $\rhohat\ll \pi$ with temperature \(\lambda>0\) is defined by its Radon-Nikodym derivative
\begin{equation}
\label{eq:ew-posterior}
    \frac{\d\rhohat}{\d\pi}(f)
    \propto \exp\!\big(-\lambda\,\Rhat_Q(f)\big),
    \quad f\in\pareto(\cF,\bR)/\bR.
\end{equation}
We call the aggregated predictor $f_{\rhohat}(\cdot):=\EE_{f\sim \rhohat}[f(\cdot)]$ the \emph{Pareto Exponential Weights} (Pareto EW).

Note that the $t$-Pareto set concentrates in different regions of the Pareto set (\cref{fig:Pareto-geometry-tikz}), which ``biases'' the ERM according to the geometry.
For the Pareto EW, this is even more explicit, as it directly endows the Pareto set with a prior.
In both cases, the better trade-offs there are, the more both estimators \emph{hedge} by preferring these trade-offs over other Pareto optimal models.\footnote{The bounds do not differ significantly between the two estimators; we include the Pareto EW to highlight the mechanism of hedging depending on the Pareto front through \cref{thm:limiting-distribution}.\arxivonly{ In particular, \cref{thm:Geometry-complexity-upper-bound} relies on the prior putting sufficiently much mass on near-optimal models (\cref{lem:mass-model-ball}).}}
We formalize the benefits of hedging for transfer in the following upper bounds on the excess risk $\excessRisk{\fhat}{\cF}=\cR_Q(\fhat)-\inf_{f\in\cF}\cR_Q(f)$. Recall that we assume that the loss is bounded $\ell\in [0,1]$, $\uppermodulus(t)$ is the upper modulus defined in  \cref{def:modulus-of-monotonicity}.
\begin{theorem} 
\label{thm:Geometry-complexity-upper-bound}
    With probability at least $1-\delta$, the Pareto ERM $\Paretoerm{t}$ achieves \cameraonly{\vspace{-0.2cm}}
    \begin{equation*}
        \excessRisk{\Paretoerm{t}}{\cF}\leq \uppermodulus(t) + \sqrt{\frac{2\log (2\Npar(t,\cF,\bR)/\delta)}{n}}.
    \end{equation*}
    Moreover, in the setting of \cref{thm:limiting-distribution}, assume that for some $\fstar\in\argmin_{f\in\cF}\cR_Q(f)$ its benchmark values $\bR(\fstar)$ are in the relative interior point of $\fR$ and that the loss is convex. Then there exists a $t_0>0$ such that for all $t\in(0,t_0]$ 
    and for $\lambda=\sqrt{8n\log((4\Npar(t,\cF,\bR))/\delta)}$, the EW posterior \(\rhohat\) achieves with probability at least \(1-\delta\)
\begin{align*}
    \excessRisk{\ParetoEW}{\cF}\leq \uppermodulus(t) +\sqrt{\frac{2\log\prn{4\Npar(t,\cF,\bR)/\delta}}{n}}.
\end{align*}
\end{theorem}
The proof is in \cref{proof:Geometry-complexity-upper-bound}.
First, observe that both estimators are consistent whenever there exists a sequence $t_n\to 0$ so that $\uppermodulus(t_n)\to 0$ as $n\to \infty$ (i.e., weak monotonicity holds). At the same time, note that the bounds hold even when monotonicity is violated: if $\uppermodulus(0)>0$, we simply incur a potentially 
irreducible approximation error of the Pareto set, as discussed in \cref{subsec:modulus-of-monotonicity}.

For example, when a decreasing upper bound on $\uppermodulus$ is known, we can choose $t$ to balance the two terms appearing in \cref{thm:Geometry-complexity-upper-bound}, exemplified in the following corollary for a linearly decaying modulus as $t\to 0$.
\begin{corollary}
\label{cor:tstar}
    Suppose that $\uppermodulus(t)\leq Lt$ with $L$ known to the learner. Define
    \begin{equation*}
        t_n = \inf\crl{t>0: t^2\geq \frac{2\log\prn{2\Npar(t,\cF,\bR)}}{L^2 n}}.
    \end{equation*}
    Then for any $t>t_n$, \arxivonly{$\Paretoerm{t}$ achieves }with probability at least $1-\delta$, 
    \begin{equation*}
        \excessRisk{\Paretoerm{t}}{\cF}\leq  2Lt+\sqrt{\frac{2\log (1/\delta)}{n}}.
    \end{equation*}
    If $n>\frac{2\log 2}{L^2}$ and $K>1$, then
    it holds that $t_n^2\leq 1\wedge \tfrac{K-1}{L^2n}\log(1+\tfrac{16L^2n}{K-1})$. An analogous bound holds under the additional assumptions of \cref{thm:Geometry-complexity-upper-bound} for the Pareto EW.
\end{corollary}
The proof is a direct consequence of \cref{thm:Geometry-complexity-upper-bound} and can be found in \cref{proof:tstar}. 
The second part of the corollary uses a worst-case upper bound on the Pareto covering number (\cref{lem:linf-bound-covering}) to obtain a uniform bound on $t_n$. It implies that the complexity of the transfer problem is at most of order $\sqrt{K/n}$ (up to log factors).
In \cref{prop:geometry-fast-rates} of \cref{subsec:aggregation-for-transfer}, we demonstrate that for strongly convex losses, by using a modified version of Pareto ERM, we can achieve fast rates 
in the worst case of order $K/n$.

\citet{mansour2021theory} also prove a $\sqrt{K/n}$ rate for target tasks that are mixtures of the sources, i.e., which are expectations of the same loss over distributions $Q\in\conv(P_1,\dots,P_K)$. Under their mixture assumption, 
we have $\overline\omega(t)\le t$ (\cref{lem:Q-in-convex-hull,lem:bound-modulus-Lipschitz}) and the corollary applies with the choice of $t_n \asymp \sqrt{K/n}$. Thus our result recovers their worst-case rate, while yielding potentially smaller bounds (by a constant factor characterized by $\mu(\fR)$ with $\mu$ from \cref{thm:limiting-distribution}) in benign regimes where $\log \Npar(t,\cF,\bR)\ll K\log(1/t)$ (e.g., when the loss function exhibits strong curvature). Indeed, as a result of using an $\ell_1$-norm discretization of the simplex, their bound is agnostic to the geometry of the Pareto front, whereas ours is adaptive (cf.\ \cref{fig:Pareto-covering}). In \cref{sec:separation-mansour}, we prove that the separation between the algorithm proposed by \citet{mansour2021theory} and Pareto ERM is not only in the upper bound; it is indeed suboptimal even when the target is a mixture.

\paragraph{A matching lower bound.}
In \citet[Theorem 7]{mansour2021theory}, a matching lower bound of order $\sqrt{K/n}$ is also established.
However, it is only under the very worst-case geometry that this lower bound applies.
We now show a refined lower bound for any \emph{fixed} Pareto covering number, and that \cref{cor:tstar} is tight in the following sense: If we fix $t$, the Pareto covering number $\Npar$ at scale $t$, and the modulus $\uppermodulus(t)=t$, we can find a problem instance where the rate obtained by \cref{cor:tstar} is tight up to constants.
\begin{theorem}
\label{thm:Geometry-complexity-lower-bound}
     Let $K\geq 4$. Then, for every $5\leq  N\leq \exp\prn{K-1}$, $n \geq 64 \log (N/4)$ and $t\leq \sqrt{\log(N/4)/64 n}$,
     there exists a $(\cF,\bR,\crl{\cR_Q:Q\in\cQ})$ so that (i) $\uppermodulus(t)=t$ for all $Q\in\cQ$, (ii) $\Npar(t,\cF,\bR)=N$ and (iii), the minimax excess risk is lower bounded by
    \begin{equation*}
        \inf_{\fhat}\sup_{Q\in\cQ} \EE\brk{\excessRisk{\fhat}{\cF}}\geq t,
    \end{equation*}
    where the infimum is taken over all \arxivonly{(potentially improper) }estimators.
\end{theorem}
The proof is in \cref{proof:Geometry-complexity-lower-bound}. For $\log N = K-1$, we again obtain \emph{in the worst case} a rate of $\sqrt{K/n}$, matching \citet[Theorem 7]{mansour2021theory}.
\arxivonly{A caveat of our bound is that it constructs both a Pareto front and a set of risks. One may also want to prove that for \emph{any} given Pareto front, the covering number tightly characterizes statistical hardness in a minimax sense. But it turns out that this is a much more difficult problem, not least because packing and covering numbers need not scale the same for quasi metrics, and resolving it would likely require \emph{localization} arguments in both upper and lower bounds.
We leave it as interesting future work to resolve the exact hardness \emph{per Pareto front}.}

\subsection{A Parameter-free \arxivonly{Consistent }Estimator for Finite Classes}
\label{sec: beyond exact monotonicity}
In the previous section, we derived upper bounds that are consistent whenever $\uppermodulus(t)\to 0$ as $t\to 0$, i.e., weak monotonicity holds. 
However, \arxivonly{a primary motivation of the setting proposed in this paper is foundation model benchmarking, where $\cF$ is a finite \emph{model zoo} and the sources are benchmark tasks\textemdash and} we saw in \cref{fig:modulus HELM and VHELM} that for instance both \HELM and \VHELM only approximately satisfy monotonicity. 
\arxivonly{Even though the upper bound in \Cref{thm:Geometry-complexity-upper-bound} holds when monotonicity is violated and $\uppermodulus(0) > 0$, the estimators introduced so far are not consistent and incur a bias.}
In this section, we discuss estimators that can always achieve consistency, but converge faster the closer to exact monotonicity the setting is. 

We begin by introducing the following threshold\arxivonly{ that depends on the target distribution $Q$}:
\[
\tlower := \inf\crl{t\in[0,1]:\ \lowermodulus(t)\ge 0}.
\]
\ifarxivelse{Intuitively, i}{I}t captures the minimum amount that all risks need to be improved so that the target risk is never worse. 
This quantity is closely related to the pivot considered by \citet[Definition~15]{hanneke2024adaptive} in the single-source setting. 
\cref{fig:modulus HELM and VHELM} illustrates this threshold \arxivonly{empirically }for \HELM and \VHELM.

\paragraph{Margin ERM.}
Define the $\gamma$-margin pruned class for $\gamma=0$ as $\Fpruned{0}:= \pareto(\cF,\bR)$ and for $\gamma> 0$ as
\ifarxivelse{
\begin{equation}
\label{eqn:Fpruned}
    \begin{aligned}
    \Fpruned{\gamma} &= \crl{f\in\cF: \text{there exists no } f'\in\cF \text{ such that } \bR(f')\llcurly \bR(f)-\gamma\one} \\
    &=\crl{f\in \cF: \Gamma(f) \leq \gamma} \quad \text{where} \quad  \Gamma(f) := -\inf_{f'\in\mathcal F}\Paretodistance(f',f).
\end{aligned}
\end{equation}
}{
\begin{equation}
\label{eqn:Fpruned}
    \begin{aligned}
    \Fpruned{\gamma} &=\crl{f\in \cF: \Gamma(f) \leq \gamma},
\end{aligned}
\end{equation}
where $\Gamma(f) := -\inf_{f'\in\mathcal F}\Paretodistance(f',f)$.
}
Intuitively, $\Gamma(f)\leq \gamma$ means that no other $f'\in\cF$ can strictly beat $f$ by more than $\gamma$ in all coordinates at the same time, and $f$ is within a ``margin'' of $\gamma$ to the Pareto set.
Define the $\gamma$-margin ERM as $\smash{\Thresholderm{\gamma}\in \argmin_{f\in\Fpruned{\gamma}} \Rhat_Q(f)}$.
When $\smash{\gamma> \tlower}$, the definition of $\smash{\tlower}$ guarantees that at least one optimal model must lie within the margin $\gamma$, and hence the margin ERM
can be consistent\ifarxivelse{. A quite general error bound for this strategy, more discussion, and an example are given in \cref{app_subsec: monotonicity above a threshold}}{ (\cref{app_subsec: monotonicity above a threshold})}.
However, \ifarxivelse{these benefits can only be achieved with}{this requires} oracle knowledge of a valid $\smash{\gamma > \tlower}$. 
Instead, for \emph{finite} hypothesis classes, we now propose a\cameraonly{n} \arxivonly{hyperparameter-free alternative that uses an }adaptive choice of the parameter $\gamma$ based on \ifarxivelse{intuition reminiscent of }{a version of }structural risk minimization. \arxivonly{In particular, the more data we have, the less we enforce the inductive bias\textemdash in this case, focusing on the Pareto set\textemdash and we include more of the entire function class by increasing $\gamma$. This yields the \emph{adaptive margin ERM}.}

\paragraph{Adaptive Margin ERM.}
The adaptive method is summarized by the following simple algorithm. 
\begin{enumerate}[leftmargin=*,itemsep=-0.1pt,topsep=-2pt,label=(\roman*)]
    \item For all $f\in\cF$ compute the margin $\Gamma(f)$, and sort increasingly to obtain $\gamma_{(1)}\leq \dots \leq \gamma_{(\abs{\cF})}$.\footnote{Here $\gamma_{(j)}$ denotes the $j$-th element of the increasingly sorted list of margins $\crl{\Gamma(f):f\in\cF}$.}
    \item Let $m_n := \min\{|\mathcal F|,\max\{n,|\pareto(\cF,\bR)|\}\}$ and choose $\tadapt:=\gamma_{(m_n)}$.
    \item Return $\AdaptiveThresholdERM\in \arg\min_{f\in\Fpruned{\tadapt}}\Rhat_Q(f)$ where $\Fpruned{\tadapt}$ is defined in \cref{eqn:Fpruned}.
\end{enumerate}
The adaptive estimator is an instance of the margin ERM with $\gamma = \gamma_n$.
Roughly speaking, as $n$ increases, this method runs ERM on an expanding subset of models ordered by their margin to the Pareto set and of size approximately $m_n$\arxivonly{; ties may cause this subset to contain more than $m_n$ models}. For large sample sizes $n$, this strategy recovers regular ERM on all of $\cF$. For this strategy, we can prove an adaptive excess risk bound (proof in \cref{proof:monotonicity-above-threshold-adaptive}).
\begin{restatable}{proposition}{boundadaptiveMAT}
\label{prop:monotonicity-above-threshold-adaptive}
    With probability at least $1-\delta$ it holds that $\excessRisk{\AdaptiveThresholdERM}{\cF} \leq \eps_n + 40\sqrt{\log(8|\Fpruned{\tadapt}|/\delta)/n}$,
    for a sample size-dependent error $\eps_n$ specified in the proof.
    Moreover, for all $n\geq n_0:= \min_{\gamma>\tlower}\abs{\Fpruned{\gamma}}$,
    \arxivonly{the approximation error vanishes, i.e.,} it holds $\eps_n=0$.
\end{restatable}
\ifarxivelse{Since $\Fpruned{0}=\pareto(\cF,\bR)$, we have that in the case of exact monotonicity, \cref{prop:monotonicity-above-threshold-adaptive} essentially recovers the bound in \cref{thm:Geometry-complexity-upper-bound} for $t=0$ (as then $\eps_n=0$). More generally, i}{I}f monotonicity holds approximately with a small threshold $\smash{\tlower}$, 
the error term $\eps_n$ vanishes already for small $n$ and the estimator benefits from its inductive bias. On the other hand, it is also guaranteed to eventually match the performance of the ERM for large sample sizes and achieve consistency.
This approach can be extended to infinite hypothesis spaces via standard structural risk minimization arguments\arxivonly{, for example through uniform convergence}.

\subsection{Experimental Results on \HELM and \VHELM}
\label{subsec:transfer-learning-HELM-VHELM}
We empirically evaluate the adaptive estimator $\AdaptiveThresholdERM$
(together with natural baselines) on \HELM and \VHELM\arxivonly{, for which we can see in \Cref{fig:modulus HELM and VHELM} that $\tlower$ is small}.
Let $\cF$ denote the finite \emph{zoo} of models evaluated on the benchmark scenarios\footnote{Although \HELM and \VHELM use the term \emph{scenarios}, we use it interchangeably with \emph{tasks} throughout this section.} $\mathcal S$.
Our experimental setup is as follows: We form \emph{transfer combinations} by selecting $K$ source tasks $S=\{S_1,\dots,S_K\}\subset\mathcal S$ and a distinct target task $Q\in\mathcal S\setminus S$. For each $(S,Q)$, we 
sample (without replacement) $n$ test instances from $Q$, compute the empirical target risks $\Rhat_Q(f)$, and apply each selection rule to obtain estimators $\fhat$. For each $(S,Q)$ and each $n$ on a predefined grid, we repeat the target subsampling $500$ times (for all transfer combinations) and report the average excess risk where the target risks $\cR_Q(\fhat)$ are evaluated on the full target pool. 
More details (with also an ablation on $K$) are in \cref{app_sec: additional experiments}.

\ifarxivelse{
\begin{wrapfigure}{r}{0.45\textwidth}
    \vspace{-0.2cm}
    \centering
    \includegraphics[width=\linewidth]{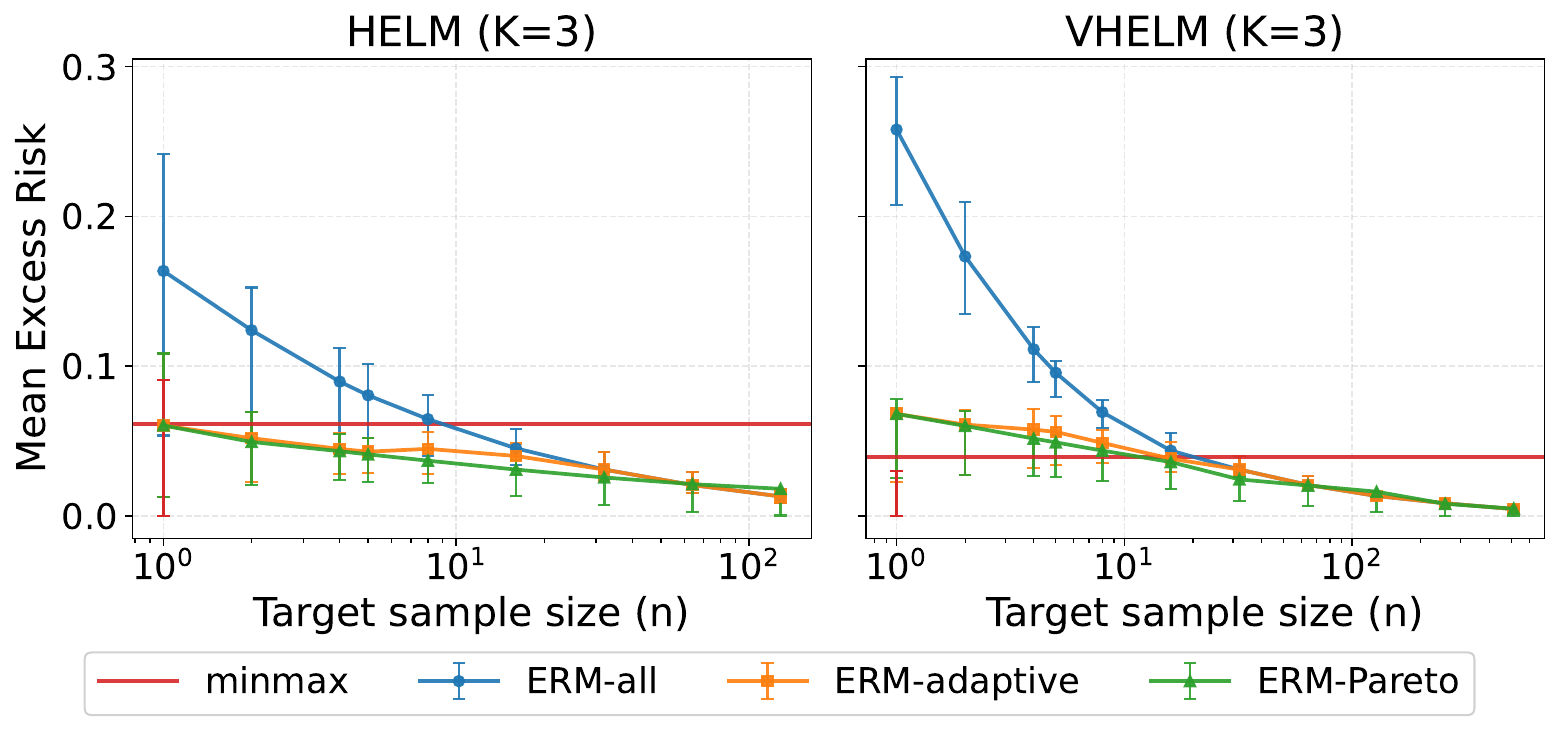}
    \caption{Mean (over all possible $(S,Q)$) excess risks.\vspace{-0.5cm}}
    \label{fig:mean excess risks HELM and VHELM}
    \vspace{-0.4cm}
\end{wrapfigure}
}{
\begin{figure}[ht!]
    \centering
    \includegraphics[width=\linewidth]{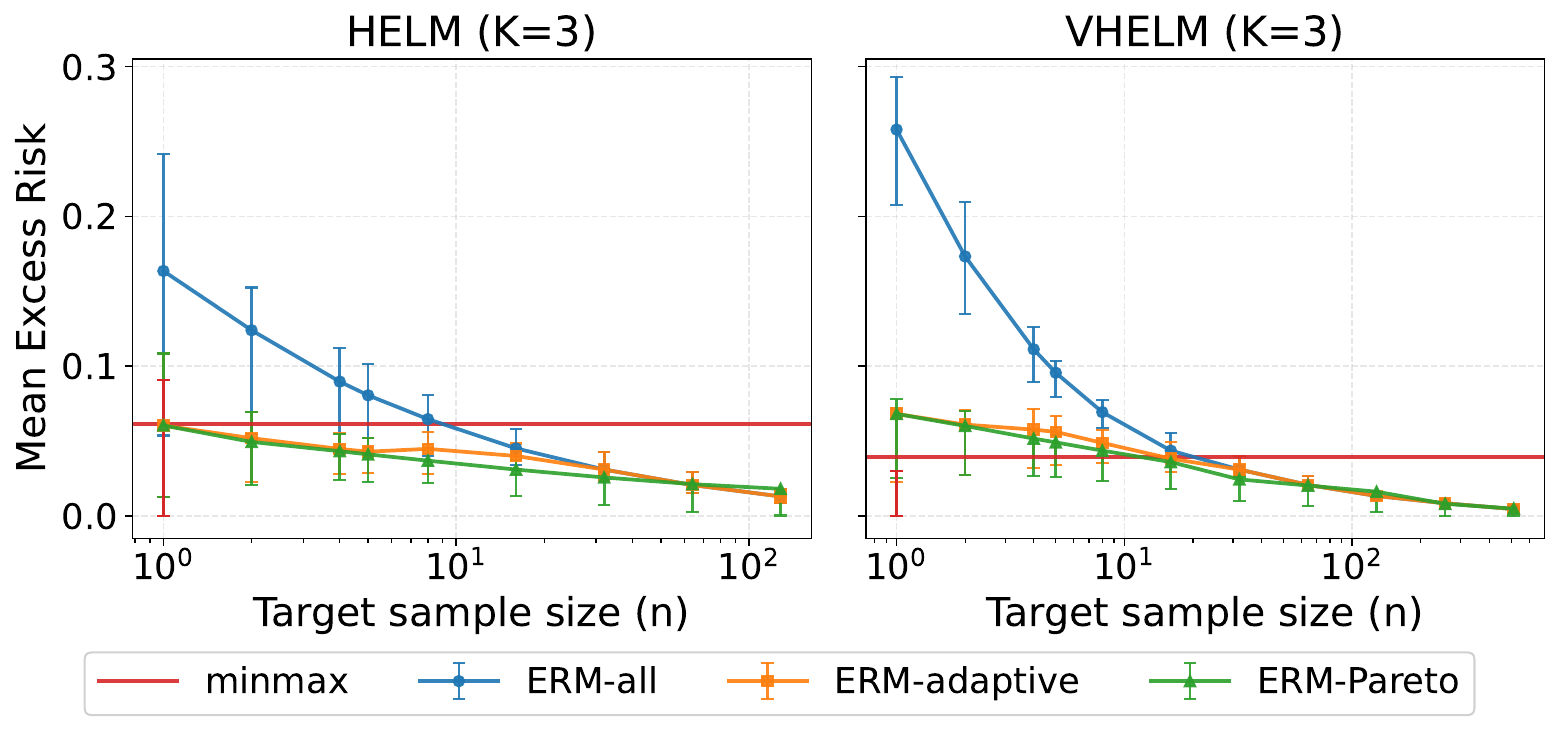}
    \caption{Mean (over all possible $(S,Q)$) excess risks.}
    \vspace{-0.2cm}
    \label{fig:mean excess risks HELM and VHELM}
\end{figure}
}
\textbf{Methods.} We compare the following selection rules:
\begin{itemize}[leftmargin=*,itemsep=-1pt,topsep=-2pt]
    \item ERM-all: $\smash{\widehat f_{\mathrm{all}}\in\argmin_{f\in \cF}\Rhat_Q(f)}$;
    \item ERM-Pareto: $\smash{\Paretoerm{0}\in\argmin_{f\in \mathrm{Par}(\cF, \bR)}\Rhat_Q(f)}$;
    \item ERM-adaptive: $\smash{\AdaptiveThresholdERM\in\argmin_{f\in \Fpruned{\tadapt}}\Rhat_Q(f)}$; 
    \item Min-Max: $\smash{\widehat{f}_\mathrm{rob}=\argmin_{f\in\cF} \max_{k\in[K]}\cR_k(f)}$.
\end{itemize}
The Min-Max does not use target samples and is a popular procedure for robust (worst-case) generalization over mixtures \citep{sagawadistributionally,mohri2019agnostic}. \arxivonly{Since we assume black-box access to models through benchmark evaluations, we do not compare with representation-based selection methods \citep{dong2022zood,zhang2023model}.}

\paragraph{Results and Discussion.} The results \arxivonly{of the above experiment }are reported in \cref{fig:mean excess risks HELM and VHELM}. Across both \HELM and \VHELM, restricting ERM to the frontier yields large gains at small sample size $n$: $\fhat_{\mathrm{all}}$ suffers from high variance\arxivonly{ when optimizing over all models}, whereas restricting to the frontier reduces this variance without introducing too large of a bias\arxivonly{; this is due to the (approximate) \emph{monotonicity} of the datasets}. By design, $\AdaptiveThresholdERM$ is no worse than $\fhat_\mathrm{all}$ eventually, but introduces large gains for small sample sizes. 
\arxivonly{Unsurprisingly, for large sample sizes, in \HELM we can observe an inversion in the trend, with $\widehat{f}_\mathrm{all}$ becoming better than $\Paretoerm{0}$ (although this is more visible for other choices of $K$, see \cref{app_subsec: ablation on number of sources K}), consistent with the fact that the target oracle need not lie on the source-induced Pareto frontier for every transfer combination.} 
These results support the general message of our paper: when source risks induce a meaningful preorder, biasing towards the Pareto frontier can substantially reduce the statistical complexity of transfer, especially at small target sample sizes.

\section{Model Selection Aggregation under Monotonicity}\label{sec: aggregation}

In \cref{sec:transfer-learning} we saw that, under monotonicity, we can achieve a worst-case error of order $\sqrt{K/n}$ for bounded losses, and $K/n$ for strongly convex losses. This can be prohibitive for a large number of benchmarks $K$, and we may be forced to aim for a weaker guarantee: maybe a model that performs well on \emph{one} of the benchmarks already performs well on the new task. More formally, if we denote the minimizer of benchmark $\cR_k$ on $\cF$ as $f_k$, how can we find a model $\fhat$ that performs as well as $\argmin_{k\in[K]}\cR_Q(f_k)$?

This is a particular instantiation of classical model selection aggregation \citep{tsybakov2003optimal}\ifarxivelse{, which studies the following problem: for a dictionary $\crl{f_1,\ldots,f_K}$ of functions $\cX\to [-1,1]$ and $n$ i.i.d.\ samples from a distribution $Q$ on $\cX\times [-1,1]$, find a predictor $\fhat$ that performs not much worse than the best dictionary element with high probability}{, in which the comparator class is usually an arbitrary finite dictionary $\crl{f_1,\ldots,f_K}$}. 
Without any curvature assumptions on the loss, the minimax rate of this problem is known to be $\sqrt{\log (K) /n}$ \citep{lecue2014optimal}, and for strongly convex losses it is $\log (K) / n$ \citep{tsybakov2003optimal}. 
Let $\rho\in \cP([K])\cong \triangle^{K-1}$ be a distribution on the dictionary indices, and denote $f_\rho := \sum_{k\in[K]}\rho_kf_k$. 
Classical aggregators achieve fast rates under strongly convex losses by \emph{hedging}: they return a $f_{\rhohat}$ in the convex hull of the dictionary. Hedging is necessary because any proper method that selects one of the dictionary elements is minimax suboptimal \citep{juditsky2008learning}.
However, these methods are completely oblivious to the origin of the dictionary. This raises the question: Can we leverage the benchmarks together with monotonicity for model selection aggregation?

Let us assume throughout this section that $\cF$ contains $\conv(f_1,\ldots,f_K)$. Then, roughly speaking, both the convex hull and the Pareto set in $\cF$ can be identified with the simplex $\triangle^{K-1}$ through the following maps: $\triangle^{K-1}\ni\rho\mapsto f_\rho$ and $\triangle^{K-1}\ni\rho\mapsto \psi_\rho$ defined as
\begin{equation}
\label{eqn:aggregation-map-Pareto-set}
    \begin{aligned}
    \psi_\rho &=\argmin_{f\in\cF} \max_{k\in[K]}\crl{\cR_k(f)-\cR_k(f_\rho)}, \\
    \Psi(\rho)&= \min_{k\in[K]}\crl{\cR_k(f_\rho)-\cR_k(\psi_\rho)} \geq 0.
\end{aligned}
\end{equation}
Notice that $\bR(\psi_\rho)\preceq \bR(f_\rho)-\Psi(\rho)\one$, and so each point on the convex hull can be dominated by a point on the Pareto set (see \cref{fig:fast-rate-condition}). 
Under the monotonicity assumption, it is therefore natural to return a function on the Pareto set instead of the convex hull, and to hedge there instead. 
We show that, similar to \cref{sec:transfer-learning}, the geometry of the Pareto front determines how much such an alternative strategy improves the excess risk.
In \cref{subsubsec:tilted-exponential-weights}, we discuss a simple method that only requires the loss $\ell(f,z)$ to be bounded in $[0,1]$ and convex in $f$ to benefit from monotonicity. Here, we focus on a potentially more intriguing benefit.

\subsection{Fast Rates from Strongly Concave Gap}
\label{subsec:fast-rates-aggregation}

We now show that an assumption on the separation of the Pareto set from the convex hull is enough to get \emph{fast rates} even when the loss is \emph{not strongly convex}. As mentioned, this is in contrast to the classical aggregation setting\arxivonly{, where only slow rates are achievable when the loss exhibits no curvature}.

\ifarxivelse{
\begin{minipage}{0.525\textwidth}
\begin{assumption}[Strongly concave Pareto gap]
\label{asm:fast-rate-aggregation}
There exists a function $\Phi:\cP([K])\to [0,\infty)$ such that 
(i) $\Phi(\delta_k)=0$ for all $k$, where $\delta_k$ denotes the Dirac delta,  (ii) $\Phi$ is $\eta$-strongly concave: for all $\alpha\in[0,1]$ and $\rho,\gamma\in\cP([K])$, it holds that
\begin{align}
    &\Phi(\alpha \rho+(1-\alpha)\gamma)  \label{eqn:strong-concavity-gap} \\
    &\geq \alpha \Phi(\rho)+(1-\alpha)\Phi(\gamma)+\tfrac{\eta}{2}\alpha(1-\alpha)\nrm{f_\rho-f_\gamma}_{L^2(Q_X)}^2, \nonumber
\end{align}
and (iii) the following equation has a solution $\phi_\rho\in\cF$:
\begin{equation*}
    \bR(\phi_\rho)=\bR(f_\rho)-\Phi(\rho)\one.
\end{equation*}
\end{assumption}

See \cref{fig:fast-rate-condition} for a visualization of \cref{asm:fast-rate-aggregation}, where we show that the risks $\bR(\phi_\rho)$ lie between the image of the convex hull and the Pareto set, $\Phi\leq \Psi$. 
\end{minipage}
\hfill
\begin{minipage}{0.45\textwidth}
    \centering
    \captionsetup{hypcap=false}
    \resizebox{\linewidth}{!}{
        \begin{tikzpicture}[baseline=(img.south)]
    \clip (-0.2,0.45) rectangle (8.25,3.7);
  \node[anchor=south west, inner sep=0, outer sep=0] (img) at (0,0)
    {\includegraphics[width=8.25cm]{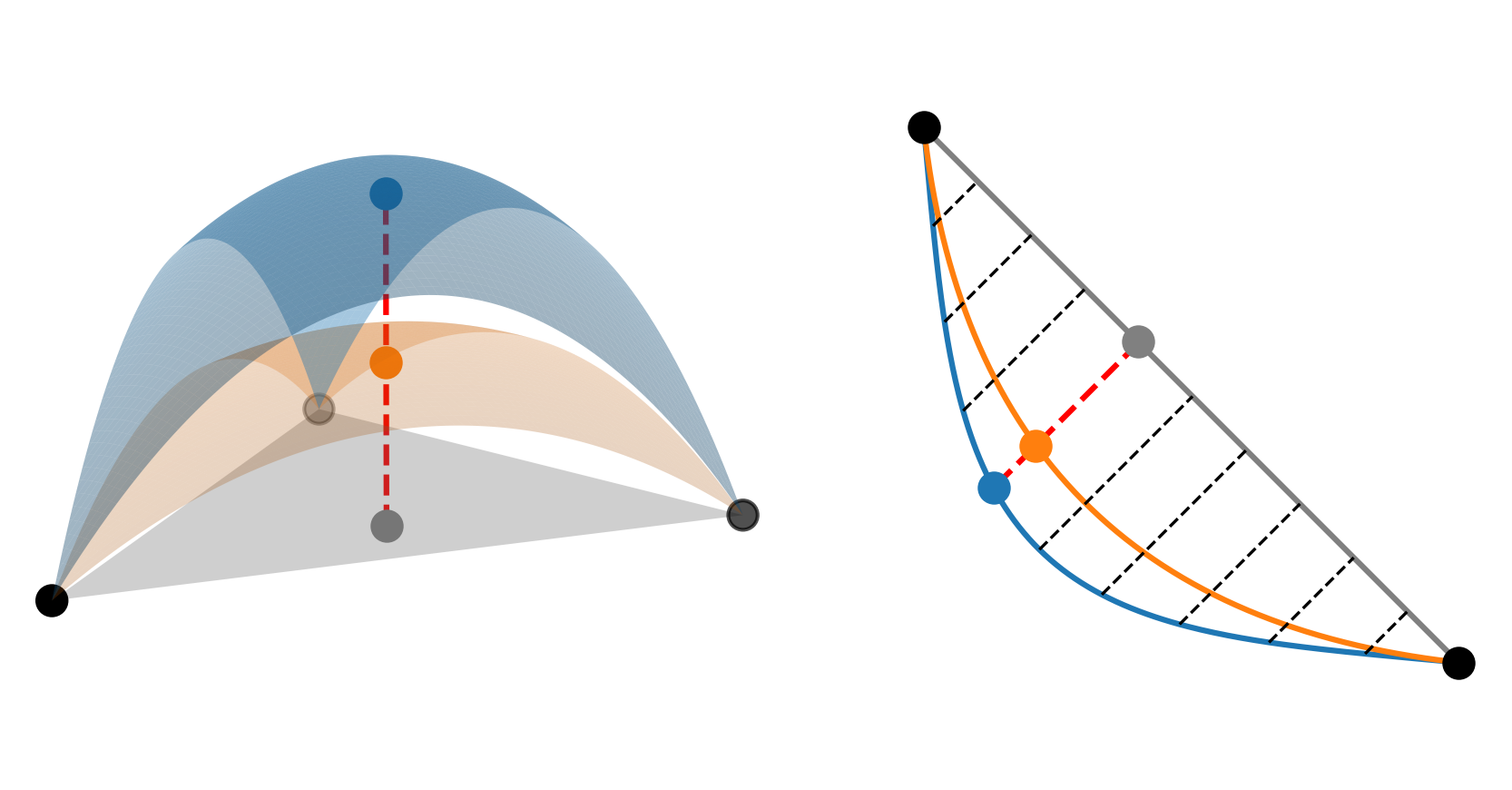}};
    % \node[black] at (5.1,1.3) {$\bR(\phi_\rho)$};
    \node[black] at (6.8,2.5) {$\bR(f_\rho)$};
    %\node[black] at (2.9,2.1) {$\bR(f_1)$};
    \node[black] (Phi) at (7.5,1.9) {$\Phi(\rho)\one$};
    \node[black] (Psi) at (6,3.4) {$\Psi(\rho)\one$};
    \node[] (front) at (5.7,0.6){Pareto front $\fR$};
    % \node (A) at (3.9,1.03)[circle,fill,inner sep=1.5pt]{};
    % \node (B) at (4.2,1.315)[circle,fill,inner sep=1.5pt]{};
    % \draw[red, -, thick] (A) -- (B);
    \draw[-] (Phi) -- (6.06,2.04);
    \draw[-] (Psi) -- (5.7,2.2);
    \draw[->] (front) --(6.0,1.0);
    % \node[black, scale=1.2] at (6.7,0.1) {$f_2$};
    % \node[black, scale=0.8] at (5.0,3.9) {$f_3$};
    % \node[black] at (3,0.5) {$\conv(f_1,f_2,f_3)$};
    % \node[black] at (1.5,1.5) {$\pareto(\cF,\bR)$};
    %\node[black] at (2.9,3.3) {$\Delta(\rho)$};
    \draw[decorate,
        decoration={brace,mirror}, thick]
        (5.75,1.85) -- (6.25,2.35);

    \draw[decorate,
        decoration={brace,mirror}, thick]
        (6.13,2.47) -- (5.4,1.75);

    % \draw[->] (3.5,0.7) -- (4.5,1.5);
    % \draw[->] (1.5,1.8) -- (2,3.5);

    \node[black, scale=0.9] at (2.4,2.3) {$\phi_\rho$};
    \node[black, scale=0.9] at (2.4,3.3) {$\psi_\rho$};
    \node[black, scale=0.9] at (2.4,1.4) {$f_\rho$};
    \node[black] at (0,1.0) {$f_1$};
    \node[black, scale=0.9] at (4.2,1.3) {$f_2$};
    \node[black, scale=0.6] at (1.5,2.0) {$f_3$};
    \node[black] (hull) at (1.5,0.7) {$\conv(f_1,f_2,f_3)$};
    \node[black] (pareto) at (0.7,3.5) {$\pareto(\cF,\bR)$};
    %\node[black] at (2.9,3.3) {$\Delta(\rho)$};

    \draw[->] (hull) -- (1.5,1.4);
    \draw[->] (pareto) -- (1.2,2.8);
\end{tikzpicture}

    }
    \captionof{figure}{
    \emph{Left:} the Pareto set is curved away from the convex hull. The functions $\psi_\rho,\phi_\rho$ are guaranteed to dominate $f_\rho$.
    \emph{Right:} visualization of \cref{asm:fast-rate-aggregation}: if we can fit a strongly concave gap $\Phi$ between the Pareto front and the image of the convex hull, we obtain fast rates by \cref{thm:fast-rate-MSA}. }
    \label{fig:fast-rate-condition}
\vspace{0.07cm}
\end{minipage}
}{
\begin{assumption}[Strongly concave Pareto gap]
\label{asm:fast-rate-aggregation}
There exists a function $\Phi:\cP([K])\to [0,\infty)$ such that 
(i) $\Phi(\delta_k)=0$ for all $k$, where $\delta_k$ denotes the Dirac delta,  (ii) $\Phi$ is $\eta$-strongly concave: for all $\alpha\in[0,1]$ and $\rho,\gamma\in\cP([K])$, it holds that
\begin{align}
    &\Phi(\alpha \rho+(1-\alpha)\gamma)  \label{eqn:strong-concavity-gap} \\
    &\geq \alpha \Phi(\rho)+(1-\alpha)\Phi(\gamma)+\tfrac{\eta}{2}\alpha(1-\alpha)\nrm{f_\rho-f_\gamma}_{L^2(Q_X)}^2, \nonumber
\end{align}
and (iii) $\phi_\rho\in\cF$ satisfies $\bR(\phi_\rho)=\bR(f_\rho)-\Phi(\rho)\one$.
\end{assumption}
\begin{figure}
    \centering
    \resizebox{\linewidth}{!}{
        \input{figures/rate_conditions.tikz}
    }
    \caption{
    \emph{Left:} the Pareto set is curved away from the convex hull. The functions $\psi_\rho,\phi_\rho$ are guaranteed to dominate $f_\rho$.
    \emph{Right:} visualization of \cref{asm:fast-rate-aggregation}: if we can fit a strongly concave gap $\Phi$ between the Pareto front and the image of the convex hull, we obtain fast rates by \cref{thm:fast-rate-MSA}. }
    \vspace{-0.3cm}
    \label{fig:fast-rate-condition}
\end{figure}
}
By definition, we have $\Phi\leq \Psi$ and the risks $\bR(\phi_\rho)$ lie between the image of the convex hull and the Pareto set with strongly concave gap $\Phi$, as illustrated in \cref{fig:fast-rate-condition}. In fact, \Cref{asm:fast-rate-aggregation} is weaker than directly assuming strong concavity of $\Psi$ defined in \eqref{eqn:aggregation-map-Pareto-set}; such a function $\Phi$ can exist even when $\Psi$ is not strongly concave, e.g., in \cref{exm:lin-reg-aggregation}. 
In \cref{subsec:examples-aggregation}, we show how the two \cref{ex:fast-rate,exm:lin-reg-aggregation} satisfy \Cref{asm:fast-rate-aggregation}: in the former, $\Psi$ itself is strongly concave, whereas the latter requires the ``intermediate'' $\Phi$. \arxivonly{They also help to clarify the (admittedly cryptic) \cref{asm:fast-rate-aggregation}.}

We now consider the estimator $\phi_{\rhofast}$ that is inspired by the $Q$-aggregation estimator \cite{dai2012deviation,lecue2014optimal,mourtada2023local}, defined via
\ifarxivelse{
\begin{equation*}
     \rhofast \in \argmin_{\rho\in\cP([K])} \crl{\frac{1}{2}\Rhat_Q(\phi_\rho) + \frac{1}{2} \EE_{k\sim \rho}\Rhat_Q(f_k)}.
\end{equation*}
}{
\begin{equation*}
     \rhofast \in \argmin_{\rho\in\cP([K])} \crl{\frac{1}{2}\Rhat_Q(\phi_\rho) + \frac{1}{2} \EE_{k\sim \rho}\Rhat_Q(f_k)}.
\end{equation*}
\vspace{-0.2cm}
}
\begin{theorem}[Fast rates from strongly concave gap]
\label{thm:fast-rate-MSA}
    Let \cref{asm:fast-rate-aggregation} hold with $\Phi,\phi$ and $\eta$ known to the algorithm. Further assume that $\lowermodulus(t)\geq t$ and $\uppermodulus(t)\leq t$ for all $t\in[0,1]$ (implying \ref{eqn:weak-monotone}), that all functions $f_\rho,\phi_\rho$ and $Y$ are in $[0,1]$, and that the loss $\ell(f,(x,y))=\ell(f(x),y)$ is $\kappa$-strongly convex in $f(x)$ (with $\kappa = 0$ allowed), as well as 
    \begin{alignat*}{2}
        \forall \yhat,\yhat'\in[0,1]:&&\ \abs{\ell(\yhat,y)-\ell(\yhat',y)}&\leq L_{\ell} \abs{\yhat-\yhat'}, \\ 
        \forall x\in \cX:&&\ \abs{\phi_\rho(x)-\phi_{\gamma}(x)} &\leq L_{\phi}\abs{f_{\rho}(x)-f_{\gamma}(x)}.
    \end{alignat*}
    Choose $\lambda= \frac{1}{11L_\ell}  \min\crl{\frac{1}{L_\phi},\frac{\eta+\kappa }{ L_\ell + \eta+\kappa }}$.  Then, for all $\delta\in(0,1)$, with probability at least $1-\delta$,
    \begin{equation*}
        \excessRisk{\phi_{\rhofast}}{\crl{f_k}_{k=1}^K} \leq \frac{2}{\lambda}\frac{\log (K/\delta) }{n}. 
    \end{equation*}
\end{theorem}
The proof is in \cref{proof:fast-rate-MSA} and is very similar to that of \citet{lecue2014optimal} for the regular $Q$-aggregation estimator, despite the fact that our estimator maps to a completely different space. 
In order to compute $\phi_{\rhofast}$, the learner requires access to $\Phi$ from \cref{asm:fast-rate-aggregation} with a certificate of the $\eta$-strong concavity with respect to $L^2(Q_X)$. Luckily, this can be a weak requirement, depending on the problem instance.
Importantly, \emph{without} assuming any Bernstein-type assumption, nor strong convexity of the loss, the estimator $\phi_{\rhofast}$ achieves fast rates whenever $\eta>0$.

\subsection{Some Intuition}
Classical aggregation procedures output a model in the convex hull of the dictionary: to achieve the fast minimax rate $\log(K)/n$, this is required because by hedging (that is, assigning non-zero weight to more than one model), the estimators can exploit the Jensen gap between $\cR_Q(f_\rho)$ and $\EE_{k\sim \rho} \cR_Q(f_k)$
\citep{alquier2024user,mourtada2023local,audibert2004theorie,audibert2007progressive,audibert2009fast,lecue2009aggregation}, similar to the offset term discussed in the theory of offset Rademacher complexity \citep{liang2015learning,vijaykumar2021localization,kanade2024exponential}.
Now, by definition of the modulus (\cref{def:modulus-of-monotonicity}) and \cref{asm:fast-rate-aggregation}, we have
\begin{equation}
\label{eqn:gap-convex-Pareto-Phi}
    \cR_Q(\phi_\rho) \leq \cR_Q(f_\rho) -\lowermodulus(\Phi(\rho)).
\end{equation}
The intuitive reason we can achieve fast rates is then that combining \cref{eqn:gap-convex-Pareto-Phi,eqn:strong-concavity-gap} with $\lowermodulus(t)\geq t$ yields an ``artificial Jensen's gap'' or ``modulus of convexity'' for all $\rho$
\begin{equation}
\label{eqn:doubly-variance-reduced}
    \cR_Q(\phi_\rho)-\EE_{k\sim \rho}\cR_Q(f_k) \leq  -\frac{\eta+\kappa}{2} \EE_{X\sim Q_X}\Var_{k\sim \rho}f_k(X),
\end{equation}
which is negative as long as $\eta>0$ and $\rho$ is in the interior of the simplex (neglecting degenerate cases). The estimator can exploit this gap to achieve fast rates.
\ifarxivelse{\par Finally, we verify that \cref{thm:fast-rate-MSA} can genuinely yield fast rates where standard aggregation methods (in the convex hull) would fail to achieve fast rates.
We prove \cref{thm:lower-bound-convex-hull} in \cref{proof:lower-bound-convex-hull}.
\begin{theorem}
\label{thm:lower-bound-convex-hull}
    Assume $K\geq 4$ and $\log(K)/n \leq 1$.
    There exists a problem setting $(\cF,\bR,\crl{\cR_Q:Q\in\cQ})$ where $\cQ$ is a family of distributions such that (i) any estimator $f_{\rhohat}$ in the convex hull has worst-case risk over $Q\in \cQ$ lower bounded by
    \begin{equation*}
        \sup_{Q\in \cQ} \EE\brk{\excessRisk{f_{\rhohat}}{\crl{f_k}_{k=1}^K}} \geq \frac{1}{38}\sqrt{\frac{\log K}{n}}, 
    \end{equation*}
    and (ii) the assumptions of \cref{thm:fast-rate-MSA} are satisfied with $(\phi,\Phi)$ chosen as $(\psi,\Psi)$ from \cref{eqn:aggregation-map-Pareto-set} and the constants $L_\ell=1,L_\phi=2,\eta=2,\kappa=0$, and $\lambda = 1/22$ for all $Q\in \cQ$. Hence, our estimator achieves
    \begin{equation*}
        \sup_{Q\in \cQ} \EE\brk{\excessRisk{\phi_{\rhofast}}{\crl{f_k}_{k=1}^K}}  \leq 22 \frac{\log K}{n}. 
    \end{equation*}
\end{theorem}
}{
In \cref{subsubsec:lower-bound-convex-hull} we verify that \cref{thm:fast-rate-MSA} can genuinely yield fast rates where standard aggregation methods (in the convex hull) would fail to achieve fast rates.
}

\section{Conclusion}

\ifarxivelse{
This work studies the sample complexity of transfer learning and model selection aggregation under (approximate) monotonicity, motivated by recent benchmarking trends where multiple benchmarks are implicitly assumed to be informative for downstream tasks. We show that the geometry of the frontier determines the sample complexity in both settings, and empirically find that biasing towards the frontier indeed helps in small sample sizes.
Throughout, we assume finite sample access to the target distribution, but full access to the benchmarks. Future work may aim to generalize our work to finite source samples, similar to \citet{hanneke2024adaptive}, and connect to the line of work aiming to learn Pareto fronts \citep{sukenik2024generalization, wegel2025sample}.
Moreover, it may be interesting to exactly characterize the sample complexity for fixed fronts (as described beneath \cref{thm:Geometry-complexity-lower-bound}), and to see if our \cref{thm:limiting-distribution} has applications beyond the learning setting described here.
}{
Motivated by the surge in benchmarking, we study transfer learning and model selection aggregation under \emph{monotonicity}. We show that the sample complexity is governed by the Pareto front geometry, and empirically on language benchmarks that biasing toward the frontier helps in the low-data regime. Our bounds assume finite target samples but full benchmark access; extending the results to finite source samples \citep{hanneke2024adaptive} and connecting to multi-objective learning \citep{sukenik2024generalization,wegel2025sample} are natural directions for future work.
}

\section*{Acknowledgements}

TW was supported by SNSF Grant 204439 and FDG by SNSF Grant 218343. GS was partially supported by the NSF award CCF-2112665 (TILOS). 
\arxivonly{The authors acknowledge the use of LLMs to improve exposition, explore proof ideas, and generate code. The authors take full responsibility for the content of the paper.}

%\newpage

\bibliographystyle{abbrvnat} % Use plainnat for natbib compatibility
{
\small
\bibliography{bibliography}
}

\appendix 

\crefalias{section}{appendix}
\crefalias{subsection}{appendix}
\crefalias{subsubsection}{appendix}
\crefname{appendix}{Appendix}{Appendices}
\Crefname{appendix}{Appendix}{Appendices}

\newpage
\part{Appendix}
\parttoc

\begin{table}[t]
\centering
\caption{Notation.}
\label{tab:notation}
\begin{tabular}{@{}cl@{}}
\toprule
\textbf{Symbol} & \textbf{Meaning} \\
\midrule
$\v\preceq \w$ & partial ordering of vectors $\v,\w\in\RR^K$: $v_k\leq w_k$ for all $k$ \\
$\v\prec \w$ & strict partial ordering of vectors $\v,\w\in\RR^K$: $v_k\leq w_k$ for all $k$ \emph{and} $\v\neq \w$ \\
$\v\llcurly \w$ & strong partial ordering of vectors $\v,\w\in\RR^K$: $v_k< w_k$ for all $k$ \\
$\one$ & all-ones vector $\one=(1,\ldots,1)^\top$ \\
$\one\crl{\cdot}$ & the indicator function \\
$\cF$ & underlying hypothesis space \\
$\bR=(\cR_k)_{k=1}^K$ & vector of benchmarks / source risks \\
$\preceq_\bR$ & Partial order induced by $\bR$ \\
$\cR_Q$ & target risk on distribution $Q$ \\
$\excessRisk{\fhat}{\cF}$ & excess risk $\excessRisk{\fhat}{\cF}=\cR_Q(\fhat)-\inf_{f\in\cF}\cR_Q(f)$\\
$\pareto(\cF,\bR)$ & Pareto set in $\cF$ with respect to $\bR$ (\cref{def:Pareto-optimality}) \\
$\Paretodistance$ & Pareto distance (\cref{def:Pareto-distance}) \\
$\uppermodulus(t),\lowermodulus(t)$ & upper and lower modulus of monotonicity at scale $t$ (\cref{def:modulus-of-monotonicity}) \\
$\fR$ & Pareto front, that is, $\bR(\pareto(\cF,\bR))$ \\
$\Npar(t,\cF,\bR)$ & Pareto covering number at scale $t$ (\cref{def:Pareto-cover})\\
$\cP([K])$ & the set of distributions on $[K]=\crl{1,\ldots,K}$\\
$\triangle^{K-1}$ & the $(K-1)$-dimensional simplex $\{\v\in \RR^K:v_k\geq 0 , \sum_{k=1}^K v_k =1\}$\\
$\Paretoerm{t}$ & Pareto ERM at scale $t$ (\cref{subsec:Pareto-ERM}) \\
$f_{\rhohat}$ & Pareto EW (\cref{subsec:Pareto-ERM}) \\
$\ParetoStar{t}$ & Pareto Star Estimator at scale $t$ (\cref{subsec:aggregation-for-transfer}) \\
$\Thresholderm{\gamma}$ & Margin ERM at margin $\gamma$ (\cref{sec: beyond exact monotonicity}) \\
$\AdaptiveThresholdERM$ & Adaptive Margin ERM (\cref{sec: beyond exact monotonicity}) \\
$\phi_{\rhofast}$ & Fast rate aggregator (\cref{subsec:fast-rates-aggregation}) \\
\bottomrule
\end{tabular}
\end{table}

\section{Additional Experimental Results}\label{app_sec: additional experiments}

In what follows, we describe the experimental protocol in detail and report additional ablation studies.

\subsection{Datasets and Methods}
\label{app_subsec: dataset}

We run our experiments on two real-world benchmark suites: \HELM (Holistic Evaluation of Language Models)~\cite{liang2023holistic} and \VHELM (Holistic Evaluation of Vision Language Models)~\cite{lee2024vhelm}.

\paragraph{\HELM.}
In its complete release, \HELM measures (whenever applicable) seven top-level desiderata---\textit{accuracy, calibration, robustness, fairness, bias, toxicity,} and \textit{efficiency}---over a suite of $16$ \emph{core} scenarios, often reporting multiple sub-metrics per desideratum. Whenever a metric admits multiple variants, we use the \emph{primary} metric reported on the \HELM leaderboard. Since \HELM is a \emph{living} benchmark that evolves over time, we fix a single public snapshot and run all experiments on release \texttt{v1.0.0} to avoid version drift and minimize missing entries. 

Because our goal is not an extensive empirical validation but rather to illustrate that our \emph{monotonicity} assumption is reasonably consistent with practice, we work with the \HELMLite setup. In \texttt{v1.0.0}, \HELMLite contains evaluations of $31$ LLMs on $7$ scenarios:
\textsc{OpenBookQA}, \textsc{GSM8K}, \textsc{LegalBench}, \textsc{MATH}, \textsc{MedQA}, \textsc{MMLU}, and \textsc{NarrativeQA}.
For each scenario, we extract a \emph{single} scalar performance measure given by its primary leaderboard metric (e.g., exact match for multiple-choice QA, F1 for \textsc{NarrativeQA}, and equivalence-based accuracy for \textsc{MATH}). We treat each scenario as providing one ``accuracy-style'' objective, and convert metrics to risks via a fixed monotone transformation (e.g., $1-\textit{score}$ when larger is better). Finally, to increase variability in the Pareto set, we remove the \texttt{GPT-4} models (which are substantially stronger than the remaining models), yielding a final model set of size $29$. With $7$ scenarios, the total number of source/target combinations for a fixed $K\in\{1,\dots,6\}$ equals $7\cdot \binom{6}{K}$.

\paragraph{\VHELM.}
We use the \texttt{v2.0.0} release of \VHELM, which evaluates $26$ vision-language models across multiple metrics; we focus on the accuracy-style metric and restrict attention to $9$ scenarios for which sufficient data are available for all models:
\textsc{a\_okvqa}, \textsc{blink}, \textsc{fair\_face}, \textsc{hateful\_memes}, \textsc{math\_vista}, \textsc{mm\_star}, \textsc{mmmu}, \textsc{seed\_bench}, and \textsc{unicorn}.
We discard scenarios with insufficient coverage across models.
With $9$ scenarios, the total number of source/target combinations for a fixed $K\in\{1,\dots,8\}$ equals $9\cdot \binom{8}{K}$.

\paragraph{Baselines.} For completeness, we restate here the model-selection methods reported in the experiments.
\begin{itemize}
    \item ERM-all: $\widehat{f}_{\mathrm{all}}\in\arg\min_{f\in \cF}\Rhat_Q(f)$;
    
    \item ERM-Pareto:  $\Paretoerm{0}\in\arg\min_{f\in \mathrm{Par}(\cF, \bR)}\Rhat_Q^{(n)}(f)$ as defined in \cref{subsec:Pareto-ERM} (note here that by setting the Parameter $t=0$ we effectively use the entire Pareto frontier);
    
    \item ERM-adaptive: we compute the parameter $\tadapt$ and the corresponding pruned hypothesis space $\Fpruned{\tadapt}$; yielding the adaptive margin ERM $\AdaptiveThresholdERM\in\argmin_{f\in \Fpruned{\tadapt}}\Rhat_Q(f)$;
    
    \item Min-Max: $\widehat{f}_\mathrm{rob}=\min_{f\in\cF} \max_{k\in[K]}\cR_k(f)$, the robust estimator that does not use any target data.
\end{itemize}
Since we assume black-box access to models through benchmark evaluations, we do not compare with representation-based selection methods \citep{dong2022zood,zhang2023model}.

\paragraph{Experimental protocol.} Let $\cF$ denote the finite set of models evaluated on the benchmark scenarios $\mathcal S$ (\HELM or \VHELM). We form \emph{transfer combinations} by selecting $K$ source scenarios $S=\{S_1,\dots,S_K\}\subset\mathcal S$ and a distinct target scenario $Q\in\mathcal S\setminus S$. For each $(S,Q)$, we compute the Pareto set $\mathrm{Par}(\cF,\bR)$ using all available source data. To emulate limited labeled target data, we then sample (without replacement) $n$ target instances from $Q$, compute the empirical target risks $\widehat{\cR}_Q(f)$, and apply each selection rule to obtain $\fhat$. We evaluate its target risk $\cR_Q(\fhat)$ on the full target pool. For each $(S,Q)$ and each $n$, we repeat the target subsampling $B=500$ times and report the average excess risk
\[
\widehat\cE_B(\fhat) \;=\; \frac{1}{B}\sum_{i=1}^B \Big(\cR_Q\big(\hat f^{(i)}\big)\;-\;\min_{f\in\cF}\cR_Q(f)\Big),
\]
where $\hat f^{(i)}$ is the model selected on trial $i$ of the subsampling. Finally, for each $n$ in a predefined grid, we aggregate results by averaging $\widehat\cE_B(\fhat)$ over all transfer combinations.

\subsection{Ablation on the Number of Benchmarks}
\label{app_subsec: ablation on number of sources K}
\begin{figure}[ht]
    \centering
    \includegraphics[width=0.75\linewidth]{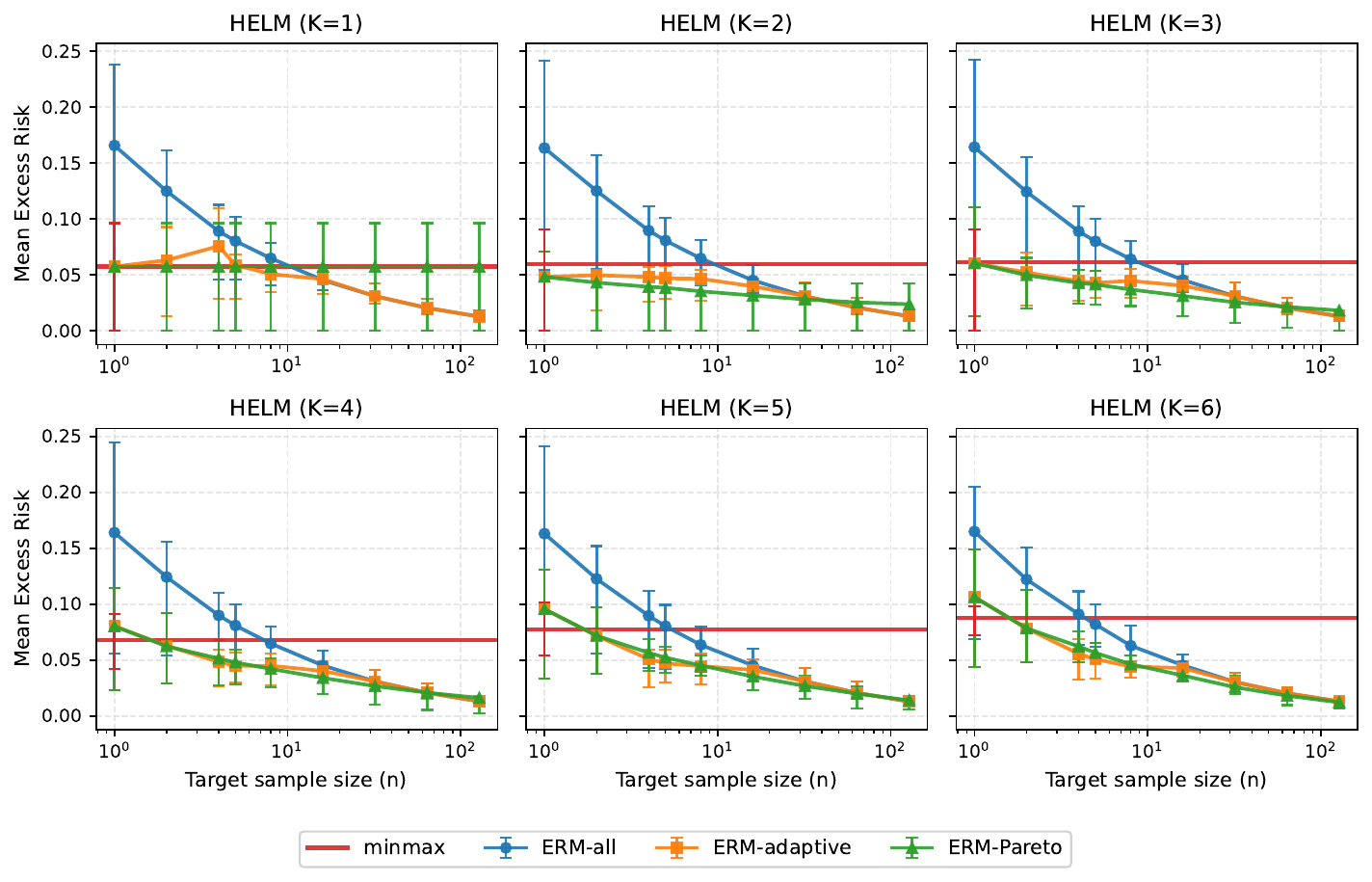}
    \caption{Mean (over all possible combinations of $K$ sources and a target $Q$) excess risk of different model selection methods. Vertical bars are (25-75 quartiles).}
    \label{fig:mean excess risks HELM}
\end{figure}
\begin{figure}
    \centering
    \includegraphics[width=\linewidth]{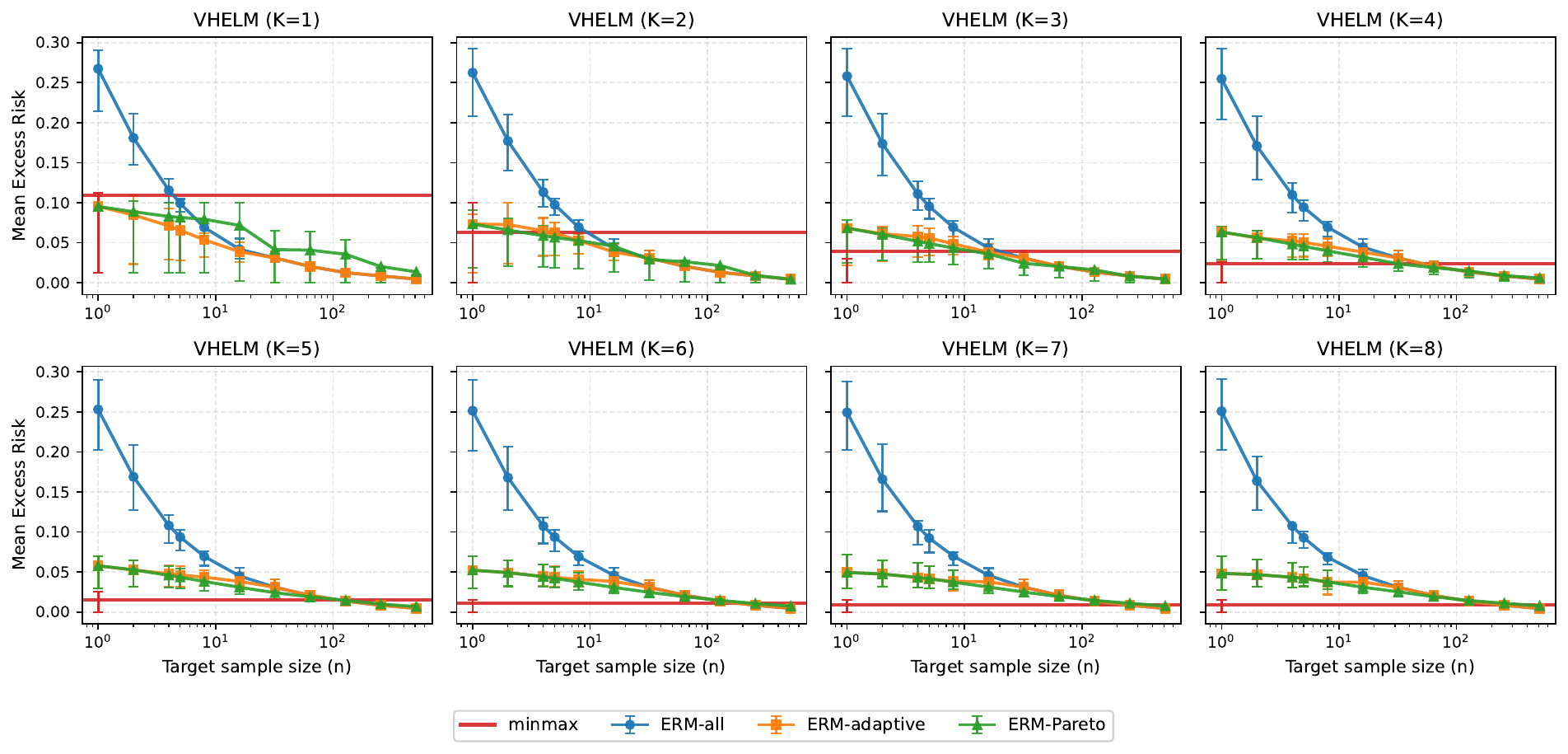}
    \caption{Mean (over all possible combinations of $K$ sources and a target $Q$) excess risk of different model selection methods. Vertical bars are (25-75 quartiles).}
    \label{fig:mean excess risks VHELM}
\end{figure}
The results of \cref{fig:mean excess risks HELM and VHELM} are presented for a fixed value $K$ of source tasks. Following, we presented the same experiment by varying $K$. In particular, \HELM has $7$ scenarios and thus $K\in\{1,...,6\}$; similarly, \VHELM has $9$ scenarios and $K\in\{1,...,8\}$. We highlight some interesting observations from the plots in \cref{fig:mean excess risks HELM,fig:mean excess risks VHELM}. First of all, the main qualitative pattern is consistent across both benchmark suites: \emph{frontier-biased} selection methods (Pareto ERM and adaptive ERM) provide a smaller excess risk compared to ERM on the full hypothesis class in the low-data regime, where ERM over the full class suffers from substantial variance. More precisely, for small target sample sizes $n$, both ERM-Pareto and ERM-adaptive significantly improve upon ERM-all for every $K$.  As $n$ grows, the gap narrows and all methods converge, reflecting the fact that estimation noise on the target decreases with $n$.
On the other hand, for bigger values of $n$, ERM on the Pareto set can be worse; this can be appreciated for example in the case of $K=1$ for both \HELM and \VHELM and this is due to the fact that the oracle model $f^\star\in \argmin_{f\in\cF}\cR_Q(f)$ is not alwys in the Pareto frontier induced by the source tasks. In the \HELM experiments with $K=1$, we observe an extreme version of this phenomenon: for essentially all combinations the Pareto set is a singleton, and the corresponding model is not the target oracle. Consequently, ERM-Pareto becomes identical to a purely source-based selection rule (it does not depend on $n$) and its excess risk matches the constant $\min\max$ baseline. 

Finally, the adaptive procedure $\AdaptiveThresholdERM$ of \cref{sec: beyond exact monotonicity} mitigates this potential bias by enlarging the candidate set with $n$: by construction $\Fpruned{\tadapt}$ contains the Pareto set and grows with $n$, interpolating between strict frontier restriction at very small $n$ and ERM over $\cF$ for larger $n$. Empirically, this yields performance close to ERM-Pareto in the low-data regime (where variance dominates) while reducing the risk of a plateau when $n$ is large (where approximation bias becomes the limiting factor).

\section{Additional Results on Monotonicity}

In this section, we provide some more discussion and results about our monotonicity assumption \eqref{eqn:weak-monotone}. We restate here for completeness a couple of previously introduced definitions that will be needed for the discussion. 
Recall the definitions of monotonicity and Pareto optimality from \cref{subsec:weak-monotonicity}.

We begin by mentioning that the sufficiency from \cref{lem:weak-sufficiency} can be exploited more explicitly.
Recall that by \cref{lem:weak-sufficiency}, the representation $\bR(f)$ gives sufficient information to compute $\cR_Q$, as long as we can learn the mapping $s$. Prior work has studied algorithms that explicitly learn this map from a fixed class of scalarizations \citep{mansour2021theory}. However, this may prove suboptimal, either when the chosen class of scalarizations is unnecessarily complex (see \cref{sec:transfer-learning}), or on the contrary, not expressive enough. This can happen, for instance, when the Pareto front is \emph{non-convex} (see \cref{ex:convex-hull-counterexample} below). In contrast, our approach completely avoids learning the scalarization $s$.

\subsection{Results for the Pareto Distance and Moduli}
\label{app_subsec:modulus-of-monotonicity}
We now provide two results about the Pareto distance and the moduli (\cref{def:Pareto-distance,def:modulus-of-monotonicity}).
Recall that the Pareto distance measures the smallest amount $t$ needed to improve all objectives $\cR_k(f)$ in order to achieve a performance that is no worse than $\cR_k(f')$. 
Note that, for arbitrary functions $f$ and $g$, $\Paretodistance(f,g)$ can be negative and asymmetric. In particular, this is not a metric even when restricted to the Pareto set: symmetry can fail, and $\Paretodistance(f,g)=0$ does not imply $f=g$. On the Pareto front, we can prove the following.
\begin{lemma} \label{lem:pareto-distance-axioms}
    Let $f, g \in \cF$. The Pareto distance (\cref{def:Pareto-distance}) satisfies $\Paretodistance(f,g) = \min \{t\in\RR : \bR(f) - t\one \preceq \bR(g) \}$. Moreover, if $f, g, h \in \pareto(\cF,\bR)$ are restricted to the Pareto set, it satisfies the following properties:
    \begin{enumerate}
        \item[(i)] Non-negativity: $\Paretodistance(f,g) \geq 0$ and $\Paretodistance(f,f) = 0$.
        \item[(ii)] Triangle inequality: $\Paretodistance(f,h) \leq \Paretodistance(f,g) + \Paretodistance(g,h)$.
    \end{enumerate}
\end{lemma}

\begin{proof}
    If $t$ is the smallest value satisfying $\bR(f) - \bR(g) \preceq t\one$, then it must be larger than $\max_{k \in [K]} \cR_k(f) - \cR_k(g)$, thus
    \(\Paretodistance(f,g) \leq \min \big\{t : \bR(f) - t\one \preceq \bR(g)\big\}.\) 
    On the other hand, $\bR(f) - \Paretodistance(f,g) \one \preceq \bR(g)$, so that
    \(\Paretodistance(f,g) \geq \min \big\{t : \bR(f) - t\one \preceq \bR(g) \big\}.\)
    And so, equality holds. When $f,g,h\in\pareto(\cF,\bR)$, we also have:
    \begin{enumerate}
        \item Since $f, g \in \pareto(\cF, \bR)$, neither dominates the other. That is, there exists some $k \in [K]$ such that 
        \(\cR_k(f) \geq \cR_k(g).\)
        The maximum over $k \in [K]$ is there non-negative. That $\Paretodistance(f,f) = 0$ is immediate by the definition.
        \item By definition, we have
        \(\bR(f) - \Paretodistance(f,g)\one \preceq \bR(g)\) and \(\bR(g) - \Paretodistance(g,h)\one \preceq \bR(h).\)
        Combining these, we obtain that
        \[\bR(f) - \big(\Paretodistance(f,g) + \Paretodistance(g,h)\big) \one \preceq \bR(h),\]
        which implies that $\Paretodistance(f,h) \leq \Paretodistance(f,g) + \Paretodistance(g,h)$.
    \end{enumerate}
    That concludes the proof.
\end{proof}

Now, recall the definition of the moduli (\cref{def:modulus-of-monotonicity}), which are the quantities that directly follow from the distance $\Paretodistance(\cdot,\cdot)$.
We provide the following bounds for scalarizations from \cref{lem:weak-sufficiency} that are in a coordinate-separable form.
\begin{lemma}
\label{lem:bound-modulus-Lipschitz}
    Assume \eqref{eqn:weak-monotone} and denote $s$ the scalarization from \cref{lem:weak-sufficiency}. 
    If $s$ is of the form $s(\r) = \sum_{k=1}^{K} q_kg(r_k)$ for some non-negative weights $\q=(q_1,\ldots,q_K)\in[0,\infty)^K$ and non-decreasing $g:[0,\infty)\to \RR$, then, for
    \begin{align*}
        \overline{\Delta}_g(t) = \sup_{u\in[0,1]} g(u+t)-g(u), \quad \text{and} \quad 
        \underline{\Delta}_g(t) = \inf_{u\in[t,1]} g(u)-g(u-t),
    \end{align*}
    it holds $\uppermodulus(t)\leq \nrm{\q}_1 \overline{\Delta}_g(t)$ and $\lowermodulus(t)\geq \nrm{\q}_1 \underline{\Delta}_g(t)$.
    In particular, if $s$ is a weighted average (a.k.a.\ linear scalarization), then $\uppermodulus(t)\leq t$ and $\lowermodulus(t)\geq t$.
\end{lemma}
While these bounds can be loose, they are essentially tight for weighted averages (except for some edge cases). This lemma also demonstrates what happens under \emph{monotone transformations} of the source risks. For example, if $g(t)=\sqrt{t}$, then the upper modulus is bounded on order of $\sqrt{t}$.
\begin{proof}
    For any $t\geq0$, let $f,f'\in\cF$ be such that $\bR(f)\leq \bR(f')+t\one$. Then, using $\cR_Q=s\circ \bR$ from \cref{lem:weak-sufficiency}, the assumption that $s$ is of the form $s(\r)=\sum_{k=1}^K q_k g(r_k)$, and the definition of $\overline{\Delta}_g(t)$, we get that
    \begin{align*}
        \cR_Q(f)-\cR_Q(f') &=s(\bR(f))-s(\bR(f')) \tag{\cref{lem:weak-sufficiency}}\\
        &= \sum_{k=1}^K q_k (g(\cR_k(f))-g(\cR_k(f'))) \tag{$s(\r)=\sum_{k=1}^K q_k g(r_k)$} \\
        &\leq \sum_{k=1}^K q_k (g(\cR_k(f')+t)-g(\cR_k(f')))  \tag{$g$ is non-decreasing and $\cR_k(f)\leq \cR_k(f')+t$}\\
        &\leq \sum_{k=1}^K q_k \prn{\sup_{u\in[0,1]} g(u+t)-g(u)} \tag{$\cR_k(f')\in [0,1]$} \\
        &= \nrm{\q}_1 \overline{\Delta}_g(t). \tag{Definition of $\overline{\Delta}_g(t)$}
    \end{align*}
    By taking the supremum over all such $f,f'\in\cF$ and \cref{def:modulus-of-monotonicity} of the upper modulus, it holds that $\uppermodulus(t)\leq \nrm{\q}_1 \overline{\Delta}_g(t)$. 

    Similarly, for any $t\geq0$, let $f,f'\in\cF$ be such that $\bR(f)\leq \bR(f')-t\one$. Then, using $\cR_Q=s\circ \bR$ from \cref{lem:weak-sufficiency} and the assumption that $s$ is of the form $s(\r)=\sum_{k=1}^K q_k g(r_k)$, and the definition of $\underline{\Delta}_g(t)$, we get that
    \begin{align*}
        \cR_Q(f')-\cR_Q(f) &= \sum_{k=1}^K q_k (g(\cR_k(f'))-g(\cR_k(f))) \\
        &\geq \sum_{k=1}^K q_k \prn{g(\cR_k(f'))-g(\cR_k(f')-t)} \tag{$g$ is non-decreasing and $\cR_k(f)\leq \cR_k(f')-t$} \\
        &\geq \sum_{k=1}^K q_k \prn{\inf_{u\in [t,1]} g(u)-g(u-t)} \tag{$\cR_k(f)\leq \cR_k(f')-t$ implies $\cR_k(f')\geq t$} \\
        &\geq \nrm{\q}_1 \underline{\Delta}_g(t). 
    \end{align*}
    By taking the infimum over all such $f,f'\in\cF$ and \cref{def:modulus-of-monotonicity} of the lower modulus, it holds that $\lowermodulus(t)\geq \nrm{\q}_1 \underline{\Delta}_g(t)$. 

    Finally, note that for a weighted average we can take $g(x)=x$ (which clearly satisfies $\overline{\Delta}_g(t)=\underline{\Delta}_g(t)=t$) and $\nrm{\q}_1=1$.
\end{proof}

\subsection{Monotonicity in the Statistical Learning Setting}

Let $\ell_k(f,z)\in[0,1]$ be a loss function; in the statistical learning setting, we assume that the source risks are also risks over distributions $P_k$ on $\cZ$, that is, $\cR_k(f) = \EE_{z \sim P_k} \ell_k (f,z)$ for all $k = 1,\ldots, K$.

\subsubsection{Mixture target distributions and monotonicity.} 
We first state a basic sufficient condition for monotonicity in the statistical learning setting when the target distribution is a convex mixture.
\begin{lemma}
\label{lem:Q-in-convex-hull}
    If $Q\in \conv(P_1,\ldots,P_K)$ and $\ell_k=\ell$ for all $k$, then $\cR_Q$ is monotonic with respect to $\bR$, and $s$ from \cref{lem:weak-sufficiency} is a weighted average.
\end{lemma}

\begin{proof}
    Note that for some weights $\q =(q_1,\ldots,q_K)\in\triangle^{K-1}$ we can write 
    \begin{equation*}
        \cR_Q(f)=\int \ell(f,z)\d Q(z) = \int \ell(f,z)\d \prn{\sum_{k=1}^K q_kP_k}= \sum_{k=1}^K q_k \int \ell(f,z)\d P_k = \sum_{k=1}^K q_k \cR_k(f).
    \end{equation*}
    Clearly, this implies that $\cR_Q$ is monotonic in $\bR$.
\end{proof}

We now prove a partial converse of \cref{lem:Q-in-convex-hull} which demonstrates that under fairly strong assumptions (roughly speaking that the span of $\crl{\ell(f,\cdot):f\in\cF}$ is expressive enough) the target risk must be a convex combination of the source risks. 

\begin{lemma}[Mixture from monotonicity] \label{prop:monotone-to-convex}
    Let $\ell_k=\ell$ for all $k$.
    Define the loss class $\cL = \{\ell_f : f \in \cF\}$, where $\ell_f \equiv \ell(f,\cdot)$. Let $\cH = \mathrm{span}(\cL)$ form a Hilbert space such that:
    \begin{enumerate}
        \item Every function $\phi \in \cH$ can be expressed in the form $\phi = \lambda (\ell_f - \ell_{f'})$, where $\lambda \in \RR$ and $\ell_f, \ell_{f'} \in \cL$,
        \item The constant function $\mathbf{1}$ is contained in $\cH$. Moreover, let $\cR_1,\ldots, \cR_K$ be linearly independent in $\cH^*$.
    \end{enumerate}
    If $\cR_Q$ is weakly monotonic with respect to $(\cR_1,\ldots, \cR_K)$, then 
    \[\cR_Q \in \mathrm{conv}_{\cH^*}(\cR_1,\ldots, \cR_K).\]
\end{lemma}

To explain these assumptions in words, we may view the functionals $\cR_1,\ldots, \cR_K$ and $\cR_Q$ as elements in the dual space $\cH^*$. The first condition of \Cref{prop:monotone-to-convex} states that the model class $\cF$ is rich enough so that we can probe the risk functional $\cR_Q$ using only pairwise comparisons of models in $\cF$, where for $\phi = \lambda( \ell_f - \ell_{f'})$,
\[\hspace{2em}\langle \cR_Q, \phi\rangle_{\cH} = \lambda\big[\cR_Q(f) - \cR_Q(f')\big].\]
The second condition enables a more interpretable and aesthetically pleasing result (though it is not `morally' necessary), although it rules out cases such as $\cR_1 = - \cR_2$. 

\Cref{prop:monotone-to-convex} shows that if we require monotonicity to hold for a sufficiently expressive set of models $\cF$, then this imposes a lot of constraints on how the new task $\cR_Q$ relates to the old ones $\bR$.
However, there exist numerous meaningful examples where the conclusion of \cref{prop:monotone-to-convex} can fail, especially when the loss function varies or the number $K$ of source risks is large. 
We exemplify this in the following simple example.
\begin{example}[Monotonicity is strictly more general than convex hull]
\label{ex:convex-hull-counterexample}
    Consider the instances $\cZ=\crl{z_1,z_2,z_3}$ and models $\cF=\crl{f_1,f_2,f_3, f_4}$, incurring the following losses:
\begin{center}
\begin{tabular}{@{}c@{\hspace{1em}}c@{}}
\begin{tabular}{c|cccc}
$\ell$ & $f_1$ & $f_2$ & $f_3$ & $f_4$ \\ \hline
$z_1$ & $0.5$ & $1.0$ & $0.8$ & $1.0$ \\
$z_2$ & $0.5$ & $0$   & $0.3$ & $0.5$ \\
$z_3$ & $0.5$ & $0.5$ & $0.3$  & $0.0$
\end{tabular}
&\hspace{1cm}
\raisebox{-0.5\height}{\includegraphics[width=0.25\linewidth]{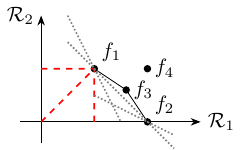}}
\end{tabular}
\end{center}
    Let $P_1=\delta_{z_1},P_2=\delta_{z_2}$. Then for any $\widetilde Q\in \conv(P_1,P_2)$ we have that $\EE_{z\sim \widetilde Q} \ell(f,z)$ is minimized by either $f_1$ or $f_2$, as visualized by the gray dotted lines. But if $Q=\uniform{\cZ}$, then $f_3$ is the target minimizer. Importantly, $\cR_Q$ remains monotonic in the sources.
\end{example}
This and many other examples (e.g., with varying losses) show that, even in the statistical learning setting where source risks can be written as expected losses, we generally cannot conclude that $\cR_Q$ or $Q$ must be in the convex hull of the source risks $\cR_k$ or source distributions $P_k$, so our setting is strictly more general.

Finally, before proving \cref{prop:monotone-to-convex}, the following lemma is needed:
\begin{lemma} \label{lem:linear-span}
    Let $\cH$ be a Hilbert space containing vectors $u, v_1,\ldots, v_N$ and $\ker(v):= \ker(w\mapsto \inner{w}{v}_\cH)=v^\perp$. Then:
    \[u \in \mathrm{span}(v_1,\ldots, v_N) \quad \Longleftrightarrow \quad \bigcap_{n \in [N]} \ker(v_n) \subset \ker(u).\]
\end{lemma}
\begin{proof} Let $V = \mathrm{span}(v_1,\ldots, v_N)$.
Suppose that $u \in V$ where $u = \sum_n \lambda_n v_n$. Let $w \in \bigcap \ker(v_n)$. Then $w \in \ker(u)$ since \(\langle w, u\rangle_{\cH} = \sum_{n \in [N]} \lambda_n \langle w, v_n\rangle_{\cH} = 0.\)
Suppose on the other hand that $u \notin V$. Since $\cH$ is a Hilbert space, $u$ decomposes into components in $V$ and $V^\perp$,
\(u = u_V + u_{V^\perp},\)
where the orthogonal component is nonzero, $u_{V^\perp} \ne 0$. Then, $u_{V^\perp} \in \bigcap \ker(v_n)$, but $u_{V^\perp} \notin \ker(u)$.
\end{proof}

\begin{proof}[Proof of \cref{prop:monotone-to-convex}]
    Throughout, we view $\cR_1,\ldots, \cR_K$ and $\cR_Q$ as elements of the dual space $\cH^*$. \Cref{lem:linear-span} shows that:
    \begin{equation} \label{eqn:kernel-implication}
        \bigcap_{k \in [K]} \ker(\cR_k) \subset \ker(\cR_Q)  \implies \cR_Q \in \mathrm{span}_{\cH^*}(\cR_1,\ldots, \cR_K).
    \end{equation}
    Assume for now that the left-hand side holds. The right-hand side almost gives the result, except that we need to ensure that $\cR_Q$ is not just a linear span, but a convex combination. Let $\cR_Q = \sum \lambda_k \cR_k$. Then, in fact, it is a convex combination:
    
    By the second condition of \Cref{prop:monotone-to-convex}, $\mathbf{1} \in \cH$. Thus, we obtain
    \[1 = \cR_Q(\mathbf{1}) = \sum_{k \in [K]} \lambda_k \cR_k(\mathbf{1}) = \sum_{k \in [K]} \lambda_k.\]
     
    Fix any $j \in [K]$. We show that $\lambda_j \geq 0$. By linear independence, the risk functionals $\cR_k$ ranging over $k \in [K]$ are linearly independent. \Cref{lem:linear-span} implies that there exists $\phi \in \cH$ such that:
    \[\phi \in \bigcap_{k \in [K] \setminus \{j\}} \ker(\cR_k) \quad \textrm{and}\quad \phi \notin \ker(\cR_j).\]
    By the first condition of \Cref{prop:monotone-to-convex}, there exist $f, f' \in \cF$ such that $\phi \in \mathrm{span}(\ell_f- \ell_{f'})$. For all $k \ne j$, we deduce from $\phi \in \ker(\cR_k)$ and from $\phi \notin \ker \cR_j$ that:
    \(\cR_k(f) = \cR_k(f')\) and \(\cR_j(f)\ne \cR_j(f')\).
    Swapping $f$ and $f'$ if needed, this shows that $f$ is strictly $\bR$-dominated by $f'$. We also obtain
    \begin{align*}
        \cR_Q(f) - \cR_Q(f') = \lambda_j \big(\cR_j(f) - \cR_j(f')\big).
    \end{align*}
    Since $\cR_j(f) - \cR_j(f') \ne 0$, weak monotonicity forces $\lambda_j \geq 0$.

    To finish the proof, we just need to show that the left-hand side of the implication in \Cref{eqn:kernel-implication} holds, namely, $\bigcap_{k \in [K]} \ker(\cR_k) \subset \ker(\cR_Q)$. Let $\phi \in \bigcap \ker(\cR_k)$. Again, by the first condition of \Cref{prop:monotone-to-convex}, there are $f, f' \in \cF$ such that $\phi \in \mathrm{span}(\ell_f - \ell_{f'})$. This implies that $\cR_k(f) = \cR_k(f')$ for all $k \in [K]$. By \Cref{lem:weak-sufficiency}, weak monotonicity implies that $\cR_Q(f) = \cR_Q(f')$. Thus, $\phi \in \ker(\cR_Q)$.
\end{proof}

\subsubsection{Covariate and label shift.}
Assume that the source and target risks share the same feature space \(\cX\), label space \(\cY\), and loss function \(\ell\), but may differ in their distributions. Let \(\cZ=\cX\times\cY\), let \(z=(x,y)\), and write \(\ell(f,z):=\ell(f(x),y)\). Let \((X,Y)\sim Q\) and \((X^k,Y^k)\sim P_k\) for \(k\in[K]\).

We say that \(Q\) satisfies \emph{covariate shift} relative to \(P_1,\dots,P_K\) if, for every \(k\in[K]\), the regular conditional law of \(Y^k\) given \(X^k\) under \(P_k\) coincides with the regular conditional law of \(Y\) given \(X\) under \(Q\), \(P_{k,X}\)-almost surely. Similarly, we say that \(Q\) satisfies \emph{label shift} relative to \(P_1,\dots,P_K\) if, for every \(k\in[K]\), the regular conditional law of \(X^k\) given \(Y^k\) under \(P_k\) coincides with the regular conditional law of \(X\) given \(Y\) under \(Q\), \(P_{k,Y}\)-almost surely. 

\begin{lemma}
\label{lem:covariate-and-label-shift}
    Covariate shift and label shift are each neither necessary nor
sufficient for monotonicity.
\end{lemma}

The examples that we use below to prove this lemma highlight the following difference: monotonicity is more robust to benign structural changes (like noise, \cref{exm:covariate-label-not-necessary}) but more sensitive to marginal shifts (\cref{exm:covariate-not-sufficient,exm:label-not-sufficient}).

\begin{proof}
We prove the claim via three examples. Throughout, we consider a single source domain when constructing counterexamples; this is without loss of generality since failure (or validity) of weak monotonicity for $K=1$ implies the same for general $K$ by ignoring the other sources.

\begin{example}[Covariate and label shift are not necessary for weak monotonicity]
\label{exm:covariate-label-not-necessary}
Weak monotonicity may hold even if the regular conditional laws of \(Y^1\) given \(X^1\) under \(P_1\) and of \(Y\) given \(X\) under \(Q\) differ. Consider binary classification with \(0\)--\(1\) loss. Let \(X^1\sim \Ber(\tfrac12)\) and \(Y^1=X^1\) almost surely. Define the target by \(X=X^1\) and \(Y=Y^1\oplus N\), where \(N\sim \Ber(p)\) with \(0<p<\tfrac12\), independent of \((X^1,Y^1)\). Then covariate shift fails since, for instance, the conditional law of \(Y^1\) given \(X^1=1\) is \(\delta_1\), whereas the conditional law of \(Y\) given \(X=1\) is \((1-p)\delta_1+p\delta_0\). Similarly, label shift fails since the conditional law of \(X^1\) given \(Y^1=1\) under \(P_1\) is \(\delta_1\), whereas the conditional law of \(X\) given \(Y=1\) under \(Q\) is \((1-p)\delta_1+p\delta_0\).
Moreover, for every classifier \(f\), \(\cR_Q(f)=\PP_{(X,Y)\sim Q}(f(X)\neq Y)=(1-p)\PP_{(X^1,Y^1)\sim P_1}(f(X^1)\neq Y^1)+p\,\PP_{(X^1,Y^1)\sim P_1}(f(X^1)=Y^1)=p+(1-2p)\cR_1(f)\). Since \(1-2p>0\), \(\cR_1(f)\le \cR_1(f')\) implies \(\cR_Q(f)\le \cR_Q(f')\). Thus weak monotonicity holds although covariate and label shift fail.
\end{example}

\begin{example}[Covariate shift alone is not sufficient for weak monotonicity]
\label{exm:covariate-not-sufficient}
Even if the regular conditional laws of \(Y^1\) given \(X^1\) under \(P_1\) and of \(Y\) given \(X\) under \(Q\) coincide, weak monotonicity may fail. Consider binary classification with \(0\)--\(1\) loss, let \(\cX=\{a,b\}\), and let \(Y^1\equiv 0\) and \(Y\equiv 0\) almost surely. Thus covariate shift holds. Let \(\PP(X^1=a)=0.9\), \(\PP(X^1=b)=0.1\), \(\PP(X=a)=0.1\), and \(\PP(X=b)=0.9\). Define classifiers \(f(a)=0,f(b)=1\) and \(f'(a)=1,f'(b)=0\). Then \(\cR_1(f)=\PP_{(X^1,Y^1)\sim P_1}(f(X^1)\neq Y^1)=0.1<0.9=\PP_{(X^1,Y^1)\sim P_1}(f'(X^1)\neq Y^1)=\cR_1(f')\), whereas \(\cR_Q(f)=\PP_{(X,Y)\sim Q}(f(X)\neq Y)=0.9>0.1=\PP_{(X,Y)\sim Q}(f'(X)\neq Y)=\cR_Q(f')\). Hence weak monotonicity fails.
\end{example}

The next example is essentially identical to \cref{exm:covariate-not-sufficient}, but for label shift.
\begin{example}[Label shift alone is not sufficient for weak monotonicity]
\label{exm:label-not-sufficient}
Even if the regular conditional laws of \(X^1\) given \(Y^1\) under \(P_1\) and of \(X\) given \(Y\) under \(Q\) coincide, weak monotonicity may fail. Let \(\cX=\{x_0\}\), \(\cY=\{0,1\}\), and use \(0\)--\(1\) loss. Then label shift holds trivially. Let \(\PP(Y^1=0)=0.9\) and \(\PP(Y=0)=0.1\). For the constant classifiers \(f\equiv 0\) and \(f'\equiv 1\), we have \(\cR_1(f)=\PP_{(X^1,Y^1)\sim P_1}(f(X^1)\neq Y^1)=0.1<0.9=\PP_{(X^1,Y^1)\sim P_1}(f'(X^1)\neq Y^1)=\cR_1(f')\), whereas \(\cR_Q(f)=\PP_{(X,Y)\sim Q}(f(X)\neq Y)=0.9>0.1=\PP_{(X,Y)\sim Q}(f'(X)\neq Y)=\cR_Q(f')\). Hence weak monotonicity fails.
\end{example}

Together, these three examples show that covariate shift and label shift are each neither necessary nor sufficient for weak monotonicity.
\end{proof}

\section{Additional Discussion of the Pareto Covering}\label{sec:additional-Pareto-covering}

\cameraonly{
\subsection{Computation of a Pareto covering}
In general, \emph{exactly} computing a minimal $t$-Pareto cover is computationally hard \citep{zitzler2008quality,papadimitriou2000approximability,chvatal1979greedy}. 
For finite hypothesis classes, the problem becomes much simpler to analyze: Let $M$ be the number of models, $K$ the number of source benchmarks, and $P \leq M$ the size of the Pareto set. 
For $t=0$ (used in our experiments), we only need to compute the exact Pareto set. This can be done by checking pairwise dominance across all models, which requires at most $O(KM^2)$ comparisons. In practice, this may already yield a good candidate set.
For $t>0$, one can first restrict attention to the Pareto set and then build a $t$-Pareto cover on top using standard greedy set-cover heuristics to obtain a small approximate cover with a logarithmic size overhead and computational cost $O(KP^2)$ \citep{chvatal1979greedy}. Notice that a logarithmic size overhead is essentially irrelevant in the upcoming bounds.
}

\subsection{Comparison of Different Coverings}
\label{subsec:covering-comparison}

We begin by providing a more in-depth comparison with the two other ways of covering a Pareto front (beyond the Pareto covering from \cref{def:Pareto-cover}) that we discussed in \cref{subsec:geometry-covering}.

\paragraph{Norm covering.} The first option is to construct a covering in some norm, such as the $\ell_2$-norm. An example of this is visualized in \cref{fig:Pareto-covering-comparison} on the right. It is clear that this method of covering adapts to the geometry, and if the front has low intrinsic volume, this covering number will be small. 
And indeed, perhaps unsurprisingly, the $\ell_2$-norm covering converges to a uniform distribution with respect to the Hausdorff measure under some regularity conditions \citep{borodachov2007asymptotics}.
It also has other appealing properties, some of them discussed in \citet{zhang2024gliding}. 
From a statistical perspective, however, when a notion of monotonicity holds, it is not necessary to cover the Pareto front point-wise in this norm. After all, if a point lies far away in $\ell_2$-norm, but is almost dominated by another, there is no need to spend any statistical budget on it. In particular, this is reflected in the fact that when better trade-offs become available, this covering number can actually \emph{increase}, as visualized in \cref{fig:Pareto-covering-comparison}, and numerically validated in \cref{subfig:eps-pareto-set-curve}. Moreover, the distribution induced by this kind of covering does not avoid improper regions, see also \cref{fig:Pareto-covering}.
As demonstrated in \cref{sec:transfer-learning}, this is contrary to the actual statistical hardness under monotonicity, so these covering numbers do not offer a good perspective on the problem.

\paragraph{Simplex $\ell_1$-covering.} Another option is to cover the simplex $\triangle^{K-1}$ in $\ell_1$-norm, as done in \citet{mansour2021theory} (cf.\ \cref{sec:separation-mansour}), and use it as weights for a weighted average. In particular, standard covering results show that creating a $t$-covering $\Lambda\subset \triangle^{K-1}$ in $\ell_1$-norm requires a set of size $\abs{\Lambda}=\Theta(t^{-(K-1)})$. Because, under the assumption that the Pareto front is \emph{convex}, it can be fully recovered by solving optimization problems of the form
\begin{equation*}
    f_\lambda \in \argmin_{f\in\cF} \sum_{k=1}^K \lambda_k\cR_k(f),
\end{equation*}
it is natural to cover the Pareto front with points $\crl{\bR(f_\lambda):\lambda\in\Lambda}$, see center panel of \cref{fig:Pareto-covering-comparison}.
However, this approach has three major limitations. 
Firstly, the \emph{number} of covering points is completely oblivious to the Pareto front, even when it only contains one point. Hence bounds using this covering directly cannot show adaptivity to the geometry. Moreover, as demonstrated in \cref{lem:linf-bound-covering} below, the Pareto covering number is never larger than $\Theta(t^{-(K-1)})$. 
Secondly, while the induced distribution on the Pareto set changes with the front, it can produce unnecessary redundancies and emphasis on ``almost dominated'' parts of the front, as visualized in \cref{fig:Pareto-covering}. Informally, it is not hard to see that under suitable regularity conditions, the limiting distribution of such a covering as $t\to 0$ is the push-forward of the uniform distribution on $\triangle^{K-1}$ with the map $\lambda\mapsto \bR(f_\lambda)$. In the setting of \cref{ex:limiting-distribution} for $p>1$, this corresponds to a density on $[\eps,1-\eps]$ given by 
\begin{equation*}
    \frac{\d \pi_{\operatorname{simplex-}\ell_1}}{\d \lambda}(x) \propto \frac{x^{p-2}(1-x)^{p-2}}{(x^{p-1}+(1-x)^{p-1})^2},
\end{equation*}
where by abuse of notation here $\lambda$ denotes the Lebesgue measure. Note that this is different from the limiting distribution of the Pareto covering number.
And thirdly, when the front is non-convex, and only monotonicity holds, it may not reach the optimal solution (although this is ruled out under the stronger mixture assumption from \citet{mansour2021theory}).
\begin{figure}
    \centering
    \includegraphics[width=0.9\linewidth]{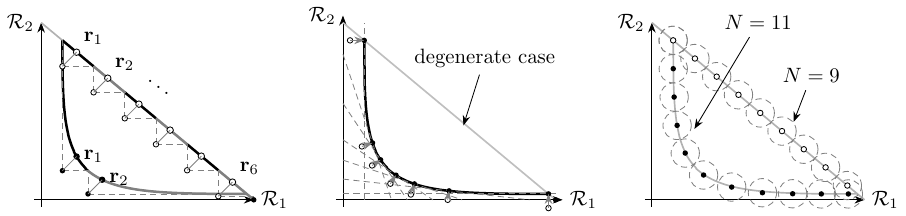}
    \caption{Comparison of different coverings. \emph{Left:} Pareto covering. \emph{Center:} $\ell_1$-simplex covering and weighted sum. \emph{Right:} $\ell_2$-covering.}
    \label{fig:Pareto-covering-comparison}
\end{figure}

\cameraonly{
In \cref{fig:Pareto-covering}, we visualize the setting of \cref{ex:limiting-distribution} for the values $p=2$ and $p=4$. \cref{subfig:eps-pareto-set-fronts} shows the Pareto fronts with numerically computed minimal Pareto covering as well as the two different coverings from above. \cref{subfig:eps-pareto-set-distributions} shows that each covering method induces a different distribution over the decision space, and depends differently on the Pareto front geometry.
Clearly, the Pareto covering number puts more mass on parts corresponding to the ``elbow'' of the Pareto fronts. \cref{subfig:eps-pareto-set-scale} shows that, unsurprisingly, all coverings scale as $t^{-(K-1)}=t^{-1}$, but there is a constant factor gap between them, with the Pareto covering being much smaller at the same scale. Therefore, the other two coverings can be ``unnecessarily large''. Finally, \cref{subfig:eps-pareto-set-curve,subfig:eps-pareto-set-distributions} together show that only the Pareto covering adapts to the Pareto front geometry as one may expect: the $\ell_2$-norm covering number \emph{increases} as more favorable trade-offs become available ($p=4$), and even puts more mass on improper regions. At the same time, naturally the simplex $\ell_1$-covering number is completely agnostic to the geometry of the front, and while the induced distributions are more similar to those of the Pareto covering, they still do not avoid improper regions in the case $p=4$.
\begin{figure}[t]
    \begin{minipage}[c]{0.72\textwidth}
        \begin{subfigure}{\linewidth}
            \includegraphics[width=\linewidth, trim={0cm 0cm 0cm 0cm}, clip]{figures/eps_pareto_set_horizontal_pareto_fronts.pdf}
            \caption{Three different covers of two Pareto fronts at the same scale $t$.}
            \label{subfig:eps-pareto-set-fronts}
        \end{subfigure}
        \begin{subfigure}{\linewidth}
            \includegraphics[width=\linewidth, trim={0cm 0.2cm 0cm 0cm}, clip]{figures/eps_pareto_set_horizontal_decision_space.pdf}
            \caption{The density estimates and limiting distributions in decision space.}
            \label{subfig:eps-pareto-set-distributions}
        \end{subfigure}
    \end{minipage}
    \hfill
    \begin{minipage}[c]{0.275\textwidth}
    \begin{subfigure}{\linewidth}
        \includegraphics[width=\linewidth, trim={0cm 0.4cm 0cm 0cm}, clip]{figures/eps_pareto_set_horizontal_sizes_vs_scale.pdf}
        \caption{Cover size dependence on $t$.}
        \label{subfig:eps-pareto-set-scale}
    \end{subfigure}
    \begin{subfigure}{\linewidth}
        \includegraphics[width=\linewidth, trim={0cm 0.5cm 0cm 0cm}, clip]{figures/eps_pareto_set_horizontal_sizes_at_eps.pdf}
        \caption{Cover size dependence on $p$.}
        \label{subfig:eps-pareto-set-curve}
    \end{subfigure}
    \end{minipage}
    \caption{(\subref{subfig:eps-pareto-set-fronts}) For a fixed scale $t$ and two different Pareto fronts from \cref{ex:limiting-distribution}, we plot a Pareto covering and two different covers from \cref{subsec:covering-discussion,subsec:covering-comparison}. 
    (\subref{subfig:eps-pareto-set-distributions}) We plot a kernel density estimate and the limiting distribution from \cref{thm:limiting-distribution} (for the other two limiting distributions, see \cref{subsec:covering-comparison}). 
    (\subref{subfig:eps-pareto-set-scale}) As a function of the scale $t$, the covering numbers are all of order $t^{-(K-1)}=t^{-1}$, but the constant factor depends on the geometry and the covering method.
    (\subref{subfig:eps-pareto-set-curve}) The Pareto covering number decreases with more favorable trade-offs, whereas the simplex $\ell_1$ cover is indifferent, and the $\ell_2$-norm cover even increases.}
    \label{fig:Pareto-covering}
\end{figure}
}

\subsection{A Worst-case Bound on the Pareto Covering Number}

Recall the definition of Pareto covering number from \cref{def:Pareto-cover}. Here we prove a bound that is independent of $\bR$ and hence worst-case over all possible Pareto front geometries.
\begin{lemma}
\label{lem:linf-bound-covering}
    Suppose that $\bR:\cF\to [0,1]^K$. Then $\Npar(t,\cF,\bR) \leq \prn{1+\frac{1}{t}}^{K-1}$ for all $t\in (0,1]$.
\end{lemma}
\begin{proof}
    It is a well-known fact that we can cover the cube $[0,1]^K$ in $\ell_\infty$-norm by $(1+1/t)^K$ points, which would immediately imply the weaker bound $\Npar(t,\cF,\bR) \leq \prn{1+\frac{1}{t}}^{K}$. The difference here is that we want to show an exponent of $K-1$ rather than $K$.

    Fix $t\in(0,1]$ and define the grid
    \begin{equation*}
        G= \crl{t,2t,\dots,\ceil{\frac{1}{t}}t}\subset[0,1+t].
    \end{equation*}
    Then the grid has cardinality at most $|G|=\ceil{\frac{1}{t}}\leq 1+\frac1t$. By replicating this grid on the first $K-1$ coordinates, we obtain the net $G^{K-1}$.
    Now, for each point in the grid $\a=(a_1,\dots,a_{K-1})\in G^{K-1}$ define the slice
    \begin{equation*}
        S_{\a} = \{\r\in \bR(\cF):\ r_j\le a_j\ \forall j\le K-1\}.
    \end{equation*}
    We ignore all $\a$ for which $S_{\a}=\emptyset$. For the rest,
    we let $\alpha(\a)= \inf\{r_K: \r\in S_{\a}\}\in[0,1]$. By definition of the infimum, there exists some $\r(\a)\in S_{\a}$ such that in the $K$th coordinate it holds $r_K(\a) \le \alpha(\a)+t$.
    Now pick any $f_{\a}\in\cF$ with $\bR(f_{\a})=\r(\a)$ (which is possible since $\r(\a)\in\bR(\cF)$) and let $\cI\subseteq G^{K-1}$ be the index set of all $\a$ with $S_{\a}\neq\emptyset$. Then
    \begin{equation*}
        |\cI|\le |G|^{K-1} \le \left(1+\frac1t\right)^{K-1}.
    \end{equation*}

    We claim that $\{f_{\a}\}_{\a\in\cI}$ is a $t$-Pareto cover of $\cF$:
    Take any $f\in\cF$ and write $\r'=\bR(f)\in[0,1]^K$.
    Choose $\a\in G^{K-1}$ by
    \[
    a_j := t \max\crl{1,\ceil{\frac{r_j'}{t}}},\qquad \forall j\in[K-1].
    \]
    Then it holds that $r_j' \le a_j \le r_j'+t$ for all $j\le K-1$, and hence $\r'\in S_{\a}$. In particular, we obtain that
    $S_{\a}\neq\emptyset$, so we know that ${\a}\in\cI$.
    Since $\r(\a)\in S_{\a}$, we have $r_j(\a)\le a_j\le r_j'+t$ for all $j\le K-1$.
    Moreover, because $\r'\in S_{\a}$, we get $\alpha(\a)\le r_K'$, hence
    $
    r_K(\a) \le \alpha(\a)+t \le r_K'+t.
    $
    Therefore, in all coordinates, $r_j(\a)\leq r_j'+t,j\in[K]$, i.e.,
    \[
    \bR(f_{\a})=\r(\a) \preceq \r'+t\one = \bR(f)+t\one.
    \]
    Thus for every $f$ there exists some $f_{\a},\a\in\cI$ with
    $\bR(f_\a)\preceq\bR(f)+t\one$, proving that $\Npar(t,\cF,\bR)\le |\cI|
    \le (1+1/t)^{K-1}$. Note that technically we require the $t$-Pareto set to only contain Pareto optimal models, which our construction does not guarantee. However, we may simply replace any $f_{\a}$ by a function that is Pareto optimal and dominates $f_{\a}$ using \cref{lem:monotonicity-pareto}. That concludes the proof.
\end{proof}

\section{Separation from LMSA by \texorpdfstring{\citet{mansour2021theory}}\ \  on Mixture Distributions}\label{sec:separation-mansour}

In this section, we compare the Pareto ERM algorithm (\cref{subsec:Pareto-ERM}) with the \emph{Limited Target Data Multiple Source Adaptation} (LMSA) algorithm from \cite{mansour2021theory} from a learning perspective. 
As mentioned in \cref{subsec:covering-discussion,subsec:covering-comparison}, if monotonicity holds but the target is not a mixture of the sources, it is obvious that the Pareto ERM can outperform LMSA simply because LMSA cannot reach all parts of the Pareto frontier, as \cref{ex:convex-hull-counterexample} demonstrates. And even if the target is a mixture (and by \cref{lem:Q-in-convex-hull} monotonicity holds), the upper bound of order $\sqrt{K\log(n)/n}$ by \citet{mansour2021theory} for the LMSA algorithm can be much worse (but not better) than the bound from \cref{cor:tstar} for Pareto ERM.

However, this does not yet imply that LMSA actually performs worse than Pareto ERM \emph{on mixture distributions}; only that our bound is adaptive. At the same time, \cref{fig:Pareto-covering} shows that the algorithms do actually fundamentally work differently; and in particular, that the covering used by LMSA can be overly redundant in certain parts of the Pareto frontier. We now demonstrate that this indeed can have an effect on the \emph{statistical} rates achieved by LMSA, and there is a true separation.

For simplicity, we show this separation in the case $K=2$; we leave the interesting but more involved case of $K>2$ for future work. To that end, let us first introduce the LMSA algorithm formally. 

\paragraph{LMSA for $K=2$.} Let $\fG$ be the set of uniform grids on $[0,1]$ with varying widths $h>0$:
\begin{equation*}
    \fG=\crl{(\theta+h\mathbb Z)\cap[0,1]:0<h\le1,\theta\in[0,h)}.
\end{equation*}
For $\Lambda\in\fG$, $\LMSA$ is defined as
\begin{enumerate}[leftmargin=*,itemsep=-0.1pt,topsep=-2pt,label=(\roman*)]
    \item Compute the set of minimizers for each fixed set of weights 
    $$\cF_{\Lambda}(\bR):=\Big\{f_\lambda:\;f_\lambda\in\argmin_{f\in\cF}\lambda\cR_1(f)+(1-\lambda)\cR_2(f),\lambda\in\Lambda\Big\},$$ 
    where we use any fixed tie-breaking rule (the choice does not matter for the following result).
    \item Return the
empirical target-risk minimizer $\fLMSA\in \argmin_{f\in \cF_{\Lambda}(\bR)}\Rhat_Q(f)$ over $\cF_\Lambda(\bR)$. 
\end{enumerate}
We get the following separation result.

\begin{theorem}[Separation from LMSA on mixture distributions]
\label{thm:mixture-separation}
There are universal constants $c,C>0$ such that, for all sufficiently large sample sizes
$n\in\NN$, there exists a class $\cC_n$ of instances
$(\cF,\bR,\cR_Q)$ with source risks in $[0,1]^2$ and target losses bounded in
$[0,1]$ satisfying: (i) Each target risk is an exact mixture of the two source risks: for
every instance there is a $\lambda^\star\in[0,1]$ such that $\cR_Q(f)=\lambda^\star\cR_1(f)+(1-\lambda^\star)\cR_2(f)$ for all $f\in\cF$. (ii) Every point of the Pareto front is supported by a weighted sum of source objectives.
(iii) Even after optimizing the grid, the $\LMSA$ estimator is lower bounded by
\begin{equation*}
    \inf_{\Lambda\in\fG}
    \sup_{(\cF,\bR,\cR_Q)\in\cC_n}
    \EE\brk{\excessRisk{\fLMSA}{\cF}}
    \ge
    c\sqrt{\frac{\log n}{n}}.
\end{equation*}
(iv) Uniformly over $\cC_n$, Pareto ERM with scale $1/n$ satisfies
\begin{equation*}
    \sup_{(\cF,\bR,\cR_Q)\in\cC_n}
    \EE\brk{\excessRisk{\Paretoerm{1/n}}{\cF}}
    \le
    \frac{C}{\sqrt n}.
\end{equation*}
\end{theorem}
The proof is in the following \cref{proof:mixture-separation}.
Note that the gap is only of order $\sqrt{\log n}$, which is essentially the largest gap one may hope for in the case $K=2$. We suspect that a similar construction for $K>2$ would yield a bigger gap (depending on $K$), but we leave this to future work.

\subsection{Proof of \texorpdfstring{\cref{thm:mixture-separation}}{Theorem \ref{thm:mixture-separation}}}
\label{proof:mixture-separation}
Set $\gamma:=a_0\sqrt{\log(n)/n}$, with $a_0>0$ a sufficiently small universal
constant, and set $M:=\floor{1/\gamma}\geq 1$. Write
$\Slambda{\lambda}(\r)=\lambda r_1+(1-\lambda)r_2$ for linear scalarizations with weight $\lambda$ on the first objective. The only geometric fact we
use is that if two vectors $\r,\r'$ differ as $\r'-\r=\Delta(1,-\theta/(1-\theta))^\top$ with some $\Delta>0,\theta\in(0,1)$ then the scalarizations of $\r,\r'$ are tied at exactly $\theta$, that is, $\Slambda{\theta}(\r)=\Slambda{\theta}(\r')$, because for any $\lambda\in[0,1]$ we have
\begin{equation}
\label{eqn:tie-weight-formula}
    \Slambda{\lambda}(\r')-\Slambda{\lambda}(\r) = \lambda \Delta -(1-\lambda)\frac{\Delta\theta}{1-\theta} =
    \Delta\frac{\lambda(1-\theta)-(1-\lambda)\theta}{1-\theta}=
    \Delta\frac{\lambda-\theta}{1-\theta}.
\end{equation}

We make a case distinction depending on the grid width $h$. The main effort of the lower bound is in the fine case (small $h$), as we need to carefully construct a Pareto set on which LMSA will use an overly redundant covering (in a statistically meaningful way), while Pareto ERM remains statistically efficient. For coarse grids, we can invoke a simpler approximation error argument.

\paragraph{Case 1: Fine grids.}
Suppose $\Lambda\in\fG$ has width $h\le1/(40M)$.
Then
$\Lambda$ contains $M$ points $\lambda_1>\cdots>\lambda_M$ in $(1/5,1/3)$ because $(1/3-1/5)/h=2/(15h)>\frac{1}{40 h}(1+1/M)\geq M+1$.
Also, $\Lambda$ contains at least one point $\lambda^\star\in(11/20,3/5)$, since $(3/5-11/20)/h=1/(20h)=2/(40 h) \geq 2M\geq 2$. 

\emph{Construction of Pareto front.} Let $\g=\r_1=(1/2,1/2)$, choose
$\theta_L=2/3$, $\theta_R=1/2$, and define
\begin{equation}
\label{eqn:r0-r2-def}
    \r_0=\g-\gamma\frac{1-\theta_L}{\theta_L-\lambda^\star}
    \begin{pmatrix}1\\-\theta_L/(1-\theta_L)\end{pmatrix},
    \qquad
    \r_2=\g+\gamma\frac{1-\theta_R}{\lambda^\star-\theta_R}
    \begin{pmatrix}1\\-\theta_R/(1-\theta_R)\end{pmatrix}.
\end{equation}
Then $\g$ is selected over $\r_0,\r_2$ by scalarizing with $\lambda^\star$, because both $\r_0$ and $\r_2$ have scalarized risk $\Slambda{\lambda^\star}(\g)+\gamma$:
\begin{equation}
\label{eqn:scalarized-risks-grs}
    \begin{aligned}
         \Slambda{\lambda^\star}(\g) &= \lambda^\star \frac{1}{2}+(1-\lambda^\star)\frac{1}{2} =\frac{1}{2}, \\
    \Slambda{\lambda^\star}(\r_0) &= \lambda^\star \prn{\frac{1}{2}-\gamma\frac{1-\theta_L}{\theta_L-\lambda^\star}}+(1-\lambda^\star)\prn{\frac{1}{2}+\gamma \frac{\theta_L}{\theta_L-\lambda^\star}} 
    = \frac{1}{2}+\gamma, \\
    \Slambda{\lambda^\star}(\r_2) &= \lambda^\star \prn{\frac{1}{2}+\gamma\frac{1-\theta_R}{\lambda^\star-\theta_R}}+(1-\lambda^\star)\prn{\frac{1}{2}-\gamma \frac{\theta_R}{\lambda^\star-\theta_R}} = \frac{1}{2}+\gamma.
    \end{aligned}
\end{equation}
\begin{wrapfigure}{r}{0.45\textwidth}
    \vspace{-0.4cm}
    \includegraphics[width=\linewidth]{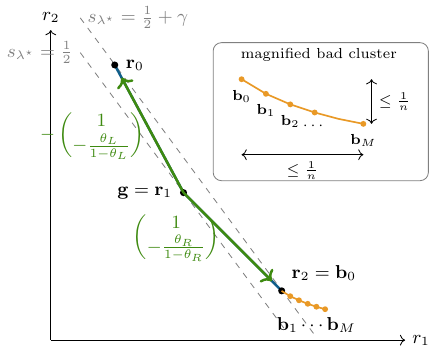}
    \caption{The Pareto front for fine grids.}
    \label{fig:LMSA-counter}
    \vspace{-2cm}
\end{wrapfigure}
Now we ``attach a bad cluster'' of points below $\r_2$. Set $\vartheta_0=2/5$, choose $\vartheta_j=(\lambda_{j+1}+\lambda_j)/2$ for $j=1,\dots,M-1$, and let $\b_0:=\r_2$ and recursively for $j=1,\dots,M$
\begin{equation*}
    \b_j
    =
    \b_{j-1}
    +
    \frac{1}{10nM}
    \begin{pmatrix}
        1\\
        -\vartheta_{j-1}/(1-\vartheta_{j-1}).
    \end{pmatrix}
\end{equation*}
We visualize the construction in \cref{fig:LMSA-counter}.
Importantly, the cluster $\crl{\r_2,\b_1,\dots,\b_M}$ lies within a $\ell_\infty$-ball of radius much smaller than
$1/n$:
\begin{equation}
\label{eqn:width-bad-cluster}
    \begin{aligned}
        \sum_{j=1}^M\abs{b_{j,1}-b_{j-1,1}} &=\sum_{j=1}^M \frac{1}{10nM} = \frac{1}{10n},\\ 
    \sum_{j=1}^M\abs{b_{j,2}-b_{j-1,2}} &=\sum_{j=1}^M \frac{1}{10nM} \frac{\vartheta_{j-1}}{1-\vartheta_{j-1}} \leq \frac{1}{15 n}
    \end{aligned}
\end{equation}
where we used that $\vartheta_0=2/5$ and $\vartheta_j\leq 2/5$ for all $j$.
Since, by definition of $\b_j$ and \cref{eqn:tie-weight-formula}, $\b_{j-1}$ and $\b_j$ tie for scalarization weight $\vartheta_{j-1}$ (that is, $\Slambda{\vartheta_{j-1}}(\b_j)=\Slambda{\vartheta_{j-1}}(\b_{j-1})$) and,
for $j<M$, $\b_j$ and $\b_{j+1}$ tie for weight $\vartheta_j$, the point
$\b_j$ is selected by minimizing $\Slambda{\lambda_j}(\r)$ over $\crl{\r_0,\g,\r_2,\b_1,\ldots,\b_M}$. Also, by combining \cref{eqn:scalarized-risks-grs,eqn:width-bad-cluster} and some more calculations we can see that
\begin{equation}
\label{eqn:scalarized-risk-bj}
    \gamma
    \le
    \Slambda{\lambda^\star}(\b_j)-\Slambda{\lambda^\star}(\g)
    \le
    \gamma+\frac1n,
    \qquad j=1,\dots,M.
\end{equation}
It is also clear that all points are supported by a weighted sum of source objective: $\r_0$ by sufficiently large weight $\lambda$, $\g$ at $\lambda^\star$, and $\b_1,\ldots,\b_M$ each with weights between each tie weights. 

\emph{Construction of distributions and risks.}
Let $\cF=\crl{\r_0,\g,\r_2,\b_1,\dots,\b_M}$ and identify each hypothesis with
its vector of source risks, so that $\bR(\r)=\r$. We can now realize these risks by
actual source distributions. For an observation $Z$, we define the whole loss vector
$(L_{\r})_{\r\in\cF}\in[0,1]^{\cF}$ where $L_\r = \ell(\r,Z)$. We may choose $Z$ and $\ell$ so that the distribution of $L$ is as follows: Under source $k\in\crl{1,2}$, $L_{\g}$ is
the constant $1/2$, $L_{\r_0}$ and $L_{\r_2}$ are the constants $r_{0,k}$ and
$r_{2,k}$, and the coordinates $L_{\b_1},\dots,L_{\b_M}$ are mutually
independent Bernoulli random variables with
\begin{equation*}
    \PP_k(L_{\b_j}=1)=b_{j,k},
    \qquad
    \PP_k(L_{\b_j}=0)=1-b_{j,k}.
\end{equation*}
We can then see that then exactly $\bR(\r)=\r$ for all $\r\in\cF$:
\begin{align*}
    \bR(\r_0) &= \EE\brk{L_{\r_0}} = \r_0, \quad \bR(\g) = \EE\brk{L_{\g}} = \g, \quad \bR(\r_2) = \EE\brk{L_{\r_2}} = \r_2, \\
    \bR(\b_j) &= \EE\brk{L_{\b_j}} = \begin{pmatrix}
        \PP_1(L_{\b_j}=1) \\
        \PP_2(L_{\b_j}=1)
    \end{pmatrix} 
    = \b_j, \qquad j=1,\ldots,M.
\end{align*}
We choose the target distribution is the mixture
$Q=\lambda^\star P_1+(1-\lambda^\star)P_2$, equivalently it first draws
$T\in\crl{1,2}$ with $\PP(T=1)=\lambda^\star$ and then draws the loss vector $L$
from $P_T$. Hence, for every $\r\in\cF$,
\begin{equation*}
    \cR_Q(\r)
    =\lambda^\star r_1+(1-\lambda^\star)r_2
    =\Slambda{\lambda^\star}(\r).
\end{equation*}

\emph{Lower bound for LMSA.}
Let $T_1,\dots,T_n$ be the source indices in the target sample and
$N_1:=\sum_{i=1}^n\mathbf{1}\crl{T_i=1}$. On the event
$E:=\crl{|N_1/n-\lambda^\star|\le\gamma/4}$, which has probability at least $1/2$ for all large $n$: By Hoeffding's inequality and $\gamma= a_0\sqrt{\log(n)/n}$ we have that
\begin{equation}
\label{eqn:E-bound}
    \PP(E^c)  \leq 2\exp\prn{-2n\frac{\gamma^2}{16}} = 2n^{-a_0^2/8},
\end{equation}
so for $n$ large enough, $\PP(E)\geq 1/2$.
On $E$ write $q:=N_1/n$. 
Conditional on the source indices, $\widehat{\cR}_Q(\b_j)$ is the average of
independent (but not necessarily identically distributed) Bernoulli variables with parameters in $[1/3,2/3]$, because all
coordinates of the front lie in this interval:
\begin{equation*}
    \Rhat_Q(\b_j) = \frac{1}{n} \sum_{i=1}^n \underbrace{\one\crl{T_i=1}(L_{\b_j})_1}_{\sim \Ber(b_{j,1})\text{ conditional on }T_i} + \underbrace{\one\crl{T_i=2}(L_{\b_j})_2}_{\sim \Ber(b_{j,2})\text{ conditional on }T_i}=:\frac{1}{n} \sum_{i=1}^n X_{j,i}.
\end{equation*}
Hence, conditional on
the source indices $T_1,\ldots,T_n$, the mean of $\widehat{\cR}_Q(\b_j)$ is
\begin{equation*}
    \EE\brk{\widehat{\cR}_Q(\b_j)\mid T_1,\ldots,T_n}=qb_{j,1}+(1-q)b_{j,2}
    =\Slambda{\lambda^\star}(\b_j)+(q-\lambda^\star)\prn{b_{j,1}-b_{j,2}}.
\end{equation*}
Since $\b_j\in[0,1]^2$ and $|q-\lambda^\star|\le\gamma/4$ on $E$, $\abs{(q-\lambda^\star)\prn{b_{j,1}-b_{j,2}}}\le\frac{\gamma}{4}$.
Combining this with
$1/2+\gamma\le\Slambda{\lambda^\star}(\b_j)\le1/2+\gamma+1/n$ from \cref{eqn:scalarized-risk-bj} gives, for all
large $n$,
\begin{equation*}
    \frac12+\frac{3\gamma}{4}
    \le
    \EE\brk{\widehat{\cR}_Q(\b_j)\mid T_1,\dots,T_n}
    \le
    \frac12+\frac{5\gamma}{4}+\frac1n
    \le
    \frac12+\frac{3}{2}\gamma . %C_0 = 3/2
\end{equation*}
We now use the celebrated Berry-Esseen theorem (see \citet{petrov1975sums} for different versions).
\begin{lemma}[Theorem 3 in Section 5.2 of \citet{petrov1975sums}]
\label{lem:Berry-Esseen}
    There exists a positive universal constant $C>0$ such that the following is true.
    Suppose $X_1,\ldots,X_n $ are independent (but not necessarily identically distributed) real-valued random variables with means $\mu_i$, variance $\sigma^2_i$, and third moment $\rho_i = \EE\brk{\abs{X_i-\mu_i}^3}<\infty$. Then for the standard normal c.d.f.\ $\Phi$ it holds
    \begin{equation*}
        \sup_{x\in\RR}\abs{\PP\prn{\frac{\sum_{i=1}^n [X_i -\mu_i]}{\sqrt{\sum_{i=1}^n \sigma_i^2}}\leq x}-\Phi(x)} \leq C\frac{\sum_{i=1}^n\rho_i}{(\sum_{i=1}^n\sigma_i^2)^{3/2}}.
    \end{equation*}
\end{lemma}

We can apply this to our setting. Specifically, let $S_j:=n\widehat{\cR}_Q(\b_j)=\sum_{i=1}^n X_{j,i}$, $\mu_{j,i}=\EE\brk{X_{j,i}\mid T_1,\ldots,T_n}$, $m_j:=\EE[S_j\mid T_1,\dots,T_n] =\sum_{i=1}^n \mu_{j,i}$, $v_j^2:=\Var(S_j\mid T_1,\dots,T_n)=\sum_{i=1}^n \Var(X_i\mid T_1\ldots,T_n)$, $\rho_{j,i} = \EE\brk{\abs{X_{j,i}-\mu_{j,i}}^3\mid T_1\ldots,T_n}$ and $a_j:=m_j-n/2$. Then, on $E$,
\begin{equation*}
    \frac{3n\gamma}{4}\le a_j \le \frac{3n\gamma}{2}
    \qquad\text{and}\qquad
    \frac{2n}{9}\le v_j^2\le\frac{n}{4} \qquad \text{and}\qquad \sum_{i=1}^n\rho_{j,i}\leq n.
\end{equation*}
The variance bounds use $p(1-p)\in[2/9,1/4]$ for $p\in[1/3,2/3]$. We obtain that
\begin{equation}
\label{eqn:Berry-Esseen-applied-bound}
    \frac{\sum_{i=1}^n \rho_{j,i}}{v_j^3} \leq \frac{n}{(2n/9)^{3/2}} \leq \frac{10}{\sqrt{n}}.
\end{equation}
Let $x_j:=a_j/v_j$ and notice $0< x_j\le \frac{9}{2\sqrt{2}}\sqrt{n}\gamma$.
and \cref{lem:Berry-Esseen} yields
\begin{align*}
    \PP\prn{S_j\leq \frac{n}{2}-1\mid T_1,\dots,T_n} &= \PP\prn{\sum_{i=1}^n [X_{j,i}-\mu_{j,i}]\leq \frac{n}{2}-m_j-1 \mid T_1,\ldots,T_n} \\
    &= \PP\prn{\frac{\sum_{i=1}^n [X_{j,i}-\mu_{j,i}]}{v_j}\leq \frac{\frac{n}{2}-m_j}{v_j}-\frac{1}{v_j} \mid T_1,\ldots,T_n} \\
    &=\PP\prn{\frac{\sum_{i=1}^n [X_{j,i}-\mu_{j,i}]}{v_j}\leq -x_j-\frac{1}{v_j} \mid T_1,\ldots,T_n} \\
    &\ge\Phi\prn{-x_j-\frac{1}{v_j}}-\frac{10}{\sqrt{n}}. \tag{\cref{lem:Berry-Esseen} and \eqref{eqn:Berry-Esseen-applied-bound}}
\end{align*}
As shown in \citet[bottom of page 15]{gasull2014approximating}, Mills' ratio can be approximated to yield the lower bound $\Phi(-y)\ge \kappa (1+y)^{-1}\exp(-y^2/2)$ (for $\kappa=1/\sqrt{2\pi}$) for every
$y>0$. Since $x_j\le \frac{9}{2\sqrt{2}}\sqrt{n}\gamma$ for all large $n$ and $1/v_j \leq \sqrt{9/(2n)}$, we know that for large enough $n$ we have $x_j+1/v_j\leq 5\sqrt{n}\gamma$. The
last display then yields the bound
\begin{align*}
    \PP\prn{\Rhat_Q(\b_j)< \frac12 \mid T_1,\dots,T_n}&\ge
    \frac{c}{1+\sqrt{n}\gamma}\exp(- 13 n\gamma^2)-\frac{10}{\sqrt{n}} \tag{for a small enough $c>0$} \\
    &=\frac{c}{1+a_0\sqrt{\log n}}\exp(- 13 a_0^2\log n)-\frac{10}{\sqrt{n}} \tag{$\sqrt{n}\gamma=a_0\sqrt{\log n}$} \\
    &\geq \frac{c}{1+a_0 \sqrt{\log n}}n^{-13a_0^2} -\frac{10}{\sqrt{n}} \\
    &\geq c'\sqrt{\frac{\log n}{n}} \tag{for $13a_0^2<1/2$ and large enough $n$} \\
    &= \frac{c'}{a_0} \gamma \geq  c_0 \frac{1}{M}. \tag{by definition $M^{-1}\le 2\gamma$}
\end{align*}
Thus, after adjusting constants appropriately, uniformly over all source indices in $E$,
\begin{equation*}
    \PP\left(
        \widehat{\cR}_Q(\b_j)<\frac12
        \mid T_1,\dots,T_n
    \right)
    \ge
    \frac{c_0}{M}.
\end{equation*}
Conditional on the source indices, the events
$\crl{\widehat{\cR}_Q(\b_j)<1/2}$ are independent across $j$. Hence
\begin{align*}
    \PP\prn{\min_{1\le j\le M}\widehat{\cR}_Q(\b_j)<\frac12}&\geq \EE\brk{\one_E\PP\prn{\bigcup_{j=1}^M\crl{\Rhat_Q(\b_j)<\frac{1}{2}}\mid T_1,\ldots,T_n}} \\
    &=\PP\prn{E}\EE\brk{ \prn{1-\prod_{j=1}^M\prn{1-\PP\prn{\Rhat_Q(\b_j)<\frac{1}{2}\mid T_1,\ldots,T_n}}}\mid E} \\
    &\ge \frac{1}{2}\left(1-\left(1-\frac{c_0}{M}\right)^M\right)\ge c_1. \tag{by \eqref{eqn:E-bound} and the lower bound above}
\end{align*}
On this event, empirical minimization over $\cF_\Lambda(\bR)$ selects a bad
vertex because $\g\in\cF_\Lambda(\bR)$ has empirical risk exactly $1/2$. Hence, by \cref{eqn:scalarized-risk-bj} and because all other vertices have excess risk at least $\gamma$, we have
\begin{equation*}
    \EE_Q\brk{\excessRisk{\fLMSA}{\cF}}
    \ge \PP\prn{\min_{1\le j\le M}\widehat{\cR}_Q(\b_j)<\frac12} \gamma =c_1a_0 \sqrt{\frac{\log n}{n}}.
\end{equation*}

\emph{Upper Bound for Pareto ERM.}
It is easy to see that any minimal $1/n$-Pareto cover must contain $\g$ and must be of size $3$:
Since the bad cluster $\crl{\r_2=\b_0,\b_1,\dots,\b_M}$ has coordinate diameter at most $1/n$ (\cref{eqn:width-bad-cluster}), choosing any element will cover the cluster in Pareto distance. Moreover, no other point can cover $\r_0$ or $\g$ at scale $1/n$ since in each coodinate the distance is at least $\gamma \gg 1/n$, so $\r_0$ and $\g$ must be contained in any minimal Pareto covering.
For instance, the set $\crl{\r_0,\g,\r_2}$ is a minimal $1/n$-Pareto cover. 

Empirical target-risk minimization over $\crl{\r_0,\g,\r_2}$ has finite-class estimation
cost at most $\sqrt{2\log(6)/n}$ for losses in $[0,1]$. In particular, by \cref{thm:Geometry-complexity-upper-bound} (since $\uppermodulus(1/n)\leq 1/n$ by \cref{lem:Q-in-convex-hull,lem:bound-modulus-Lipschitz}), we get for all $u>0$
\begin{equation*}
    \PP\prn{\excessRisk{\Paretoerm{1/n}}{\cF} > \frac{1}{n}+u} \leq 6\exp\prn{-\frac{nu^2}{2}}.
\end{equation*}
By simple tail integration we obtain that
\begin{align*}
    \EE_Q\brk{\excessRisk{\Paretoerm{1/n}}{\cF}}
    &\leq \frac{1}{n}+ \int_{0}^\infty 6\exp\prn{-nu^2/2} \d u \leq \frac{4}{\sqrt n}.
\end{align*}

\paragraph{Case 2: Coarse grids.}
Now suppose the grid width is $h>1/(40M)$. There is an interval
$[\theta_R,\theta_L]\subset[11/20,3/5]$ with $[\theta_R,\theta_L]\cap \Lambda = \emptyset$ of length at least
$\min\crl{h/2,1/40}$, since $3/5-11/20=1/20$. Set the parameter $\delta=(\theta_L-\theta_R)/10$.
Let $\lambda^\star$ be its midpoint, set $\g=(1/2,1/2)$, and define
$\r_0,\r_2$ by \eqref{eqn:r0-r2-def} but with these values of $\theta_L,\theta_R$ and $\gamma$ replaced by $\delta$. Then
$\g$ is the target minimizer and both neighbors have target excess $\delta$, by the same calculation as \cref{eqn:scalarized-risks-grs}:
\begin{equation*}
    \Slambda{\lambda^\star}(\r_0)=\Slambda{\lambda^\star}(\r_2)= \Slambda{\lambda^\star}(\g)+\delta .
\end{equation*}
In this case we do not add the bad cluster to the Pareto frontier.

\emph{Lower bound for LMSA.}
Since $\Lambda$ does not intersect $[\theta_R,\theta_L]$, the good vector $\g$ is not in
$\cF_\Lambda(\bR)$. Every selected candidate is $\r_0$ or $\r_2$. We can, again, realize the
sources by deterministic loss coordinates $L_{\r}=r_k$ under source $k$, and
let the target again draw source 1 with probability $\lambda^\star$ and source
2 otherwise. Then $\cR_Q(\r)=\Slambda{\lambda^\star}(\r)$, and deterministically,
\begin{equation*}
    \excessRisk{\fLMSA}{\cF}\ge \delta\geq \frac{1}{10} \min\crl{\tfrac{h}{2},\tfrac{1}{40}} \geq \frac{1}{800 M} \geq \frac{1}{800}\gamma = \frac{a_0}{800} \sqrt{\frac{\log n}{n}}.
\end{equation*}

\emph{Upper bound for Pareto ERM.} The coarse instances already have only three vectors so the bound follows immediately from the same calculation as above. 

Take $\cC^{\operatorname{mix}}_n$ to be the union over fine and coarse instances constructed above.
Then, combining the fine and coarse cases show that
\begin{equation*}
    \inf_{\Lambda\in\fG}\sup_{(\cF,\bR,\cR_Q)\in \cC^{\operatorname{mix}}_n}\EE_Q\brk{\excessRisk{\fLMSA}{\cF}}
    \ge
    c\sqrt{\frac{\log n}{n}}  \quad \text{while} \quad \sup_{(\cF,\bR,\cR_Q)\in \cC^{\operatorname{mix}}_n}\EE_Q\brk{\excessRisk{\Paretoerm{1/n}}{\cF}}
    \le\frac{C}{\sqrt n}.
\end{equation*}
That concludes the proof.

\section{Additional Learning Results}

Here we provide an additional result for \cref{sec:transfer-learning} on transfer learning when the loss is strongly convex in \cref{subsec:aggregation-for-transfer}, examples and a general bound for monotonicity above a threshold in \cref{app_subsec: monotonicity above a threshold}, and a learning bound for model selection aggregation with general loss functions and examples for \cref{thm:fast-rate-MSA} in \cref{subsec:additional-model-selection}.

\subsection{Fast Rates in Transfer Learning}
\label{subsec:aggregation-for-transfer}

While for general loss functions running ERM on the Pareto covering is sufficient, for strongly convex losses we would like to achieve fast rates instead. In this section, we assume that the loss $\ell(\yhat,y)$ is $\kappa$-strongly convex in $\yhat$ and $L$-Lipschitz on $[-1,1]$.
Luckily, we can achieve those fast rates by running an aggregation procedure on a $t$-Pareto set. In particular, we consider the algorithm that for some fixed $t>0$
\begin{enumerate}[label=(\roman*)]
    \item builds a minimal $t$-Pareto set $\crl{f_1,\ldots,f_N}$ and
    \item computes the \emph{star estimator} \citep{audibert2007progressive,audibert2009fast} on the covering, that is,
\begin{equation*}
    \begin{aligned}
    \ParetoStar{t}&\in\argmin_{f\in \starhull(\Paretoerm{t},\crl{f_k}_{k=1}^N)} \Rhat_Q(f) \\ 
    &\text{where} \quad \starhull(\Paretoerm{t},\crl{f_k}_{k=1}^N) = \crl{\alpha\Paretoerm{t}+(1-\alpha) f_i: i\in [N], \alpha\in[0,1]},
    \end{aligned}
\end{equation*}
\end{enumerate}
and recall that $\Paretoerm{t}$ is the Pareto ERM on the same $t$-Pareto set.

\begin{proposition}
\label{prop:geometry-fast-rates}
    There exists a constant $C>0$, so that for every $t>0$,
    with probability at least $1-\delta$, the Pareto star-aggregator $\ParetoStar{t}$ achieves
    \begin{equation*}
        \cR_Q(\ParetoStar{t})-\inf_{f\in\cF}\cR_Q(f) \leq \uppermodulus(t) + \frac{CL^2}{\kappa}\frac{\log (\Npar(t,\cF,\bR)/\delta)}{n}.
    \end{equation*}
\end{proposition}

\begin{proof}
    Let $N=\Npar(t,\cF,\bR)$ and $\crl{f_1,\ldots, f_N}$ be the minimal $t$-Pareto set in $\cF$ on which we compute the Pareto star aggregator $\ParetoStar{t}$. By definition, for every $f\in \cF$ there is a $f_j$ with $\cR_k(f_j)\leq \cR_k(f)+t$ for all $k$. By \cref{def:modulus-of-monotonicity}, we know that hence $\cR_Q(f_j)\leq \cR_Q(f) +\uppermodulus(t)$. 

    Let $\fstar\in\argmin_{f\in\cF}\cR_Q(f)$.
    By applying the the main result of \citet{kanade2024exponential} to the star estimator (see their Appendix A.1), there exists a universal constant $C>0$ so that with probability at least $1-\delta$
    \begin{equation*}
        \cR_Q(\ParetoStar{t})-\min_{j\in[N]}\cR_Q(f_j) \leq \frac{C L^2}{\kappa} \frac{\log(N/\delta)}{n}=:r.
    \end{equation*}
    Hence, the same error decomposition as in the proof of \cref{thm:Geometry-complexity-upper-bound} yields that
    \begin{align*}
        &\cR_Q(\ParetoStar{t})-\cR_Q(\fstar)= \underbrace{\cR_Q(\ParetoStar{t})-\min_{j\in[N]} \cR_Q(f_j)}_{\leq r} + \underbrace{\min_{j\in[N]} \cR_Q(f_j)-\cR_Q(\fstar)}_{\leq \uppermodulus(t)},
    \end{align*}
    and the bound of \cref{prop:geometry-fast-rates} follows.
\end{proof}

\subsection{More on the Margin ERM}
\label{app_subsec: monotonicity above a threshold}

Recall that in \cref{subsec:Pareto-ERM}, we derived upper bounds that are consistent whenever $\uppermodulus(t)\to 0$ as $t\to 0$ and hence monotonicity holds.
However, the latter need not hold; instead the upper and lower moduli may stay bounded away from zero near the origin.
Recall from \cref{sec: beyond exact monotonicity}, that we defined $\tlower := \inf\crl{t\in[0,1]: \lowermodulus(t)\geq 0}$.
We now demonstrate this in an example where the lower modulus and the threshold $\smash{\tlower}$ have a closed form solution.
 We then prove a general bound on the excess risk of the margin ERM when $\smash{\gamma>\tlower}$.

\begin{minipage}{0.48\textwidth}
\begin{example}[Monotonicity above a threshold]
\label{exm:monotone-above-threshold-linear}
    Consider the hypothesis space $\cF=\crl{f_\theta\equiv \theta\in[0,1]^2: \theta_1+\theta_2\geq 1}$, the source risks $\cR_1(f_\theta)=\theta_1, \cR_2(f_\theta)=\theta_2$ (e.g., expected absolute loss) and, for $\gamma_0\in(0,1/2)$, the target risk $\cR_Q(f_\theta)=\abs{\theta_1+\theta_2-(1+\gamma_0)}$. Then a calculation shows that for all $t\geq 0$ (cf. \cref{fig:pruned-threshold-set-linear})
    \begin{align*}
        \lowermodulus(t) &= 
        \begin{cases}
            -\gamma_0 & t\leq \gamma_0/2, \\
            2(t-\gamma_0) & \gamma_0/2 < t \leq 1/2, \\
            1 & t>1/2.
        \end{cases}
    \end{align*}
    so that $\tlower = \gamma_0$ and the pruned sets are $\Fpruned{\gamma} = \crl{\theta:1\leq \theta_1+\theta_2\leq 1+2\gamma }$.
\end{example}
\end{minipage}
\hfill
\begin{minipage}{0.5\textwidth}
    \centering
    \captionsetup{hypcap=false}
    \includegraphics[width=\linewidth]{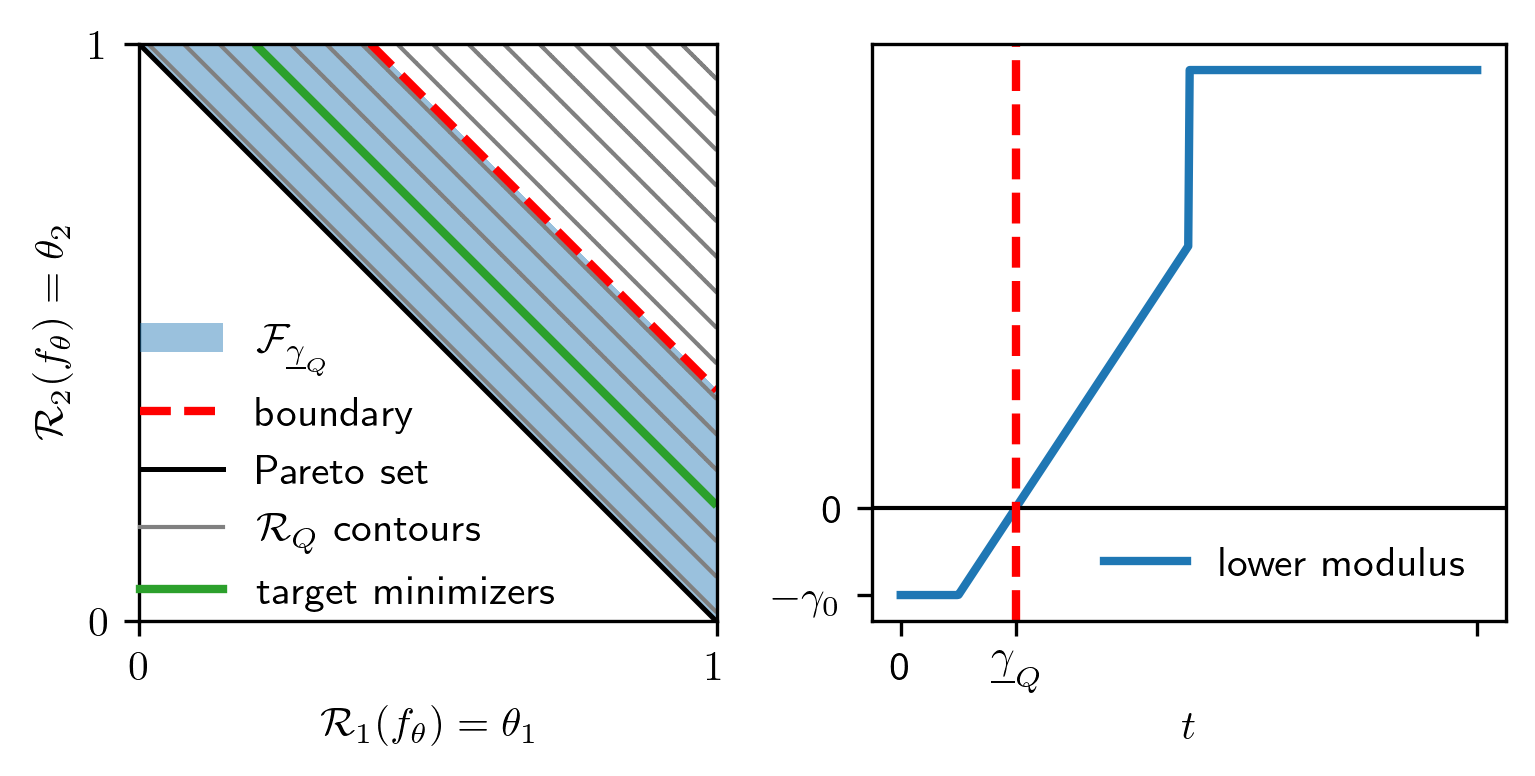}
    \captionof{figure}{\emph{Left:} the objective space of \cref{exm:monotone-above-threshold-linear} with the modulus-pruned set $\Fpruned{\gamma_0}$, the Pareto front and contour lines of the target risk. \emph{Right:} the lower modulus as a function of $t$ in this example.}
    \label{fig:pruned-threshold-set-linear}
\end{minipage}

In \cref{exm:monotone-above-threshold-linear}, monotonicity is not satisfied because $\lowermodulus(0)<0$ and hence also $\uppermodulus(t)$ stays bounded away from zero. See \cref{fig:pruned-threshold-set-linear} for a visualization, and \cref{fig:modulus HELM and VHELM} for the moduli of HELM and VHELM.

Notice that in \cref{exm:monotone-above-threshold-linear} the threshold satisfies $\smash{\tlower = \gamma_0 <\infty}$. 
Hence, any functions $f,f'$ for which $\bR(f)\llcurly \bR(f')-\gamma_0\one$ satisfy $\cR_Q(f) \leq \cR_Q(f')$, and hence we may prune $f'$ from the candidate set, resulting in $\smash{\Fpruned{\tlower}}$, visualized in \cref{fig:pruned-threshold-set-linear}. Of course, in \cref{exm:monotone-above-threshold-linear}, if we knew $\tlower$, we could improve upon this approach; here it just serves as an easy example. 

Also recall from \cref{sec: beyond exact monotonicity} that for $\gamma>0$ the pruned hypothesis space is defined as
\begin{equation*}
    \Fpruned{\gamma} =  \crl{f\in\cF: \text{there exists no } f'\in\cF \text{ such that } \bR(f')\llcurly \bR(f)-\gamma\one},
\end{equation*}
and the margin ERM $\Thresholderm{\gamma}$ is defined by running ERM on $\Fpruned{\gamma}$; that is, $\Thresholderm{\gamma}\in\argmin_{f\in \Fpruned{\gamma}} \Rhat_Q(f)$. We obtain the following result, which is also used in the proof of the parameter-free strategy \cref{prop:monotonicity-above-threshold-adaptive}.
\begin{proposition}
\label{prop:monotonicity-above-threshold}
    For any $\gamma>\tlower$, with probability at least $1-\delta$, the margin ERM $\Thresholderm{\gamma}$ achieves
    \begin{align*}
        &\cR_Q(\Thresholderm{\gamma})-\inf_{f\in\cF}\cR_Q(f) \leq 4\radComp(\ell\circ \Fpruned{\gamma})+2\sqrt{\frac{2\log(4/\delta)}{n}},
    \end{align*}
    where $\radComp (\ell\circ \Fpruned{\gamma})$ denotes the standard Rademacher complexity of the pruned hypothesis class composed with the loss
    $$\radComp(\ell\circ \cF) = \EE_{z_1^n\sim Q^{\otimes n}}\EE_{\bsigma}\sup_{f\in\cF}\abs{\frac{1}{n} \sum_{i=1}^n \sigma_i \ell(f,z_i)}$$ with i.i.d.\ Rademacher random variables $\bsigma=(\sigma_1,\ldots,\sigma_n)$.
\end{proposition}

\begin{proof}
Let $\fstar\in \argmin_{f\in \cF}\cR_Q(f)$. Then
\begin{align*}
    \cR_Q(\Thresholderm{\gamma})-\cR_Q(\fstar) = \cR_Q(\Thresholderm{\gamma})-\inf_{f\in\Fpruned{\gamma}}\cR_Q(f) + \inf_{f\in\Fpruned{\gamma}}\cR_Q(f)-\cR_Q(\fstar).
\end{align*}
Recalling that $\Thresholderm{\gamma}$ is the empirical risk minimizer on $\Fpruned{\gamma}$, standard ERM analysis for bounded losses \citep[Theorem 26.5]{shalev2014understanding} shows that with probability at least $1-\delta$
\begin{equation*}
    \cR_Q(\Thresholderm{\gamma})-\inf_{f\in\Fpruned{\gamma}}\cR_Q(f) \leq 4\radComp(\ell\circ \Fpruned{\gamma})+2\sqrt{\frac{2\log(4/\delta)}{n}}.
\end{equation*}

We argue now that $\inf_{f\in\Fpruned{\gamma}}\cR_Q(f)-\cR_Q(\fstar)=0$, in other words, that a minimizer of the target risk lies in $\Fpruned{\gamma}$. First, recall that 
\begin{equation*}
    \tlower := \inf\crl{t\geq 0: \lowermodulus(t)\geq 0},
\end{equation*}
and $t\mapsto\lowermodulus(t)$ is a non-decreasing function.
Hence, for $\gamma>\tlower$,
we know that any functions $f,f'$ with $\bR(f)\preceq \bR(f')-\gamma\one$ satisfy $\cR_Q(f) \leq \cR_Q(f')$ by \cref{def:modulus-of-monotonicity} of the lower modulus. Recalling that $\Fpruned{\gamma}$ is defined as
\begin{equation*}
    \Fpruned{\gamma} =  \crl{f\in\cF: \text{there exists no } f'\in\cF \text{ such that } \bR(f')\llcurly \bR(f)-\gamma\one}
\end{equation*}
we can consider two cases: if $\fstar\in \Fpruned{\gamma}$, then we are done. If $\fstar$ is not in $\Fpruned{\gamma}$, then there is another function $f\in \cF$ so that $\bR(f)\llcurly \bR(\fstar)-\gamma\one$, and hence $\cR_Q(f) \leq \cR_Q(\fstar)$, so $f$ must also be a minimizer of $\cR_Q$, call that minimizer $\fstar_2=f$. By induction, we obtain a chain $\fstar_i$ of minimizers, each satisfying $\bR(\fstar_i)\llcurly \bR(\fstar_{i-1})-\gamma\one$. Since $\gamma>\tlower\geq 0$, and $\bR$ is bounded from below, this chain cannot be infinite, which means there exists an $i$ so that $\fstar_i\in\Fpruned{\gamma}$, yielding the claim.
\end{proof}

\subsection{Model Selection Aggregation}
\label{subsec:additional-model-selection}
In this section, we further discuss the model selection aggregation problem from \cref{sec: aggregation} and provide some additional results and examples.

First, note that under monotonicity and a convex loss, it holds by definition for any estimator $\rhohat$ that $\cR_k(\psi_{\rhohat})\leq \cR_k(f_{\rhohat})-\Psi(\rhohat)$, which then implies the following gap between the risks of $\psi_{\rhohat}$ and $f_{\rhohat}$:
\begin{equation}
\label{eqn:gap-convex-Pareto}
\begin{aligned}
    \cR_Q(\psi_{\rhohat}) \leq \cR_Q(f_{\rhohat})- \lowermodulus(\Psi(\rhohat)) \leq \EE_{k\sim \rhohat}\cR_Q(f_k)-\lowermodulus(\Psi(\rhohat)).
\end{aligned}
\end{equation}
The first inequality holds by definition of $\lowermodulus$ (\cref{def:modulus-of-monotonicity}), and the second is Jensen's inequality. If $\cF$ contains the convex hull $\conv(f_1,\ldots,f_K)$, it holds that $\lowermodulus(\Psi(\rho))\geq 0$ for all $\rho$ under monotonicity.

\subsubsection{Tilted Exponential Weights}
\label{subsubsec:tilted-exponential-weights}
In this section, we discuss a \emph{tilted exponential weights} estimator, defined for $\lambda, J \geq 0,\pi\in\cP([K])$ as
\begin{equation*}
    \rhotilt := \argmin_{\rho\in \cP([K])} \crl{\EE_{k\sim \rho} \Rhat_Q(f_k) +\frac{\KL(\rho,\pi)}{\lambda}-J\Psi(\rho)},
\end{equation*}
where we recall the definition of $\Psi$ from \cref{eqn:aggregation-map-Pareto-set}.
Note here that for $J=0$ we recover the regular exponential weights estimator \cite{alquier2024user}.
We call it a \emph{tilted} exponential weights posterior, because the function $\Psi$ tilts the objective towards the center of the simplex, depending on where the gap between the Pareto front and the convex hull, as measured by $\Psi$, is largest.

To state the bound for this estimator, we use the following \emph{maximal additive gain} $G_J(Q)$, defined as
\begin{equation*}
    G_J(Q):= \sup_{\rho\in\cP([K])}\crl{J\Psi(\rho)-\prn{\EE_{k\sim \rho}\cR_Q(f_k)-\min_{k\in[K]}\cR_Q(f_k)}}.
\end{equation*}
Under monotonicity \eqref{eqn:weak-monotone}, $G_J(Q)\in[0,J]$: $G_J(Q)\geq 0$ follows from choosing $\rho$ as a point mass on the best function in the dictionary and noting that then $\Psi(\rho)=0$, and $G_J(Q)=0$ can be attained, e.g., when one of the dictionary elements is optimal in $\cF$, that is, $f_k\in\argmin_{f\in\cF}\cR_Q(f)$. On the other hand $G_J(Q)\leq J$ follows from boundedness of the loss.
\begin{proposition} \label{lem:tilted-EW-additive}
Let the loss $\ell(f,z)$ be bounded in $[0,1]$ and convex in $f$, and assume that $\lowermodulus(t)\geq J t$ for all $t\geq 0$, where $J$ is known to the algorithm, in particular implying \eqref{eqn:weak-monotone}. Choose the prior $\pi$ as uniform and $\lambda=\sqrt{8n \log (2K/\delta)}$.
    Then it holds with probability at least $1-\delta$ that
    \begin{equation*}
        \cR_Q(\psi_{\rhotilt})-\min_{k\in[K]} \cR_Q(f_k) \leq \sqrt{\frac{2\log (2K/\delta) }{n}}-G_J(Q).
    \end{equation*}
\end{proposition}

The first term is the minimax rate $\sqrt{\log(K)/n}$ from the standard aggregation setting with convex loss.
However, the bound can be negative for large $G_J(Q)$ (respectively, large $n$), because we are in an improper setting. But its clear that when one of the functions in the dictionary is optimal in all of $\cF$, then $G_J(Q)$ will be zero and there is no gain. 

\begin{proof}
From the assumption that $\lowermodulus(t)\geq J t$ and \cref{eqn:gap-convex-Pareto}, we know that $\cR_Q(\psi_\rho)\leq \EE_{k\sim \rho}\cR_Q(f_k)-J\Psi(\rho)$.
Combining this with Theorem 2.1 in \citet{alquier2024user} yields that
\begin{equation*}
    \PP\prn{\forall\rho\in\cP([K]):\quad \cR_Q(\psi_\rho) \leq \EE_{k\sim \rho} \Rhat_Q(f_k) + \frac{\lambda}{8n}+\frac{\KL(\rho,\pi)+\log(1/\delta)}{\lambda}-J\Psi(\rho)}\geq 1-\delta.
\end{equation*}
Minimizing the bound over $\rho$, we get the tilted posterior; that is,
\begin{equation*}
    \PP\prn{\cR_Q(\psi_{\rhotilt}) \leq \inf_{\rho\in\cP([K])} \crl{ \EE_{k\sim \rho} \Rhat_Q(f_k) + \frac{\lambda}{8n}+\frac{\KL(\rho,\pi)+\log(1/\delta)}{\lambda}-J\Psi(\rho)}}\geq 1-\delta.
\end{equation*}
To turn this into an oracle inequality, we can use the same trick as \citet[Theorem 4.2]{alquier2024user}, that is, plugging in their Equation (4.2) and using a union bound yields that
\begin{equation*}
    \PP\prn{\cR_Q(\psi_{\rhotilt}) \leq \inf_{\rho\in\cP([K])} \crl{ \EE_{k\sim \rho} \cR_Q(f_k) + \frac{\lambda}{4n}+2\frac{\KL(\rho,\pi)+\log(2/\delta)}{\lambda}-J\Psi(\rho)}}\geq 1-\delta.
\end{equation*}
Now note that for a uniform prior $\pi$ on $[K]$ we have that $\KL(\rho,\pi) = \sum_{k=1}^K \rho_k\log\prn{K\rho_k}\leq \log K$
for all $\rho\in\triangle^{K-1}$. Plugging this and the choice of $\lambda$ into the previous bound, we get that with probability $1-\delta$,
\begin{equation*}
    \cR_Q(\psi_{\rhotilt}) \leq \inf_{\rho\in \cP([K])} \crl{\EE_{k\sim \rho} \cR_Q(f_k) +\sqrt{\frac{2\log(2K/\delta)}{n}}-J\Psi(\rho)}.
\end{equation*}
Subtracting $\min_{k\in[K]}\cR_Q(f_k)$ on both sides and plugging in the definition of $G_J(Q)$ yields the result.
\end{proof}

\subsubsection{Examples for Fast Rates}
\label{subsec:examples-aggregation}

\begin{wrapfigure}{r}{0.22\textwidth}
    \vspace{-0.4cm}
    \includegraphics[width=\linewidth]{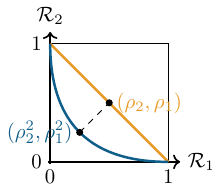}
    \caption{Construction of \cref{ex:fast-rate}.}
    \label{fig:fast-rate-example}
    \vspace{-0.4cm}
\end{wrapfigure}

We describe two examples of where \cref{asm:fast-rate-aggregation} holds: in the first one we may use $(\Phi,\phi)=(\Psi,\psi)$ from \cref{eqn:aggregation-map-Pareto-set}, whereas in the latter we have to construct it separately.

In the first example, any estimator whose range lies in the convex hull will have excess risk lower bounded by $1/\sqrt{n}$ in the worst case, while our estimator achieves the fast rate of order $1/n$. The same construction also serves as the basis for the construction that appears in the proof of the lower bound in \cref{thm:lower-bound-convex-hull}, essentially by stacking multiple copies of this problem when $K>2$.

\begin{example}
\label{ex:fast-rate}
    We begin by constructing a tuple $(\bR,\cF,\crl{\cR_Q:Q\in \cQ})$ describing the problem; we then verify the assumptions of \cref{thm:fast-rate-MSA}. 
    Let $K=2$ and $\cX=\crl{x_1,x_2}$. We define two distributions $P_1,P_2$: Under $P_1$ let $X= x_1$ and $Y= 0$ almost surely, whereas under $P_2$ we let $X=x_2$ and $Y= 1$ almost surely. Using absolute loss, we then find that 
    $\cR_1(f)=\EE_{(X,Y)\sim P_1}\abs{f(X)-Y}=f(x_1)$ and $\cR_2(f)=\EE_{(X,Y)\sim P_2}\abs{f(X)-Y}=1-f(x_2)$
    for any function $f:\cX\to [0,1]$. Take the dictionary $\crl{f_1\equiv 0,f_2\equiv 1}$ so that the convex combination always outputs $f_\rho = \rho_1f_1+\rho_2f_2 =\rho_2$ and the benchmark values of the convex combination are $\bR(f_\rho)=(\rho_2,\rho_1)$. Now let the hypothesis space be
    $$\cF=\conv(\crl{f_1,f_2})\cup \crl{g_u:u\in[0,1]}$$ 
    where we define the functions $g_u$ as $g_u(x_1)=u^2$ and $g_u(x_2)=2u-u^2$, so that $\bR(g_u)=(u^2,(1-u)^2)$; see \cref{fig:fast-rate-example}. Let the target risk be $\cR_Q=q_1\cR_1+q_2\cR_2$ for $\q\in\triangle^{1}$ which we may obtain by also using absolute loss and target distribution $Q\in\cQ=\crl{q_1P_1 +q_2P_2:\q\in \triangle^1}$. 
    Then, on the convex hull of the dictionary, we have that $\cR_Q(f_\rho)=q_1\rho_2 +q_2\rho_1$. 

    Due to using absolute loss, standard aggregation methods are not guaranteed to achieve a fast rate in this setting (and in fact, \cref{thm:lower-bound-convex-hull} uses the same idea of this setting to prove this formally). Our estimator, on the other hand, achieves a fast rate.

    To that end, we now verify the assumptions of \cref{thm:fast-rate-MSA}. We begin with \cref{asm:fast-rate-aggregation} by showing that $\Psi$ from \cref{eqn:aggregation-map-Pareto-set} itself is strongly concave and hence we may choose $(\Phi,\phi)=(\Psi,\psi)$.
    Let us compute the map $\rho\mapsto \psi_\rho$ parameterizing the Pareto set.
    It is not hard to see that the minimizer $\psi_\rho$ is given by $\psi_\rho = g_{\rho_2}$ and $\Psi(\rho) = \rho_1\rho_2$.
    Indeed, we have that $\cR_1(g_u)-\cR_1(f_\rho) = u^2-\rho_2$ and $\cR_2(g_u)-\cR_2(f_\rho)=(1-u)^2-\rho_1$. And for $u=\rho_2$, by the fact that $\rho_1+\rho_2=1$, these two terms are balanced and equal to $\rho_2^2-\rho_2=-\rho_1\rho_2$. Hence, we also obtain that $\bR(\psi_\rho)=(\rho_2^2,\rho_1^2)$.
    A calculation now verifies that $\Psi(\rho)=\rho_1\rho_2$ is $2$-strongly concave in $\nrm{f_\rho-f_\gamma}_2^2 = (\rho_2-\gamma_2)^2$ and $\bR(\psi_\rho)=\bR(f_\rho)-\Psi(\rho)\one$, and so the pair $(\psi,\Psi)$ satisfies \cref{asm:fast-rate-aggregation}. 
    Further, we can verify the other assumptions of \cref{thm:fast-rate-MSA}: first, notice that in this setting we have $\uppermodulus(t)\leq t$ by \cref{lem:bound-modulus-Lipschitz}, and the absolute loss is $L_\ell=1$-Lipschitz on $[0,1]$. Second, since $\abs{\psi_\rho(x_1)-\psi_\gamma(x_1)}=\abs{g_{\rho_2}(x_1)-g_{\gamma_2}(x_1)}=\abs{\rho_2^2-\gamma_2^2}\leq 2\abs{\rho_2-\gamma_2}=2\abs{f_\rho(x_1)-f_\gamma(x_1)}$, and analogously for $x_2$ we have $\abs{\psi_\rho(x_2)-\psi_\gamma(x_2)} \leq 2\abs{f_\rho(x_2)-f_\gamma(x_2)}$, the last requirement holds with $L_\phi=2$.
    
    Consequently, this is a setting in which we may apply \cref{thm:fast-rate-MSA}. That is, our estimator $\phi_{\rhofast}=\psi_{\rhofast}$ achieves excess risk of order at most $1/n$ for all distributions $Q\in \cQ=\crl{q_1P_1+q_2P_2:\q\in \triangle^{1}}$: For some constant $c>0$ and with probability at least $1-\delta$, we have
    \begin{equation*}
        \excessRisk{\phi_{\rhofast}}{\crl{f_1,f_2}} \leq \frac{c\log(1/\delta)}{ n}.
    \end{equation*}
\end{example}

The next example is in the setting of linear regression with absolute loss. We consider the case of $K=d=2$ for the example to be simpler to follow; it can be extended to any dimension $d$ and number of source tasks $K$. As opposed to \cref{ex:fast-rate}, now the map $\Psi$ itself is not strongly concave. We visualize \cref{exm:lin-reg-aggregation} in \cref{fig:lin-reg-aggregation}.
\begin{example}
\label{exm:lin-reg-aggregation}
Again, we begin by constructing the problem setting.
Let $K= d=2$ and consider the hypothesis space $\cF=\crl{\x\mapsto \inner{\x}{\w}:\w\in B_2^2}$.
Let $P_1, P_2$ be two joint distributions of the random vectors $X \in B_2^d$ and $Y \in [-1,1]$, defined via $X= \Sigma_k^{1/2} U$, where $U\sim \uniform{\SSS^{d-1}}$ and $Y = \inner{X}{\w_k}$ 
with ground truths $\w_1 = \e_1, \w_2=-\e_1$ and shape matrices 
\begin{equation*}
    \Sigma_1 = \frac{\pi^2}{16(1+\sigma)}
    \begin{pmatrix}
        1 & \sigma \\
        \sigma & 1
    \end{pmatrix}
    \qquad \text{and} \qquad \Sigma_2 = \frac{\pi^2}{16(1+\sigma)}
    \begin{pmatrix}
        1 & -\sigma \\
        -\sigma & 1
    \end{pmatrix}
\end{equation*}
where $\sigma \in [0,1)$. Note that the largest eigenvalues of both shape matrices $\Sigma_k$ are bounded by $1$ so that $\abs{Y}=\abs{\inner{X}{\w_k}}\leq 1$ almost surely.
Further, we consider the risk $\cR_k(\w) = \EE_{(X,Y)\sim P_k}\abs{\inner{X}{\w}-Y}$ and notice that $\cR_k(\w) = \tfrac{2}{\pi}\sqrt{ (\w-\w_k)^\top\Sigma_k(\w-\w_k)}$. Therefore, the two benchmark risks are given by
\begin{align*}
    \cR_1(\w) &= c_\sigma\sqrt{(w_1-1)^2+2\sigma(w_1-1)w_2 +w_2^2},\quad
    \cR_2(\w) = c_\sigma\sqrt{(w_1+1)^2-2\sigma(w_1+1)w_2 +w_2^2},
\end{align*}
where $c_\sigma=\frac{1}{2\sqrt{1+\sigma}}$.
The scaling of $\Sigma_k$ ensures that $\cR_k$ maps into $[0,1]$.
We can now write $\bR(\w) = (\cR_1(\w),\cR_2(\w))$.
Define for $\rho\in\triangle^{1}$ the convex combination $\w_\rho = \rho_1 \w_1+\rho_2\w_2$ and notice that
    $
        \bR(\w_\rho) = c_\sigma\prn{1-(\rho_1-\rho_2),1+(\rho_1-\rho_2)}.
    $
    Finally, define the family of distributions $Q_\q= q_1P_1+ q_2P_2$ for $\q\in \triangle^1$. A calculation shows that then
    \begin{equation*}
        \cR_{\q}(\w)\equiv \cR_{Q_\q}(\w) := \EE_{(X,Y)\sim Q_\q} \abs{\inner{X}{\w}-Y} = q_1 \cR_1(\w)+q_2\cR_2(\w).
    \end{equation*}

Now, given this setting, in the following \cref{lem:lin-reg-example-aggregation} we verify that \cref{asm:fast-rate-aggregation} is satisfied, and show that weights $\v_\rho\in B_2^2$ exist that satisfy $\bR(\v_\rho)=\bR(\w_\rho)-\Phi(\rho)\one$.
\begin{lemma}
\label{lem:lin-reg-example-aggregation}
    The function $\Phi:\triangle^{1}\to [0,\infty)$ defined as $\Phi(\rho) = \eta \frac{\pi^2}{8(1+\sigma)} \rho_1\rho_2$
    with the constant 
    $$
    \eta=\frac{8\sqrt{1+\sigma}}{\pi^2}\prn{1-\sqrt{\frac{1-\sigma^2}{1+3\sigma^2}}}\prn{1+\sqrt{\frac{1-\sigma^2}{1+3\sigma^2}}}^{-1}
    $$
    satisfies \cref{asm:fast-rate-aggregation}: (i) $\Phi(\e_k)=0$ for $k\in\crl{1,2}$, (ii) $\Phi(\rho)$ is $\eta$-strongly concave w.r.t. $\nrm{\inner{\cdot}{\w_\rho-\w_\gamma}}_{L^2(Q)}$, that is, for all $\rho,\gamma\in \triangle^{1}$, $\alpha\in[0,1]$,
    \begin{equation*}
        \Phi(\alpha\rho+(1-\alpha)\gamma) \geq \alpha \Phi(\rho) + (1-\alpha) \Phi(\gamma) + \frac{\eta}{2}\alpha(1-\alpha)\frac{\pi^2}{4(1+\sigma)}(\rho_1-\gamma_1)^2,
    \end{equation*}
     and (iii) the weights $\v_\rho \in B_2^2$ that satisfy the equation $\bR(\v_\rho ) = \bR(\w_\rho)-\Phi(\rho)\one$ exist and are given by
    \begin{equation*}
        \v_\rho =\prn{\frac{(\rho_1-\rho_2)(1-\Phi(\rho)/c_\sigma)}{1-\sigma v_2}, v_2}
    \end{equation*}
    where $v_2$ is given through the polynomial equation $P(v_2)=0$ and $P$ is given in \cref{eqn:polynomial}.
\end{lemma}
Note that $\eta>0$ iff $\sigma>0$, which is exactly the case when the Pareto set and the convex hull do not coincide and the set $\crl{\v_\rho:\rho\in \triangle^1}$ ``bends away'' from the convex hull. This is visualized in \cref{fig:lin-reg-aggregation}, where we also show that in this case, we do \emph{not} output models in the Pareto set, because the Pareto front does not exhibit sufficient curvature.\footnote{While \cref{asm:fast-rate-aggregation} is satisfied, the Lipschitz condition of \cref{thm:fast-rate-MSA} is not, so we cannot apply it black box. However, it is not hard to see that with some additional work, fast rates can also be achieved here.}
\begin{figure}
    \centering
    \includegraphics[width=0.9\linewidth]{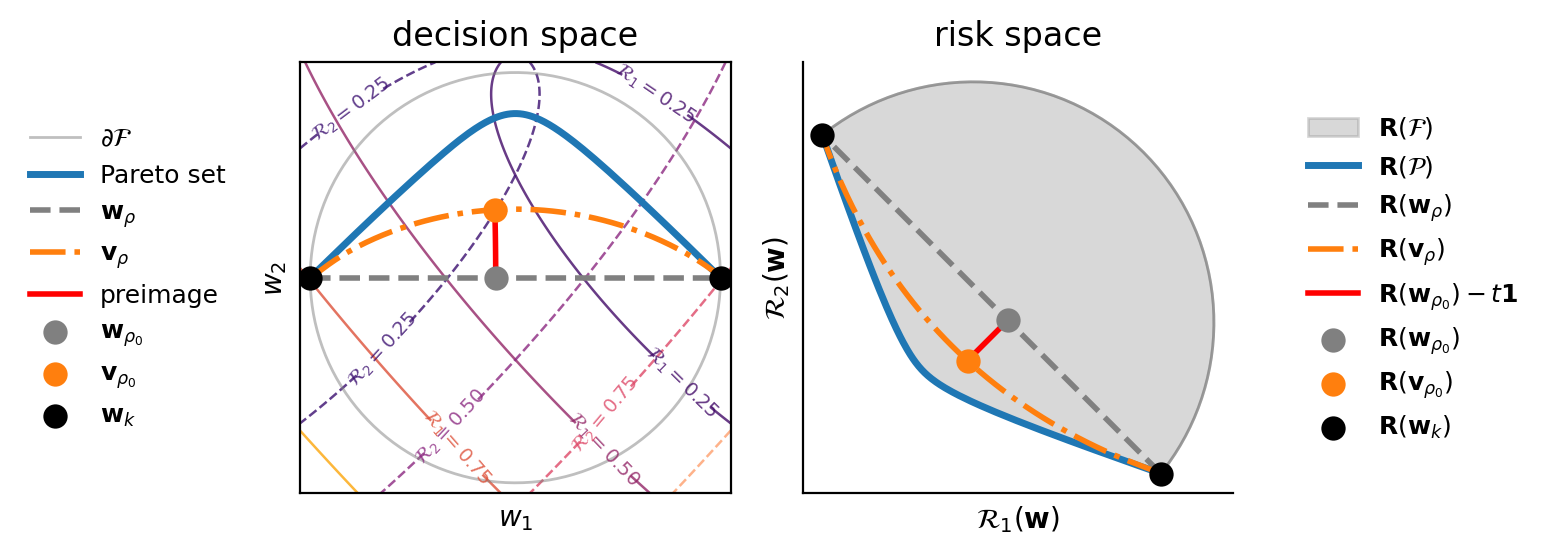}
    \caption{The two-dimensional \cref{exm:lin-reg-aggregation,lem:lin-reg-example-aggregation} visualized. On the left the decision space with contour lines of the two benchmarks, and on the right the risk space.}
    \label{fig:lin-reg-aggregation}
\end{figure}
\end{example}

\begin{proof}[Proof of \cref{lem:lin-reg-example-aggregation}]
    First, denote the risk minimizer by $\w_\q = \argmin_{\w\in B_2^2} q_1 \cR_1(\w)+ q_2 \cR_2(\w)$
    (assume for now that $\w_\q$ is in the interior of $B_2^2$). Moreover, we can write the $L^2(Q)$-norm of $\inner{\cdot}{\w_\rho-\w_\gamma}$, because $\w_\rho-\w_\gamma = (2(\rho_1-\gamma_1),0)$, as
    $
        \nrm{\inner{\cdot}{\w_\rho-\w_\gamma}}_{L^2(Q)}^2=\sum_{k=1}^K q_k\nrm{\w_\rho-\w_\gamma}_{\Sigma_k}^2 = \frac{\pi^2}{4(1+\sigma)}(\rho_1-\gamma_1)^2
    $
    independently of $\q$. 

    The proof of the first point then is obvious.
    The strong concavity of $\Phi$ holds because the function $g(p):= \Phi((p,1-p))=\eta \frac{\pi^2}{8(1+\sigma)} p(1-p)$ has second derivative $g''(p)=-\frac{\pi^2}{4(1+\sigma)}\eta$ and is therefore $\frac{\pi^2}{4(1+\sigma)}\eta$-strongly concave in squared distance on $[0,1]$. Hence $\Phi$ is $\eta$-strongly concave in $\frac{\pi^2}{4(1+\sigma)}(\rho_1-\gamma_1)^2$.

    Finally, to show that $\v_\rho$ exists, fix some $\rho$ and denote $a=\rho_1-\rho_2$, $t=\Phi(\rho)$. The condition that $\bR(\v ) = \bR(\w_\rho)-\Phi(\rho)\one$ is equivalent to the system of two equations
    \begin{equation*}
        \begin{aligned}
            c_\sigma\sqrt{(v_1-1)^2+2\sigma(v_1-1)v_2 +v_2^2}&=c_\sigma(1-a) -t \\
            c_\sigma\sqrt{(v_1+1)^2-2\sigma(v_1+1)v_2 +v_2^2} &= c_\sigma(1+a)-t
        \end{aligned}
    \end{equation*}
    Subtracting one from the other yields $ v_1(1-\sigma v_2)=a(1-t/c_\sigma)$.
    For $v_2\in[-1,1]$ and $\sigma\in[0,1)$ we know $1-\sigma v_2>0$ and so
    \begin{equation*}
        v_1=\frac{a(1-t/c_\sigma)}{1-\sigma v_2}.
    \end{equation*}
    Substituting that back into the first equality from our system of equations yields
    \begin{equation}
    \label{eqn:polynomial}
        P(v_2) = \sigma^3 v_2^4-(\sigma^4+3\sigma^2)v_2^3+(3\sigma^3+3\sigma)v_2^2-(3\sigma^2+1)v_2+\sigma\prn{1-a^2(1-t/c_\sigma)^2}=0.
    \end{equation}
    A sign test with $v_2 = -1$ and $v_2=1$ shows that this quartic has a solution on $(-1,1)$, and with some additional but elementary work, it can be shown that there exists a solution that satisfies $\nrm{\v}_2\leq 1$.
\end{proof}

\cameraonly{
\subsubsection{A Separation from Aggregation in the Convex Hull}
\label{subsubsec:lower-bound-convex-hull}
Finally, we verify that \cref{thm:fast-rate-MSA} can genuinely yield fast rates where standard aggregation methods (in the convex hull) would fail to achieve fast rates.
We prove \cref{thm:lower-bound-convex-hull} in \cref{proof:lower-bound-convex-hull}.
\begin{theorem}
\label{thm:lower-bound-convex-hull}
    Assume $K\geq 4$ and $\log(K)/n \leq 1$.
    There exists a problem setting $(\cF,\bR,\crl{\cR_Q:Q\in\cQ})$ where $\cQ$ is a family of distributions such that (i) any estimator $f_{\rhohat}$ in the convex hull has worst-case risk over $Q\in \cQ$ lower bounded by
    \begin{equation*}
        \sup_{Q\in \cQ} \EE\brk{\excessRisk{f_{\rhohat}}{\crl{f_k}_{k=1}^K}} \geq \frac{1}{38}\sqrt{\frac{\log K}{n}}, 
    \end{equation*}
    and (ii) the assumptions of \cref{thm:fast-rate-MSA} are satisfied with $(\phi,\Phi)$ chosen as $(\psi,\Psi)$ from \cref{eqn:aggregation-map-Pareto-set} and the constants $L_\ell=1,L_\phi=2,\eta=2,\kappa=0$, and $\lambda = 1/22$ for all $Q\in \cQ$. Hence, our estimator achieves
    \begin{equation*}
        \sup_{Q\in \cQ} \EE\brk{\excessRisk{\phi_{\rhofast}}{\crl{f_k}_{k=1}^K}}  \leq 22 \frac{\log K}{n}.
    \end{equation*}
\end{theorem}
The proof idea is based on \cref{ex:fast-rate}, and constructs approximately $\log_2 K$ copies of the problem.
}

\section{Proof of \texorpdfstring{\cref{thm:limiting-distribution}}{Theorem \ref{thm:limiting-distribution}}}\label{proof:limiting-distribution}
Throughout this section, let $\fR$ be the Pareto front associated with $\pareto(\cF, \bR)$ that is \emph{nice} (\cref{assumption: nice}). \cref{thm:limiting-distribution} computes the mass of certain subsets $A \subset \fR$ with respect to the limiting distribution of the Pareto covering as its granularity goes to zero. In particular, these subsets $A$ are assumed to be \emph{Jordan measurable}, which means that their boundaries have zero measure (see \cref{def:Jordan-measurable} below). We begin with a few remarks on these assumptions. 

Niceness ensures that the Pareto front is a smooth manifold; the smoothness structure on the boundary of $\fR$ is handled in the usual way by smoothly extending the manifold within $\RR^K$ \citep{lee2003smooth}. It also lower bounds each component of the normal vector $\n(\r)$ away from zero, where $n_k(\r) > \eta > 0$ for all $k \in [K]$ and $\r \in \fR$. 
To interpret this assumption, suppose that there were an interior point $\r$ in $\fR$ whose normal vector $\n(\r)$ achieves zero on some component $n_k(\r) = 0$, as exemplified in \cref{subfig:inflection}. Such a point is an \emph{improper} trade-off, which intuitively means that it is almost dominated: it is possible to improve on some other objective at the cost of only an infinitesimal amount of utility in the $k$th objective \citep{geoffrion1968proper}. In this sense, the lower bound condition ensures that points in $\fR$ are not `barely' Pareto optimal.

Jordan measurability is a standard assumption that helps us avoid pathological situations, such as when the subset $A \subset \fR$ is a space-filling curve or a countable dense set in $\fR$. In these cases, the Hausdorff measure of $A$ is zero, while its covering numbers at all scales coincide with that of $\fR$; the connection between the Hausdorff measure and limiting distribution necessarily breaks down for such irregular sets.

As a high-level roadmap, the first section \Cref{subsubsec:smooth-machinery} provides some technical machinery for working with smooth Pareto fronts through approximations via linear Pareto fronts. In linear fronts, Pareto balls correspond to simplexes, so the next section \Cref{subsubsec:simplex-coverings} works out the asymptotics of simplex coverings. Finally, \Cref{subsubsec:fine-bound} proves \Cref{thm:limiting-distribution} by passing through a limit of linear approximations of the nice Pareto front. See \cref{fig:proof-limiting-distribution-graph} for an overview. All technical proofs are provided afterwards in \Cref{subsec:technical-proofs}.

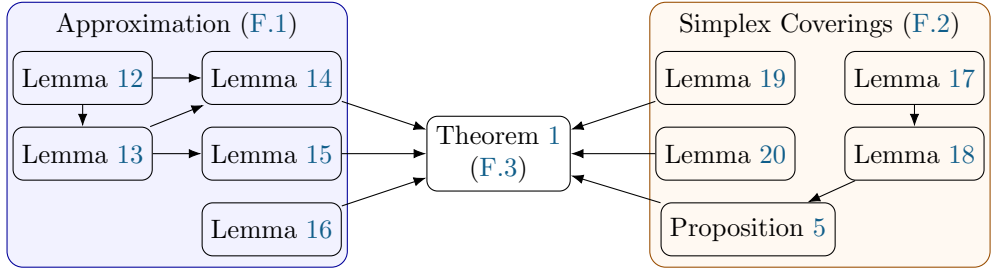
\begin{figure}[H]
    \centering
    \begin{tikzpicture}[
  >=Latex,
  node distance=10mm and 10mm,
  every node/.style={
    draw,
    rounded corners,
    align=center,
    minimum width=18mm,
    minimum height=7mm
  }
]

\node (theorem) at (0,0) {\cref{thm:limiting-distribution} \\ (\ref{subsubsec:fine-bound})};

%approximation
\begin{scope}[shift={(-1.5,-2)}]
\filldraw[
  fill=blue!5,
  draw=blue!60!black,
  rounded corners=6pt
] (-5,4) rectangle (-0.5,0.5);
\node[draw=none, fill=none, inner sep=0pt, outer sep=0pt] (approximation) at (-2.75,3.7){Approximation (\ref{subsubsec:smooth-machinery})};
\node (local) at (-4,3) {\cref{lem:local-linear-approx}};
\node (comparability) at (-4,2) {\cref{lem:comparability}};
\node (piecewise) at (-1.5,3) {\cref{lem:piecewise-approximation}};
\node (distortion) at (-1.5,1) {\cref{lem:distortion-to-covering}};
\node (coarse) at (-1.5,2) {\cref{lem:coarse-bounds}};
\end{scope}

%implications
\draw[->] (local) -- (comparability);
\draw[->] (local) -- (piecewise);
\draw[->] (comparability) -- (piecewise);
\draw[->] (comparability) -- (coarse);
\draw[->] (distortion) -- (theorem);
\draw[->] (piecewise) -- (theorem);
\draw[->] (coarse) -- (theorem);

%simplex coverings
\begin{scope}[shift={(7,-2)}]
\filldraw[
  fill=orange!5,
  draw=orange!60!black,
  rounded corners=6pt
] (-5,4) rectangle (-0.5,0.5);
\node[draw=none, fill=none, inner sep=0pt, outer sep=0pt] (simplexcovering) at (-2.75,3.7){Simplex Coverings (\ref{subsubsec:simplex-coverings})};
\node (minkowski) at (-3.7,1) {\cref{prop:covering-to-minkowski}};
\node (density) at (-4,3) {\cref{lem:covering-density}};
\node (boxes) at (-1.5,2) {\cref{lem:covering-for-boxes}};
\node (parallelepiped) at (-1.5,3) {\cref{lem:centers-in-fundamental-parallelepiped}};
\node (volumesimplex) at (-4,2) {\cref{lem:volume-simplex}};
\end{scope}

\draw[->] (minkowski) -- (theorem);
\draw[->] (density) -- (theorem);
\draw[->] (volumesimplex) -- (theorem);
\draw[->] (boxes) -- (minkowski);
\draw[->] (parallelepiped) -- (boxes);

\end{tikzpicture}
    \caption{Logical proof structure of \cref{thm:limiting-distribution}.}
    \label{fig:proof-limiting-distribution-graph}
\end{figure}

\subsection{Approximations of Nice Pareto Fronts}
\label{subsubsec:smooth-machinery}
Before we prove \Cref{thm:limiting-distribution}, we first provide five basic results about nice Pareto fronts. The proofs of these results are fairly standard: they use of Taylor's theorem and compactness. We further defer them to \Cref{sec:technical-proofs-thm1}. 

These results are about approximating the Pareto front linearly. First, we show that we can approximate the Pareto front \emph{locally} at one point with its tangent space (\cref{lem:local-linear-approx}). We then use this to establish that the Pareto distance and the Euclidean distance are equivalent for any given Pareto front (\cref{lem:comparability}). These two results together let us prove the main approximation, \cref{lem:piecewise-approximation}, that establishes a \emph{global} approximation consisting of a finite number of maps mapping into local tangent spaces. Importantly, these approximations distort the Pareto quasi-metric between the Pareto front and tangent spaces arbitrarily little.
Moreover, \cref{lem:local-linear-approx,lem:comparability} let us prove a first \emph{coarse} bound on the Pareto covering number (\cref{lem:coarse-bounds}). Finally, we prove a basic result about covering number under distortions (\cref{lem:distortion-to-covering}), which is independent of the previous steps but in the main proof will be applied to the approximation maps (which incur little distortion). This approach generally allows us to prove results for linear fronts and extend them to smooth fronts.

We begin with \Cref{lem:local-linear-approx} which shows that a smooth Pareto front $\fR$ can locally be approximated \emph{as a manifold} by a linear subspace, its tangent space, to arbitrary accuracy.
The proof of \cref{lem:local-linear-approx} is in \cref{proof:local-linear-approx}.

\begin{lemma}[Local linear approximation of smooth Pareto fronts]
    \label{lem:local-linear-approx}
    Let $\fR \subset \RR^K$ be a smooth Pareto front. Fix any $\r_0 \in \fR$ and define the tangent space at $\r_0$ as the hyperplane: 
    \[\mathsf{T}_{\r_0}\fR := \big\{\r_0 + \mathbf{v} \in \RR^K: \mathbf{v} \cdot \n(\r_0) = 0\big\}.\]
    For any $\varepsilon > 0$, there are subsets $U$ relatively open in $\fR$ and $V$ relatively open in $\mathsf{T}_{\r_0}\fR$, and a diffeomorphism $\phi : U \to V$ such that for any $1$-Lipschitz-smooth function $f: \RR^K \to \RR$, the following holds for all $\r, \r' \in U$:
    \begin{equation} \label{eqn:bounded-difference}
        \left|\big(f(\r) - f(\r')\big) - \big(f\big(\phi(\r)\big) - f\big(\phi(\r')\big)\big)\right| \leq \varepsilon \|\r - \r'\|_2.
    \end{equation}
\end{lemma}
\begin{figure}[t]
\centering
\begin{subfigure}[b]{0.45\linewidth}
    \centering
    \includegraphics[height=3cm]{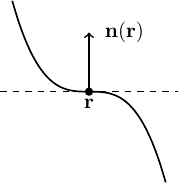}
    \caption{A point excluded by the normal vector condition.}
    \label{subfig:inflection}
\end{subfigure}
\begin{subfigure}[b]{0.45\linewidth}
    \centering
    \includegraphics[height=3cm]{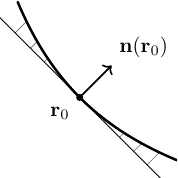}
    \caption{The linear approximation in \cref{lem:local-linear-approx}.}
    \label{subfig:linear-approx}
\end{subfigure}
\caption{(\subref{subfig:inflection}): An example of a point that is excluded by the normal vector condition, since the second coordinate $n_2(\r)=0$ vanishes. (\subref{subfig:linear-approx}): A visualization of the local linear approximation of $\fR$ at $\r_0$ constructed in \Cref{lem:local-linear-approx}. The thick curved line is a smooth Pareto front $\fR$. It is approximated by a tangent space $\mathsf{T}_{\r_0}\fR$, which has normal direction $\n(\r_0)$. The lines connecting $\fR$ with its tangent space represent the local approximation $\phi : U \to V$, where $U \subset \fR$ and $V \subset \mathsf{T}_{\r_0} \fR$.}
\label{fig:linear-approx}
\end{figure}

See \cref{subfig:linear-approx} for this construction.
For asymptotic or non-quantitative results, the following equivalence between Pareto and Euclidean distances is useful and follows from \cref{lem:local-linear-approx}: 
\Cref{lem:comparability} shows that the Pareto distance is bi-Lipschitz with the standard Euclidean distance on $\fR$. In particular, these two distances induce the same topology on the Pareto front.

\begin{lemma}[Comparability of Pareto and Euclidean distances] \label{lem:comparability}
    Let $d$ be the Pareto distance on a nice Pareto front $\fR$. There is a constant $c_\fR > 0$ depending on $\fR$ such that  
    \begin{equation} \label{eqn:metric-equivalence}
        \hspace{2em}c_\fR \cdot \|\r - \r'\|_2 \leq d(\r, \r') \leq  \|\r - \r'\|_2,\qquad \forall \r, \r' \in \fR.
    \end{equation}
\end{lemma}
The proof of \cref{lem:comparability} is in \cref{proof:comparability}.
Combining \cref{lem:local-linear-approx,lem:comparability} we can get the main approximation result \cref{lem:piecewise-approximation}. 
It provides a piecewise linear approximation to the Pareto front. This approximation distorts distances, but it turns out that by increasing the number of pieces, we can incur arbitrarily little distortion. By \emph{distortion}, we mean:

\begin{definition}[$\varepsilon$-distortion]
\label{def:distortion}
    Let $(U, d_U)$ and $(V, d_V)$ be quasi-metric spaces. A bijection $\phi : U \to V$ incurs \emph{$\varepsilon$-distortion} if
    \[\hspace{2em}(1 - \varepsilon) \cdot d_V\big(\phi(\r), \phi(\r')\big) \leq d_U(\r,\r') \leq (1 + \varepsilon) \cdot d_V\big(\phi(\r), \phi(\r')\big), \qquad \forall \r, \r' \in U.\]
    When $\varepsilon = 0$, we say that $\phi$ is an \emph{isometry}.
\end{definition}

In particular, \Cref{lem:piecewise-approximation} shows that a nice Pareto front $(\fR, d)$ can be partitioned into finitely many pieces, where each piece is arbitrarily well-approximated \emph{as a quasi-metric space} by a corresponding linear front. 
\begin{lemma}[Piecewise linear approximation of smooth Pareto fronts]\label{lem:piecewise-approximation}
    Let $\fR \subset \RR^K$ be a nice Pareto front and let $d$ be the Pareto distance. For any $\varepsilon > 0$, there exist a finite collection of disjoint open sets $U_1,\ldots, U_M \subset \fR$ where the diameter of each set $U_m$ is at most $\varepsilon$ and the following hold:
    \begin{itemize}
        \item For each $m \in [M]$, there is a diffeomorphism $\phi_m : (U_m, d) \to (V_m, d)$ that achieves $\varepsilon$-distortion, where  $V_m \subset \RR^{K}$ is contained in some hyperplane, i.e.
        \[V_m \subset \big\{\v \in \RR^K : \n_m \cdot \v = 0\big\},\]
        and each coordinate of $\n_m$ is positive. Moreover, the determinant of the Jacobian is bounded $1 - \varepsilon < |J\phi_m| < 1 + \varepsilon$ on all of $U_m$.
        \item The remainder region $Z = \fR \setminus \bigcup U_m$ is contained in the finite union of the boundaries of a set of balls:
        \[Z \subset \bigcup_{m \in [M]} \partial B_m,\]
        where each $B_m$ is a ball in $\fR$ with respect to the Euclidean distance, and is $\sH^{K-1}$-measure zero. 
    \end{itemize}
\end{lemma}
The proof of \cref{lem:piecewise-approximation} is in \cref{proof:piecewise-approximation}.

As both a useful and illuminating step, we now show that the comparability of the Pareto and Euclidean distances in \cref{lem:comparability} also lets us obtain a coarse upper and lower bound on the Pareto covering number with respect to the usual Euclidean covering number; \Cref{lem:coarse-bounds} shows that the Pareto and Euclidean covering numbers in $\fR$ are comparable.

\begin{lemma}[Coarse bound on Pareto covering numbers] \label{lem:coarse-bounds}
    Let $(\fR, d)$ be a nice Pareto front and its Pareto distance. Further, denote by $N_2(t,A)$ the $t$-covering number of $A$ with respect to the Euclidean distance. There is some $C_\fR > 0$ depending on $\fR$ so that for all $A \subset \fR$, the Pareto and Euclidean covering numbers equivalent up to multiplicative constants:
    \[C_\fR \cdot N_2(t,A)\leq \Npar(t,A) \leq N_2(t,A).\]
    Moreover, let $U \subset \fR$ be a ball in $\fR$ with respect to the Euclidean distance and let $\partial U$ be its boundary. Then, the growth rate of the $t$-Pareto covering number of $\partial U$ is strictly dominated by $t^{K-1}$: 
    \[\lim_{t\to 0}\, \Npar(t, \partial U) \cdot t^{K-1} = 0.\]
\end{lemma}
The proof of \cref{lem:coarse-bounds} is in \cref{proof:coarse-bounds-and-boundary-covering}.
This result follows almost immediately from \Cref{lem:comparability}. The second part is a corollary of the first; we can bound the Pareto covering number of low-dimensional sets in $\fR$. This is easy to do since asymptotically, the Pareto and Euclidean covering numbers qualitatively behave the same. 

To relate the covering numbers on the Pareto front to those on the linear approximation, we need to quantify how much an $\eps$-distortion affects the covering numbers.
The last lemma of this technical preparation allows us to control exactly that; we prove it in \cref{proof:distortion-to-covering}. Importantly, recall that the Pareto distance is a quasi-metric on the Pareto manifold by \cref{lem:pareto-distance-axioms}.

\begin{lemma}[Covering numbers under distortions] \label{lem:distortion-to-covering}
    Let $\phi : (U,d_U) \to (V,d_V)$ be a map between \emph{quasi-metric} spaces incurring $\varepsilon$-distortion. For each $A \subset U$, let $N_U(t,A)$ denote the $t$-covering number of $A$, and let $N_V(t,\phi(A))$ be the $t$-covering number of $\phi(A)$ under their respective quasi-metrics. Then: 
    \[N_V\big(t/(1 - \varepsilon), \phi(A)\big) \leq N_U(t,A) \leq N_V\big(t/(1 + \varepsilon),\phi(A)\big).\]
\end{lemma}
These results let us reduce covering of the smooth Pareto front to covering of its piecewise linear approximation. In the next section we hence focus on how to cover linear fronts.

\cameraonly{\clearpage}
\subsection{Simplex Coverings}\label{subsubsec:simplex-coverings}
\cameraonly{
\begin{wrapfigure}{r}{0.25\textwidth}
    \centering
    \input{figures/simplex_view.tex}
    \caption{The Pareto ball in the tangent space of a Pareto front is a simplex.}
    \label{fig:tangent}
    \vspace{-0.5cm}
\end{wrapfigure}
}
In this section, we bound the covering number for linear Pareto fronts. It turns out that Pareto balls in linear fronts are simplices (\cref{fig:tangent} and \cref{eqn:Pareto-ball-is-simplex} below), so this problem reduces to one of computing \emph{simplex covering numbers}.
We focus on covering subsets $A \subset \RR^\kappa$ of a $\kappa$-dimensional linear space with simplices. In the context of \Cref{thm:limiting-distribution}, $\kappa = K-1$ is the dimension of the linear front.
\begin{definition}[Simplex] \label{def:simplex}
    A subset $\Omega \subset \RR^\kappa$ is a $\kappa$-dimensional \emph{simplex} if it is the convex hull of $\kappa+1$ points in general linear position. That is, there exist $\mathbf{w}_0,\ldots, \mathbf{w}_\kappa \in \RR^\kappa$ such that any collection of $\kappa$ vectors is linearly independent, and
    \[\Omega = \mathrm{conv}(\mathbf{w}_0,\ldots, \mathbf{w}_\kappa).\]
\end{definition}

Throughout, let $\mathrm{Vol}$ denote the $\kappa$-dimensional Lebesgue measure.
And whenever $A, B \subset \RR^\kappa$ are subsets, $v \in \RR^\kappa$ is a vector, and $t > 0$ is a scalar, we use the notation:
\(A + B = \{\mathbf{a} + \mathbf{b}: \mathbf{a} \in A \textrm{ and } \mathbf{b} \in B\}\), \(A + \mathbf{v} = \{\mathbf{a} + \v : \mathbf{a} \in A\}\) and \(tA = \{t\mathbf{a} : \mathbf{a} \in A\}\).

In particular, it is not hard to see that Pareto balls in nice linear Pareto fronts are always simplices: Take the linear front $V=\crl{\v\in\RR^K: \v\cdot \n =0}$ where $n_k>0$ for all $k$. Then the Pareto ball centered at zero can be written as
\begin{equation}
\label{eqn:Pareto-ball-is-simplex}
    \begin{aligned}
    \crl{\v\in V: d(\mathbf{0},\v)\leq t}
    &=t\,\crl{\v\in V: v_k\geq -1 \text{ for all }k\in[K]} \\
    &=t\,\conv\prn{\w^{(1)},\ldots,\w^{(K)}} \quad \text{with} \quad w^{(j)}_k = 
    \begin{cases}
        -1 & j\neq k, \\
        \frac{\one \cdot \n}{n_k}-1 & j=k.
    \end{cases}
    \end{aligned}
\end{equation}
And hence it coincides with a simplex, which we denote $\Omega_\n$. We plot two views of such simplices in \cref{fig:simplex-Pareto}.
\begin{figure}[H]
    \centering
    \includegraphics[height=0.25\linewidth]{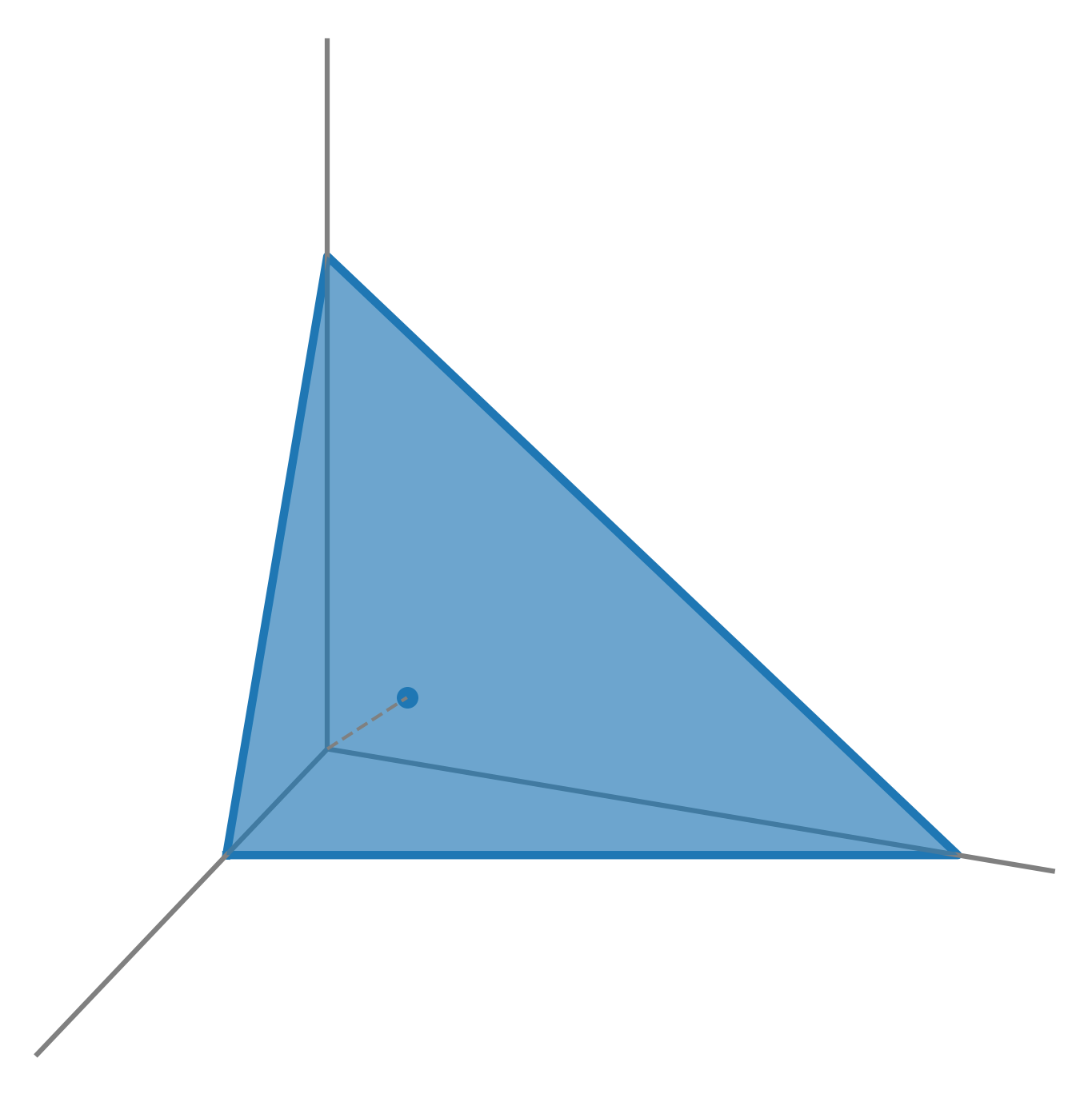} \hspace{2em}
    \includegraphics[height=0.25\linewidth]{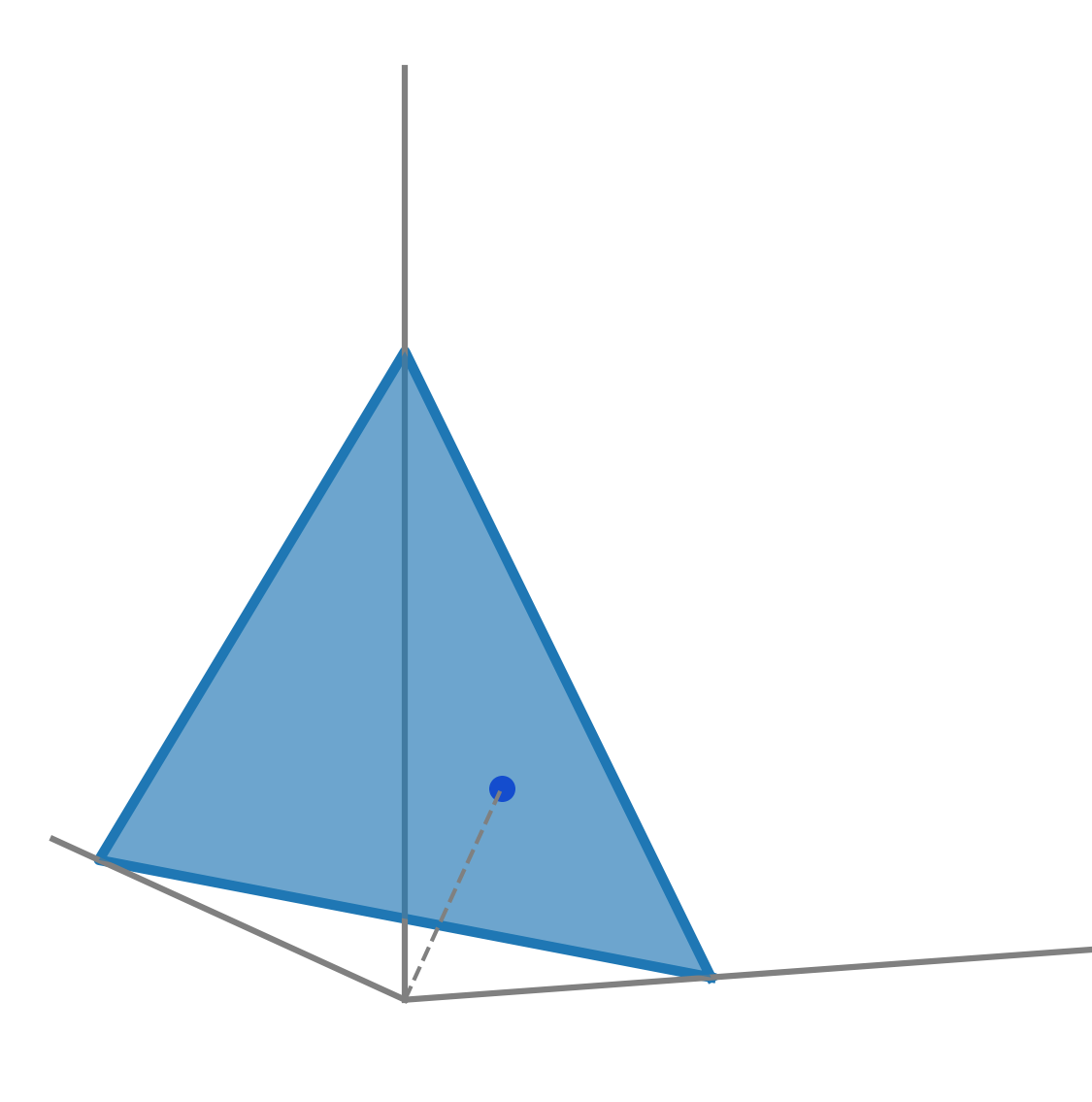}
    \caption{Two different views of a Pareto ball in a linear Pareto front. This is a higher-dimensional analog of \Cref{fig:Pareto-geometry-tikz}, see also \cref{fig:tangent}.}
    \label{fig:simplex-Pareto}
\end{figure}

A standard result in geometric measure theory shows that if $A \subset \RR^\kappa$ is a closed $\kappa$-dimensional rectifiable set, then its Minkowski content, defined as $\mathcal{M}(A) = \limsup_{t \to 0}\, \mathrm{Vol}\big(A + tB\big)$ where $B$ is the unit $\ell_\infty$-ball, coincides with the Hausdorff measure \citep[Theorem 3.2.39]{federer1996geometric}. This tells us that the asymptotic covering number is governed by the Minkowski content. This section first provides an analogous result, \Cref{prop:covering-to-minkowski}, but for the setting where the covering is done with general compact sets $\RR^\kappa$, before specializing to $\kappa$-dimensional simplices. 

\paragraph{A covering limit via periodic translative covering densities.}
We now define efficient or `economical' periodic coverings of $\RR^\kappa$ by a convex set $\Omega$, which we can think of as a repetitive tiling of $\RR^\kappa$ by many translated copies of $\Omega$. We are interested in the \emph{periodic translative covering density} $\theta(\Omega)$, which measures on average how many copies of $\Omega$ cover a random point in space. If $\Omega$ can tessellate the space, then the optimal density is one. But usually, these translations of $\Omega$ will need to overlap in order to cover the space. When $\Omega$ is a convex body, the value $\theta(\Omega)$ exists and is upper bounded by $O(\kappa \ln \kappa)$, as established in \citet{rogers1957note}.

Before we define these periodic coverings, we recall the standard notion of a lattice (capturing the periodicity of the covering). The following is based off of the reference \citet{micciancio2002complexity}.
\begin{definition}[Lattice, \citet{micciancio2002complexity}]
    A (full-rank) lattice $\Lambda \subset \RR^\kappa$ is a set
    \[\Lambda = \big\{z_1 \v_1 + \dotsm + z_\kappa \v_\kappa : z_1,\ldots, z_\kappa \in \ZZ\big\},\]
    where $V = \begin{bmatrix} \v_1& \dotsm & \v_\kappa\end{bmatrix} \in \RR^{\kappa \times \kappa}$ form a linearly independent set called a \emph{basis} of $\Lambda$. The \emph{fundamental parallelepiped} spanned by the basis $V$ is defined as the set
    \[P(V) = \big\{\lambda_1 \v_1 + \dotsm + \lambda_\kappa \v_\kappa : \lambda_1,\ldots, \lambda_\kappa \in [0,1)\big\}.\]
    The \emph{determinant} of the lattice is the $\kappa$-dimensional volume of $P(V)$, given by $\det\Lambda = |\det V|$.
\end{definition}

It is well-established that the determinant of any fixed lattice is independent of the choice of basis \cite{micciancio2002complexity}. We now define periodic coverings and the covering density $\theta(\Omega)$, following the exposition given by \citet{naszodi2018flavors}. 

\begin{definition}[Periodic translative covering density]
\label{def:periodic-covering-density}
    Let $\Omega \subset \RR^\kappa$ be convex, $\Lambda \subset \RR^\kappa$ be a lattice, and let $S \subset \RR^\kappa$ be a finite set. The \emph{periodic arrangement of translates}
    \[\mathcal{C}(\Omega, \Lambda, S) = \big\{\Omega + \v + \mathbf{s}: \v \in \Lambda \textrm{ and } \mathbf{s} \in S\big\}\]
    is an \emph{periodic $\Omega$-covering} of $\RR^\kappa$ if it covers $\RR^\kappa$.
    The \emph{periodic translative covering density} of $\Omega$ is the following infimum taken over all periodic $\Omega$-coverings:
    \[\theta(\Omega) := \inf \big\{\delta(\mathcal{C}) : \mathcal{C} \textrm{ is an periodic $\Omega$-covering of $\RR^\kappa$}\big\},\qquad \textrm{where } \delta(\mathcal{C}) = \frac{|S| \cdot \mathrm{Vol}(\Omega)}{\det \Lambda}, \]
    where the infimum is taken over all lattices $\Lambda$ and finite sets $S$, and $\delta(\mathcal{C})$ is called the \emph{density} of $\mathcal{C}$.
\end{definition}

\begin{definition}[$\Omega$-covering number]
\label{def:Omega-covering-number}
    Let $\Omega \subset \RR^\kappa$ be a compact set whose interior contains $\mathbf{0} \in \mathrm{int}(\Omega)$. Given any $\w \in \RR^\kappa$, we say that $\Omega + \w$ is a translation of $\Omega$ \emph{centered} at $\w$. Let $A \subset \RR^\kappa$ be any set.  The \emph{$\Omega$-covering number} $N(A;\Omega)$ of $A$ is defined as the minimal number of translations of $\Omega$ centered at points in $A$ needed to cover $A$, i.e.,
    \[N(A;\Omega) := \inf_{S \subset A} \, \big\{|S| : A \subset \Omega + S\big\}.\]
\end{definition}

Before we prove the main result of this section (\Cref{prop:covering-to-minkowski}), we provide two lemmata. We leave the proofs to \Cref{sec:technical-proofs-thm1}. 
The first lemma counts the number of centers of a periodic covering $\cC(\Omega, \Lambda, S)$ that lands inside any fundamental parallelepiped of $\Lambda$. This number is at most $|S|$, and in fact, equality holds unless there are distinct points $\mathbf{s}, \mathbf{s}' \in S$ that are perfectly spaced apart so that $\mathbf{s} - \mathbf{s}' \in \Lambda$. The proof is in \cref{proof:centers-in-fundamental-parallelepiped}.

\begin{lemma}[Centers in the fundamental parallelepiped] \label{lem:centers-in-fundamental-parallelepiped}
    Let $\Omega \subset \RR^\kappa$ be a compact set whose interior contains $\mathbf{0} \in \mathrm{int}(\Omega)$. Suppose that $\cC = \cC(\Omega, \Lambda, S)$ is an $\Omega$-periodic covering of $\RR^n$. Let $V$ be any basis of $\Lambda$ and let $P(V)$ be its fundamental parallelepiped. Let $\cC(V) = \{ \v + \mathbf{s} : \v \in \Lambda, \mathbf{s} \in S\} \cap P(V)$ be the set of centers contained in $P(V)$. Then $|\cC(V)| \leq |S|$. 
\end{lemma}

We use \cref{lem:centers-in-fundamental-parallelepiped} to prove the second lemma which is a simple case of the upcoming \Cref{prop:covering-to-minkowski}, when restricted to boxes; as it presents a significant first step, we mention it here explicitly.

\begin{lemma}[$\Omega$-covering number of boxes]
    \label{lem:covering-for-boxes}
    Let $\Omega \subset \RR^\kappa$ be a compact set with $\mathbf{0} \in \mathrm{int}(\Omega)$. Let $A = [0,1]^\kappa$. Then: 
    \[\lim_{t \to 0}\, N(A; t\Omega) \cdot t^\kappa = \frac{\theta(\Omega) \cdot \mathrm{Vol}(A)}{\mathrm{Vol}(\Omega)}.\]
\end{lemma}
The proof of \cref{lem:covering-for-boxes} is in \cref{proof:covering-for-boxes}.
The following result then generalizes \Cref{lem:covering-for-boxes} from boxes to general Jordan measurable sets $A$. To that end, recall the definition of Jordan measurable sets.
\begin{definition}[Jordan measurable set]
\label{def:Jordan-measurable}
    A set $A$ in a metric measure space is \emph{Jordan measurable} if $A$ is bounded and the boundary of $A$ has measure zero.
\end{definition}
Intuitively, we can generalize \cref{lem:covering-for-boxes} to Jordan measurable sets, because the boundary of some set $A$ having measure zero implies that there is a countable disjoint union of boxes $B_i$ contained in $A$ so that $A\setminus \bigcup B_i$ has volume zero. We obtain \cref{prop:covering-to-minkowski}, which (perhaps surprisingly) to the best of our knowledge seems to not have been established in the literature.

\begin{proposition}[$\Omega$-covering number of Jordan measurable sets] \label{prop:covering-to-minkowski}
    Let $\Omega \subset \RR^\kappa$ be compact with $\mathbf{0} \in \mathrm{int}(\Omega)$. Let $A \subset \RR^\kappa$ be a Jordan measurable set, so that it is bounded and has measure zero boundary $\mathrm{Vol}(\partial A)=0$. Then: 
    \[\lim_{t \to 0}\, N(A; t\Omega) \cdot t^\kappa = \frac{\theta(\Omega) \cdot \mathrm{Vol}(A)}{\mathrm{Vol}(\Omega)}.\]
\end{proposition}
The proof makes use of \cref{lem:covering-for-boxes} in the way described above; we defer it to \Cref{proof:covering-to-minkowski}.

\paragraph{Application to simplices.} We now show that the form of the limit in \cref{prop:covering-to-minkowski} is special when we cover with \emph{simplices}.
In particular, it turns out that the covering density $\theta(\Omega)$ of all simplices is equal. It follows from the fact that any two simplices $\Omega$ and $\Omega'$ are affinely isomorphic: the same isomorphism maps periodic $\Omega$-coverings into periodic $\Omega'$-coverings. 

\begin{lemma}[Covering density of simplices is unique] \label{lem:covering-density}
    Let $\Omega, \Omega' \subset \RR^\kappa$ be two $\kappa$-dimensional simplices. Then $\theta(\Omega) = \theta(\Omega')$. 
\end{lemma}
The proof is in \cref{proof:covering-density}.
As the final result of this section, we compute the volume of a simplex corresponding to the unit Pareto ball on a linear Pareto front with normal vector $\n$. Effectively, this is where the form of the density in \cref{thm:limiting-distribution} comes from, as we can plug it into \cref{prop:covering-to-minkowski}. The proof is in \cref{proof:volume-simplex} and is just a calculation.
\begin{lemma}[Volume of the simplex] \label{lem:volume-simplex}
    Let $\n = (n_1,\ldots, n_K)$ be a normal vector where $\|\n\|_2 = 1$ and $n_k > 0$ for all $k \in [K]$. Define, as in \cref{eqn:Pareto-ball-is-simplex}, the set $\Omega_\n := \big\{\v \in \RR^K : \v \cdot \n = 0\big\} \cap \big\{\v \in \RR^K : v_k \geq - 1\big\}$.
    Then, $\Omega_\n$ is a $(K-1)$-dimensional simplex and for $t\geq 0$, the $(K-1)$-dimensional Hausdorff volume of $t \Omega_n$ is 
    \[\sH^{K-1}(t \Omega_\n) = \frac{1}{(K-1)!}\frac{(\mathbf{1} \cdot \n)^{K-1}}{\prod_{k \in [K]} n_k} \cdot t^{K-1}.\]
\end{lemma}

\subsection{Main Proof of \texorpdfstring{\cref{thm:limiting-distribution}}{Theorem \ref{thm:limiting-distribution}}}
\label{subsubsec:fine-bound}

    We are now ready to prove \Cref{thm:limiting-distribution}. Let $A \subset \fR$ be a Jordan measurable set in a nice Pareto front. 

    Let $\theta_K$ be the covering density (\cref{def:periodic-covering-density}) of $(K-1)$-dimensional simplices, which is independent of the choice of simplex by \Cref{lem:covering-density}. As shown in \cite{rogers1957note}, this covering density satisfies $1 \leq \theta_K \leq O(K \log K)$. Choosing $C_K = \theta_K \cdot (K-1)!$ we have 
    \begin{equation} \label{eqn:beta}
        \beta\big(\n\big) = \theta_K \cdot (K-1)!\cdot \frac{\prod_{k \in [K]} n_k}{(\mathbf{1} \cdot \n)^{K-1}}.
    \end{equation}
    For any $\varepsilon > 0$, by \Cref{lem:piecewise-approximation} we can construct a piecewise linear Pareto manifold achieving $\varepsilon$-distortion.
    In particular, there is a decomposition $\{\phi_m : U_m\to V_m : m \in [M]\}$ where $U_m$ has diameter at most $\varepsilon$, $V_m$ is contained in the hyperplane normal to $\n_m = \n(\r_m)$ for some $\r_m \in U_m$, and there is bounded distortion of the volume $1 - \varepsilon < |J \phi_m| < 1 + \varepsilon$ for all $m$. We use the following notation: we denote the remainder region by $Z=\fR\setminus \bigcup U_m$ and let $N_m$ denote the covering number in the linear Pareto front $V_m$ under the Pareto distance. For any normal vector $\n$, we denote by $\Omega_{\n}$ the simplex defined in \Cref{lem:volume-simplex}, by which we also know that 
    \begin{equation} \label{eqn:local-simplex-volume}
        \sH^{K-1}(\Omega_{\n}) = \frac{1}{(K-1)!}\frac{(\mathbf{1} \cdot \n)^{K-1}}{\prod_{k \in [K]} n_k} = \frac{\theta_K}{\beta(\n)}.
    \end{equation}
    
    \paragraph{Upper bound.} We can compute an upper bound on the $t$-Pareto covering number:
    \begin{align*} 
        \Npar(t,A) &\leq \sum_{m \in [M]} \Npar\big(t, A \cap U_m\big) + \Npar(t, Z)
        \\&\leq \sum_{m \in [M]} N_m\big(t/(1 + \varepsilon), \phi_m(A \cap U_m)\big) + \Npar(t,Z),
    \end{align*}
    where the first step follows because the covering number is subadditive, and the second applies \Cref{lem:distortion-to-covering}. Let $\pi_m : V_m \to \mathbb{R}^{K-1}$ be any injective affine map. We can use the simplex covering notation from \Cref{subsubsec:simplex-coverings} and \Cref{prop:covering-to-minkowski} to obtain an upper bound on the (rescaled) first term
    \begin{align*} 
        \limsup_{t \to 0} \, N_m\big(t/(1 + \varepsilon), \phi_m(A\cap U_m)\big) \cdot t^{K-1} &\overset{(i)}{=} \limsup_{t \to 0}\, N\big(\pi_m (\phi_m(A \cap U_m)), (t/(1+\varepsilon)) \pi_m(\Omega_{\n_m})\big) \cdot t^{K-1}
        \\&\overset{(ii)}{=} \frac{\theta_K \cdot \mathrm{Vol}(\pi_m(\phi_m(A \cap U_m)))}{\mathrm{Vol}(\pi_m(\Omega_{\n_m}))} \cdot (1 + \varepsilon)^{K-1}
        \\&\overset{(iii)}{=} \frac{\theta_K \cdot \sH^{K-1}(\phi_m(A\cap U_m))}{\sH^{K-1}(\Omega_{\n_m})} \cdot (1 + \varepsilon)^{K-1}
        \\&\overset{(iv)}{\leq} \frac{\theta_K \cdot \sH^{K-1}(A\cap U_m)}{\sH^{K-1}(\Omega_{\n_m})} \cdot (1 + \varepsilon)^{K}
        \\&\overset{(v)}{=} \beta(\n_m) \cdot \sH^{K-1}(A \cap U_m) \cdot (1 + \varepsilon)^{K}.
    \end{align*}
    Here (i) uses the fact that the Pareto covering number is the same as the simplex covering by the derivations in \cref{subsubsec:simplex-coverings}, (ii) applies \Cref{prop:covering-to-minkowski}, (iii) uses the fact that $\pi_m$ is affine and applies the area formula (Theorem 3.2.3 in \cite{federer1996geometric}). Moreover, (iv) again applies the area formula and the fact that the determinant of the Jacobian is bounded $|J \phi_m| \leq 1 + \varepsilon$, so
    \[\sH^{K-1}\big(\phi_m(A \cap U_m)\big) = \int_{\phi_m(A \cap U_m)}\d \sH^{K-1} = \int_{A \cap U_m} |J\phi_m| \d \sH^{K-1} \leq (1 + \varepsilon) \cdot \sH^{K-1}(A \cap U_m),\]
    and (v) substitutes in \Cref{eqn:local-simplex-volume}.  Further, we can bound the second term using \Cref{lem:coarse-bounds}, so we have the following limit:
    \[\limsup_{t \to 0}\, \Npar(t,Z) \cdot t^{K-1} = 0.\]
    
    Therefore, the $t$-Pareto covering number $\Npar(t,A)$ satisfies:
    \begin{align}
        \limsup_{t \to 0}\, \Npar(t,A)\cdot t^{K-1} &\leq (1 + \varepsilon)^{K} \cdot \sum_{m \in [M]} \beta(\n_m) \cdot \sH^{K-1}(A \cap U_m) \notag
        \\&= (1 + \varepsilon)^{K} \cdot \int_A \sum_{m \in [M]} \beta\big(\n(\r_m)\big) \cdot \mathbf{1}_{ U_m} \, \d \sH^{K-1}. \label{eqn:integrand}
    \end{align}
    The function $(\beta \circ \n)(\r)$ is uniformly continuous, since $(\beta \circ \n)(\r)$ is smooth and $\fR$ is compact. It follows that for every $\delta > 0$, there is a sufficiently small $\varepsilon > 0$ such that whenever $\|\r - \r'\|_2 < \varepsilon$, then $|(\beta \circ \n)(\r) - (\beta \circ \n)(\r')| < \delta$. In particular, since each $U_m$ has diameter at most $\varepsilon$, we have that the integrand in \Cref{eqn:integrand} is bounded above by $\beta(\n(\r)) + \delta$. We obtain
    \[\limsup_{t \to 0} \, \Npar(t,A) \cdot t^{K-1} \leq (1 + \varepsilon)^K \cdot \int_{A \setminus Z} \big(\beta(\n(\r)) + \delta\big)\, \d \sH^{K-1}(\r).\]
    Finally, by sending $\delta \to 0$, which also forces $\varepsilon \to 0$, we deduce the upper bound:
    \[\limsup_{t\to 0} \, \Npar(t,A) \cdot t^{K-1} \leq \int_{A\setminus Z} \beta(\n(\r))\, \d \sH^{K-1}(\r) = \int_{A} \beta(\n(\r))\, \d \sH^{K-1}(\r),\]
    where in the last step we use the fact that $Z$ is $\sH^{K-1}$-measure zero. 
    
    \paragraph{Lower bound.} We now compute a matching lower bound. For each $s > 0$ and $U_m$, define the set of points in $U_m$ that are $s$-close to the boundary of $U_m$ in the following sense:
    \[\partial U_m^s = \big\{\r \in U_m : \exists \r' \in \fR \setminus U_m \textrm{ s.t. } d(\r', \r) < s\big\}.\]
    Notice that $\partial U_m^s$ converges to $\emptyset$ as $s$ goes to zero: for each $\r$, once $s$ becomes sufficiently small, $\r \notin \partial U_m^s$. Fix any $s > 0$ and Jordan measurable $A \subset \fR$. For all $0 < t < s$, we have:
    \[\Npar(t,A) \geq \sum_{m \in [M]} \Npar(t, A \cap U_m) - \Npar(t, \partial U_m^s).\]
    To see this, let $S \subset A$ be any minimal $t$-Pareto covering of $A$ and let $S_m \subset \partial U_m^s$ be a minimal $t$-Pareto covering of $\partial U_m^s$. We claim that the set $(S \cap U_m) \cup S_m$ forms a $t$-Pareto covering of $A \cap U_m$. Let's assume the claim for now. It implies the upper bound
    \[|S \cap U_m| + |S_m| \geq \Npar(t, A \cap U_m).\]
    We deduce the above inequality by taking a summation over $m \in [M]$, and combining with the facts that
    \[\Npar(t,A) = \sum_m |S \cap U_m| + |S \cap Z| \geq \sum_m |S \cap U_m|\qquad \textrm{and}\qquad |S_m| = \Npar(t,\partial U_m^s) .\]
    As for the claim, let $\r \in A \cap U_m$. Then, either it is covered by an element of $S \cap U_m$, or it was covered by an element of $S$ centered at a point $\r' \notin U_m$. In this case, $d(\r', \r) < t < s$, which implies that $\r \in \partial U_m^s$, and it is covered by an element of $S_m$. 

    Applying \Cref{lem:distortion-to-covering} and multiplying through by $t^{K- 1}$, we have:
    \begin{align*}
        \Npar(t,A) \cdot t^{K-1}\geq \sum_{m \in [M]} N_m(t/(1 - \varepsilon), \phi_m(A \cap U_m))\cdot t^{K-1} - N_m(t/(1 + \varepsilon), \phi_m(\partial U_m^s)) \cdot t^{K-1},
    \end{align*}
    where recall that $N_m$ gives the covering number for sets in $V_m$.
    By \Cref{prop:covering-to-minkowski}, the limit of the right-hand side exists as $t \to 0$, and so we obtain that:
    \begin{align*} 
        \liminf_{t \to 0} &\, \Npar(t,A) \cdot t^{K-1} 
        \\&\geq \sum_{m \in [M]} \frac{\theta_K \cdot \sH^{K-1}(\phi_m(A \cap U_m))}{\sH^{K-1}(\Omega_{\n_m})} \cdot (1 - \varepsilon)^{K-1} - \frac{\theta_K \cdot \sH^{K-1}(\phi_m(\partial U_m^s))}{\sH^{K-1}(\Omega_{\n_m})} \cdot (1 + \varepsilon)^{K-1}.
    \end{align*}
    Let $Z^s = \bigcup \partial U_m^s \cup Z$. Once again, relying on the uniform continuity of $(\beta \circ \n)(\r)$, we can apply the same argument used for the upper bound to deduce that for all $\delta > 0$, 
    \begin{align*} 
        \liminf_{t \to 0} &\, \Npar(t,A) \cdot t^{K-1} 
        \\&\geq (1 - \varepsilon)^K \cdot \int_{A \setminus Z} \big(\beta(\n(\r)) - \delta\big)\, \d \sH^{K-1}(\r) - (1 + \varepsilon)^K \cdot \int_{Z^s} \big(\beta(\n(\r)) + \delta\big)\, \d \sH^{K-1}(\r).
    \end{align*}
    This holds for all $s > 0$, so we can also let $s$ go to zero, where $Z^s$ converges to $Z$. Since $\sH^{K-1}(Z) = 0$, not only does the second term vanishes, but for the first term, it is equivalent to taking the integral over $A$ instead of $A\setminus Z$. We obtain: 
    \[\liminf_{t \to 0}\, \Npar(t,A) \cdot t^{K-1} \geq (1 - \varepsilon)^K \cdot \int_{A} \big(\beta(\n(\r)) - \delta\big)\, \d \sH^{K-1}(\r).\]
    Since this holds for all $\delta > 0$, we can let $\delta$ go to zero. This also forces $\varepsilon$ to zero:
    \[\liminf_{t \to 0}\, \Npar(t,A) \cdot t^{K-1} \geq \int_{A} \beta(\n(\r))\,\d\sH^{K-1}(\r).\]
    
    Since the limit infimum and limit supremum coincide, the limit exists and is given by the theorem statement.

\subsection{Technical Proofs} \label{sec:technical-proofs-thm1} \label{subsec:technical-proofs}

\subsubsection{Proof of \texorpdfstring{\Cref{lem:local-linear-approx}}{Lemma \ref{lem:local-linear-approx}}}
\label{proof:local-linear-approx}

    Choose any set of smooth coordinates $\psi: \tilde{U} \to \RR^{K-1}$ centered at $\r_0$, so that $\tilde{U} \subset \fR$ is an open set relative to $\fR$ containing $\r_0$, and $\psi(\r_0) = 0$. By shrinking $\tilde{U}$ if needed, we may assume that $\psi$ is $L$-bi-Lipschitz:
    \[\hspace{2em}\frac{1}{L} \|\r - \r'\|_2 \leq \|\psi(\r) - \psi(\r')\|_2 \leq L \|\r - \r'\|_2,\qquad \forall \r, \r'\in \tilde{U}.\]
    Define the following map $\theta : \mathsf{T}_{\r_0} \fR \to \RR^{K-1}$, which is also $L$-bi-Lipschitz:
    \[\theta(\r_0 + \mathbf{v}) = \nabla \psi(\r_0)\mathbf{v}.\]
    Define $\phi : U \to V$ by $\phi = \theta^{-1} \circ \psi$ where we set $U \subset \tilde{U}$ later and let $V = \phi(U)$. Also see \Cref{fig:linear-approx}.

    We now show that $\phi$ approximately preserves smooth maps. Let $f : \RR^K \to \RR$ be any smooth, 1-Lipschitz map. Let $\r, \r' \in U \subset \tilde{U}$ come from a sufficiently small region around $\r_0$. Taylor's theorem implies that:
    \begin{align*}
        f(\r) - f(\r') &= (f \circ \psi^{-1})(\psi(\r)) - (f \circ \psi^{-1})(\psi(\r'))
        \\&= \nabla (f \circ \psi^{-1})(\psi(\r'))\big(\psi(\r) - \psi(\r')\big) + o\big(\|\psi(\r) - \psi(\r')\|_2\big)
    \end{align*}
    For short, let $\hat{\r} = \phi(\r)$ and $\hat{\r}' = \phi(\r')$. Taylor's theorem also implies that:
    \begin{align*}
        f(\hat{\r}) - f(\hat{\r}') &= (f \circ \theta^{-1})(\psi(\r)) - (f \circ \theta^{-1})(\psi(\r'))
        \\&= \nabla (f \circ \theta^{-1})(\psi(\r')) \big(\psi(\r) - \psi(\r')\big) + o\big(\|\psi(\r) - \psi(\r')\|_2\big)
    \end{align*}
    Let $M_f(\r') = \nabla (f \circ \psi^{-1})(\psi(\r')) - \nabla(f \circ \theta^{-1})(\psi(\r'))$. The term in the absolute value of \Cref{eqn:bounded-difference} is:
    \begin{align*}
        \big(f(\r) - f(\r')\big) - \big(f(\hat{\r}) - f(\hat{\r}')\big) &= M_f(\r') \big(\psi(\r) - \psi(\r')\big) + o\big(\|\psi(\r) - \psi(\r')\|_2\big).
    \end{align*}
    By bi-Lipschitzness, we have that $\|\psi(\r) - \psi(\r')\|_2 = \Theta(\|\r - \r'\|_2)$, and so we have that:
    \[\big|\big(f(\r) - f(\r')\big) - \big(f(\hat{\r}) - f(\hat{\r}')\big)\big| \leq \|M_f(\r')\|_2 \cdot O(\|\r - \r'\|_2).\]
    It suffices to show that when $U \subset \tilde{U}$ is a sufficiently small region around $\r_0$, then $\|M_f(\r')\|_2$ also becomes arbitrarily small for all 1-Lipschitz-smooth $f$. Indeed, this holds:
    \begin{align*} 
        \|&M_f(\r')\|_2 
        \\&\overset{(i)}{=} \|\nabla f (\psi^{-1}(\psi(\r')))\nabla \psi^{-1}(\psi(\r')) - \nabla f(\theta^{-1}(\psi(\r')) \nabla \theta^{-1}(\psi(\r'))\|
        \\&\overset{(ii)}{=} \|\nabla f (\psi^{-1}(\psi(\r')))\nabla \psi^{-1}(\psi(\r')) - \nabla f (\psi^{-1}(\psi(\r')))\nabla \theta^{-1}(\psi(\r')) \\
        &\phantom{= \| \nabla f (\psi^{-1}(\psi(\r')))\nabla \psi^{-1}(\psi(\r'))}\,\, + \nabla f (\psi^{-1}(\psi(\r')))\nabla \theta^{-1}(\psi(\r')) -  \nabla f (\theta^{-1}(\psi(\r')))\nabla \theta^{-1}(\psi(\r'))\|
        \\&\overset{(iii)}{\leq} \|\nabla f(\r')\| \cdot \|\nabla \psi^{-1} (\psi(\r')) - \nabla \theta^{-1}(\psi(\r'))\| + \|\nabla f(\r') - \nabla f(\theta^{-1}(\psi(\r'))\| \cdot \|\nabla \theta^{-1} (\psi(\r'))\|
        \\&\overset{(iv)}{\leq} \|\nabla \psi^{-1}(\psi(\r')) - \nabla \psi^{-1}(\psi(\r_0))\|  + L \cdot \|\r' - \phi(\r')\|,
    \end{align*}
    where (i) applies the chain rule, (ii) adds and subtracts the inner terms, (iii) applies triangle inequality, (iv) simplifies the first term by using the facts that (a) the gradient $\nabla f$ is 1-Lipschitz and (b) the gradient $\nabla\theta^{-1} \equiv \nabla\psi^{-1}(\r_0)$ is constant, and (iv) also simplifies the second term also using the fact that $\nabla f$ is 1-Lipschitz and that $\theta$ is $L$-bi-Lipschitz. Both terms in (iv) can be made to be arbitrarily close to zero by controlling the size of $U$ around $\r_0$. Both terms go to zero as $\r' \to \r_0$.

\subsubsection{Proof of \texorpdfstring{\Cref{lem:comparability}}{Lemma \ref{lem:comparability}}}
\label{proof:comparability}

We begin by making the following observation on comparability of Pareto and Euclidean distances for \emph{linear} Pareto fronts.   
Let $\n \in \RR^K_{>0}$ be a unit normal vector where $n_k > \eta > 0$ is bounded away from 0. Let $\fR = \{\r \in \RR^K : \r \cdot n = 0\}$ be the hyperplane with normal vector $\n$. Then:
\begin{equation}
    \label{eqn:linear-comparability}
    \frac{\eta}{K(1 + \eta)} \big\|\r - \r'\big\|_2 \leq d(\r, \r').
\end{equation}
\begin{proof}[Proof of \Cref{eqn:linear-comparability}]
    Let $\v = \r - \r'$ for short, and let $I = \{k \in [K]: v_k > 0\}$ be the coordinates where $r_k' < r_k$. Let $J = [K] \setminus I$ be the remaining coordinates. Thus, $d(\r, \r') = \max_{k \in [K]} v_k$. Because $\v \cdot \n = 0$, we have:
    \[\sum_{i \in I} n_iv_i = -\sum_{j \in J} n_jv_j.\]
    We also have that $\|\v\|_1 = \sum_{i \in I} v_i - \sum_{j \in J} v_j$. The unit vector $\n$ satisfies $1 \geq n_k > \eta > 0$. Together:
    \[\sum_{i \in I} v_i \geq \sum_{i \in I} n_i v_i = - \sum_{j \in J} n_j v_j \geq - \, \eta \cdot \sum_{j \in J} v_j = \eta \cdot \|\v\|_1 - \eta \sum_{i \in I} v_i.\]
    Rearranging, we obtain that:
    \[\sum_{i \in I} v_i \geq \frac{\eta}{ 1 + \eta} \|\v\|_1 \geq \frac{\eta}{1 + \eta} \|\v\|_2.\]
    Dividing through by $K$ and using the fact that $\max_{k\in[K]} v_k \geq \frac{1}{K} \sum_{i \in I} v_i$ gives the result. 
\end{proof}

We can use \Cref{eqn:linear-comparability} combined with \cref{lem:local-linear-approx} to prove \cref{lem:comparability}.

\begin{proof}[Proof of \Cref{lem:comparability}]
    The Pareto distance is bounded above by the $\ell_\infty$-distance, which is upper bounded by the $\ell_2$-distance:
    \[d(\r,\r') = \max_{k \in [K]} \,\big(r_k - r_k'\big) \leq \|\r - \r'\|_\infty \leq \|\r- \r'\|_2.\]

    As for the lower bound, we use the niceness of $\fR$. We claim that for each $\r_0 \in \fR$, there is a sufficiently small open Euclidean ball $B_{\r_0}:= B(\r_0, \alpha_{\r_0}) \cap \fR$ such that the following has a positive lower bound:
    \begin{equation} \label{eqn:lower-bound-comparability}
        \inf_{\r \in B_{\r_0}}\inf_{\r' \in \fR} \,\frac{d(\r, \r')}{\|\r - \r'\|_2} > c_{\r_0} > 0.
    \end{equation}
    The result follows from compactness of $\fR$. In particular, there is a finite subcover of $\fR$ by balls $B_{\r_1},\ldots, B_{\r_N}$, and we let the constant in the lemma statement be defined as $c = \min_{n\in[N]} c_{\r_n}$. 

    Fix $\r_0$ and some $\varepsilon > 0$. \Cref{lem:local-linear-approx} constructs a diffeomorphism $\phi : U \to V$ from a relatively open set in $\fR$ containing $\r_0$ to a locally linearized Pareto front satisfying \Cref{eqn:bounded-difference}. Choose any ball $B_{\r_0} := B(\r_0, \alpha_{\r_0}) \cap \fR$ such that the ball $B_{\r_0}'$ with twice the radius is contained in $U$:
    \[B_{\r_0}' := B(\r_0, 2\alpha_{\r_0}) \cap \fR \subset U.\]
    We show \Cref{eqn:lower-bound-comparability} for two cases: (1)  $\r' \in B_{\r_0}'$ and (2) $\r' \notin B_{\r_0}'$.
    \begin{enumerate}
        \item \underline{Local}. Let $\r \in B_{\r_0}$ and $\r' \in B_{\r_0}'$. Let $\hat{\r} = \phi(\r)$ and $\hat{\r}' = \phi(\r')$. For each $k \in [K]$, the coordinate map projecting $\r$ down to its $k$ component $r_k$ is $1$-Lipschitz smooth. Thus, \Cref{lem:local-linear-approx} shows that:
        \[\phantom{\qquad \forall k \in [K]}\frac{r_k - r_k'}{\|\r - \r'\|_2} \geq \frac{\hat{r}_k - \hat{r}_k'}{\|\r - \r'\|_2} - \varepsilon, \qquad \forall k \in [K].\]
        Taking the max over $k \in [K]$ implies that:
        \begin{align*} 
            \frac{d(\r, \r')}{\|\r- \r'\|_2} &\geq \frac{d(\hat{\r}, \hat{\r}')}{\|\r - \r'\|_2} - \varepsilon
            \geq \frac{d(\hat{\r}, \hat{\r}')}{\|\hat{\r} - \hat{\r}'\|_2} \frac{\|\hat{\r} - \hat{\r}'\|_2}{\|\r - \r'\|_2} - \varepsilon
            \geq \frac{\eta}{K(1 + \eta)} \cdot (1 - \varepsilon \sqrt{K}) - \varepsilon,
        \end{align*}
        where in the last inequality, we use the facts that (a) \Cref{lem:local-linear-approx} shows that the map $\phi$ has bounded distortion, and (b) $V \subset \mathsf{T}_{\r_0}\fR$ is a linear Pareto front, so that \Cref{eqn:linear-comparability} shows that $d(\hat{\r}, \hat{\r}') / \|\r - \r'\| \geq \eta/(K(1 + \eta))$. Once $\varepsilon$ is sufficiently small, the lower bound is strictly positive. 
        
        \item \underline{Global}. Consider $\r \in B_{\r_0}$ and $\r' \in \fR \setminus B_{\r_0}'$. The product set $\mathcal{K} = \overline{B}_{\r_0} \times (\fR \setminus B_{\r_0}') \subset \fR^2$ is compact, while the function $(\r, \r') \mapsto d(\r, \r') / \|\r - \r'\|_2$ is continuous. Thus, it attains its minimum on $\mathcal{K}$. Since $\overline{B}_{\r_0}$ and $(\fR \setminus B_{\r_0}')$ are disjoint, this function is strictly greater than zero. This is because distinct points on $\fR$ are mutually non-dominating; $\r'$ does not dominate $\r$, so there is a coordinate $k \in [K]$ such that $r_k - r_k' > 0$. Thus, the numerator  $d(\r,\r')=\max_{k\in[K]}(r_k-r_k')>0$ for $(\r,\r')\in \mathcal{K}$. The denominator is upper bounded by the diameter of $\fR$, which is finite as $\fR$ is compact.
    \end{enumerate}
    That concludes the proof of \cref{lem:comparability}.
\end{proof}

\subsubsection{Proof of \texorpdfstring{\Cref{lem:piecewise-approximation}}{Lemma \ref{lem:piecewise-approximation}}}
\label{proof:piecewise-approximation}

We begin by proving the following fact that inverses of bounded distortions are bounded distortions:
Let $\varepsilon \in(0,1/2)$ and let $\phi : (U, d_U) \to (V, d_V)$ be an $\varepsilon$-distortion (recalling \cref{def:distortion}). Then the map $\phi^{-1} : (V,d_V) \to (U,d_U)$ is a $2\varepsilon$-distortion. 

    Indeed, whenever $\r, \r' \in U$, define $\hat{\r} = \phi(\r)$ and $\hat{\r}' = \phi(\r')$. By definition, the map $\phi$ satisfies:
     \[\hspace{2em}(1 - \varepsilon) \cdot d_V\big(\hat{\r}, \hat{\r}'\big) \leq d_U(\r,\r') \leq (1 + \varepsilon) \cdot d_V\big(\hat{\r}, \hat{\r}'\big), \qquad \forall \r, \r' \in U.\]
     Since $(1 + 2\varepsilon)(1 - \varepsilon)  - 1 = \varepsilon (1 - 2\varepsilon) > 0$ whenever $\varepsilon \in (0,1/2)$, we have: 
     \[d_V(\hat{\r}, \hat{\r}') \leq \frac{1}{1 - \varepsilon} \cdot d_U(\r, \r') \leq (1 + 2\varepsilon) \cdot d_U(\r, \r').\]
     And since $(1 + \varepsilon)(1 - \varepsilon) - 1 = - \varepsilon^2 < 0$, we have that:
     \[d_V(\hat{\r}, \hat{\r}') \geq \frac{1}{1+ \varepsilon} \cdot d_U(\r, \r') \geq (1 - \varepsilon) \cdot d_U(\r, \r').\]
     Together, these show that $\phi^{-1}$ is a $2\varepsilon$-distortion. 

    We now come to the proof of \cref{lem:piecewise-approximation}.
    We first prove the local result, showing that for each $\r_0 \in \fR$, there is a local linear approximation that preserves the Pareto distance. This is almost an immediate consequence of \Cref{lem:local-linear-approx}. The second step applies compactness to cover $\fR$ with a finite collection of well-approximated local regions.
    
    \paragraph{Step 1: Low-distortion for small neighborhoods.} Fix any $\r_0 \in \fR$ and without loss of generality, we may fix $\varepsilon \in (0,1/2)$. \Cref{lem:local-linear-approx} constructs a diffeomorphism $\phi : U \to V$ satisfying \Cref{eqn:bounded-difference}. For any $\r, \r' \in U$, let $\hat{\r} = \phi(\r)$ and $\hat{\r}' = \phi(\r')$ for short. Applying \Cref{eqn:bounded-difference} to the standard basis functions of $\RR^K$, where $\pi_k : \RR^K \to \RR$ projects to the $k$th coordinate. Each of these maps are 1-Lipschitz smooth, and so we obtain that for each $k \in [K]$:
    \[(r_k - r_k') = (\hat{r}_k - \hat{r}_k')  + \Delta_k,\]
    where $|\Delta_k| \leq \varepsilon \|\r - \r'\|_2$. We obtain the upper bound:
    \[d(\r, \r') = \max_{k \in [K]}\, r_k - r_k' = \max_{k \in [K]} \hat{r}_k - \hat{r}_k' + \Delta_k \leq d(\hat{\r}, \hat{\r}') + \varepsilon \|\r - \r'\|_2.\]
    The analogous lower bound holds. \Cref{lem:comparability} shows that there is a constant $c_\fR > 0$ depending on $\fR$ with
    \[
        \hspace{2em}c_\fR \cdot \|\r - \r'\|_2 \leq d(\r, \r') \leq  \|\r - \r'\|_2,\qquad \forall \r, \r' \in \fR.
    \]
    Let $\varepsilon' = \max\{\varepsilon, \varepsilon/c_\fR\}$. It follows that:
    \[ \left(1 - \varepsilon'\right) \cdot d(\r, \r') \leq d(\hat{\r}, \hat{\r}') \leq (1 + \varepsilon') \cdot d(\r, \r').\]
    Thus, $\phi^{-1}$ is an $\varepsilon'$-distortion. As shown above, this means $\phi$ is a $2\varepsilon'$-distortion. Reparametrizing $\varepsilon'$, we obtain a construction for an $\varepsilon$-distortion $\phi$ around any point $\r_0$ in $\fR$, where $\phi$ maps into a linear Pareto front with normal vector $\n(\r_0)$. By construction, $J \phi(\r_0)$ is the identity, so $|J\phi(\r_0)| = 1$. We may restrict the domain of $\phi$ to an open ball $B(\r_0, \alpha_{\r_0}) \cap \fR \subset U$ so that for all $\r$ in the domain of $\phi$,
    \[1 - \varepsilon < |J\phi(\r)| < 1 + \varepsilon.\]

    \paragraph{Step 2: Partitioning the Pareto front.} In the previous step, we saw that for any $\r_0 \in \fR$, there is a triple $(B, \phi, \n)$ where $B$ is a relatively open ball in $\fR$ containing $\r_0$, and $\phi$ is a $\varepsilon$-distortion mapping $B$ into a subset of $\{x \in \RR^K: x \cdot \n = 0\}$. We can always shrink the ball so that its radius is less than $\varepsilon$. Since $\fR$ is compact, there must be a finite collection of such triples:
    \[(B_1, \phi_1, \n_1),\ldots, (B_M, \phi_M, \n_M),\]
    where $B_1,\ldots, B_M$ cover $\fR$. To obtain the collection $\phi_m : (U_m, d) \to (V_m, d)$ in the statement, we set:
    \[U_m := B_m \setminus \bigcup_{i < m} \overline{B}_i \qquad \textrm{and}\qquad V_m =\phi_m(U_m).\]
    By construction, the first condition is satisfied. The second condition is also satisfied, since the only points potentially not covered by the family $\{U_m : m \in [M]\}$ are the boundary points of $B_m$ that we removed.

\subsubsection{Proof of \texorpdfstring{\Cref{lem:coarse-bounds}}{Lemma \ref{lem:coarse-bounds}}}
\label{proof:coarse-bounds-and-boundary-covering}

We begin with a basic fact about the scaling of $\ell_2$-covering numbers. Let $A \subset \RR^K$ and $c, t > 0$. Then: 
\begin{equation}
\label{eqn:covering-scaling}
    N_2(t,A) \leq \big(1 + 2 c\big)^K \cdot N_2(t/c, A).
\end{equation}

\begin{proof}[Proof of \Cref{eqn:covering-scaling}]
    First, recall the following fact (e.g., Corollary 4.2.13 in \cite{vershynin2018high}):
    Let $t > 0$ and let $B \subset \RR^K$ be the unit ball in $\ell_2$-norm. Then the covering and packing numbers with respect to the Euclidean distance scale as:
    \begin{equation}
    \label{eqn:l2-covering-standard}
        N_2(t, B) \leq P_2(t,B) \leq \left(1 + \frac{2}{t}\right)^K.
    \end{equation}

    When $c \geq 1$, then $N_2(t/c,A) \geq N_2(t,A)$, and so the inequality holds. Assume in the remainder that $c < 1$. Greedily construct a $t$-covering $S$ of $A$ as follows: while $A$ is not fully covered by $S$, choose any point in $A$ that is distance at least $t$ from $S$, and include it in $S$. The resulting $t$-covering is also a $t$-packing.
    
    Suppose that $S' \subset A$ is a $t/c$-covering of $A$. If $s' \in S'$, then by \eqref{eqn:l2-covering-standard}, we must have:
    \[\big|B(s', t/c) \cap S\big| \leq P_2(1/c, B) \leq \left(1 + 2c\right)^K.\]
    Every point in $A$ is within a distance of $t/c$ from a point in $S'$, so it follows that:
    \[|S| \leq \sum_{s' \in S'}\big|B(s', t/c) \cap S\big| \leq (1 + 2c)^K \cdot N_2(t/c, A)\]
    which is \Cref{eqn:covering-scaling}. 
\end{proof}

\begin{proof}[Proof of \Cref{lem:coarse-bounds}]
    This follows almost immediately from \Cref{lem:comparability}. By the upper bound on $d(\r,\r')$ in \Cref{eqn:metric-equivalence}, every $t$-Euclidean covering of a set $A$ is a $t$-Pareto covering, implying the first inequality:
    \[\Npar(t,A) \leq N_2(t, A).\]
    The lower bound also implies that a $t$-Pareto covering is a $t/c_\fR$-Euclidean covering, so that:
    \begin{align*} 
        \Npar(t,A) &\geq N_2(t/c_\fR, A)
        \\&\geq \left(1 + 2c_\fR\right)^{-K} N_2(t, A),
    \end{align*}
    where the last inequality uses \Cref{eqn:covering-scaling}.  

    Now, since $U = B \cap \fR$ is the intersection of a ball with $\fR$, the boundary $\partial U$ is a piecewise smooth, compact $(K-2)$-dimensional manifold. It is a standard fact that the Minkowski dimension of such manifolds exists so that we have (Section 3.2 of \cite{falconer2013fractal}):
    \[\lim_{t \to 0}\, \frac{\log N_2(t, \partial U)}{\log (1/t)} = K-2.\]
    This shows that $N_2(t, \partial U) \leq t^{-(K-2) \cdot (1 + o(1))}$ as $t\to 0$, which implies that
    \[\lim_{t \to 0}\, N_2(t, \partial U) \cdot t^{K-1} = 0.\]
    The second claim then follows immediately from the coarse bound on the covering numbers above, concluding the proof of \cref{lem:coarse-bounds}.
\end{proof}

\subsubsection{Proof of \texorpdfstring{\cref{lem:distortion-to-covering}}{Lemma \ref{lem:distortion-to-covering}}}
\label{proof:distortion-to-covering}

Let $S \subset A$ be a minimal subset such that $\phi(S)$ is a $(t/(1+\varepsilon))$-covering of $A$. Because $\phi$ is a bijection, the size of $S$ is $N_V(t/(1+\varepsilon), \phi(A))$. For every $\phi(\r') \in \phi(A)$, there is some $\phi(\r) \in \phi(S)$ such that $d_V(\phi(\r), \phi(\r')) < t/(1 + \varepsilon)$. Since $\phi$ is an $\varepsilon$-distortion, this implies that $d_U(\r, \r') < t$. In particular, $S$ is a $t$-covering of $A$. This implies the upper bound
\[N_U(t,A) \leq N_V\big(t/(1 + \varepsilon),\phi(A)\big).\]
The lower bound is similar: if $S$ is a minimal $t$-covering of $A$, then $\phi(S)$ is a $(t/(1 - \varepsilon))$-covering of $\phi(A)$.

\subsubsection{Proof of \texorpdfstring{\Cref{lem:covering-density}}{Lemma \ref{lem:covering-density}}}
\label{proof:covering-density}
    We claim that there is an affine isomorphism $T : \RR^\kappa \to \RR^\kappa$ such that $T(\Omega) = \Omega'$.
    
    Define the simplex $\Omega_0 := \mathrm{conv}(\mathbf{0}, \mathbf{e}_1,\ldots, \mathbf{e}_\kappa)$, where $\mathbf{e}_i \in \RR^\kappa$ is a standard basis vector. It suffices to show that there is an affine isomorphism from $\Omega_0$ to $\Omega_1$ where $\Omega_1 = \mathrm{conv}(\v_0, \v_1, \ldots, \v_\kappa)$. Let $L : \RR^\kappa \to \RR^\kappa$ be the following linear map:
    \[L\mathbf{e}_i = \v_i - \v_0.\]
    It is invertible because the points $\{\v_k: k = 0,1,\ldots, \kappa\}$ are in general linear position.  
    Set $T(\mathbf{u}) = \v_0 + L(\mathbf{u})$. Notice that the determinants of $T$ and $V$ coincide, $\det T = \det L$.

    Let $\mathcal{C}(\Omega, \Lambda, S)$ be an $\Omega$-periodic covering of $\RR^\kappa$. Then, Fact 2.1 of \citet{naszodi2018flavors} shows that the following is an $\Omega$-periodic covering of $\RR^\kappa$,
    \[T(\mathcal{C}):= \mathcal{C}(T(\Omega), L(\Lambda), T(S)).\] 
    Because any basis $V$ of $\Lambda$ is mapped to a basis $LV$ of $L(\Lambda)$, the density is preserved:
    \[\delta(T(\mathcal{C})) = \frac{|T(S)| \cdot \mathrm{Vol}(T(\Omega))}{\det L(\Lambda)} = \frac{|S| \cdot \det T \cdot \mathrm{Vol}(\Omega)}{\det L \det \Lambda} = \delta(\mathcal{C}),\]
    since $\det L(\Lambda) = \det (LV) = \det L  \det V = \det L \det \Lambda$, and where the last step used the fact that $\det T = \det L$.
    Thus, the infimum is also preserved, $\theta(\Omega) = \theta(T(\Omega))$.

\subsubsection{Proof of \texorpdfstring{\cref{lem:centers-in-fundamental-parallelepiped}}{Lemma \ref{lem:centers-in-fundamental-parallelepiped}}}
\label{proof:centers-in-fundamental-parallelepiped}
    We first observe that the following set of differences contains only one lattice point, namely $\mathbf{0} \in \Lambda$,
    \[\big\{\w - \w' : \w, \w' \in P(V)\big\} = \big\{\lambda_1 \v_1 + \dotsm + \lambda_\kappa \v_\kappa : \lambda_1,\ldots, \lambda_\kappa \in (-1,1)\big\}.\]
    Set this aside for now. We now construct an injective map $\iota : \cC(V) \to S$, which implies the result.
    
    Let $\mathbf{w} \in \cC(V)$. By definition of $\cC(V)$, there is some $\v \in \Lambda$ and $\mathbf{s} \in S$ such that $\mathbf{w} = \v + \mathbf{s}$. Choose any such $\mathbf{s}$ and define $\iota(\mathbf{w}) = \mathbf{s}$. We claim that $\iota$ is injective. Suppose that there are $\w, \w' \in \cC(V)$ such that $\iota(\w) = \iota(\w')$. It follows that
    \[\w - \w' = \big(\w - \iota(\w)\big) - \big(\w' - \iota(\w')\big) = \v-\v' \in \Lambda.\]
    By the observation at the beginning, we deduce that $\w = \w'$. 

\subsubsection{Proof of \texorpdfstring{\cref{lem:covering-for-boxes}}{Lemma \ref{lem:covering-for-boxes}}}
\label{proof:covering-for-boxes}

    For any $t > 0$, let $A_t$ denote the scaled set $(1/t) A = [0, 1/t]^\kappa$. To prove the result, we show that the limit supremum and infimum coincide. 

    \paragraph{Lower bound.}
    For each $t > 0$, let $S_t \subset A$ be a set such that the collection $\{t\Omega + \mathbf{s} : \mathbf{s} \in S_t\}$ forms a minimal covering of $A$ by $t\Omega$. Then, the collection $\{\Omega + (1/t) \mathbf{s} : \mathbf{s} \in S_t\}$ covers $A_t$. As the set $A_t$ tessellates $\RR^\kappa$, we can construct the $\Omega$-periodic covering $\mathcal{C}(\Omega, (1/t) \mathbb{Z}^\kappa,  (1/t)S_t)$ of $\RR^\kappa$ (recall \cref{def:periodic-covering-density}). By definition, the density of this covering $\mathcal{C}$ is at least $\theta(\Omega)$, and so
    \[\theta(\Omega) \leq |S_t| \cdot \mathrm{Vol}(\Omega) \cdot t^\kappa,\]
    where we use the fact that the set $(1/t) \mathrm{Id} \in \RR^{\kappa \times \kappa}$ forms a basis of $(1/t)\ZZ^\kappa$, so that $\det(1/t) \ZZ^\kappa = t^{-\kappa}$. 
    Dividing through by $\mathrm{Vol}(\Omega)$ and writing $|S_t| = N(A;t\Omega)$, we obtain the limit infimum:
    \[\liminf_{t\to 0}\,  N(A; t\Omega) \cdot t^\kappa \geq \frac{\theta(\Omega)}{\mathrm{Vol}(\Omega)}.\]

    \paragraph{Upper bound.} 
    For the upper bound, we will not only need $A_t = [0,1/t]^\kappa$, but for all $s < 1/t$, we will also let $A_{t,s} \subset A_t$ be the slightly smaller open box:
    \[A_{t,s} = \left(s, \frac{1}{t}- s\right)^\kappa.\]
    Let $\mathcal{C} = \cC(\Omega, \Lambda, S)$ be an $\Omega$-periodic covering of $\RR^\kappa$ such that $\delta(\mathcal{C}) < \theta(\Omega) + \varepsilon$. We now show that for sufficiently small $t > 0$, we can use $\mathcal{C}$ to construct an economical covering of $A$ by $t\Omega$, with:
    \[\limsup_{t \to 0} \, N(A;t\Omega) \cdot t^\kappa <\frac{\theta(\Omega) + \varepsilon}{ \mathrm{Vol}(\Omega)}. \]
    And since $\varepsilon > 0$ is arbitrary, the result follows for $A = [0,1]^\kappa$ by taking the limit as $\varepsilon$ goes to zero.

    Fix any basis $V$ of the lattice $\Lambda$, and let $D_V = \mathrm{diam}_{\ell_\infty}(P(V))$ and $D_\Omega = \mathrm{diam}_{\ell_\infty}(\Omega)$ be the diameters of the fundamental parallelepiped and of $\Omega$ under the $\ell_\infty$-distance. As a high-level roadmap, we will first cover nearly all of $A_t$ by a tessellation of $P(V)$, where the corresponding covering number can be related to the volume. Then, we can transfer to a simplex covering using \Cref{lem:centers-in-fundamental-parallelepiped}, which counts the number of elements of $\cC$ that are centered in $P(V)$. The main subtlety will be points in $A_t$ that are very close to its boundary, but the covering number of this region grows at a slower rate.
    
    We begin by constructing the \emph{lattice tesselation.} For each $t > 0$, let $\cL_t \subset \Lambda$ be the (possibly empty) set of lattice-valued translations $\v$ such that $P(V) + \v$  is fully contained in $A_t$,
    \[\cL_t = \big\{\v : \v \in \Lambda \textrm{ and }  P(V) + \v \subset A_t\big\}.\]
    Notice that every point $\mathbf{a} \in A_t$ is covered by exactly one translated parallelepiped $P(V) + \v $ where $\v \in \Lambda$, since $P(V)$ tessellates $\RR^\kappa$. Moreover, if this point $\mathbf{a}$ is $D_V$-separated from the boundary $\partial A_t$ under the $\ell_\infty$-distance and $\mathbf{a} \in P(V) + \v$, then we also have that $P(V) + \v \subset A_t$. It follows that:
    \[A_{t, D_V} \quad\subset \quad \bigcup_{\v \in \cL_t} P(V) + \v \quad \subset\quad A_t,\]
    where $A_{t,D_V}$ is the slightly smaller box. The volume of each term is:
    \[\left(\frac{1}{t} - 2D_V \right)^\kappa \quad\leq\quad |\cL_t| \det \Lambda\quad \leq \quad  \frac{1}{t^\kappa}.\hspace{3em}\]
    And so, the following limit exists: 
    \begin{equation} \label{eqn:lattice-limit}
        \lim_{t \to 0}\,  |\cL_t| \cdot t^\kappa = \frac{1}{\det \Lambda}.
    \end{equation}
    
    We now construct the \emph{simplex covering.} Recall that $\cC$ is a collection whose elements $\Omega + \w$ are translations of $\Omega$. Let $\cC_t \subset \cC$ consist of those whose centers $\w$ are covered by some $P(V) + \v$ for $\v \in \cL_t$:
    \[\cC_t = \left\{\Omega + \w \in \cC : \w \in \bigcup_{\v \in \cL_t} P(V) + \v\right\}.\]
    By \Cref{lem:centers-in-fundamental-parallelepiped}, the size of $\cC_t$ is at most $|\cC_t| \leq |S| \cdot |\cL_t|$. This subcollection $\cC_t$ covers almost all of $A_t$, except possibly points near the boundary $\partial A_t$. By a similar argument as before, if $\mathbf{a} \in A_t$ is $(D_V + D_\Omega)$-separated from the boundary and if $\mathbf{a} \in \Omega + \w$ is covered by an element of $\cC$, then $\w$ must be $D_V$-separated from the boundary. This further implies that $\w \in P(V) + \v$ for some $\v \in \cL_t$. In particular, this argument shows that the smaller box $A_{t,D_V + D_\Omega}$ is covered by $\cC_t$:
    \begin{equation} \label{eqn:small-box-cover}
        A_{t,D_V + D_\Omega} \quad \subset \quad \bigcup_{\Omega + \w \in \cC_t} \Omega + \w.
    \end{equation}
    We are almost done, since this shows that the asymptotic covering number of the set on the left satisfies:
    \[\limsup_{t \to 0}\, N(A_{t,D_V + D_\Omega}; \Omega) \cdot t^\kappa \overset{(i)}{\leq} \lim_{t \to 0}\, |S| \cdot |\cL_t| \cdot t^{\kappa} \overset{(ii)}{=} \frac{|S|}{\det \Lambda} \overset{(iii)}{=} \frac{\delta(\cC)}{\mathrm{Vol}(\Omega)} \overset{(iv)}{<} \frac{\theta (\Omega) + \varepsilon}{\mathrm{Vol}(\Omega)},\]
    where (i) combines \cref{eqn:small-box-cover} with the above upper bound $|\cC_t| \leq |S| \cdot |\cL_t|$, (ii) applies \cref{eqn:lattice-limit}, (iii) uses the definition of $\delta(\cC)$, and (iv) follows from our choice of $\cC$ to be an $\Omega$-periodic covering where $\delta(\cC) < \theta(\Omega) + \varepsilon$.
    To finish the proof, we show that the remaining region $Z_t = A_t \setminus A_{t, D_V + D_\Omega}$ satisfies
    \[\limsup_{t \to 0}\, N(Z_t;\Omega) \cdot t^\kappa = 0.\]
    This follows because $Z_t$ is nearly the $(\kappa-1)$-dimensional object $\partial A_t$. It has just been fattened up by a constant amount $D_V + D_\Omega$, which is independent of the size of the box $1/t$. 
    
    Let's formally show that $N(Z_t, \Omega) = O(t^{-(\kappa-1)})$. As $\mathbf{0}$ is in the interior of $\Omega$, for some $r > 0$, this set contains some $r$-box under the $\ell_\infty$-distance,
    \[[-r,r]^\kappa \subset \Omega.\]
    The region $Z_t$ is a thin shell at the boundary $\partial A_t$ with thickness $D_V + D_\Omega$. The number of $(\kappa-1)$-dimensional faces of the box $A_t$ is $2\kappa$; we can tile $Z_t$ with $r$-boxes centered in $Z_t$ using no more than
    \[2 \kappa \times \left\lceil \frac{1}{tr}\right\rceil^{\kappa-1} \times \left\lceil \frac{D_V + D_\Omega}{r}\right\rceil \textrm{ $r$-boxes.}\]
    It follows that $N(Z_t, \Omega) = O(t^{-(\kappa-1)})$ grows at a rate dominated by $t^{-\kappa}$. The result follows, since
    \begin{align*} 
        \limsup_{t \to 0}\, N(A;t\Omega) \cdot t^\kappa &= \limsup_{t \to 0}\,  N(A_t; \Omega) \cdot t^\kappa 
        \\&\leq \limsup_{t \to 0} \, N(A_{t, D_V + D_\Omega}; \Omega) \cdot t^\kappa + \limsup_{t \to 0}\,  N(Z_t; \Omega) \cdot t^{\kappa}< \frac{\theta(\Omega) + \varepsilon}{\mathrm{Vol}(\Omega},
    \end{align*}
    using the fact that $N(Z_t, \Omega) \cdot t^\kappa = O(t)$ goes to zero with $t$. 
        
    We can also compute the asymptotic covering number of boxes with any length $\ell A = [0,\ell]^\kappa$, since 
    \[\lim_{t \to 0} \, N\big(\ell A; t\Omega\big) \cdot t^\kappa = \lim_{t \to 0}\, N\big(A; (t/\ell) \Omega\big) \cdot (t/\ell)^\kappa \cdot \ell^\kappa = \frac{\theta(\Omega) \cdot \mathrm{Vol}(\ell A)}{\mathrm{Vol}(\Omega)}.\]

\subsubsection{Proof of \texorpdfstring{\Cref{prop:covering-to-minkowski}}{Proposition \ref{prop:covering-to-minkowski}}}
\label{proof:covering-to-minkowski}

    Let $A \subset \RR^\kappa$ be Jordan measurable. Let $D_\mathrm{in}, D_\mathrm{out} > 0$ satisfy $B_\infty(\mathbf{0}, D_\mathrm{in}) \subset \Omega \subset B_\infty(\mathbf{0}, D_\mathrm{out})$.

    \paragraph{Lower bound.} Since $A$ is Jordan measurable (recall \cref{def:Jordan-measurable}), there is a countable disjoint union of boxes $B_i \subset A$ covering almost all of $A$,
    \[\mathrm{Vol}\left(A \setminus \bigcup_{i \in \mathbb{N}} B_i\right) = 0,\]
    where each $B_i = B_\infty(\mathbf{v}_i, r_i)$ is an $\ell_\infty$-ball. As it is enough to cover the interior of $A$, this follows from the Whitney covering lemma \citep{grafakos2008classical}. Without loss of generality, we can assume that $r_i \leq 1$. For each $t > 0$, let $S_t \subset A$ be a set such that the collection $\{t\Omega + \mathbf{s} : \mathbf{s} \in S_t\}$ forms a minimal covering of $A$ by translations of $t\Omega$. We now show that there is a constant $C$ depending only on $\Omega$ such that for each box $B_i$, the size of the collection $B_i \cap S_t$ is bounded:
    \[|B_i \cap S_t| \geq N\big(B_i, t\Omega\big)  - Ct^{-(\kappa-1)}.\]
    To see this, we augment the centers in $B_i \cap S_t$ by a set $T_i$ so that $\{t \Omega + \mathbf{s} : \mathbf{s} \in (B_i \cap S_t) \cup T_i\}$ covers $B_i$. The centers in $B_i \cap S_t$ must cover all points in $B_i$, except possibly those within a distance of $D_\mathrm{out}$ to the boundary of $B_i$ (since those may have been covered by simplices with centers outside $B_i$). Thus, to ensure that we cover all of $B_i$, we just need to cover a thin shell near the boundary of $B_i$. There are $2\kappa$ faces of the ball $B_i$, and for each face, we need to cover an $r_i\times \cdots\times r_i \times D_\mathrm{out}$ box. Recall that we restricted $r_i \leq 1$. Thus, there is a choice of $T_i$ such that:
    \[|T_i| \leq 2\kappa \cdot \left\lceil \frac{D_\mathrm{out}}{D_\mathrm{in}}\right\rceil \cdot \left\lceil\frac{r_i}{tD_\mathrm{in}}\right\rceil^{\kappa-1} \leq C t^{-(\kappa - 1)},\]
    where $C$ is a constant depending only on $\Omega$. We use this to show that:
    \begin{equation} \label{eqn:lower-general-cover}
        \liminf_{t \to 0}\, N(A;t\Omega) \cdot t^\kappa \geq \frac{\theta(\Omega) \cdot \mathrm{Vol}(A)}{\mathrm{Vol}(\Omega)}.
    \end{equation}

    For any $\varepsilon > 0$, choose $M_\varepsilon \in \mathbb{N}$ sufficiently large so that
    \[\mathrm{Vol}\left(A \setminus \bigcup_{i =1}^{M_\varepsilon} B_i\right) < \varepsilon.\]
    For each $t > 0$, we obtain the chain of inequalities:
    \begin{align*}
       N(A;t\Omega) = |S_t| &\geq \sum_{i=1}^{M_\varepsilon} |B_i \cap S_t| \geq \sum_{i=1}^{M_\varepsilon} N(B_i; t\Omega) - CM_\varepsilon t^{-(\kappa - 1)}.
    \end{align*}
    Multiplying through by $t^\kappa$ and taking a limit infimum, we apply \Cref{lem:covering-for-boxes} to obtain
    \[
        \liminf_{t\to0}\, N(A;t\Omega)\cdot t^\kappa \geq \sum_{i=1}^{M_\varepsilon} \frac{\theta(\Omega) \cdot \mathrm{Vol}(B_i)}{\mathrm{Vol}(\Omega)} \geq \frac{\theta(\Omega) \cdot (\mathrm{Vol}(A) - \varepsilon)}{\mathrm{Vol}(\Omega)},
    \]
    where the last step uses $\sum_i \Vol(B_i) \geq  \Vol(A)-\eps$. Since this holds for all $\varepsilon > 0$, \Cref{eqn:lower-general-cover} holds.

    \paragraph{Upper bound.} Fix $r > 0$ and consider a tessellation of $\mathbb{R}^\kappa$ by $r$-boxes $B_\infty(\mathbf{v}, r)$. Because $A$ is bounded, there is a finite number of such boxes that are fully contained in $A$. Let us call these $B_1,\ldots, B_N$, and let $Z$ denote the remainder
    \[Z = A \setminus \bigcup_{i=1}^N B_i.\]
    For each $s > 0$, let $\partial A^s$ denote the $s$-expansion of $\partial A$, consisting of the points whose distance from $\partial A$ is less than $s$ under the $\ell_\infty$-distance. Notice that if $\mathbf{a} \in A \setminus \partial A^{2r}$ has $\ell_\infty$-distance more than $2r$ from $\partial A$, then every $r$-box containing $\mathbf{a}$ is fully contained in $A$. Thus, $Z \subset \partial A^{2r}$. 

    For every $t > 0$, the following holds:
    \[N(A;t\Omega) \leq \sum_{i=1}^N N(B_i;t\Omega) + N(Z;t\Omega).\]
    Each $N(B_i;t\Omega)$ is bounded by \cref{lem:covering-for-boxes}.  To bound $N(Z;t\Omega)$, we construct a $t\Omega$-covering of $Z$. To this end, let $S_{t;Z} \subset Z$ be a maximal $tD_\mathrm{in}$-packing under the $\ell_\infty$-norm. Since $S_{t;Z}$ is maximal, this means that every point in $Z$ must be within a distance of $tD_\mathrm{in}$ of a center in $S_{t;Z}$. As $t\Omega \supset B_\infty(\mathbf{0}, t D_\mathrm{in})$, this implies that this set is a $t\Omega$-covering of $Z$. 

    To bound the size of $|S_{t;Z}|$, note that as $S_{t;Z}$ form a $tD_\mathrm{in}$-packing, the collection of boxes with $\ell_\infty$-radii $tD_\mathrm{in}/2$ centered at $S_{t;Z}$ do not overlap. And because the centers of these boxes are contained in $Z \subset \partial A^{2r}$, each of these boxes must be contained in the larger expansion $\partial A^{2r + tD_{\mathrm{in}}}$. It follows that:
    \[ |S_{t;Z}| \cdot (t D_\mathrm{in})^\kappa \leq \mathrm{Vol}(\partial A^{2r + t D_\mathrm{in}}).\]
    Rearranging, we obtain an upper bound on $|S_{t;Z}| \cdot t^\kappa$. Putting these upper bounds together, we obtain:
    \begin{align*}
    N(A;t \Omega) \cdot t^\kappa &\leq \sum_{i=1}^N N(B_i; t\Omega) \cdot t^\kappa + |S_{t;Z}| \cdot t^\kappa
    \\&\leq \sum_{i=1}^N N(B_i; t\Omega) \cdot t^\kappa + \frac{1}{D_\mathrm{in}^\kappa} \cdot \mathrm{Vol}(\partial A^{2r + tD_\mathrm{in}}).
    \end{align*}
    By the continuity of the Lebesgue measure, the limit of the right-hand side exists as $t \to 0$, and we obtain:
    \begin{align*} 
    \limsup_{t \to 0}\, N(A;t\Omega) \cdot t^\kappa &\leq \sum_{i=1}^N \frac{\theta(\Omega) \cdot \mathrm{Vol}(B_i)}{\mathrm{Vol}(\Omega)} + \frac{1}{D_\mathrm{in}^\kappa} \cdot \mathrm{Vol}(\partial A^{2r})
    \\&\leq \frac{\theta(\Omega) \cdot \mathrm{Vol}(A)}{\mathrm{Vol}(\Omega)} + \frac{1}{D_\mathrm{in}^\kappa} \cdot \mathrm{Vol}(\partial A^{2r}),
    \end{align*}
    where the first inequality uses \Cref{lem:covering-for-boxes}, and the second inequality uses the fact that the $B_i$ are disjoint subsets of $A$. Finally, because $\partial A$ is a compact set, we have that $\partial A^{2r} \downarrow \partial A$ as $r \to 0$. By the continuity of the Lebesgue measure, we also have that $\mathrm{Vol}(\partial A^{2r}) \to \mathrm{Vol}(\partial A) = 0$, as the boundary has measure zero by assumption. This yields the upper bound:
    \[
    \limsup_{t \to 0}\, N(A;t\Omega) \cdot t^\kappa \leq \frac{\theta(\Omega)\cdot \mathrm{Vol}(A)}{\mathrm{Vol}(\Omega)}.
    \]
    This upper bound along with the lower bound \Cref{eqn:lower-general-cover} gives the result. 

\subsubsection{Proof of \texorpdfstring{\cref{lem:volume-simplex}}{Lemma \ref{lem:volume-simplex}}}
\label{proof:volume-simplex}
Let $\cH=\big\{\v \in \RR^K : \v \cdot \n = 0\big\}$ be the hyperplane through the origin. Parameterize $\cH$ via the map $\Phi:\RR^{K-1}\to \RR^K$ defined as $\Phi(\bu) = (u_1,...,u_{K-1},-\tfrac{1}{n_K}\sum_{k=1}^{K-1}u_kn_k)$.
Then the Hausdorff measure $\sH^{K-1}$ on $\cH$ can be written as $\sH^{K-1}(\Phi(A)) = \sqrt{\det \prn{(J\Phi)^\top (J\Phi)}} \mathrm{Vol}(A)$ (see Theorem 3.2.3 in \cite{federer1996geometric}),
where the Jacobian is given by
\begin{equation*}
    J\Phi = 
    \begin{pmatrix}
        &\I_{K-1}& \\
        -n_1/n_K &\dots & -n_{K-1}/n_K
    \end{pmatrix}
    \in\RR^{K\times (K-1)}.
\end{equation*}
A computation then yields that $\sqrt{\det \prn{(J\Phi)^\top (J\Phi)}}=1/n_K$, which implies
$\sH^{K-1}(\Phi(A)) = \tfrac{1}{n_K} \mathrm{Vol}(A)$. One can now verify that the preimage of $\Omega_\n$ under $\Phi$ is given by the set
\begin{equation*}
    A_t = \crl{\a\in\RR^{K-1}: a_k\geq -t, \sum_{k=1}^{K-1}a_kn_k \leq tn_K} = \crl{\a\in\RR^{K-1}:a_k\geq 0, \sum_{k=1}^{K-1} a_kn_k \leq t \sum_{k=1}^Kn_k}-t\one.
\end{equation*}
Another computation using the Jacobian of the map $\a\mapsto \sum_{k=1}^{K-1}n_ka_k$ yields that the volume of this set is given by
\begin{equation*}
    \sH^{K-1}(\Omega_\n)=\frac{1}{n_K}\mathrm{Vol}(A_t)= \frac{t^{K-1}}{(K-1)!} \frac{\prn{\sum_{k=1}^K n_k}^{K-1}}{\prod_{k=1}^{K}n_k}.
\end{equation*}
That yields the claim. 

\section{Deferred Proofs}

\subsection{Proofs for Monotonicity and Pareto Optimality}
\label{proof:basics-monotonicity}

\subsubsection{Proof of \texorpdfstring{\cref{lem:monotonicity-pareto}}{Lemma \ref{lem:monotonicity-pareto}}}

    Fix any $f\in\cF$ and set $\r= \bR(f)\in\bR(\cF)$. Consider the set
    \begin{equation*}
        S = \{\bu\in\bR(\cF): \bu\preceq \r\}
    = \bR(\cF)\cap\prod_{k=1}^K(-\infty,r_k],
    \end{equation*}
    where $\prod_{k=1}^K(-\infty,r_k]$ is the Cartesian product of sets. Clearly $S\neq\emptyset$ since $\r\in S$. Because $\bR(\cF)$ is compact by our standing assumption, and $S$ is a
    closed subset of $\bR(\cF)$, the set $S$ is also compact.
    Define the continuous scalarization $s:\RR^K\to\RR$ as
    $s(\bu)= \sum_{k=1}^K u_k$.
    By compactness of $S$, there exists $\bu^\star\in S$ such that
    \begin{equation*}
        \bu^\star \in \arg\min_{\bu\in S}s(\bu).
    \end{equation*}
    Since $\bu^\star\in\bR(\cF)$, there exists $f'\in\cF$ with $\bR(f')=\bu^\star$.
    By construction, $\bu^\star\preceq \r$, hence $\bR(f')\preceq\bR(f)$.
    
    It remains to show that $f'$ is Pareto optimal. Suppose otherwise. Then there
    exists $f''\in\cF$ such that $\bR(f'')\prec\bR(f')=\bu^\star$. In particular,
    $\bR(f'')\preceq \bu^\star\preceq \r$, so $\bR(f'')\in S$. Moreover,
    $s(\bR(f''))<s(\bu^\star)$ since at least one coordinate is strictly smaller and
    none is larger, contradicting the minimality of $\bu^\star$ over $S$.
    Therefore $f'\in\pareto(\cF,\bR)$ and $\bR(f')\preceq\bR(f)$. Since $f$ was arbitrary, the claim follows.

    Now let $\fstar\in \argmin_{f\in\cF}\cR_Q(f)$, which exists by our standing assumptions. Then there exists a $f$ with $\bR(f)\preceq \bR(\fstar)$ and $f\in \pareto(\cF,\bR)$ by the previous argument. By monotonicity, we hence know that $\cR_Q(f)\leq \cR_Q(\fstar)$, and $f$ is also a minimizer of $\cR_Q$. That concludes the proof.

\subsubsection{Proof of \texorpdfstring{\cref{lem:weak-sufficiency}}{Lemma \ref{lem:weak-sufficiency}}}

    For any $\r \in \range(\bR)$ with some $f \in \cF$ satisfying $\bR(f) = \r$, define $s$ as
    \[s(\r) = \cR_Q(f),\]
    and let $s$ be arbitrarily monotonic elsewhere. The map $s$ is well-defined: for any $f, f'$ such that $\bR(f) = \bR(f')$, we have by monotonicity
    \[\cR_Q(f) \leq \cR_Q(f') \qquad \textrm{and}\qquad \cR_Q(f') \leq \cR_Q(f),\]
    in which case $\cR_Q(f) = \cR_Q(f')$. By construction, we have that $\cR_Q = s \circ \bR$.

\subsection{Proof of \texorpdfstring{\cref{thm:Geometry-complexity-upper-bound}}{Theorem \ref{thm:Geometry-complexity-upper-bound}}}
\label{proof:Geometry-complexity-upper-bound}

We split the proof into the bound for Pareto ERM and the bound for Pareto EW.
\subsubsection{Proof for Pareto ERM}
    Fix $t>0$, let $N=\Npar(t,\cF,\bR)$ and let $f_1,\ldots, f_N$ be the minimal $t$-Pareto set in $\cF$ on which we compute the Pareto ERM $\Paretoerm{t}$. By \cref{def:Pareto-cover}, this means that for every $f\in \cF$ there is a $f_j$ with $\cR_k(f_j)\leq \cR_k(f)+t$ for all $k$. By \cref{def:modulus-of-monotonicity}, we know that hence $\cR_Q(f_j)\leq \cR_Q(f) +\uppermodulus(t)$. 
    Now consider the event
    \begin{equation*}
        \cE=\crl{ \max_{j\in[N]} \abs{\Rhat_Q(f_j)-\cR_Q(f_j)}\leq r} \quad \text{with} \quad r=\sqrt{\frac{\log (2N/\delta)}{2n}}.
    \end{equation*}
    Recalling that $\Rhat_Q(f)=\frac{1}{n}\sum_{i=1}^n \ell(f,z_i)$, we may apply Hoeffding's bound to these i.i.d.\ $[0,1]$-valued random variables $\ell(f,z_i)$ with expectation $\cR_Q(f)$. 
    In particular, we get that
    \begin{equation*}
        \PP(\cE^c) \leq \sum_{j\in[N]} \PP\prn{\abs{\sum_{i=1}^n \ell(f_j,z_i)-n\cR_Q(f_j)}>nr} \leq 2N\exp\prn{-\frac{2n^2r^2}{n}} = \delta
    \end{equation*}
    and so, $\PP(\cE)\geq 1-\delta$.
    Let $\fstar\in\argmin_{f\in\cF}\cR_Q(f)$.
    On $\cE$, for $i=i(\fstar)$ being the index in the covering that corresponds to $\fstar$ (meaning that $\cR_Q(f_{i})\leq \cR_Q(\fstar)+\uppermodulus(t)$), we get  that
    \begin{align*}
        &\cR_Q(\Paretoerm{t})-\cR_Q(\fstar)\\
        &\leq \underbrace{\cR_Q(\Paretoerm{t})-\Rhat_Q(\Paretoerm{t})}_{\leq r}+  \underbrace{\Rhat_Q(\Paretoerm{t})-\Rhat_Q(f_i)}_{\leq 0} +\underbrace{\Rhat_Q(f_{i})-\cR_Q(f_i)}_{\leq r} +\underbrace{\cR_Q(f_i)-\cR_Q(\fstar)}_{\leq \uppermodulus(t)},
    \end{align*}
    where we used that $\Paretoerm{t}$ is the empirical risk minimizer on the Pareto covering.
    And so, $\cR_Q(\Paretoerm{t})-\cR_Q(\fstar)\leq 2r+\uppermodulus(t)$, concluding the proof.
\subsubsection{Proof for Pareto EW}

Throughout the proof, we assume that $\fR$ is a nice Pareto front (\Cref{assumption: nice}), so that the limiting distribution $\mu$ computed in \Cref{thm:limiting-distribution} is well-defined. Recall $\mu$ induces the prior $\pi$ over the Pareto set, obtained by pulling back and normalizing,\footnote{Technically, the distribution is defined on the quotient space $\pareto(\cF, \bR)/\bR$. But, without loss of generality, we can assign a canonical representative to each class, and assume that $\bR$ is injective.}
\[\pi = \frac{1}{\mu(\fR)}\cdot \bR^* \mu,\]
so that $\pi(A) \propto \mu\big(\bR(A)\big)$ whenever $A \subset \pareto(\cF, \bR)$ is a measurable set.

We begin with an observation, which relates the Pareto covering number and the limiting distribution. First, we lower bound the mass of near-optimal models close to $\fstar \in \argmin_{f\in\cF} \cR_Q(f)$, where recall that we assumed $\fstar\in\pareto(\cF,\bR)$. In particular, for any $t > 0$, define the following subset:
\begin{equation}
\label{eq:Nt-def}
    \cN_t:=\crl{f\in\pareto(\cF,\bR):\ \bR(f)\leq \bR(f^\star)+t\one}.
\end{equation}

\begin{lemma}[Mass on near-optimal models] \label{lem:mass-model-ball}
    Let $\fR$ be a nice Pareto front and let $\bR(\fstar)$ be in the relative interior of $\fR$. For every $\gamma > 0$, there is some $t_0 > 0$ such that for all \(t\in(0,t_0]\),
    \begin{align}
        \label{eq:ass-local-prior-mass}
        \pi(\cN_t) &\geq \frac{1 - \gamma}{\Npar(t, \cF, \bR)}.
    \end{align} 
\end{lemma}

\begin{proof}
    The set $\bR(\cN_t)$ is a Pareto ball in $\fR$. Let $\r^\star = \bR(\fstar)$. Then $\bR(\cN_t) = B_t$, where:
    \[B_t:= \{\r  \in \fR : d(\r^\star, \r) \leq t\}\qquad \textrm{and}\qquad \pi(\cN_t) = \frac{\mu(B_t)}{\mu(\fR)}.\]
    We hence proceed by lower bounding the volume of Pareto balls centered at $\r^\star$ in $\fR$ by linearizing the manifold at $\r^\star$. Let $\n^\star = \n(\r^\star)$, where $\n(\r)$ is the normal vector of $\fR$ at $\r$. Niceness imposes a lower bound $n_k(\r) > \eta > 0$ for all $k \in [K]$.  Also recall that $\beta(\n)$ is defined in \Cref{eqn:beta} as:
    \[\beta\big(\n\big) = \theta_K \cdot (K-1)!\cdot \frac{\prod_{k \in [K]} n_k}{(\mathbf{1} \cdot \n)^{K-1}},\]
    where $\theta_K$ is the covering density of $(K-1)$-dimensional simplices (see \Cref{lem:covering-density}). Since the covering density is lower bounded $\theta_K \geq 1$ \citep{rogers1957note}, by niceness we have that $\beta(\n^\star) > 0$ is bounded away from zero.

    For any $\varepsilon > 0$, \Cref{lem:piecewise-approximation} shows that there exists a smooth diffeomorphism $\phi : U \to V$ centered at $\r^\star$, where $U \subset \fR$ is relatively open in $\fR$ containing $\r^\star$ and $V \subset \{\v \in \RR^K : \v \cdot \n^\star = 0\}$ is in the linear subspace orthogonal to $\n^\star$. The diffeomorphism can be constructed so that $\phi(\r^\star) = \r^\star$, the map $\phi$ is an $\varepsilon$-distortion, and it is locally volume-preserving at $\r^\star$, where the determinant of the Jacobian is $|J\phi(\r)| = 1$. Define the ball $\tilde{B}_{t,\varepsilon}$ by:
    \[\tilde{B}_{t,\varepsilon} := \big\{ \r \in V : d(\r^\star, \r) < t/(1 + \varepsilon) \big\}.\]
    Because $\r^\star$ is in the relative interior of $\fR$ and $\phi$ is an $\varepsilon$-distortion, we have that $\tilde{B}_{t, \varepsilon} \subset \phi(B_t)$.

    We now lower bound $\mu(B_t)$ using the definition of $\mu$ from \cref{thm:limiting-distribution}, given by $\mu(B_t) = \int_{B_t} \beta(\n(\r)) \d \sH(\r)$.
    The functions $\phi(\r)$ and $\beta(\n(\r))$ are smooth, so for any $\gamma > 0$, by shrinking $U$ to a sufficiently small relatively open set around $\r^\star$, we can ensure that
    \begin{equation}
        \label{eq:lowerbounds}
        |J\phi(\r)| > 1 - \gamma \quad \textrm{and}\quad \beta(\n(\r)) > (1 - \gamma) \cdot \beta(\n^\star),
    \end{equation}
    where the second inequality is possible because $\beta(\n^\star) > 0$.
    We have:
    \begin{align*} 
        \mu(B_t)
        &\geq(1 - \gamma) \cdot \beta(\n^\star) \cdot \int_{\phi(B_t)} \big|J\phi(\r)\big|\d \sH(\r)
        \\&\geq (1 - 2\gamma) \cdot \beta(\n^\star) \cdot \sH\big(\phi(B_t)\big)\vphantom{\int_\phi}
        \\&\geq(1 - 2 \gamma) \cdot \beta(\n^\star) \cdot \sH\big(\tilde{B}_{t,\varepsilon}\big),
    \end{align*}
    where the first two inequalities follow from \cref{eq:lowerbounds} 
    and the fact that $(1 - \gamma)(1 - \gamma') \geq 1 - (\gamma + \gamma')$ whenever $\gamma, \gamma' \geq 0$, and the last inequality uses the fact that $\phi(B_t)$ contains $\tilde{B}_{t,\varepsilon}$. 
    Now  \Cref{lem:volume-simplex} implies that
    \begin{align*} 
        \sH(\tilde{B}_{t,\varepsilon}) &= \frac{1}{(K-1)!} \frac{(\mathbf{1} \cdot \n^\star)^{K-1}}{\prod_{k \in [K]} n_k^\star} \cdot \left(\frac{t}{1 + \varepsilon}\right)^{K-1} \geq (1 - \gamma) \cdot \frac{1}{(K-1)!} \frac{(\mathbf{1} \cdot \n^\star)^{K-1}}{\prod_{k \in [K]} n_k^\star} \cdot t^{K-1},
    \end{align*}
    where the inequality holds for $\varepsilon < \log (1/\gamma)/ K$, since $1/(1+\varepsilon)^{K-1} \geq 1/e^{\varepsilon \cdot (K-1)}$. Continuing from the above inequality, we obtain:
    \begin{align*} 
        \mu(B_t) &\geq (1 - 2 \gamma) \cdot \beta(\n^\star) \cdot \sH\big(\tilde{B}_{t,\varepsilon}\big)
        \\&\overset{\phantom{(iv)}}{=} (1 - 2 \gamma) \cdot\beta(\n^\star)  \cdot  (1 - \gamma) \cdot \frac{1}{(K-1)!} \frac{(\mathbf{1} \cdot \n^\star)^{K-1}}{\prod_{k \in [K]} n_k^\star} \cdot t^{K-1}
        \\&\geq (1 - 3\gamma) \cdot \theta_K \cdot t^{K-1},
    \end{align*}
    where the last inequality again uses the fact that $(1 - \gamma)(1 - \gamma') \geq 1 - (\gamma + \gamma')$.
    Finally, \Cref{thm:limiting-distribution} also shows that:
    \[\lim_{t \to 0} \, \Npar(t,\cF, \bR) \cdot t^{K-1} = \mu(\fR),\]
    so that for sufficiently small $t$, we have that $\mu(\fR) \leq (1 + \gamma) \cdot \Npar(t,\cF, \bR) \cdot t^{K-1}$. We obtain that:
    \[\frac{\mu(B_t)}{\mu(\fR)} \geq \frac{(1 - 3 \gamma) \cdot \theta_K \cdot t^{K-1}}{(1 + \gamma) \cdot \Npar(t, \cF, \fR) \cdot t^{K-1}} \geq (1 - 4 \gamma) \cdot \frac{\theta_K}{\Npar(t, \cF, \fR)},\]
    where we use the fact that for $\gamma < 1$, that $1/(1 + \gamma)\geq 1 -\gamma$. 
    
    The result then follows from the lower bound $\theta_K \geq 1$, and reparametrizing $\gamma$.
\end{proof}

We now prove the bound for Pareto EW from \cref{thm:Geometry-complexity-upper-bound}.
Fix \(t\in(0,1]\) and recall $\cN_t$ from \Cref{eq:Nt-def}. We just established by \Cref{lem:mass-model-ball} that the mass $\pi(\cN_t)$ is bounded away from zero. 
Recall the definition of Pareto EW from \cref{subsec:Pareto-ERM} as the aggregated model $f_{\rhohat}=\EE_{f\sim \rho}[f]$ using the Gibbs posterior $\rhohat$ from \cref{eq:ew-posterior}.
We now show that for any \(\delta\in(0,1)\) and $\lambda>0$, with probability at least \(1-\delta\)
\begin{equation}
\label{eq: fixed lambda EW bound}
    \cR_Q(f_{\rhohat})-\inf_{f\in\cF}\cR_Q(f)
\le \uppermodulus(t)
+\frac{\lambda}{4n}
+\frac{2\log\prn{2/(\delta \pi(\cN_t)})}{\lambda}.
\end{equation}

Indeed, by Theorem~4.2 in~\citet{alquier2024user}, the exponential weights estimator satisfies with probability at least \(1-\delta\) and for every \(\lambda>0\),
\begin{equation*}
\EE_{f\sim \rhohat}\!\big[\mathcal R_Q(f)\big]
\le 
\inf_{\rho\ll\pi}\left\{
\EE_{f\sim\rho}[\mathcal R_Q(f)]
+\frac{\lambda}{4n}
+\frac{2\big(\KL(\rho,\pi)+\log(\tfrac{2}{\delta})\big)}{\lambda}
\right\}.
\end{equation*}

Now, define the localized distribution
\begin{equation*}
    \rho_t(A) := \pi(A \mid \mathcal N_t) = \frac{\pi(A\cap \mathcal N_t)}{\pi(\mathcal N_t)}.
\end{equation*}
Then $\rho$ is absolutely continuous with respect to $\pi$ and 
$
\KL(\rho_t,\pi) =\int \log\left(\frac{\d\rho_t}{\d\pi}\right)\d\rho_t =\log\left(\frac{1}{\pi(\cN_t)}\right).
$
Then the infimum is upper bounded by the evaluation at this specific $\rho_t$ so that we obtain
\begin{equation*}
\EE_{f\sim \rhohat}\brk{\mathcal R_Q(f)}
\leq 
\EE_{f\sim\rho_t}[\mathcal R_Q(f)]
+\frac{\lambda}{4n}
+\frac{2\log\prn{2/(\delta\pi(\cN_t))}}{\lambda}.
\end{equation*}

Next, since \(\rho_t\) is supported on \(\mathcal N_t\), by the definition of the upper modulus (Definition~\ref{def:modulus-of-monotonicity}, applied with comparator \(f^\star\)) we have for all \(f\in\mathcal N_t\), $\cR_Q(f)\le \cR_Q(\fstar)+\uppermodulus(t)$,
and therefore
\begin{equation*}
\EE_{f\sim\rho_t}[\mathcal R_Q(f)]
\le 
\mathcal R_Q(f^\star)+\uppermodulus(t).
\end{equation*}
Combining the last two displays yields
\begin{equation*}
\EE_{f\sim \rhohat}\brk{\mathcal R_Q(f)}
\leq 
\cR_Q(\fstar)+\uppermodulus(t)
+\frac{\lambda}{4n}
+\frac{2\log\prn{2/(\delta\pi(\mathcal N_t))}}{\lambda},
\end{equation*}
for all \(\lambda>0\) and with probability at least $1-\delta$. 
Finally, since the loss defining \(\mathcal R_Q\) is convex in the prediction, Jensen's inequality implies that the aggregated predictor satisfies
$
\cR_Q(f_{\rhohat})=\cR_Q\prn{\EE_{f\sim \rhohat}[f]}\le\EE_{f\sim \rhohat}\brk{\mathcal R_Q(f)}.
$
Therefore the same bound holds with \(\cR_Q(f_{\rhohat})\) on the left-hand side, and recalling that \(\inf_{f\in\mathcal F}\mathcal R_Q(f)=\mathcal R_Q(f^\star)\) concludes the proof of  \cref{eq: fixed lambda EW bound}.

Finally, recall that by \Cref{lem:mass-model-ball} we have that for a sufficiently small $t_0 > 0$  
\[\pi(\cN_t) \geq \frac{1}{2 \Npar(t, \cF, \fR)}, \qquad \forall t \in (0, t_0].\]
We can hence upper bound the log term $\log\prn{\frac{2}{\delta\pi(\cN_t)}}\leq \log\prn{\frac{4\,N_{\mathrm{Par}}(t,\cF,\mathbf R)}{\delta}}$.
Finally, optimizing the right-hand side of \cref{eq: fixed lambda EW bound} (where we upper bounded as above the $\log$ term) over $\lambda>0$ yields the choice
\[
\lambda=\sqrt{8n\log\!\Big(\tfrac{4\Npar(t,\cF,\bR)}{\delta}\Big)},
\]
which gives the bound for the Pareto EW stated in \cref{thm:Geometry-complexity-upper-bound}.

\subsection{Proof of \texorpdfstring{\cref{cor:tstar}}{Corollary \ref{cor:tstar}}}
\label{proof:tstar}

The first part of the proof is a direct consequence of \cref{thm:Geometry-complexity-upper-bound}. Indeed, recall that $t_n$ is defined in \cref{cor:tstar} as
\begin{equation*}
    t_n =\inf\crl{t>0: t^2\geq \frac{2\log\prn{2\Npar(t,\cF,\bR)}}{L^2 n}}.
\end{equation*}
By definition of the Pareto covering number, the function $t\mapsto 2\log\prn{2\Npar(t,\cF,\bR)}/(L^2 n)$ is non-increasing, and so for any $t>t_n$, it holds that
\begin{equation*}
    t^2\geq \frac{2\log\prn{2\Npar(t,\cF,\bR)}}{L^2 n}.
\end{equation*}
Combining this with \cref{thm:Geometry-complexity-upper-bound}, we obtain that, for any $t> t_n$,
\begin{equation*}
    \cR_Q(\Paretoerm{t})-\inf_{f\in\cF}\cR_Q(f) \leq \uppermodulus(t) + \sqrt{\frac{2\log (2\Npar(t,\cF,\bR)/\delta)}{n}} \leq Lt+Lt+\sqrt{\frac{2\log(1/\delta)}{n}}.
\end{equation*}
As this is valid for all $t> t_n$, we can conclude the first part of the proof.

For the second part, note that the assumption $L^2n>2\log 2$ together with the fact that $\bR(\cF)\subset [0,1]^K$ implies that
\begin{equation*}
    \frac{2\log(2\Npar(1,\cF,\bR))}{L^2n} =\frac{2\log 2}{L^2 n}\leq 1.
\end{equation*}
As a consequence, $t_n\leq 1$ and we can make the following case distinction based on the value of
$\theta := \frac{K-1}{L^2n}\log\prn{1+\frac{16L^2 n}{K-1}}$.
\begin{itemize}
    \item If $\theta \geq 1$ the claimed bound from \cref{cor:tstar} becomes $t_n^2 \leq 1\wedge \theta = 1 $, which we already verified.
    \item If $\theta <1$, \cref{lem:linf-bound-covering} implies that $\log \Npar(t,\cF,\bR)\leq (K-1) \log(1+\frac{1}{t})$. Thus, any solution to 
    \begin{equation*}
        t^2 \geq \frac{2\log\prn{2(1+1/t)^{K-1}}}{L^2 n}
    \end{equation*}
    also upper bounds $t_n$. Now, write $a=2(K-1)/(L^2n)$ and $b=2^{K/(K-1)}$. We claim that $t^2=\frac{a}{2}W\prn{\frac{2b^2}{a}}$ satisfies the inequality above, where $W$ denotes the Lambert $W$-function. Indeed,
    \begin{equation}
    \label{eqn:bound-t-squared}
        t^2=\frac{a}{2}W\prn{\frac{2b^2}{a}} \leq \frac{a}{2} \log\prn{1+\frac{2b^2}{a}}=\frac{K-1}{L^2n}\log\prn{1+\frac{L^2 n2^{2\tfrac{K}{K-1}}}{K-1}} \leq \theta,
    \end{equation}
    where we used $W(x)\leq \log(1+x)$ for $x>0$. Furthermore, $\theta<1$ implies $t<1$ and, consequently, 
    \begin{equation*}
        \frac{2\log\prn{2(1+1/t)^{K-1}}}{L^2 n} \leq \frac{2\log\prn{2^{K}/t^{K-1}}}{L^2 n}=\frac{2(K-1)}{L^2n}\log\prn{\frac{2^{K/(K-1)}}{t}} = a\log \prn{\frac{b}{t}}.
    \end{equation*}
    Finally, by definition of the Lambert $W$-function, we know that our choice of $t^2$ satisfies $t^2= a\log(b/t)$, and the claim follows. Hence, by \cref{eqn:bound-t-squared} we have
    \begin{equation*}
        t_n^2 \leq t^2 \leq \theta= \frac{K-1}{L^2n}\log\prn{1+\frac{16L^2 n}{K-1}}.
    \end{equation*}
    \end{itemize}

Putting both cases together yields the bound $t_n^2 \leq 1 \wedge \theta$, concluding the proof. Note that, by going through the same calculations, one can derive an analogous bound for the Pareto EW.

\subsection{Proof of \texorpdfstring{\cref{thm:Geometry-complexity-lower-bound}}{Theorem \ref{thm:Geometry-complexity-lower-bound}}}
\label{proof:Geometry-complexity-lower-bound}

The high-level road map for the proof of the lower bound is the following. We construct a set of $N$ points $\r\in[0,1]^K$, which allows us to choose a function $f$ such that $\bR(f)=\r$ for each point $\r$. We choose the points so that for any pair $\r,\r'$, neither point lies in the $t$-Pareto ball centered at the other, that is, their pairwise Pareto distance exceeds $t$. Further, we add two dominated points separated by exactly $t\one$ along the diagonal. These points do not change the $t$-Pareto covering number, but they force the modulus at scale $t$ to be exactly $t$. Finally, we construct a loss and a family of distributions $\cQ$ in a way that for each point $\r$, there is a corresponding distribution $Q\in\cQ$ under which the associated hypothesis is the unique risk minimizer and every other point has excess risk at least $C t$. Estimating a nearly optimal model therefore induces a multiple testing problem, which is lower bounded by a standard application of Fano's inequality.

We begin by noting that, by the assumptions on $n$ and $t$, we have $t\leq 1/60$.
Let $m\in[8]$ and define the values $v_m = \frac{1}{2}+8t\prn{m-\tfrac{9}{2}}$. Then, $v_m\in [0,1]$ since $t\leq 1/56$ and $\abs{v_m-\frac{1}{2}}\leq 8t\frac{7}{2}=28t\leq \frac{1}{2}$.
For $\a\in [8]^{K-1}$, define the points in risk space 
\begin{equation*}
    \r(\a):= \prn{v_{a_1},\ldots, v_{a_{K-1}},\tfrac{1}{2}}\in [0,1]^{K}.
\end{equation*}
For the sum $S(\a)=\sum_{i=1}^{K-1}a_i$, let $A_s = \crl{\a\in[8]^{K-1}:S(\a)=s}$; if $\a\prec \a'$ then clearly $S(\a)<S(\a')$. Thus, for all $\a,\a'\in A_s$ it cannot be that $\a\neq \a'$ and $\a\preceq \a'$. Since there are exactly $8^{K-1}$ different $\a\in[8]^{K-1}$ and exactly $7(K-1)+1$ different values that $S$ can take (which are $K-1,K,\ldots,8(K-1)$), by the \emph{pigeonhole principle} there must be at least one $s$ such that
\begin{equation*}
    \abs{A_s} \geq \frac{8^{K-1}}{7(K-1)+1} \geq \exp\prn{K-1} \geq N,
\end{equation*}
where in the second inequality we used that $K\geq 4$ and in the third the assumption on $N$.
Let $A\subset A_s$ be any subset of size exactly $N$. Then, for any $\a\neq \a'\in A$, it holds that $\a\not \preceq \a'$, and there is one coordinate $j\leq K-1$ with $a_j>a'_j$. As a consequence
\begin{equation*}
    r_j(\a)-r_j(\a')=v_{a_j}-v_{a_j'}=8t(a_j-a_j')\geq 8t.
\end{equation*}
Now, by definition of Pareto distance (cf. \cref{def:Pareto-distance}), it follows that
\begin{equation*}
    d(\r(\a),\r(\a')) = \min\crl{t'\geq 0: \r(\a)-t'\one\preceq \r(\a')} \geq 8t
\end{equation*}
and hence no $t$-Pareto ball can cover any two of the points $\crl{\r(\a):\a\in A}$. In particular, we have constructed a set that has $t$-Pareto covering number $\Npar(t,\crl{\r(\a):\a\in A})=N$. 

Let us introduce the hypothesis space $\cF=\crl{f_\a:\a\in A}\cup\crl{g_0,g_1}$ so that $\abs{\cF}=N+2$. Define $\bR:\cF\to[0,1]^K$ as
\begin{align*}
    \bR(f_\a) = \r(\a), \qquad \bR(g_0) = \prn{\frac{1}{2}+29t }\one, \qquad \text{and} \qquad \bR(g_1) = \prn{\frac{1}{2}+30t }\one.
\end{align*}
By construction, every $f_\a$ is Pareto optimal, and both $g_0$ and $g_1$ are dominated by all $f_\a$, implying that we retain $\Npar(t,\cF,\bR)=N$.

Let us now introduce the coordinate-wise independent distributions $\cQ=\crl{Q_\q : \q\in A}$ via $Z\in \cZ:=\crl{0,1}^{A\cup\crl{0,1}}$ that, for all $\a\in A$,
\begin{equation*}
    Z_{\a} \sim \Ber(\theta_{\q,\a}) \quad \text{with} \quad \theta_{\q,\a} = 
    \begin{cases}
        \frac{1}{2}-2t & \a=\q, \\
        \frac{1}{2} & \a\neq \q.
    \end{cases}
    \quad \text{and} \quad Z_0 \sim \Ber\prn{\frac{1}{2}+29t}, \quad Z_1 \sim \Ber\prn{\frac{1}{2}+30t}.
\end{equation*}
Moreover, let us denote the risks $\cR_\q:= \cR_{Q_\q}$ through the loss on $\cF\times \cZ$
\begin{equation*}
    \ell(f,z) = \begin{cases}
    z_\a & f=f_\a, \\
    z_0 & f=g_0, \\
    z_1 & f=g_1. \\
    \end{cases}
    \implies
    \cR_\q (f) = 
    \begin{cases}
        1/2 -2t & f=f_\q \\
        1/2 & f=f_\a , \a\neq \q \\
        1/2 +29t & f=g_0 \\
        1/2 + 30t & f=g_1
    \end{cases}
\end{equation*}
Importantly, the loss is independent of the distribution $Q$, and between $0$ and $1$. By comparing each ordered pair in $\cF$ we see that the only one within margin $t$ is $(g_1,g_0)$; $\bR(g_1)=\bR(g_0)+t\one$. Since $\cR_Q(g_1)-\cR_Q(g_0)=t$, this verifies that $\uppermodulus(t)=t$ for all $Q$.

In the following lower bound we consider proper estimators $\fhat\in \cF$, but notice that we could set the loss to $1$ outside of $\cF$, so that any improper estimator must have excess risk at least $1/2$.

Let $\fhat\in \cF$ be any estimator and define the test
\begin{equation*}
    \bqhat =
    \begin{cases}
        \a & \fhat=f_\a \\
        \a\sim \uniform{A}  & \fhat\in \crl{g_0,g_1} 
    \end{cases}.
\end{equation*}
Then, by design, on the event that $\bqhat \neq \q$
\begin{equation*}
    \cR_\q(\fhat) -\min_{f\in \cF}\cR_\q(f)\geq \frac{1}{2}-\prn{\frac{1}{2}-2t}=2t.
\end{equation*}
Therefore,
\begin{equation*}
    \sup_\q\PP\prn{\cR_\q(\fhat) -\min_{f\in \cF}\cR_\q(f)\geq 2t} \geq \sup_\q\PP\prn{\bqhat\neq \q}.
\end{equation*}
We are now ready to apply Fano's inequality.
\begin{lemma}[Fano's inequality \citep{yu1997assouad}]
\label{lem:Fano}
    Let $J\sim \uniform{[N]}$, and let $(Z|J=j)\sim Q_j$. Assume that $\KL(Q_i,Q_j)\leq \alpha$ for all $i,j$. Then, for any estimator $\psi(Z)$ computed from $Z$,
    \begin{equation*}
        \PP\prn{\psi(Z)\neq J} \geq 1-\frac{\alpha+\log 2}{\log N}.
    \end{equation*}
\end{lemma}
In order to apply \Cref{lem:Fano}, we bound the KL-divergence between the $Q_\q$ we constructed. By \citet[Lemma D.2]{wegel2025sample}
, we have
\begin{align*}
    \KL(Q_\q^{\otimes n},Q_{\q'}^{\otimes n}) &= n\prn{\KL(\Ber(\tfrac{1}{2}-2t),\Ber(\tfrac{1}{2}))+\KL(\Ber(\tfrac{1}{2}), \Ber(\tfrac{1}{2}-2t))} \\
    &\leq n\prn{4(2t)^2+4(2t)^2} = 32nt^2.
\end{align*}
As the supremum over $\q$ is lower bounded by the average over $\q$, \cref{lem:Fano} implies that
\begin{equation*}
    \sup_\q \PP(\bqhat \neq \q) \geq \frac{1}{N} \sum_{\q\in A} \PP(\bqhat \neq \q) \geq 1-\frac{32nt^2+\log 2}{\log N} \geq 1/2,
\end{equation*}
where we used that $t\leq \sqrt{\log(N/4)/(64 n)}$ and $N\geq 4$. Finally, it follows that
\begin{equation*}
    \sup_{\q\in A} \EE\brk{\cR_\q(\fhat) -\min_{f\in \cF}\cR_\q(f)} \geq 2t \sup_{\q}\PP\prn{\cR_\q(\fhat) -\min_{f\in \cF}\cR_\q(f)\geq 2t} \geq t,
\end{equation*}
which yields the result and concludes the proof of \cref{thm:Geometry-complexity-lower-bound}.

\subsection{Proof of \texorpdfstring{\cref{prop:monotonicity-above-threshold-adaptive}}{Proposition \ref{prop:monotonicity-above-threshold-adaptive}}}
\label{proof:monotonicity-above-threshold-adaptive}

Recall that $\AdaptiveThresholdERM = \Thresholderm{\tadapt}$. Therefore, to prove
\cref{prop:monotonicity-above-threshold-adaptive}, we may apply the same decomposition presented in the proof of \cref{prop:monotonicity-above-threshold}:
\begin{equation*}
    \cR_Q(\Thresholderm{\tadapt})-\cR_Q(\fstar) \leq 4\radComp(\ell\circ \Fpruned{\tadapt})+2\sqrt{\frac{2\log(4/\delta)}{n}}  + \underbrace{\inf_{f\in\Fpruned{\tadapt}}\cR_Q(f)-\cR_Q(\fstar)}_{=:\eps_n},
\end{equation*}
where, as opposed to the proof of \cref{prop:monotonicity-above-threshold}, we do not have a guarantee that $\eps_n=0$ for all $n$ since it is not guaranteed that $\tadapt>\tlower$ for all $n$.
We can further bound the Rademacher complexity of a finite function class using Massart's Lemma \citep[Lemma 26.8]{shalev2014understanding} as
\begin{equation*}
    4\radComp(\ell\circ \Fpruned{\tadapt})+2\sqrt{\frac{2\log(4/\delta)}{n}} \leq 4\sqrt{\frac{2\log (2\abs{\Fpruned{\tadapt}})}{n}}+ 2\sqrt{\frac{2\log(4/\delta)}{n}} \leq 40\sqrt{\frac{\log(8\abs{\Fpruned{\tadapt}}/\delta)}{n}}.
\end{equation*}

It remains to argue that $\eps_n=0$ if $n\geq n_0:=\min_{\gamma>\tlower}\abs{\Fpruned{\gamma}}$. As demonstrated in the proof of \cref{prop:monotonicity-above-threshold}, for every $\gamma>\tlower$, we know that $\inf_{f\in\Fpruned{\gamma}} \cR_Q(f) = \inf_{f\in \cF}\cR_Q(f)$. Thus, it is enough if the adaptive set $\Fpruned{\tadapt}$ contains $\Fpruned{\gamma}$ for some $\gamma>\tlower$.

By construction, $\tadapt = \gamma_{(m_n)}$, where recall that
\begin{equation*}
    m_n = \min\crl{\abs{\cF},\max\crl{n,\abs{\pareto(\cF,\bR)}}}.
\end{equation*}
Hence either $\tadapt=0$ and $\Fpruned{\tadapt}=\pareto(\cF,\bR)$, or $\tadapt>0$ and $\Fpruned{\tadapt}=\crl{f\in\cF:\Gamma(f)\leq \gamma_{(m_n)}}$ keeps all models up to the $m_n$-th smallest margin (including ties).
Since $\gamma\mapsto \abs{\Fpruned{\gamma}}$ is an increasing step function of $\gamma$, we know that $n_0=\min_{\gamma>\tlower}\abs{\Fpruned{\gamma}}$ is well-defined and attained by some $\gamma^\star$.
Moreover, because $\pareto(\cF,\bR)\subseteq \Fpruned{\gamma^\star}\subseteq \cF$ for any $\gamma^\star >\tlower$, we know that $n\geq \abs{\Fpruned{\gamma^\star}}$ implies the second inequality in $\abs{\Fpruned{\gamma_{(m_n)}}}\geq m_n\geq \abs{\Fpruned{\gamma^\star}}$, and therefore $\gamma_{(m_n)}\geq \gamma^\star$ which implies $\Fpruned{\gamma^\star}\subseteq \Fpruned{\gamma_{(m_n)}}=\Fpruned{\tadapt}$ by nestedness. As argued above, this yields the claim.

\subsection{Proof of \texorpdfstring{\cref{thm:fast-rate-MSA}}{Theorem \ref{thm:fast-rate-MSA}}}
\label{proof:fast-rate-MSA}

We demonstrate that the proof of \cref{thm:fast-rate-MSA} follows almost verbatim from the proof of the $Q$-aggregation estimator from \citet{lecue2014optimal}. There are, however, some important differences that we highlight along the way.

In this section we use the shorthand notation $\cR= \cR_Q$, $\Rhat=\Rhat_Q$ and $\nrm{f}_2=\nrm{f}_{L^2(Q_X)}$, where $Q$ remains the target distribution. Moreover, we associate the simplex $\triangle^{K-1}$ with the set of all distributions $\cP([K])$ through the canonical map, and switch between both notations without explicitly stating it.

Define the function $\Rtilde(\rho) = \frac{1}{2}\cR(\phi_\rho)+\frac{1}{2} \EE_{k\sim\rho} \cR(f_k)$ and denote $\elltilde_{\rho}(x,y)=\frac{1}{2}\ell(\phi_\rho(x),y)+\frac{1}{2}\EE_{k\sim \rho}\ell(f_k(x),y)$. Note that $\Rtilde(\rho)= Q\elltilde_\rho$ where we use the standard empirical process theory notation $ Qf = \int f\d Q$.
We actually prove the result for a generalized version of the estimator using a prior $\pi\in\cP([K])$ defined as
\begin{equation*}
    \rhofast \in \argmin_{\rho\in\cP([K])} \Fhat(\rho) \quad \text{where} \quad \Fhat(\rho)= \frac{1}{2}\Rhat_Q(\phi_\rho)+ \frac{1}{2}\EE_{k\sim \rho}\Rhat_Q(f_k)+\frac{K(\rho,\pi)}{\lambda n},
\end{equation*}
where $K(\rho,\pi)= \sum_{k=1}^K \rho_k \log \prn{1/\pi_k}$. It reduces to the estimator in \cref{subsec:fast-rates-aggregation} by taking the uniform prior $\pi_k=1/K$, since then $K(\rho,\pi) = \log(K) \sum_{k=1}^K \rho_k=\log K$ is independent of $\rho$ and does not affect the optimization problem.
We use the shorthand $\rhohat=\rhofast$ throughout this section.
Define the population version of $\Fhat$ and the comparator $\rhostar$ as
\begin{equation*}
    \rhostar \in \argmin_{\rho\in \cP([K])}\crl{F(\rho)+\mu V(\rho)} \quad \text{where} \quad F(\rho) = \frac{1}{2}\cR(\phi_\rho)+ \frac{1}{2}\EE_{k\sim \rho}\cR(f_k)+\frac{K(\rho,\pi)}{\lambda n}
\end{equation*}
so that $F(\rho)=\Rtilde(\rho)+\frac{K(\rho,\pi)}{\lambda n}$ and where we denote the variance $V(\rho):= \EE_{X\sim Q_X} \Var_{k\sim \rho} f_k(X) = \EE_{k\sim \rho}\nrm{f_k-f_{\rho}}_2^2$.

\subsubsection{Preliminary lemmata}
To start the proof, we need some lemmata that essentially demonstrate why the proof of the $Q$-aggregation estimator applies almost verbatim to the aggregator on the Pareto set. 
\begin{lemma}
\label{lem:upper-and-lower-modulus-equality}
    If $\lowermodulus(t)\geq t$ and $\uppermodulus(t) \leq t$, then $\Phi(\rho) = \cR(f_\rho)-\cR(\phi_\rho)$.
\end{lemma}
\begin{proof}
    Follows from the following sandwich inequalities: Since $\bR(\phi_\rho)= \bR(f_\rho)-\Phi(\rho)\one$, we have by definition of the moduli (\cref{def:modulus-of-monotonicity}) that
    \begin{equation*}
        \Phi(\rho) \leq \lowermodulus(\Phi(\rho)) \leq \cR(f_\rho)- \cR(\phi_\rho) \leq \uppermodulus(\Phi(\rho))\leq \Phi(\rho),
    \end{equation*}
    implying the equality.
\end{proof}
\cref{lem:upper-and-lower-modulus-equality} implies that we know exactly the improvement an estimator $\phi_\rho$ has over the estimator with $f_\rho$. Recall now that the strong concavity of the loss implies the following Jensen's gap.
\begin{lemma}[Proposition 2 in \citet{lecue2014optimal}]
\label{lem:Jensen-gap}
     If the loss $\ell(\cdot,y)$ is $\kappa$-strongly convex, then
\begin{equation*}
    \cR(f_\rho)\leq \EE_{k\sim \rho}\cR(f_k) -\frac{\kappa}{2}V(\rho).
\end{equation*}
Importantly, this is also true for $\kappa=0$, when the loss is merely convex.
\end{lemma}
Next, we show that the strong concavity of $\Phi$, ensured by \cref{asm:fast-rate-aggregation}, implies that we can effectively control the variance of the estimator.
\begin{lemma}
\label{lem:Delta-geq-Var}
    If $\Phi$ is $\eta$-strongly concave in $\nrm{f_\rho-f_\gamma}_{2}$, as defined in \cref{eqn:strong-concavity-gap}, then $\Phi(\rho)\geq \tfrac{\eta}{2}V(\rho)$. 
\end{lemma}
\begin{proof}
    To prove this, define the difference $D(\rho)=\Phi(\rho)-\tfrac{\eta}{2}V(\rho)$. Since $\Phi$ is $\eta$-strongly concave and $V$ is $2$-strongly concave, that is, 
    \begin{equation*}
        V(\alpha \rho+(1-\alpha)\gamma)=\alpha V(\rho)+(1-\alpha)V(\gamma)+\alpha(1-\alpha)\nrm{f_\rho-f_\gamma}_{2}^2,
    \end{equation*}
    we know that $D$ must be concave (e.g., by adding the two equations / inequalities). Let $\delta_j$ now denote the dirac delta on $j\in[K]$. Then, because $V(\delta_j)=\Phi(\delta_j)=0$ for all $j$, we know that $D(\delta_j)=0$ for all $j$. And so, by Jensen's inequality, for all $\rho\in \cP([K])$
    \begin{equation*}
        \Phi(\rho)-\tfrac{\eta}{2}V(\rho)=D(\rho)=D\prn{\sum_{k\in[K]}\rho_k\delta_k}\geq \sum_{k\in[K]}\rho_k D(\delta_k)=0 .
    \end{equation*}
    This yields the claim.
\end{proof}
Combining \cref{lem:upper-and-lower-modulus-equality,lem:Jensen-gap,lem:Delta-geq-Var} yields the proof of \cref{eqn:doubly-variance-reduced}:
\begin{equation*}
    \cR(\phi_\rho) = \cR(f_\rho)-\Phi(\rho)\leq  \EE_{k\sim \rho}\cR(f_k) -\frac{\kappa}{2}V(\rho)-\Phi(\rho) \leq  \EE_{k\sim \rho}\cR(f_k) -\frac{\kappa+\eta}{2}V(\rho).
\end{equation*}
This provides some intuition why the proof works; even when the loss is not strongly convex ($\kappa=0$), the estimator can penalize variance. Next, we prove the key property gained from the strongly convex objective. This is analogous to Proposition 4 in \citet{lecue2014optimal}.

\begin{lemma}
\label{lem:Rtilde-lower-bound}
    Let $\mu < (\eta +\kappa)/4$. It holds for any $\rho$ that
    \begin{equation*}
        \Rtilde(\rho)-\Rtilde(\rhostar)\geq -\frac{1}{\lambda n}(K(\rho,\pi)-K(\rhostar,\pi)) -\mu(V(\rho)-V(\rhostar)) +\prn{\frac{\eta+\kappa}{4}-\mu}\nrm{f_\rho-f_{\rhostar}}_2^2
    \end{equation*}
\end{lemma}
\begin{proof}
    First note that from the definition,
    \begin{equation*}
        F(\rho) = \Rtilde(\rho) +\frac{K(\rho,\pi)}{\lambda n} = \frac{1}{2}\cR(\phi_\rho)+ \frac{1}{2}\EE_{k\sim \rho}\cR(f_k)+\frac{K(\rho,\pi)}{\lambda n}
    \end{equation*}
    is $(\eta+\kappa)/2$-strongly convex with respect to $\nrm{f_\rho-f_\gamma}_2$: Since $\rho\mapsto \Phi(\rho)$ is $\eta$-strongly concave and $\rho\mapsto \cR(f_\rho)$ is $\kappa$-strongly convex (Proposition 3 in \cite{lecue2014optimal}) we have that $\rho\mapsto \cR(\phi_\rho) = \cR(f_\rho)-\Phi(\rho)$ is $\eta+\kappa$ strongly convex in $\nrm{f_\rho-f_\gamma}_2^2$, where we used \cref{lem:upper-and-lower-modulus-equality} for the equality. Hence, $\rho\mapsto \Rtilde(\rho)$ is $(\eta+\kappa)/2$-strongly convex. Finally $\rho\mapsto K(\rho,\pi)$ is convex.
    
    Moreover, since $V$ is $2$-strongly concave, that is, 
    \begin{equation*}
        V(\alpha \rho+(1-\alpha)\gamma)=\alpha V(\rho)+(1-\alpha)V(\gamma)+\alpha(1-\alpha)\nrm{f_\rho-f_\gamma}_2^2,
    \end{equation*}
    we have that $F_\mu(\rho):= F(\rho)+\mu V(\rho)$ is $(\eta+\kappa)/2-2\mu$-strongly convex.
    This strong convexity implies that 
    \begin{equation*}
        F_\mu(\alpha\rho+(1-\alpha)\rhostar)\leq \alpha F_\mu(\rho)+(1-\alpha)F_\mu(\rhostar)-\prn{\frac{\eta+\kappa}{4}-\mu}\alpha(1-\alpha)\nrm{f_\rho-f_{\rhostar}}_2^2.
    \end{equation*}
    Now, bounding $F_\mu(\rhostar)\leq F_\mu(\alpha\rho+(1-\alpha)\rhostar)$, subtracting $(1-\alpha)F_\mu(\rhostar)$ from both sides and dividing by $\alpha$ yields
    \begin{equation*}
        F_\mu(\rhostar) \leq F_\mu(\rho) -\prn{\frac{\eta+\kappa}{4}-\mu}(1-\alpha)\nrm{f_\rho-f_{\rhostar}}_2^2.
    \end{equation*}
    Taking $\alpha \to 0$ yields that $F_\mu(\rho)-F_\mu(\rhostar)\geq \prn{\frac{\eta+\kappa}{4}-\mu}\nrm{f_\rho-f_{\rhostar}}_2^2$. Therefore, we get that
    \begin{align*}
        F_\mu(\rho)-F_\mu(\rhostar)&= \mu(V(\rho)-V(\rhostar))+
        (\Rtilde(\rho)-\Rtilde(\rhostar))+\frac{1}{\lambda n}(K(\rho,\pi)-K(\rhostar,\pi)) \\
        &\geq \prn{\frac{\eta+\kappa}{4}-\mu}\nrm{f_\rho-f_{\rhostar}}_2^2
    \end{align*}
    Rearranging yields the result.
\end{proof}
\subsubsection{Main proof} We now come to the main proof. We begin by extracting the main empirical error term deterministically. This is analogous to Proposition 5 in \citet{lecue2014optimal}. Let $Q_n$ denote the empirical measure, and for some parameter $s>0$, define the random error term
\begin{equation*}
    Z_n=(Q-Q_n)(\elltilde_{\rhohat}-\elltilde_{\rhostar})-\mu \prn{2V(\rhohat)+V(\rhostar)+2\nrm{f_{\rhohat}-f_{\rhostar}}_2^2}-\frac{2}{s}K(\rhohat,\pi).
\end{equation*}
Notably, recall that the loss $\elltilde_\rho$ was defined via our map $\phi_\rho$. Hence, $Z_n$ differs from the error term in \citet{lecue2014optimal} via the loss functions and via the choice of $\rhohat$.
\begin{proposition}
\label{prop:Q-aggregation-error-pullout}
    Assume that $\mu \leq \tfrac{1}{20}(\kappa+\eta)$, $\frac{1}{\lambda n} \geq \frac{4}{s}$.
    Then it holds that
    \begin{equation*}
        \cR(\phi_{\rhohat})\leq \min_{k\in[K]} \crl{\cR(f_k)+\frac{\log (1/\pi_k)}{\lambda n}} + 2Z_n
    \end{equation*}
\end{proposition}
\begin{proof}
    From the definition of $\rhohat$ we know that $\Fhat(\rhohat)\leq \Fhat(\rhostar)$, spelled out being
    \begin{equation*}
        Q_n \elltilde_{\rhohat} +\frac{K(\rhohat,\pi)}{\lambda n} \leq Q_n \elltilde_{\rhostar} +\frac{K(\rhostar,\pi)}{\lambda n} 
    \end{equation*}
    which we can rearrange and add $\Rtilde(\rhohat)-\Rtilde(\rhostar)=Q(\elltilde_{\rhohat}-\elltilde_{\rhostar})$ on both sides to obtain
    \begin{align*}
        \Rtilde(\rhohat)-\Rtilde(\rhostar) &\leq (Q-Q_n)(\elltilde_{\rhohat}-\elltilde_{\rhostar})+\frac{1}{\lambda n}(K(\rhostar,\pi)-K(\rhohat,\pi))  \\
        &= Z_n + \mu \prn{2V(\rhohat)+V(\rhostar)+2\nrm{f_{\rhohat}-f_{\rhostar}}_2^2}+\frac{2}{s}K(\rhohat,\pi)  +\frac{1}{\lambda n}(K(\rhostar,\pi)-K(\rhohat,\pi))
    \end{align*}
    where the second equality is simply the definition of $Z_n$.
    Combining this with the lower bound from \cref{lem:Rtilde-lower-bound} and cancelling $\frac{1}{\lambda n}(K(\rhostar,\pi)-K(\rhohat,\pi))+\mu V(\rhostar)$ on both sides yields that
    \begin{equation*}
        \prn{\frac{\eta+\kappa}{4}-3\mu}\nrm{f_{\rhohat}-f_{\rhostar}}_2^2 \leq  Z_n + 3\mu V(\rhohat)+\frac{2}{s}K(\rhohat,\pi)
    \end{equation*}
    which we can plug back in using $\prn{\frac{\eta+\kappa}{4}-3\mu}^{-1}=\frac{4}{\eta+\kappa-12\mu}$ to get 
    \begin{align*}
        \Rtilde(\rhohat)-\Rtilde(\rhostar) &\leq Z_n +\mu(2V(\rhohat)+V(\rhostar))+\frac{2}{s}K(\rhohat,\pi)+ \frac{8\mu}{\eta+\kappa-12\mu} \prn{Z_n + \mu 3V(\rhohat)+\frac{2}{s}K(\rhohat,\pi)} \\
        &\qquad +\frac{1}{\lambda n}(K(\rhostar,\pi)-K(\rhohat,\pi)) \\
        &= \prn{1+\frac{8\mu}{\eta+\kappa-12\mu}}\prn{Z_n +\frac{2}{s}K(\rhohat,\pi)} + \prn{2\mu+\frac{24\mu^2}{\eta+\kappa-12\mu}}V(\rhohat) \\
        &\qquad +\frac{1}{\lambda n}(K(\rhostar,\pi)-K(\rhohat,\pi)) +\mu V(\rhostar).
    \end{align*}
    Now notice that by definition of $\Rtilde$, \cref{asm:fast-rate-aggregation} and \cref{lem:Delta-geq-Var,lem:Jensen-gap} we have that 
    \begin{align*}
        \Rtilde(\rho)&= \frac{1}{2}\cR(\phi_\rho)+\frac{1}{2}\EE_{k\sim \rho}\cR(f_k) \\
        &=\cR(\phi_\rho)-\frac{1}{2} (\cR(f_\rho)-\Phi(\rho)) +\frac{1}{2}\EE_{k\sim \rho}\cR(f_k) \\
        &\geq \cR(\phi_\rho)+\frac{ \eta}{4} V(\rho) + \frac{\kappa}{4} V(\rho) = \cR(\phi_\rho)+\frac{\eta+\kappa}{4} V(\rho),
    \end{align*}
    cf. \cref{eqn:doubly-variance-reduced}.
    It follows that
    \begin{align*}
        \cR(\phi_{\rhohat}) &\leq \Rtilde(\rhostar)+\mu V(\rhostar) + \frac{1}{\lambda n} K(\rhostar,\pi) + \prn{1+\frac{8\mu}{\eta+\kappa-12\mu}} Z_n \\
        &\qquad + \prn{2\mu+\frac{24\mu^2}{\eta+\kappa-12\mu}-\frac{\eta+\kappa}{4}}V(\rhohat) \\
        &\qquad + \prn{\prn{1+\frac{8\mu}{\eta+\kappa-12\mu}}\frac{2}{s}-\frac{1}{\lambda n}} K(\rhohat,\pi)
    \end{align*}
    Now, if the choice of parameters satisfies
    \begin{equation*}
        \quad 20\mu \leq \eta+\kappa, \quad \text{and} \quad \frac{1}{\lambda n} \geq \frac{4}{s}.
    \end{equation*}
    we can treat each term:
    $1+\frac{8\mu}{\eta+\kappa-12\mu} \leq 2$ since $\eta+\kappa-12\mu\geq 8\mu$.
    The second line is non-positive since $2\mu+\frac{24\mu^2}{\eta+\kappa-12\mu}-\frac{\eta+\kappa}{4} \leq 5\mu -\frac{\eta+\kappa}{4}\leq 0$. The third line is non-positive since $\prn{1+\frac{8\mu}{\eta+\kappa-12\mu}}\frac{2}{s}-\frac{1}{\lambda n}\leq \frac{4}{s}-\frac{1}{\lambda n}\leq 0$.

    It follows by definition of $\rhostar \in \argmin_{\rho} F(\rho)+\mu V(\rho) = \Rtilde(\rho)+\frac{1}{\lambda n}K(\rho,\pi)+\mu V(\rho)$ that
    \begin{align*}
        \cR(\phi_{\rhohat}) &\leq \Rtilde(\rhostar)+\mu V(\rhostar)  + \frac{1}{\lambda n} K(\rhostar,\pi) + 2Z_n \\
        &= \inf_{\rho} \crl{\Rtilde(\rho)+\mu V(\rho)+\frac{1}{\lambda n}K(\rho,\pi)} + 2Z_n \\
        &\leq \min_{k\in[K]} \Big\{ \underbrace{\Rtilde(\delta_{k})}_{=\cR(f_k)}+\underbrace{\mu V(\delta_{k})}_{=0}+\underbrace{\frac{1}{\lambda n}K(\delta_{k},\pi)}_{= \log (1/\pi_k)/(\lambda n)} \Big\} + 2Z_n.
    \end{align*}
    This yields the result of the proposition.
\end{proof}

Finally, we now show that a similar bound to that proved in \citet{lecue2014optimal} on ``their $Z_n$'' also applies to our $Z_n$. Again, note that while they look the same, they differ both in the loss function and the $\rhohat$. Luckily, this does not change the proof too much given the following assumptions. For completeness, we state the whole proof here.
\begin{remark}
    In the work by \citet{lecue2014optimal}, there is a mistake in this step. When deriving their Equations (3.7) and (3.8), a factor $1/2$ is missing in front of $K(\hat\theta)$. This can be fixed by using $-\frac{2}{s}K(\rhohat,\pi)$ instead of $-\frac{1}{s}K(\rhohat,\pi)$ in the definition of $Z_n$, assuming $\beta \geq 4n/s$ in their Proposition 5 and later adjusting the assumption on $\beta$ in the main result accordingly. We already applied this fix in our \cref{prop:Q-aggregation-error-pullout}.
\end{remark}
\begin{proposition}
\label{prop:Q-aggregation-Z-bound}
    Assume that all functions and $Y$ are in $[0,1]$, and that the loss and $\phi$ are Lipschitz in this sense: 
    \begin{equation*}
        \forall \yhat,\yhat'\in[0,1]:\, \abs{\ell(\yhat,y)-\ell(\yhat',y)}\leq L_{\ell} \abs{\yhat-\yhat'} \qquad \text{and} \qquad \forall x\in \cX:\, \abs{\phi_\rho(x)-\phi_{\gamma}(x)} \leq L_{\phi}\abs{f_{\rho}(x)-f_{\gamma}(x)}.
    \end{equation*}
    Assume that $\lambda n \leq s/4 $ and $s< \min\crl{n /\sqrt{3} L_\ell L_\phi, 8\mu n /(L_\ell (L_\ell +8\mu))} $. Then it holds that
    \begin{equation*}
        \PP\prn{Z_n \geq \frac{\log (1/\delta)}{\lambda n}} \leq \delta \quad \text{and} \quad \EE\brk{Z_n}\leq 0.
    \end{equation*}
\end{proposition}
\begin{proof}
    Thanks to Jensen's inequality and Chernoff's bound, it is sufficient to prove that $\EE\exp(\lambda n Z_n)\leq 1$, because then
    \begin{equation*}
        \PP(\lambda n Z_n> x) \leq \frac{\EE\exp(\lambda n Z_n)}{\exp x} \leq \exp(- x) \quad \text{and} \quad \EE\brk{Z_n} \leq \frac{1}{\lambda n }\log \EE \exp\prn{\lambda n Z_n} \leq0.
    \end{equation*}
    Denote $\ell_\rho(x,y)=\ell(\phi_\rho(x),y), \ \ell^{(k)}(x,y)=\ell(f_k(x),y)$ and write $Z_n$ as
    \begin{align*}
        Z_n &= (Q-Q_n)(\elltilde_{\rhohat}-\elltilde_{\rhostar})-\mu \prn{2V(\rhohat)+V(\rhostar)+2\nrm{f_{\rhohat}-f_{\rhostar}}_2^2}-\frac{2}{s}K(\rhohat,\pi) \\
        &=(Q-Q_n) \prn{\frac{1}{2}(\ell_{\rhohat}-\ell_{\rhostar})+\frac{1}{2} \sum_{k=1}^K (\rhohat_k-\rhostar_k)\ell^{(k)} 
        }-\mu \prn{\sum_{k\in[K]}\rhohat_k\nrm{f_k-f_{\rhostar}}_2^2+\rhohat \bH \rhostar}-\frac{2}{s}K(\rhohat,\pi) \\
        &= \underbrace{\frac{1}{2}(Q-Q_n) (\ell_{\rhohat}-\ell_{\rhostar})-\mu \sum_{k\in[K]}\rhohat_k\nrm{f_k-f_{\rhostar}}_2^2-\frac{1}{s}K(\rhohat,\pi)}_{=A_n}  \\
        &\quad + \underbrace{\frac{1}{2}(Q-Q_n) \sum_{k=1}^K (\rhohat_k-\rhostar_k)\ell^{(k)}-\mu\rhohat \bH \rhostar -\frac{1}{s}K(\rhohat,\pi)}_{=: B_n}
    \end{align*}
    where we used the $K\times K$ matrix $\bH_{j,k}=\nrm{f_j-f_k}_2^2$ and the identities
    \begin{align*}
        V(\rhohat) + \nrm{f_{\rhohat}-f_{\rhostar}}_2^2 &= \sum_{k\in [K]} \rhohat_k \nrm{f_k-f_{\rhostar}}_2^2, \\
        V(\rhohat)+V(\rhostar) +\nrm{f_{\rhostar}-f_{\rhohat}}_2^2 &= \rhohat \bH \rhostar,
    \end{align*}
    see Equations $(3.2)$ and $(3.3)$ from \citet{lecue2014optimal}.
    
    Apply Cauchy-Schwarz $\EE\exp \lambda n Z_n = \EE\exp \lambda n A_n \exp \lambda n B_n \leq \sqrt{\EE\exp 2\lambda n A_n}\sqrt{\EE\exp 2\lambda n B_n}$, and prove $\EE\exp 2\lambda n A_n \leq 1$ and $\EE\exp 2\lambda n B_n \leq 1$, respectively.

    Our assumptions ensure that $\abs{\ell_{\rhohat}-\ell_{\rhostar}}\leq L_{\ell}\abs{\phi_{\rhohat}-\phi_{\rhostar}}\leq L_\ell L_\phi \abs{f_{\rhohat}-f_{\rhostar}}$, and so symmetrization and contraction yields 
    \begin{align*}
        \EE\exp 2\lambda n A_n &\leq \EE \exp\prn{2\lambda n \max_{\rho\in \cP([K])}\crl{\frac{1}{2}(Q-Q_n) (\ell_{\rho}-\ell_{\rhostar})-\mu \sum_{k\in[K]}\rho_k\nrm{f_k-f_{\rhostar}}_2^2-\frac{1}{s}K(\rho,\pi)}} \\
        &\leq \EE \exp\prn{s \max_{\rho\in \cP([K])}\crl{L_\ell L_\phi Q_{\sigma,n} (f_{\rho}-f_{\rhostar})-\mu \sum_{k\in[K]}\rho_k\nrm{f_k-f_{\rhostar}}_2^2-\frac{1}{s}K(\rho,\pi)}}
    \end{align*}
    where we used that $2\lambda n \leq s$ and denote $Q_{\sigma,n}$ as the symmetrized measure with i.i.d. Rademacher variables. Note that the choice of $\rhostar$ is irrelevant for the rest of the proof; hence
    from here on, the proof is identical to \citet{lecue2014optimal} and we get that $\EE\exp 2\lambda n A_n\leq 1$ if $s< n /\sqrt{3} L_\ell L_\phi$.

    Moreover, by similar calculations as in Equation (3.15) in \citet{lecue2014optimal}, we have that
    \begin{align*}
        \EE\exp 2\lambda n B_n &\leq  \EE\exp \prn{  \sum_{k=1}^K \rhohat_k\sum_{j=1}^K \rhostar_j s \prn{\frac{1}{2}(Q-Q_n)(\ell^{(j)}-\ell^{(k)})-\mu\nrm{f_j-f_k}_2^2}-K(\rhohat,\pi)} \\
        &\leq \EE\exp \prn{ \max_{\rho\in\cP([K])}\crl{ \sum_{k=1}^K \rho_k \sum_{j=1}^K \rhostar_j s \prn{\frac{1}{2}(Q-Q_n)(\ell^{(j)}-\ell^{(k)})-\mu\nrm{f_j-f_k}_2^2}-K(\rho,\pi)}} \\
        &\leq  \sum_{k=1}^K \pi_k \EE \exp\prn{\sum_{j=1}^K \rhostar_j s \prn{\frac{1}{2}(Q-Q_n)(\ell^{(j)}-\ell^{(k)})-\mu\nrm{f_j-f_k}_2^2}} \\
        &\leq \sum_{k\in[K]}  \pi_k \sum_{j\in [K]} \rhostar_j \EE\exp \prn{s\prn{\frac{1}{2}(Q-Q_n)(\ell^{(j)}-\ell^{(k)})-\mu\nrm{f_j-f_k}_2^2}}
    \end{align*}
    where the first inequality follows from expanding the quadratic term $\rhohat \bH \rhostar$ and $2\lambda n \leq s$, the third one follows from 
    \begin{equation*}
        \max_{\rho\in\cP([K])} \crl{\sum_{k=1}^K \rho_k a_k -K(\rho,\pi)} = \max_{k\in[K]} \crl{a_k +\log \pi_k}, 
    \end{equation*}
    and the last one is Jensen's inequality.
    Since $f_{\delta_k}=f_k$, the Lipschitz assumption gives
    \begin{equation*}
        Q(\ell^{(j)}-\ell^{(k)})^2 \leq L_\ell^2 Q(f_j-f_k)^2 = L_\ell^2 \nrm{f_k-f_j}_2^2 \implies -\mu \nrm{f_k-f_j}_2^2 \leq -\frac{\mu}{L_\ell^2} Q(\ell^{(j)}-\ell^{(k)})^2.
    \end{equation*}
    Plugging this in and using Proposition 1 in \citet{lecue2014optimal} yields $\EE\exp 2\lambda n B_n \leq 1$ 
    whenever $s<8\mu n /(L_\ell (L_\ell +8\mu))$. That concludes the proof of this proposition.
\end{proof}
Combining \cref{prop:Q-aggregation-error-pullout} and \cref{prop:Q-aggregation-Z-bound} concludes the proof of \cref{thm:fast-rate-MSA} by noting that all conditions from \cref{lem:Rtilde-lower-bound,prop:Q-aggregation-error-pullout,prop:Q-aggregation-Z-bound} are satisfied whenever we have that
\begin{equation*}
     \mu \leq \tfrac{1}{20}(\eta+\kappa), \quad 1\leq  \lambda n \leq \frac{s}{4}<\frac{1}{4}\min\crl{\frac{n}{\sqrt{3} L_\ell L_\phi },\frac{8\mu n}{L_\ell (L_\ell +8\mu)}}.
\end{equation*}
One set of parameters that satisfies these constraints is
\begin{equation*}
    \mu = \frac{\eta+\kappa}{20}, \quad \text{and} \quad  \lambda = \frac{s}{4n}= \frac{1}{11L_\ell}  \min\crl{\frac{1}{L_\phi},\frac{\eta+\kappa }{ L_\ell + \eta+\kappa}}. 
\end{equation*}
Further, this not only yields the high-probability bound, but also a bound in expectation:
\begin{equation}
\label{eqn:fast-rate-expectation}
    \EE\brk{\excessRisk{\phi_{\rhohat}}{\crl{f_k}_{k=1}^K}} \leq \frac{\log(1/\pi_k)}{\lambda n}.
\end{equation}

\subsection{Proof of \texorpdfstring{\cref{thm:lower-bound-convex-hull}}{Theorem \ref{thm:lower-bound-convex-hull}}}
\label{proof:lower-bound-convex-hull}

The proof idea of \cref{thm:lower-bound-convex-hull} is to extend the setting of \cref{ex:fast-rate} by creating $d:=\floor{\log_2(K)}$ copies of it. To this end, choose $\cX=\crl{(i,j):i\in[d], j\in\crl{-1,1}}$ and, for some $\eps\in(0,1/4]$ to be determined later, let us define the family of distributions $\cQ = \crl{Q_{\bsigma}: \bsigma\in \crl{-1,+1}^d}$ via the random variables $(X,Y)\sim Q_{\bsigma}$ as
\begin{equation*}
     \PP(X=(i,-1))=\frac{1}{d}\prn{\frac{1}{2}+\sigma_i\eps}, \quad \PP(X=(i,1))=\frac{1}{d}\prn{\frac{1}{2}-\sigma_i\eps}, \quad (Y|X=(i,j)) = \one\crl{j=+1} .
\end{equation*}
We consider the absolute loss $\ell(f,(x,y))= \abs{y-f(x)}$ and the dictionary $\crl{f_{\btau}: \btau\in\crl{-1,+1}^d}$ with $f_{\btau}(i,j)= \one\crl{\tau_i=+1}$. $\btau\in\crl{\pm1}^d$ is the index of the dictionary and yields a dictionary of size $2^d=2^{\floor{\log_2(K)}}\in[K/2,K]$.
As the hypothesis space, we extend the dictionary to the following set
\begin{equation*}
    \cF= \conv\prn{\crl{f_{\btau}: \btau\in\crl{-1,+1}^d}} \cup \crl{g_{\bu}:\bu\in [0,1]^d},
\end{equation*}
with the functions $g_{\bu}$ defined as $g_{\bu}(i,-1) = u_i^2$ and $ g_{\bu}(i,+1)= 2u_i-u_i^2$. Finally, note that $\cR_{Q_{\bsigma}}=\sum_{\btau\in \crl{\pm}^d} q_{\btau}(\bsigma) \cR_{\btau}$. Here, $\cR_{\btau}(f) = \EE_{(X,Y)\sim P_{\btau}} \abs{f(X)-Y}$, where $P_{\btau}$ is defined by
\begin{equation*}
      \PP(X=(i,-1))=\frac{1}{d}\one\crl{\tau_i=-1}, \quad \PP(X=(i,+1))=\frac{1}{d}\one\crl{\tau_i=+1}, \quad (Y|X=(i,j)) = \one\crl{j=+1}. 
\end{equation*}
Indeed, for the weights defined as
\begin{equation*}
    q_{\btau}(\bsigma) = \prod_{i=1}^d \prn{\frac{1}{2}-\sigma_i\eps}^{\one\crl{\tau_i=+1}}\prn{\frac{1}{2}+\sigma_i\eps}^{\one\crl{\tau_i=-1}}
\end{equation*}
a direct calculation verifies that $\sum_{\btau\in\crl{\pm}^d}q_{\btau}(\bsigma)=1$ and $Q_{\bsigma}=\sum_{\btau\in\crl{\pm}^d}q_{\btau}(\bsigma) P_{\btau}$. The claim then follows from \cref{lem:Q-in-convex-hull}.

\subsubsection{Lower bound}
We begin by noticing that we can write the convex combination of dictionary elements as
\begin{equation*}
    f_\rho(i)\equiv f_\rho (i,j) = \sum_{\btau\in \crl{-1,+1}^d} \rho_{\btau}  \one\crl{\tau_i=+1}.
\end{equation*}
For any expert $f_{\btau}$, the loss reads as $\abs{f_{\btau}(i,j)-\one\crl{j=+1}} = \abs{\one\crl{\tau_i=+1}-\one\crl{j=+1}} = \one\crl{\tau_i\neq j}$ for any $(i,j)$. Therefore, we can write the risk $\cR_{\bsigma}=\cR_{Q_{\bsigma}}$ of $f_{\btau}$ as
\begin{equation*}
    \cR_{\bsigma}(f_{\btau})=\frac{1}{d}\sum_{i=1}^d \prn{\frac{1}{2}-\sigma_i \eps}\one\crl{\tau_i=-1} + \prn{\frac{1}{2}+\sigma_i \eps}\one\crl{\tau_i=+1} = \frac{1}{2}+\frac{\eps}{d}\sum_{i=1}^d \sigma_i \tau_i.
\end{equation*}
Thus, the best dictionary element is the one corresponding to $\btau=-\bsigma$ which achieves risk $1/2-\eps$, and each coordinate that differs in $\btau$ increases the excess risk of $f_{\btau}$ by $2\eps/d$. Moreover, for the aggregate $f_\rho$ we have
\begin{equation*}
    \cR_{\bsigma}(f_\rho) = \frac{1}{d}\sum_{i=1}^d \prn{\frac{1}{2}-\sigma_i \eps}(1-f_\rho(i)) + \prn{\frac{1}{2}+\sigma_i \eps}f_\rho(i) = \frac{1}{2}+\frac{\eps}{d}\sum_{i=1}^d \sigma_i(2f_\rho(i)-1),
\end{equation*}
and the excess risk is given by
\begin{align*}
    \cR_{\bsigma}(f_\rho) -\min_{\btau} \cR_{\bsigma}(f_{\btau}) &= \eps+ \frac{\eps}{d} \sum_{i=1}^d\sigma_i(2f_\rho(i)-1). 
\end{align*}

For any estimator $\rhohat \in \triangle^{2^d-1 }\cong \cP\prn{\crl{\btau:\btau\in\crl{\pm1}^d}}$ that we may compute on $n$ i.i.d. samples from $Q_{\bsigma}$, that is, $\rhohat\equiv \rhohat(Z)$ where $Z\sim Q_{\bsigma}^{\otimes n}$, define the estimator $\bsigmahat$ as $\sigmahat_i = -\sign(2f_{\rhohat}(i)-1) \in \crl{-1,+1}$ (with arbitrary tie breaks when it is zero). Then,
\begin{equation*}
    \cR_{\bsigma}(f_{\rhohat}) -\min_{\btau} \cR_{\bsigma}(f_{\btau}) \geq \frac{\eps}{2d} \nrm{\bsigmahat-\bsigma}_1.
\end{equation*}
We can now apply Assouad's Lemma to lower bound the excess risk.
\begin{lemma}[Assouad's lower bound \citep{yu1997assouad}]
    Let $\cQ_n=\crl{Q_{\bsigma}^n:\bsigma\in\crl{-1,1}^d}$ be a family of probability measures. Write $\bsigma\sim \bsigma'$ if $\bsigma$ and $\bsigma'$ differ in only one coordinate.
    Then
    \begin{equation*}
        \max_{\bsigma} \EE_{Z\sim Q_{\bsigma}^n} \brk{\nrm{\bsigmahat(Z)-\bsigma}_1} \geq  d \min\crl{1-\sqrt{\frac{1}{2}\KL(Q_{\bsigma}^n,Q_{\bsigma'}^n)}:\bsigma\sim \bsigma'}.
    \end{equation*}
\end{lemma}
In particular, we need to bound the Kullback-Leibler divergence between the two distributions.
A calculation shows that for $\bsigma\sim \bsigma'$. Then, for $\eps\leq 1/4$ and by \citet[Lemma 2.7]{tsybakov2009introduction}
\begin{align*}
    \KL(Q_{\bsigma}^{\otimes n},Q_{\bsigma'}^{\otimes n})&=\frac{n}{d}\KL(\Ber(\tfrac{1}{2}\pm\eps),\Ber(\tfrac{1}{2}\mp\eps)) \\
    &\leq \frac{n}{d}\chi^2(\Ber(\tfrac{1}{2}\pm\eps),\Ber(\tfrac{1}{2}\mp\eps)) \\
    &= \frac{n}{d}\frac{4\eps^2}{(1/2+\eps)(1/2-\eps)} \leq \frac{22n\eps^2}{d}.
\end{align*}
We get that the minimax expected excess risk is lower bounded by
\begin{equation*}
    \max_{\bsigma}\EE_{Q_{\bsigma}^{\otimes n}}\brk{\cR_{\bsigma}(f_{\rhohat}) -\min_{\btau} \cR_{\bsigma}(f_{\btau})} \geq \frac{\eps}{2}\prn{1-\sqrt{\frac{11 n \eps^2}{d}}}  \geq \frac{\eps}{4} = \frac{1}{8}\sqrt{\frac{d}{11n}} =  \frac{1}{8}\sqrt{\frac{\floor{\log_2(K)} }{11n}} \geq  \frac{1}{38} \sqrt{\frac{\log (K)}{n}}.
\end{equation*}
where the second and third (in-)equalities hold for $\eps=\frac{1}{2}\sqrt{\frac{d}{11 n}}$ which we now show to be smaller than $1/4$. Also note that we used that for any $K\geq 4$,
\begin{equation*}
    \floor{\log_2(K)} \geq \log_2(K) -1  \geq \frac{1}{2}\log_2(K).
\end{equation*}
Indeed, since we assumed that $\log(K)/n \leq 1$, we have that
\begin{equation*}
    \eps=\frac{1}{2}\sqrt{\frac{\floor{\log_2(K)}}{11 n}} \leq \frac{1}{2}\sqrt{\frac{\log K}{22 \log (2) n}} < \frac{1}{7} \sqrt{\frac{\log(K)}{n}} \leq \frac{1}{4}.
\end{equation*}
That concludes the proof of the lower bound.

\subsubsection{Upper bound} 
Finally, we check that \cref{thm:fast-rate-MSA} applies. In particular, we claim that we may choose $(\phi,\Phi)$ as $(\psi,\Psi)$ from \cref{eqn:aggregation-map-Pareto-set}.
To begin, we calculate that for any $\btau$ and any $f$,
\begin{equation*}
    \cR_{\btau}(f) = \frac{1}{d}\sum_{i=1}^d \prn{\one\crl{\tau_i=-1}f(i,-1)+\one\crl{\tau_i=+1}(1-f(i,+1))}.
\end{equation*}
Recall the definition of $\psi_\rho$ from \cref{eqn:aggregation-map-Pareto-set}, which in the notation of this setting is
\begin{equation*}
    \psi_\rho \in \argmin_{f\in\cF} \max_{\btau\in\crl{\pm}^d} \crl{\cR_{\btau}(f)-\cR_{\btau}(f_\rho)}.
\end{equation*}
Furthermore, the objective is non-negative on the convex hull; thus, the minimizer will be one of the $g_{\bu}$. For $g_{\bu}$ we can calculate the risk compared to $f_\rho$
\begin{equation*}
    \cR_{\btau}(g_{\bu})-\cR_{\btau}(f_\rho) = \frac{1}{d}\sum_{i=1}^d \prn{\one\crl{\tau_i=-1}(u_i^2-f_\rho(i))+\one\crl{\tau_i=+1}((1-u_i)^2-(1-f_\rho(i)))} .
\end{equation*}
Maximizing over $\btau$, which can be done coordinate-wise, yields
\begin{equation*}
    \max_{\btau\in\crl{\pm}^d} \crl{\cR_{\btau}(g_{\bu})-\cR_{\btau}(f_\rho)} = \frac{1}{d}\sum_{i=1}^d \max\crl{u_i^2-f_\rho(i),(1-u_i)^2-(1-f_\rho(i))}.
\end{equation*}
On the other hand, minimizing over $\bu$ (also coordinate-wise), yields that the minimum is attained for $u_i= f_\rho(i)$. Hence, $\psi_{\rho}=g_{\bu}$ with 
\begin{equation*}
    -\Psi(\rho) = \min_{\bu\in[0,1]^d}\max_{\btau\in\crl{\pm}^d} \crl{\cR_{\btau}(g_{\bu})-\cR_{\btau}(f_\rho)} = \frac{1}{d}\sum_{i=1}^d f_\rho^2(i)-f_\rho(i) = -\frac{1}{d}\sum_{i=1}^d f_\rho(i)(1-f_\rho(i)) \leq 0.
\end{equation*}
We can now verify that $(\psi,\Psi)$ satisfies \cref{asm:fast-rate-aggregation}.
\begin{enumerate}
    \item Clearly, when we plug in any Dirac delta $\delta_{\btau}$, it holds that $\Psi(\delta_{\btau})=\frac{1}{d}\sum_{i=1}^d f_{\btau}(i)(1-f_{\btau}(i))=0$ because $f_{\btau}(i)$ is either $0$ or $1$.
    \item $\Psi$ is $\eta=2$-strongly concave in $\nrm{f_\rho-f_\gamma}_2^2= \frac{1}{d}\sum_{i=1}^d (f_\rho(i)-f_{\gamma}(i))^2$: indeed, for $h(u)=u(1-u)$, it holds $h''(u)=-2$. As a consequence, for $\bu=(f_{\rho}(1),...,f_{\rho}(d))$ we have
    \begin{equation*}
        \nabla_{\bu}^2 \Psi(\rho) = \nabla_{\bu}^2 \frac{1}{d}\sum_{i=1}^d h(u_i) = -\frac{2}{d}\I_d.
    \end{equation*}
    \item Finally, our previous calculations verify that $\cR_{\btau}(f_\rho) - \cR_{\btau}(\psi_\rho) = \Psi(\rho)$ for all $\btau$,
    and thus $\bR(\psi_\rho)=\bR(f_\rho)-\Psi(\rho)\one$.
\end{enumerate}
Moreover, all the other assumptions of \cref{thm:fast-rate-MSA} are also satisfied.
\begin{enumerate}
    \item Boundedness is obvious as all functions and $Y$ take values in $[0,1]$.
    \item The absolute loss is $1$-Lipschitz, that is, $L_\ell=1$.
    \item $\psi_\rho$ is $2$-Lipschitz in $f_\rho$: if $\psi_\rho=g_{\bu}$ and $\psi_{\gamma} = g_{\v}$, then in each coordinate we have that $\abs{\psi_\rho(i,-1)-\psi_{\gamma}(i,-1)}=\abs{u_i^2-v_i^2}\leq 2 \abs{u_i-v_i}=2\abs{f_\rho(i)-f_{\gamma}(i)}$ and $\abs{\psi_\rho(i,+1)-\psi_{\gamma}(i,+1)}=\abs{(u_i-v_i)(2-(u_i+v_i))}\leq 2\abs{u_i-v_i}=2\abs{f_\rho(i)-f_\gamma(i)}$. Hence $L_\phi=2$.
    \item Finally, since $\cR_Q$ is a linear scalarization of the source risks, \cref{lem:bound-modulus-Lipschitz} implies that $\lowermodulus(t)\geq t$ and $\uppermodulus(t)\leq t$.
\end{enumerate}
Hence, we apply \cref{thm:fast-rate-MSA} (through \cref{eqn:fast-rate-expectation}) to get
$$\lambda= \frac{1}{11L_\ell}  \min\crl{\frac{1}{L_\phi},\frac{\eta+\kappa }{ L_\ell + \eta+\kappa }}=\frac{1}{11}  \min\crl{\frac{1}{2},\frac{2+0 }{ 1 + 2+0 }}=\frac{1}{22},$$ 
by plugging in all the constants. That concludes the proof of \cref{thm:lower-bound-convex-hull}.

\end{document}